\documentclass{article}

\PassOptionsToPackage{compress, numbers}{natbib}
\usepackage[main, final]{neurips_2026}

\usepackage[utf8]{inputenc}         %
\usepackage[T1]{fontenc}            %
\usepackage{algorithm}              %
\usepackage[noend]{algpseudocode}   %
\usepackage{amsfonts}               %
\usepackage{amsmath}                %
\usepackage{amssymb}                %
\usepackage{amsthm}                 %
\usepackage{bbding}                 %
\usepackage{bbm}                    %
\usepackage{booktabs}               %
\usepackage{bm}                     %
\usepackage{caption}                %
\usepackage{colortbl}               %
\usepackage{comment}                %
\usepackage{enumitem}               %
\usepackage{footnote}               %
\usepackage{graphicx}               %
\usepackage{longtable}              %
\usepackage{mathtools}              %
\usepackage{microtype}              %
\usepackage{minitoc}                %
\usepackage{multirow}               %
\usepackage{nicefrac}               %
\usepackage{pifont}                 %
\usepackage{setspace}               %
\usepackage{soul}                   %
\usepackage{subcaption}             %
\usepackage[x11names]{xcolor}       %
\usepackage{tcolorbox}              %
\usepackage{url}                    %
\usepackage{wrapfig}                %

\usepackage{etoc}

\usepackage{tyler}                  %

\usepackage[colorlinks, citecolor=blue]{hyperref}   %
\usepackage{footnotebackref}                        %
\usepackage{cleveref}                               %

\newtheorem{theorem}{Theorem}[section]
\newtheorem{assumption}[theorem]{Assumption}

\newtheorem{definition}[theorem]{Definition}
\newtheorem{lemma}[theorem]{Lemma}
\newtheorem{proposition}[theorem]{Proposition}
\newtheorem{remark}[theorem]{Remark}

\crefname{term}{term}{terms}
\Crefname{term}{Term}{Terms}
\crefformat{term}{term~(#2#1#3)}
\Crefformat{term}{Term~(#2#1#3)}

\newtcolorbox{wheatbox}{colback=Wheat1, colframe=Wheat1!75!black}

\title{SGD Provably Prioritizes a Shortcut\\Spurious Feature in the XOR Model}

\author{%
  Tyler LaBonte$^1$ $\qquad$ Vidya Muthukumar$^{1,2}$ \\
  $^1$H. Milton Stewart School of Industrial and Systems Engineering, Georgia Institute of Technology \\
  $^2$School of Electrical and Computer Engineering, Georgia Institute of Technology \\
  \texttt{\{tlabonte,vmuthukumar8\}@gatech.edu} \\
}

\begin{document}

\maketitle

\begin{abstract}
    Neural networks are known to be susceptible to over-reliance on spurious correlations.
    However, the precise mechanism by which models exploit shortcut features is not fully understood, and algorithms to mitigate this behavior rely on as yet unjustified assumptions about the learned representations.
    In this work, we provide the first end-to-end theoretical characterization of spurious feature learning for two-layer ReLU neural networks trained by online minibatch SGD on the logistic loss.
    We consider data drawn from the high-dimensional Boolean hypercube with a quadratic signal function (namely XOR) and a linear spurious correlation.
    We show that SGD learns the spurious feature first, and exponentially fast.
    Moreover, the optimization dynamics couple the spurious and signal features, with a stronger spurious component inhibiting signal feature learning.
    Our analysis reveals precise phase transitions in the learning dynamics.
    In the first phase, alignment between the signs of the spurious feature and second-layer weight drives rapid growth of the spurious feature.
    In the second phase, large majority group margin slows learning and the signal feature remains suppressed.
    When the spurious correlation is maximally strong, we show theoretically that the spurious feature dominates even at the sample complexity threshold where XOR would be learned in isolation (\ie if the spurious feature was absent).
    In contrast, when the correlation strength is constant, we provide preliminary evidence that the model can eventually learn the XOR signal, although the spurious feature is not forgotten.
\end{abstract}

\section{Introduction}
Classification tasks in machine learning are sensitive to \emph{spurious correlations}: facile patterns which are
predictive of the class label in the training dataset but irrelevant to the target function~\citep{sagawa2020distributionally}.
These ``shortcuts'' can be as simple as using the background of an image to predict its content~\citep{beery2018recognition, xiao2021noise}, and have been observed in applications ranging from medicine~\citep{zech2018variable} to justice~\citep{chouldechova2016fair} and facial recognition~\citep{liu2015deep}.
Models which over-rely on these shortcuts can be accurate on average, but perform no better than random guessing on data where the spurious correlation is absent or opposite~\citep{shah2020pitfalls}.

Understanding how neural networks learn features is increasingly viewed as essential to explaining their reliance on spurious correlations~\citep{izmailov2022feature}.
In particular, it is unclear how quickly signal and spurious features are learned, the impact of their relative complexity and strength of the spurious correlation, and how these features are represented in parameter space.
Insights would have significant implications for the design of debiasing algorithms.
Among the most popular such methods, \emph{Just Train Twice}~\citep{liu2021just} is predicated on faster learning of spurious features than signal features, while \emph{Deep Feature Reweighting}~\citep{kirichenko2023last} assumes the model learns both signal and spurious features even if it primarily utilizes the latter --- these are highly nontrivial conditions on the learned representations which currently lack theoretical justification.
Despite recent advances in linearized settings or with modified training algorithms~\citep{bombari2024spurious, hermann2024foundations, yang2024identifying}, we do not understand the mechanism underlying spurious feature learning in nonlinear models, even for basic target functions.

In this paper, we theoretically characterize the feature learning process of stochastic gradient descent (SGD) on a two-layer ReLU neural network where the ground-truth (signal) model is the Boolean exclusive-OR (XOR) problem, but the data contains a linear spurious correlation.
As the spurious feature is quantitatively ``simpler'' than the quadratic XOR signal, our setting enables rigorous study of how neural networks exploit low-complexity shortcuts~\citep{qiu2024complexity}.
Compared to learning the XOR in isolation~\citep{glasgow2024sgd}, we find that the signal and spurious features exhibit nonlinear interdependencies which complicate the learning dynamics --- we show that the spurious feature competes with, and eventually suppresses, the signal feature.
Our main theoretical result (\Thmref{main}) states that the spurious feature inevitably dominates when the spurious correlation is maximally strong.
Our theory identifies several interesting phase transitions in the learning dynamics, illustrated in \Figref{phases}.

\begin{wheatbox}
    In more detail, our \textbf{main contributions} include the below:
    \begin{itemize}
        \item We provide the first end-to-end theoretical characterization of spurious feature learning for two-layer ReLU neural networks trained by online minibatch SGD on the logistic loss (\ie no model linearization or layer-wise training).
        \item If the spurious correlation is maximally strong, we show theoretically that the spurious feature suppresses signal learning, inducing shortcut reliance even at the sample complexity threshold for learning XOR in isolation. Our analysis reveals precise phase transitions in the learning dynamics which align with simulations.
        \item If the correlation strength is constant, we provide preliminary evidence that the model eventually learns the XOR signal. Moreover, the spurious feature is not forgotten, and the network decomposes into disjoint signal and spurious subnetworks.
    \end{itemize}
\end{wheatbox}

\paragraph{Novelty of techniques.} 
Our proofs build on the XOR feature learning analysis done by~\cite{glasgow2024sgd} without a spurious correlation.
Our framework similarly involves two phases, though our phase transitions are different and governed by the spurious feature.
Despite sharing the basic setting of~\cite{glasgow2024sgd} and some technical tools (\eg \Lemref{boolean_to_gaussian_delta}), our analysis diverges almost immediately due to nonlinear interactions between the signal and spurious features that fundamentally alter the learning dynamics.
The starkest difference is the exponential suppression of the signal in Phase II, which necessitates a disparate approach focused on the data margins.
Additionally, the spurious feature grows so rapidly that a useful Taylor approximation to the population gradients also used by~\cite{glasgow2024sgd} becomes vacuous after $O(\log\log(d)\eta^{-1})$ iterations rather than the $O(\log(d)\eta^{-1})$ of~\cite{glasgow2024sgd}, where $\eta$ is the learning rate. This is a significant compression since the full training run lasts $O(\log(d)\eta^{-1})$ iterations, and it induces substantial differences in Phase I analysis.
The signal-spurious interdependence also introduces new technical challenges,
the greatest being the analysis of the orthogonal component $\bwperp$ in Phase I. %
We elaborate on these challenges and our approach in \Secref{proof_sketch}.

\paragraph{Notation.}
We provide a full table of notation in \Secref{notation}, and describe here only what is necessary to read the main paper.
We use uppercase bold symbols to denote matrices (\eg $\mathbf{X}$), lowercase bold symbols to denote vectors (\eg $\bx$), and italicized symbols to denote scalars (\eg $x$).
We write $\be_k$ for the $k$-th standard basis vector and $\norm{\cdot}$ for the vector $\ell_2$-norm.
Let $\P_{\bx}$ and $\E_{\bx}$ denote a probability and expectation with respect to a random vector $\bx$, respectively.
Let $\text{Unif}(\mathbb{S}^{d-1}(\theta))$ denote the uniform distribution on the $\ell_2$-sphere in $\R^d$ of radius $\theta$.
We define the following asymptotic notation with respect to growing data dimension $d$: $x\ll y \iff x=o(y)$, $x\lesssim y \iff x=O(y)$, $x\gg y \iff x=\omega(y)$, $x\gtrsim y \iff x = \Omega(y)$, and $x\asymp y \iff x=\Theta(y)$.
Finally, we write $x=(1\pm o(1))\cdot y$ as shorthand for $x\in [(1-o(1))\cdot y, (1+o(1))\cdot y]$.

\subsection{Related Work}
Here we provide a brief summary of related work across three axes.

\paragraph{Spurious correlations.} The proclivity of neural networks to over-reliance on spurious features has been widely observed~\citep{geirhos2020shortcut, singla2022salient}.
These features often manifest as simple ``shortcuts'' --- including image backgrounds~\citep{beery2018recognition, xiao2021noise} and secondary objects~\citep{rosenfeld2018elephant, shetty2019not} in computer vision, and syntactical or statistical heuristics in NLP~\citep{gururangan2018annotation, mccoy2019right, niven2019probing}.
In applications, such shortcuts are known to exacerbate biases~\citep{hovy2015tagging, blodgett2016demographic, tatman2017gender, hashimoto2018fairness} and cause failure in high-stakes scenarios~\citep{liu2015deep, chouldechova2016fair, zech2018variable, oakden-rayner2019hidden}.
Substantial research in group robustness and out-of-distribution generalization has investigated algorithms to mitigate these issues~\citep{arjovsky2019invariant, sagawa2020distributionally, nam2020learning, liu2021just, idrissi2022simple, pagliardini2023agree, kirichenko2023last, qiu2023simple, labonte2023towards, vasudeva2024mitigating, noohdani2024decompose, tifrea2024frappe}.
These methods are often predicated on tenuous assumptions about the feature learning process, \eg that the spurious features are learned faster~\citep{liu2021just} or have simpler representations~\citep{vasudeva2024mitigating}, or that the signal and spurious features are jointly learned but improperly weighted in the last layer~\citep{kirichenko2023last, qiu2023simple, labonte2023towards}.
Theoretical analysis such as ours would (in the long term) clarify these assumptions and possibly be prescriptive of novel methodology.

\paragraph{Feature learning theory without spurious correlations.}
A profound advantage of neural networks over classical kernel methods (including the neural tangent kernel (NTK)~\citep{jacot2018neural}) is their ability to learn features from data~\citep{karp2021local, telgarsky2023feature, vyas2023empirical, radhakrishnan2024mechanism}.
Showing how feature learning reduces the sample complexity of learning implicitly low-dimensional functions is an active area of research~\citep{mei2018mean,li2020learning,abbe2021staircase, bietti2022learning, abbe2022merged, damian2022neural, tan2023online, mousavi-hosseini2023neural, wu2023finite, abbe2023sgd, lee2024neural}.
Many such analyses are limited by their modification of the training algorithm, \eg via layer-wise training or gradient clipping.
Our work is particularly inspired by~\cite{glasgow2024sgd}, which showed that two-layer ReLU neural networks trained by standard SGD (with no such modifications) can learn the Boolean XOR in $d$ dimensions with $d\cdot\text{polylog}(d)$ sample complexity.
Ultimately, we show that this sample complexity is insufficient to learn the XOR in the presence of the linear spurious correlation.

\paragraph{Spurious feature learning.} Theoretical insight into spurious correlations is useful even in linear models~\citep{sagawa2020investigation, nagarajan2021understanding, ye2023freeze, puli2023dont}, but recent work suggests that feature learning underlies the design of robustness algorithms~\citep{izmailov2022feature}.
Several works have studied shortcut learning in the random features and NTK regimes~\citep{bombari2024spurious, hermann2024foundations, roy2025how}.
Most closely related is~\cite{yang2024identifying}, which showed that spurious features are learned early in SGD training and characterized separability of majority and minority groups; their analysis uses square loss and a linearization similar to the NTK~\citep{hu2020surprising}.
Another related work is~\cite{qiu2024complexity}, which characterized the gradients with and without spurious correlations (under layer-wise training), but did not provide an end-to-end analysis of test error.
In contrast, we directly analyze nonlinear neural networks trained by online minibatch SGD and explicitly characterize the majority/minority group error.
We contextualize our results with the empirical findings of~\cite{qiu2024complexity} in \Secref{discussion}.

\section{Setting} \label{sec:setting}

We now describe our Boolean XOR setting with a linear spurious correlation, as well as our two-layer neural network model and training procedure.

\subsection{Data} \label{sec:data}
Define $\bmu_1\coloneqq \be_1-\be_2$ and $\bmu_2\coloneqq \be_1+\be_2$.
Let $\lambda\in (0, \tfrac{1}{2})$ denote the strength of the spurious correlation, where lower $\lambda$ indicates stronger correlation.
We model the data distribution $P_d(\lambda)$ by:\footnote{For simplicity, we detail our setting for an axis-aligned Boolean hypercube where the first three coordinates comprise the ground truth and spurious correlation. Nevertheless, all aspects of our analysis are rotationally invariant, and thus our results hold for Boolean hypercubes in any basis.}
\begin{equation}\label{eq:XOR-spurious-correlation}
    \bx = \left\{
    \begin{array}{llll}
        \hphantom{-}\bmu_1+\be_3+\bxi & \withprob \nicefrac{1}{4}-\nicefrac{\lambda}{4} & \hphantom{-}\bmu_2-\be_3+\bxi & \withprob \nicefrac{1}{4}-\nicefrac{\lambda}{4} \\
        \hphantom{-}\bmu_1-\be_3+\bxi & \withprob \nicefrac{\lambda}{4} & \hphantom{-}\bmu_2+\be_3+\bxi & \withprob \nicefrac{\lambda}{4} \\
        -\bmu_1+\be_3+\bxi & \withprob \nicefrac{1}{4}-\nicefrac{\lambda}{4} & -\bmu_2-\be_3+\bxi & \withprob \nicefrac{1}{4}-\nicefrac{\lambda}{4} \\
        -\bmu_1-\be_3+\bxi & \withprob \nicefrac{\lambda}{4} & -\bmu_2+\be_3+\bxi & \withprob \nicefrac{\lambda}{4}
    \end{array}
    \right.,
\end{equation}
where $\bxi\sim\Unif(0^3\times\{\pm 1\}^{d-3})$ so that $\bxi\perp\{\bmu_1,\bmu_2, \be_3\}$.
For the remainder of the paper, we write $\P_{\bx}\coloneqq \P_{\bx\sim P_d(\lambda)}$ and $\E_{\bx}\coloneqq \E_{\bx\sim P_d(\lambda)}$ as shorthand.
We will also write $\bx\coloneqq\bz+\bs+\bxi$ as shorthand, where $\bz \coloneqq x_1 \be_1 + x_2 \be_2$ and $\bs \coloneqq x_3 \be_3$ and the distribution of both $\bz$ and $\bs$ can be derived from \Eqref{XOR-spurious-correlation} above.
Note that for $\lambda=\tfrac{1}{2}$ we would have $\bx\sim\Unif(\{\pm 1\}^d)$, recovering the data distribution of~\cite{glasgow2024sgd}.

The target function is the XOR formula on the first two dimensions, \ie $y(\bx)\coloneqq y(\bz) \coloneqq -x_1x_2$ with no label noise.
We denote the majority group by $\Xmaj \coloneqq \{\bx\in \{\pm 1\}^d:y(\bx)=x_3\}$ and the minority group by $\Xmin \coloneqq \{\bx \in \{\pm 1\}^d:y(\bx)=-x_3\}$.
It can be verified directly from Equation~\eqref{eq:XOR-spurious-correlation} that $\P_{\bx}\left(\bx \in \Xmaj\right) = 1 - \lambda$ while $\P_{\bx}\left(\bx \in \Xmin\right) = \lambda$.
We denote the accuracy of a real-valued predictor $f:\{\pm 1\}^d\to\R$ on data uniformly drawn from a set $\X\subseteq\{\pm 1\}^d$ by
\begin{equation*}
    \textnormal{Acc}_{\X}(f) \coloneqq \P_{\bx\sim \Unif(\X)}\big(\sgn(f(\bx))=y(\bx)\big).
\end{equation*}
Note that the fully spurious predictor $\fsp(\bx)=x_3$ has $\textnormal{Acc}_{\Xmaj}(\fsp)=1$ and $\textnormal{Acc}_{\Xmin}(\fsp)=0$.

\subsection{Model and Training} \label{sec:model}

We simultaneously train both layers of a two-layer ReLU neural network with $p$ neurons $\{(a_j,\bw_j)\}_{j=1}^p$.
For initialization parameter $\theta>0$, we initialize each neuron \iid with $\bw_j\sim\Unif(\mathbb{S}^{d-1}(\theta))$ so that $\norm{\bw_j}=\theta$ and $a_j=r_j\theta$ where $r_j\sim\Unif(\{\pm 1\})$. %
We represent the prediction as an expectation over the empirical distribution $\rho$ of the neurons, \ie
\begin{equation*}
    f_\rho(\bx)\coloneqq \E_{(a,\bw) \sim \rho}[a\sigma(\bw^\top\bx)] = \frac{1}{p}\sum_{j=1}^{p} a_j \sigma(\bw_j^\top \bx),
\end{equation*}
where $\sigma(\alpha)\coloneqq\max(0, \alpha)$ is the ReLU.
We often drop the $j$ and simply refer to a neuron as $(a,\bw)$.

We train via online minibatch SGD with learning rate $\eta>0$ on the logistic loss.
Let $\gamma(\bx)\coloneqq y(\bx)f_\rho(\bx)$ denote the (unnormalized) margin of datum $\bx$ and $\psi(u)\coloneqq 1/(1+e^{-u})$ denote the sigmoid. %
Then, we write the logistic loss by the composite notation $\ell_\rho(\bx)\coloneqq h(\gamma(\bx))$ where $h(\gamma)\coloneqq-2\log(\psi(\gamma))$.
Accordingly, we define $\ell_\rho^{(1)}(\bx)\coloneqq h'(\gamma(\bx))$.
We define the population loss as $L_\rho\coloneqq \E_{\bx} [\ell_\rho(\bx)]$ and the empirical loss on the $t$-th minibatch of size $m$, denoted $M^{(t)}\sim P_d^m(\lambda)$, as $\widehat{L}_{\rho^{(t)}}\coloneqq \frac{1}{m}\sum_{\bx\in M^{(t)}} \ell_{\rho^{(t)}}(\bx)$.
We write $\nabla_{\bw} L\coloneqq p \cdot \frac{\partial L}{\partial \bw}$ for the $p$-scaled gradient of $L$ with respect to $\bw$ and $\partial_{u} L \coloneqq p\cdot \frac{\partial L}{\partial u}$ for the $p$-scaled partial derivative of $L$ with respect to $u\in \{a, w_i\}$.\footnote{We scale by $p$ following~\cite{glasgow2024sgd} to match the conventional mean-field scaling and avoid excessive $\tfrac{1}{p}$ factors.}
We denote the minibatch SGD update by
\begin{equation*}
   \atplus = \at - \eta \partial_{\at}\widehat{L}_{\rhot} \qquad \bwtplus = \bwt - \eta \nabla_{\bwt}\widehat{L}_{\rhot}.
\end{equation*}

\subsection{Feature Learning Analysis} \label{sec:feature_learning_analysis}

For a neuron $(a,\bw)$, let us write $\bw\coloneqq\bw_{1:2}+\bwsp+\bwperp$ and $\bw_{1:2}\coloneqq \bwsig+\bwopp$, where
\begin{equation*}
    \bwsig \coloneqq
    \begin{cases}
        \frac{1}{2}\bmu_1\bmu_1^\top \bw & a \geq 0 \\
        \frac{1}{2}\bmu_2\bmu_2^\top \bw & a < 0
    \end{cases}
    \qquad
    \bwopp \coloneqq
    \begin{cases}
        \frac{1}{2}\bmu_2\bmu_2^\top \bw & a \geq 0 \\
        \frac{1}{2}\bmu_1\bmu_1^\top \bw & a < 0
    \end{cases}
    \qquad \bwsp \coloneqq w_3 \be_3.
\end{equation*}
We will often write $\wsp\coloneqq w_3$ as shorthand for the scalar value of $\bwsp$.
Note that $\bwsig$ and $\bwopp$ are defined with respect to $\sgn(a)$; our analysis shows $\sgn(a)$ does not change.
We also define the ``positive'' neurons by $S^+\coloneqq \{(a,\bw):\sgn(a)=1\}$ and the ``negative'' neurons by $S^-\coloneqq \{(a,\bw):\sgn(a)=-1\}$, where importantly $\bmu_1 \parallel \bwsig$ for neurons in $S^+$ while $\bmu_2 \parallel \bwsig$ for neurons in $S^-$.

Intuitively, $\bwsig$ is the feature in the ``signal'' direction --- towards the correct classification of the XOR label --- while $\bwopp$ is the feature in the ``opposing'', or negative, direction.
Likewise, $\bwsp$ denotes the ``spurious'' feature, while $\bwperp$ comprises ``orthogonal'' features.
The majority of our technical analysis will be towards providing a high-probability quantification of the growth rates of $\norm{\bwsig}$, $\norm{\bwopp}$, $\norm{\bwsp}$, and $\norm{\bwperp}$ for each neuron (where probability is over all the minibatched data and the random initialization).

A model which isolates the ground-truth XOR signal will have $\bw=\pm\bwsig$ for all neurons $(a,\bw)$.
For example, if $\bw=\bwsig$ with $a\geq 0$ then $\bw=\frac{1}{2}\bmu_1\bmu_1^\top\bw$, so $\bw=\alpha\bmu_1$ for $\alpha\in \R$.
In particular, the smallest optimal (normalized) neural network classifier has four neurons as follows:
\begin{equation*}
    f^\star(\bx)=\frac{1}{2}\Big(\sigma(\bmu_1^\top\bx)+\sigma(-\bmu_1^\top\bx)-\sigma(\bmu_2^\top\bx)-\sigma(-\bmu_2^\top\bx)\Big).
\end{equation*}

\section{Main Result} \label{sec:main}

\begin{figure}[t]
    \centering
    \begin{subfigure}[b]{0.325\textwidth}
        \centering
        \includegraphics[width=\textwidth, height=0.18\textheight]{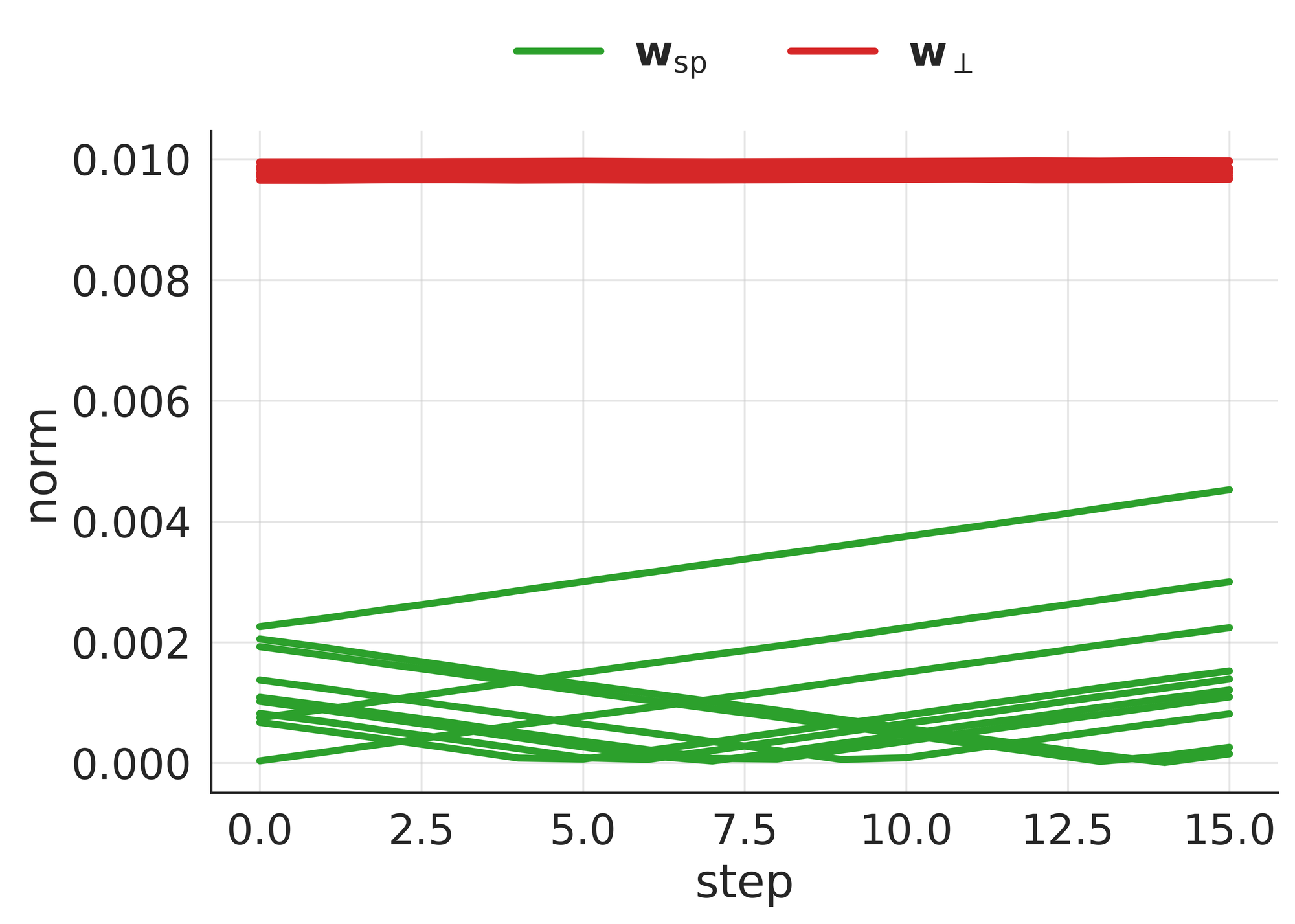}
        \caption{Phase Ia}
        \label{fig:phase1a}
    \end{subfigure}
    \begin{subfigure}[b]{0.325\textwidth}
        \centering
        \includegraphics[width=\textwidth, height=0.18\textheight]{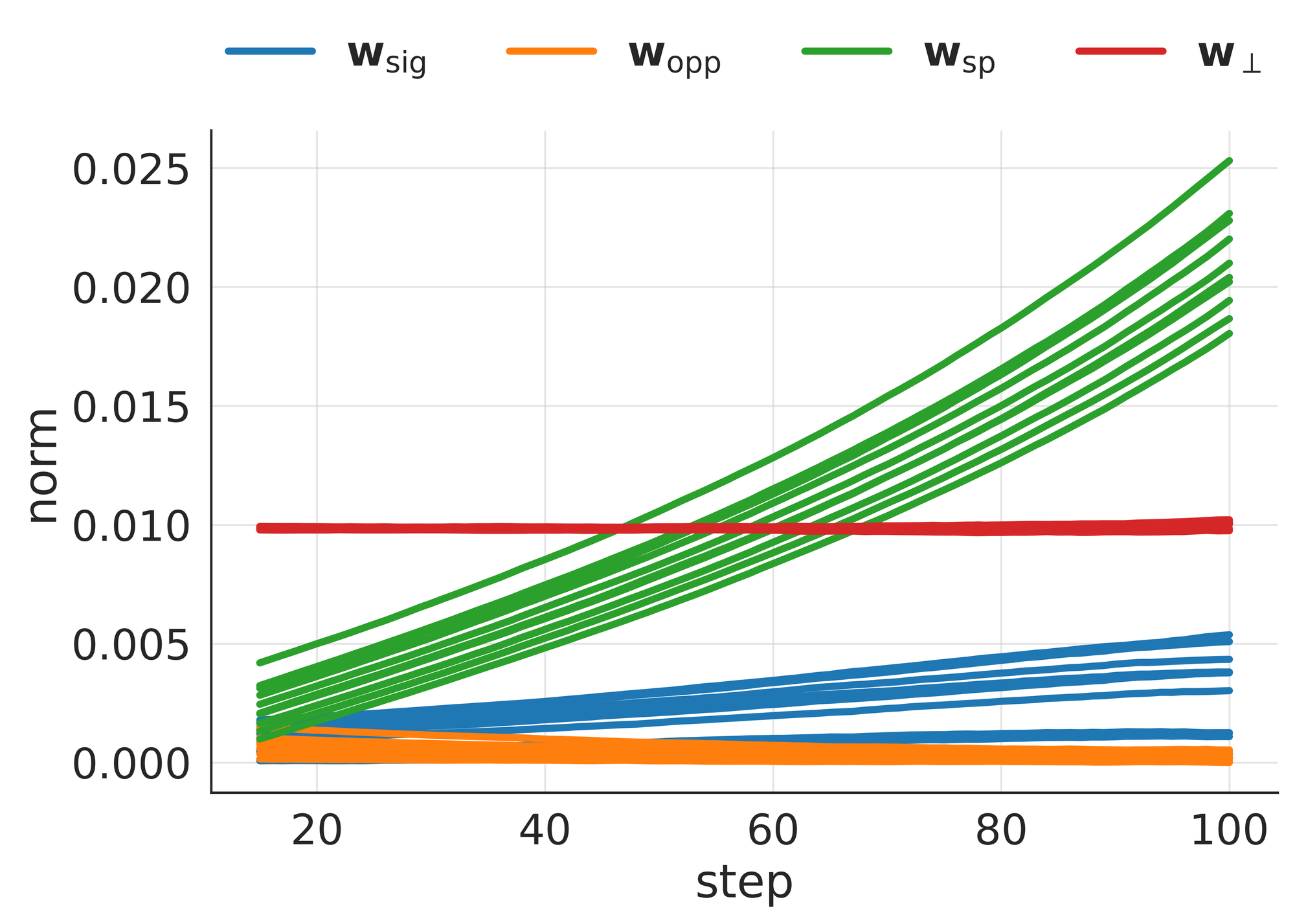}
        \caption{Phase Ib}
        \label{fig:phase1b}
    \end{subfigure}
    \begin{subfigure}[b]{0.325\textwidth}
        \centering
        \includegraphics[width=\textwidth, height=0.18\textheight]{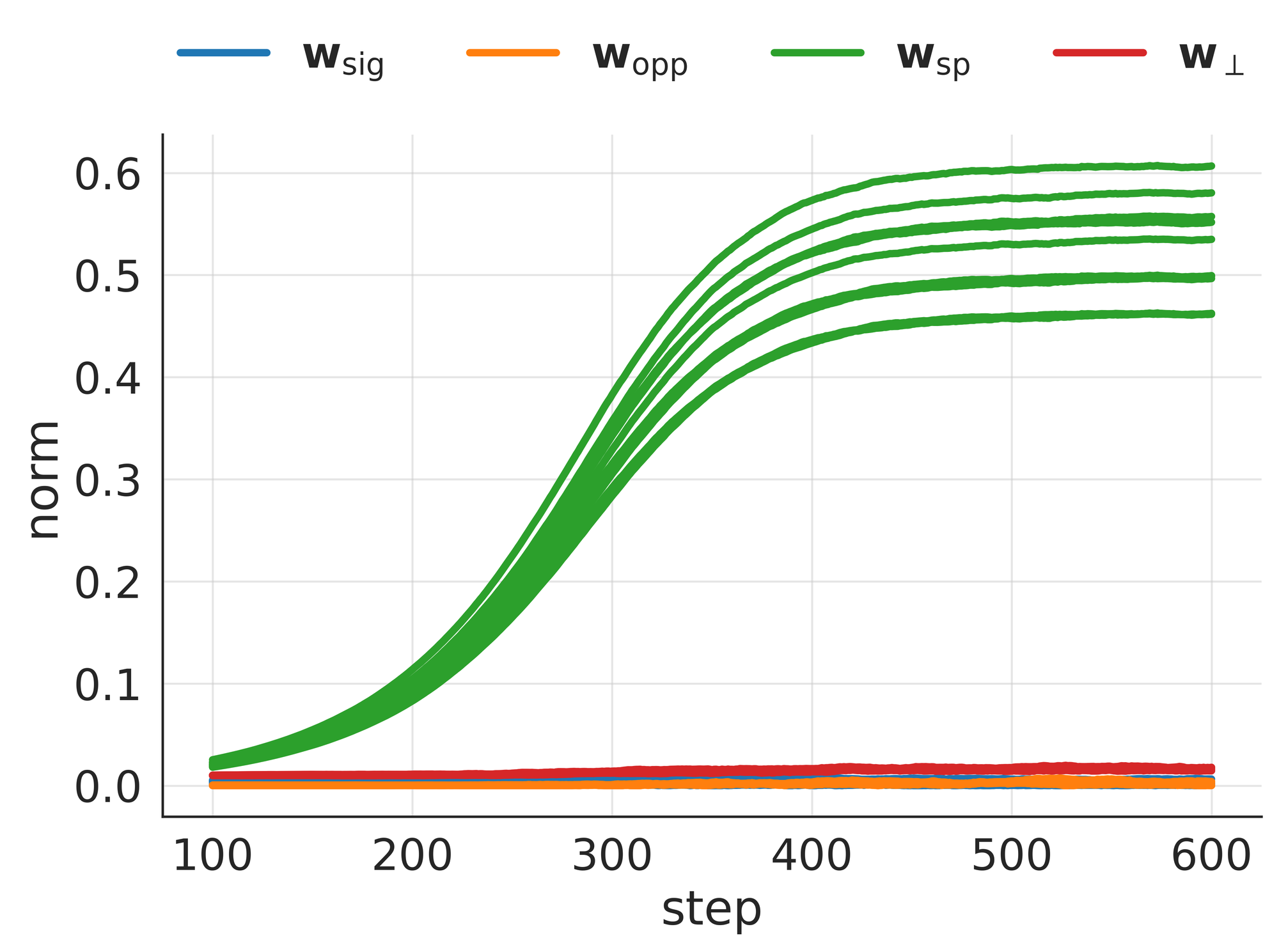}
        \caption{Phase II}
        \label{fig:phase2}
    \end{subfigure}
    \caption{\textbf{Phase transitions in spurious feature learning.} We display the results of a training run with dimension $d=100$, spurious correlation strength $\lambda=0.1$, learning rate $\eta=0.05$, width $p=10$, initialization scale $\theta=0.01$, and batch size $m=5000$. For each of the $p=10$ neurons, we plot $\norm{\bwsig}$, $\norm{\bwopp}$, $\norm{\bwsp}$, and $\norm{\bwperp}$ (defined in \Secref{feature_learning_analysis}). %
    Phase Ia is a very short initial phase which induces alignment between $\sgn(a)$ and $\sgn(\wsp)$; we have omitted $\bwsig$ and $\bwopp$ to emphasize how some $\bwsp$ lines ``bounce off'' zero.
    In Phase Ib, the spurious feature $\bwsp$ grows exponentially fast and dominates all other features.
    Finally, in Phase II, the growth of $\bwsp$ slows in a sigmoidal manner, yet continues monotonically increasing, while the signal feature $\bwsig$ is greatly suppressed.
    }
    \label{fig:phases}
\end{figure}

In this section, we introduce our main theoretical result.
It states that
the spurious feature inevitably dominates even at the sample complexity threshold for learning XOR in isolation, \ie $O(d\cdot \text{polylog}(d))$ samples~\citep{glasgow2024sgd}. %
We will need the following scalings on our model parameters.
\begin{assumption} \label{asm:main}
    We require the following for large enough constant $C>0$: the learning rate $\log(d)d^{-C} \ll \eta\ll \log^{-3}(d)$, the width $\log^5(d)\ll p \ll d^{C}$, the initialization scale $d^{-C/2}\ll \theta\ll \log^{-5C}(d)$, the batch size $m\gg d\log^{6}(d)\theta^{-2}$, and the spurious correlation strength $\lambda\ll \log^{-1}(d)$.
\end{assumption}

The following theorem is our main result.

\begin{theorem} \label{thm:main}
    There exists a large enough constant $C>0$ such that the following holds.
    If \Asmref{main} is satisfied, then with probability at least $1-d^{-C}$, upon
    \begin{equation*}
        T \asymp
        \begin{cases*}
            \log(d)(\log\log(d))^{-1}\eta^{-1} & \text{if } $\theta \asymp \textnormal{polylog}^{-1}(d)$ \\
            \log(d)\eta^{-1} & \text{if } $\theta \asymp \textnormal{poly}^{-1}(d)$
        \end{cases*}
    \end{equation*}
    iterations of online minibatch SGD under the $\ell_\rho$ loss, we have
    \begin{equation*}
        \textnormal{Acc}_{\Xmaj}(f_{\rho^{(T)}}) \geq 1-d^{-C} \qquad \textnormal{Acc}_{\Xmin}(f_{\rho^{(T)}}) \leq d^{-C}.
    \end{equation*}
\end{theorem}
\Thmref{main} implies that the majority group accuracy tends to $1$ while the minority group accuracy tends to $0$ as $d \to \infty$, and it is proved at the end of \Secref{phase2_inductive_step}.
\begin{remark} \normalfont
    Recall that since we are running online minibatch SGD, the sample complexity is $m\cdot T$ for batch size $m$ and $T$ iterations.
    \Thmref{main} holds in two regimes demarcated by the initialization scale $\theta$.
    If $\theta\asymp\textnormal{polylog}^{-1}(d)$, we show that the spurious feature dominates at the sample complexity for learning XOR, up to $\log\log$ factors, with standard $d\cdot\textnormal{polylog}(d)$ batch size.
    On the other hand, if $\theta\asymp \textnormal{poly}^{-1}(d)$ we precisely match the $O(\log(d)\eta^{-1})$ iterations of~\cite{glasgow2024sgd}, though we require a large $\textnormal{poly}(d)$ batch size.\footnote{This is a strong result from a sample complexity perspective, as one observes $\textnormal{poly}(d)$ more data than~\cite{glasgow2024sgd}, but the XOR signal is still not learned. Our perspective is that the small initialization scale slows down XOR feature learning even more.}
\end{remark}
\begin{remark} \normalfont
    Our theoretical analysis requires the strength of the spurious correlation to decay sufficiently fast, \ie $\lambda \ll \log^{-1}(d)$.
    This corresponds to an ``extreme'' correlation case where the proportion of minority group data decays to zero.
    This setting is well-studied in both the empirical~\citep{lee2023diversify, pagliardini2023agree} and theoretical~\citep{lai2024sharp, park2026spurious} literature, and it is closely related to out-of-domain generalization~\citep{koh2021wilds}.
    Notably, the condition $\lambda\ll \log^{-1}(d)$ is only required for our Phase II analysis, while our Phase I analysis holds with constant correlation strength.
\end{remark}

Our analysis involves a precise characterization of the feature learning dynamics in three distinct phases, illustrated in \Figref{phases}.
This characterization may be of independent interest.
We briefly summarize each phase here and provide details in \Secref{proof_sketch}.
\begin{itemize}
    \item Phase Ia (\Figref{phase1a}) is a very short initial phase lasting only $\TIa\asymp \log^{1/2}(d)d^{-1/2}\eta^{-1}$ iterations.
    By the end of Phase Ia, we have sign alignment between the spurious features and second-layer weights, \ie $\sgn(a)=\sgn(\wsp)$ for all neurons $(a,\bw)$.
    \item Phase Ib (\Figref{phase1b}) lasts $\TIb\asymp \log\log(d)\eta^{-1}$ iterations, wherein sign alignment implies a geometric lower bound on spurious feature growth.
    By the end of Phase Ib, the margins on majority and minority group data concentrate in the following (informally stated) way: with high probability over $\Unif(\Xmaj)$, a datum $\bxmaj$ has \emph{positive} margin proportional to $\wsp^2$; similarly, over $\Unif(\Xmin)$, a datum $\bxmin$ has \emph{negative} margin proportional to $-\wsp^2$. On this event, we already have $\sgn(f_\rho(\bx))=x_3$, \ie total dependence on the spurious correlation.
    \item Phase II (\Figref{phase2}) lasts $\TII \asymp \frac{1}{\eta}\left( \frac{\log(d)}{\log\log(d)} + \log\left(\frac{1}{\theta}\right) \right)$ iterations, \ie most of the training period.
    Two key phenomena occur during this phase.
    First, large majority group margin leads to a sigmoidal slowdown of the growth of the spurious feature, while the condition $\lambda\ll \log^{-1}(d)$ ensures that it remains monotonically increasing.
    Second, the spurious feature becomes massive enough to exponentially suppress the growth of $\norm{\bwsig}$, inhibiting acquisition of the signal feature.
\end{itemize}

\section{Proof Sketch} \label{sec:proof_sketch}
We now sketch the proof of \Thmref{main}, treating Phase I and Phase II separately.
In this section, we let $C>0$ denote a sufficiently large constant which does not vary from line to line.

\subsection{Phase I Proof Sketch} \label{sec:proof_sketch_phase1}

In Phase I, the neural network is small, and thus the loss $\ell_\rho(\bx)$ is well-approximated by a first-order Taylor expansion about $f_\rho=0$, \ie $\ell_0(\bx)\coloneqq -2\log(\tfrac{1}{2})-y(\bx)f_\rho(\bx)$.
Let us define the $L_0$ population loss by $L_0 := \E_{\bx}[\ell_0(\bx)]$.
Importantly, this first-order approximation is only an analysis technique, and \emph{does not} represent a linearization of the model.
That is, we still train a nonlinear model with the standard $\ell_\rho$ loss, but we bound feature norms by their growth under $L_0$ and the deviation between the $L_0$ and $L_\rho$ gradients.

The goal of Phase I is to verify the conditions of \Lemref{end_of_phase1} for all neurons $(a,\bw)$: that $\sgn(a)=\sgn(\wsp)$ and the spurious feature dominates in the sense that $\frac{1}{C}\norm{\bwsp}\geq \norm{\bwsig}+\norm{\bwopp}+\norm{\bwperp}\log^{1/2}(d)$.
Once both conditions hold, each active neuron contributes a term matching $\sgn(x_3)$ to $f_\rho(\bx)$ such that $\sgn(f_\rho(\bx))=x_3$ with high probability over $\bx \sim P_d(\lambda)$ --- that is, the network predicts entirely via the spurious feature on typical test data.

\paragraph{Why the spurious feature grows exponentially fast.} Under the $L_0$ approximation, \Lemref{sp_l0} yields the population gradient $-\bwsp^\top\nabla_{\bw} L_0 = a\wsp (\tfrac{1}{2}+\frac{\epsw}{4}-\lambda)$, where $\epsw$ is an asymmetry term that is exponentially small in $\norm{\bwsp}$ (\Lemref{v_ub}) and hence negligible.
Since $\lambda < \tfrac{1}{2}$, once $\sgn(a)=\sgn(\wsp)$ occurs, \Lemref{wsp_phase1_induction} gives the clean geometric recurrence
\begin{equation*}
    \wsptplus - \wspt \asymp \eta\sgn(\azero)(|\wspt|+\theta),
\end{equation*}
where we used that $\sgn(\at)=\sgn(\azero)$ (\ie the sign of $a$ does not change) by \Lemref{phase1_norms_sgn}.
While $\sgn(\azero)$ and $\sgn(\wspzero)$ may be initially misaligned, we show that they align (and remain aligned) after at most $\TIa\asymp \log^{1/2}(d)d^{-1/2}\eta^{-1}$ iterations.
Once aligned, $\norm{\bwsp}$ grows geometrically, reaching $\norm{\bwsp}\asymp \theta\log^C(d)$ after at most an additional $\TIb\asymp \log\log(d)\eta^{-1}$ iterations.

\paragraph{Why the signal feature grows slowly, but is not yet suppressed.} Under the $L_0$ approximation, \Lemref{sigopp_l0} yields the population gradient $-\bwsig^\top\nabla_{\bw}L_0=\frac{\sqrt{2}}{4}|a|\norm{\bwsig}\P_{\bxi} (|\bw^\top\bxi+\bw^\top\be_3|\leq\sqrt{2}\norm{\bwsig})$.
Unraveling the probability term requires a delicate ``perturbed Berry-Esseen'' lemma analogous to~\cite[Lemma B.4]{glasgow2024sgd}, which we present a modified version of in \Lemref{boolean_to_gaussian_delta}.
Ultimately, we get the recurrence that is stated in \Lemref{wsig_phase1_induction}, that is,
\begin{equation*}
    \norm{\bwsigtplus - \bwsigt} \lesssim \eta\norm{\bwsigt} + \eta\theta\log^C(d)d^{-1/2}.
\end{equation*}
This has the same geometric structure as the $\bwsp$ recurrence --- both grow with multiplicative rate $1+\Theta(\eta)$ --- but the additive term for $\bwsig$ is smaller by a factor of $d^{-1/2}$.
There is no suppression of signal growth at this stage: both the spurious and signal features grow geometrically, with the spurious feature simply having a much larger additive ``head start''.
Summing this recursion over $\TI\asymp\log\log(d)\eta^{-1}$ iterations yields $\norm{\bwsig}\lesssim \theta\log^{2C}(d)d^{-1/2}$, which is only a polylogarithmic factor larger than the initialization upper bound.
The opposing component $\bwopp$ follows similarly (\Lemref{wopp_phase1_induction}).

\paragraph{Why the orthogonal components stay controlled.} It remains to track $\norm{\bwperp}$ and $\norm{\bwperp}_\infty$.
While this may seem like a formality, it is one of the more mathematically involved parts of our analysis.
For $\norm{\bwperp}$, the key observation is that the $L_0$ gradient $\nabla_{\bwperp} L_0$ is nearly parallel to $\bwperp$.
The gradient takes the form $\E_{\bxi}[\bxi\cdot h(\bwperp^\top\bxi)]$ for a certain function $h$; for Gaussian inputs $\bg$, Stein's lemma gives $\E_{\bg}[\bg\cdot h(\bwperp^\top\bg)]=\E_{\bg}[h'(\bwperp^\top\bg)]\cdot \bwperp$ which is exactly parallel to $\bwperp$.
The deviation from Gaussianity is controlled by a Lindeberg exchange argument (\Lemref{lindeberg}).
For the component parallel to $\bwperp$, the gradient involves truncated moments of the form $\E_{\bxi}[|\bw^\top\bxi|\cdot \ind(|\bw^\top\bxi|\geq k)]$.
\Lemref{berry_esseen_tail} approximates these by their Gaussian counterparts uniformly over $k$, with error proportional to $\norm{\bwperp}_\infty$.
The contributions from the signal and opposing directions nearly cancel in the Gaussian approximation, giving $\norm{\bwperptplus - \bwperpt}\lesssim \eta\theta\log^{-2C}(d)$ (\Lemref{wperp_phase1_induction}).
Controlling $\norm{\bwperp}_\infty$ requires a separate coordinate-wise analysis (\Lemref{wperp_inf_phase1_induction}), which shows $\norm{\bwperptplus - \bwperpt}_\infty\lesssim \eta\theta\log^{3C}(d)d^{-1/2}$ and keeps the \Lemref{berry_esseen_tail} approximation sharp throughout Phase I.

\paragraph{Controlling approximation errors and closing the induction.} The $L_0$ approximations above are useful only if $\norm{\nabla_{\bw} \widehat{L}_\rho - \nabla_{\bw} L_0}$ remains negligible.
While~\cite{glasgow2024sgd} handle this using a leave-one-out symmetrization, the presence of the spurious correlation requires us to take a more complex \emph{leave-two-out} symmetrization instead (\Lemref{l0_lp_norm}).
Moreover, the new approximation error for $\bwsp$ is much larger, but taking the initialization scale $\theta\ll \log^{-5C}(d)$ turns out to be sufficient to analyze the Phase I dynamics (\Lemref{phase1_l0_lp_norm}).
The minibatch fluctuations $\norm{\nabla_{\bw} \widehat{L}_\rho - \nabla_{\bw} L_\rho}$ are then controlled by Hoeffding's inequality (\Lemref{empirical_concentration_ii}) with batch size $m\gg d\log^{7C}(d)$.
All five component bounds --- on $\norm{\bwsig}$, $\norm{\bwopp}$, $\norm{\bwsp}$, $\norm{\bwperp}$, and $\norm{\bwperp}_\infty$ --- are maintained jointly on a single inductive hypothesis (\Defref{phase1_hypothesis}) over $\TI=\TIa+\TIb$ iterations.

\subsection{Phase II Proof Sketch} \label{sec:proof_sketch_phase2}
By the end of Phase I, we have $\sgn(f_\rho(\bx))=x_3$ with high probability --- meaning that the network predicts entirely via the spurious feature. %
The goal of Phase II is to show this remains true for a further $O(\log(d)\eta^{-1})$ iterations, matching the sample complexity needed to learn the XOR signal in isolation~\citep{glasgow2024sgd}.
This is nontrivial because spurious feature growth slows down in Phase II, and we must rule out the possibility that the signal feature eventually overtakes it.
The key technical challenge of Phase II is that the neural network is too large for the $L_0$ approximation of Phase I to be valid, so we must bound the original $L_\rho$ gradient by characterizing the margin on typical data.

Recall that $\psi(x)\coloneqq 1/(1+e^{-x})$ denotes the sigmoid and $\gamma(\bx)\coloneqq y(\bx)f_\rho(\bx)$ denotes the margin.
Then, we have the following decomposition, which enables us to analyze the $L_\rho$ gradient:
\begin{multline*}
    \nabla_{\bw} L_\rho = p\E_{\bx}\left[\ell^{(1)}_\rho(\bx)\cdot y(\bx)\nabla_{\bw} f_\rho(\bx) \right] \\
    = 2p\E_{\bx} \big[\psi(-\gamma(\bx)) \cdot -y(\bx)\nabla_{\bw} f_\rho(\bx) \big] = 2\E_{\bx} \big[\psi(-\gamma(\bx)) \cdot \nabla_{\bw} p\ell_0(\bx) \big].
\end{multline*}

\paragraph{Why the margins are ``equal and opposite''.} The central structural fact of Phase II is that the classification margins concentrate tightly about certain average values of $\pm \wsp^2$, where the sign is positive for the majority group and negative for the minority group (see \Figref{lambda01_margin}).
Define
\begin{equation*}
    \gamma_+\coloneqq \frac{1}{p}\sum_{(a,\bw)\in S^+} (\wsp)^2 \qquad \gamma_- \coloneqq \frac{1}{p}\sum_{(a,\bw)\in S^-} (\wsp)^2.
\end{equation*}
Then, \eg in the case where $y(\bx)=1$, \Lemref{phase2_margin} states that $\gamma(\bxmaj)\approx \gamma_+$ and $\gamma(\bxmin)\approx -\gamma_-$ with high probability over $\bxmaj\sim\Unif(\Xmaj)$ and $\bxmin\sim \Unif(\Xmin)$.
This holds because the active ReLUs are the same for all majority group and minority group points, respectively.
The key object we track in Phase II is the average of the positive and negative margins, defined by $\bar{\gamma}\coloneqq \tfrac{1}{2}(\gamma_++\gamma_-)$.

\paragraph{Why the spurious feature continues to grow, despite slowing down.}%
We can compute $\nabla_{\bwsp} L_\rho$ directly using this decomposition (\Lemref{sp_lp}).
Controlling error terms in \Lemref{wsp_phase2_induction}, we obtain
\begin{equation} \label{eq:wsp_recurrence}
    \wsptplus - \wspt \asymp \eta \wspt \left(1-\lambda - \psi(\bgammat) \right).
\end{equation}
This recurrence induces a sigmoidal structure: when $\bar{\gamma}\ll 1$, we have $\psi(\bar{\gamma})\approx \tfrac{1}{2}$, so $\norm{\bwsp}$ continues to increase exponentially.
Once $\bar{\gamma}\gg 1$, we have $1-\psi(\bar{\gamma})\approx e^{-\bar{\gamma}}$, so growth decelerates dramatically as the margins become large.
We show in \Propref{phase2} that Phase II lasts $\TII\asymp \tfrac{1}{\eta}\big(\tfrac{\log(d)}{\log\log(d)}+\log(\tfrac{1}{\theta})\big)$ iterations until $\bar{\gamma}\asymp \log\log(d)$.
Crucially, monotone growth of $\bar{\gamma}$ throughout Phase II requires $e^{-\bar{\gamma}}-\lambda \geq 0$, \ie $\lambda \ll \log^{-1}(d)$.
This is the only place where we require the ``extreme'' correlation.

\paragraph{Why the signal feature is suppressed.} The key difference from Phase I is that signal growth is no longer merely slow --- it is actively suppressed by the growing spurious feature. 
We show that both terms in the signal gradient bound of \Lemref{sigopp_lp} become exponentially small in $\wsp^2$ (via margin analysis and Hoeffding's inequality, respectively).
Combining the terms in \Lemref{sigopp_phase2_induction} yields
\begin{equation*}
    \norm{\bwsigtplus - \bwsigt} \lesssim \eta e^{-\bgammat}\Big(\max_{(\at,\bwt)}(\wspt)^2\norm{\bwsigt} \Big) + \eta\theta\log^{-1}(d)d^{-1/2}.
\end{equation*}
As $\bar{\gamma}$ grows monotonically throughout Phase II, the exponential suppression only strengthens over time, guaranteeing that $\norm{\bwsig}$ grows by only a polylogarithmic factor (\Lemref{phase2_norms_wsig}).
The same decay governs the $\norm{\bwperp}$ and $\norm{\bwperp}_\infty$ inductions (\Lemref{wperp_phase2_induction} and \Lemref{wi_phase2_induction} respectively).

\paragraph{Closing the induction.} As in Phase I, all five component bounds are maintained jointly as an inductive hypothesis (\Defref{phase2_hypothesis}) across $\TII$ iterations.
Minibatch concentration requires a larger batch size $m\gg d\log^6(d)\theta^{-2}$ to keep the component norms proportional to the initialization scale $\theta$; note that this could be $\text{poly}(d)$ if $\theta\asymp \text{poly}^{-1}(d)$.
The Phase II induction concludes that $\sgn(f_\rho(\bx))=x_3$ with high probability, such that $\text{Acc}_{\Xmaj}(f_\rho)\geq 1-d^{-C}$ and $\text{Acc}_{\Xmin}(f_\rho)\leq d^{-C}$. %

\begin{figure}[t]
    \centering
    \begin{subfigure}[b]{0.245\textwidth}
        \centering
        \includegraphics[width=\textwidth, height=0.15\textheight]{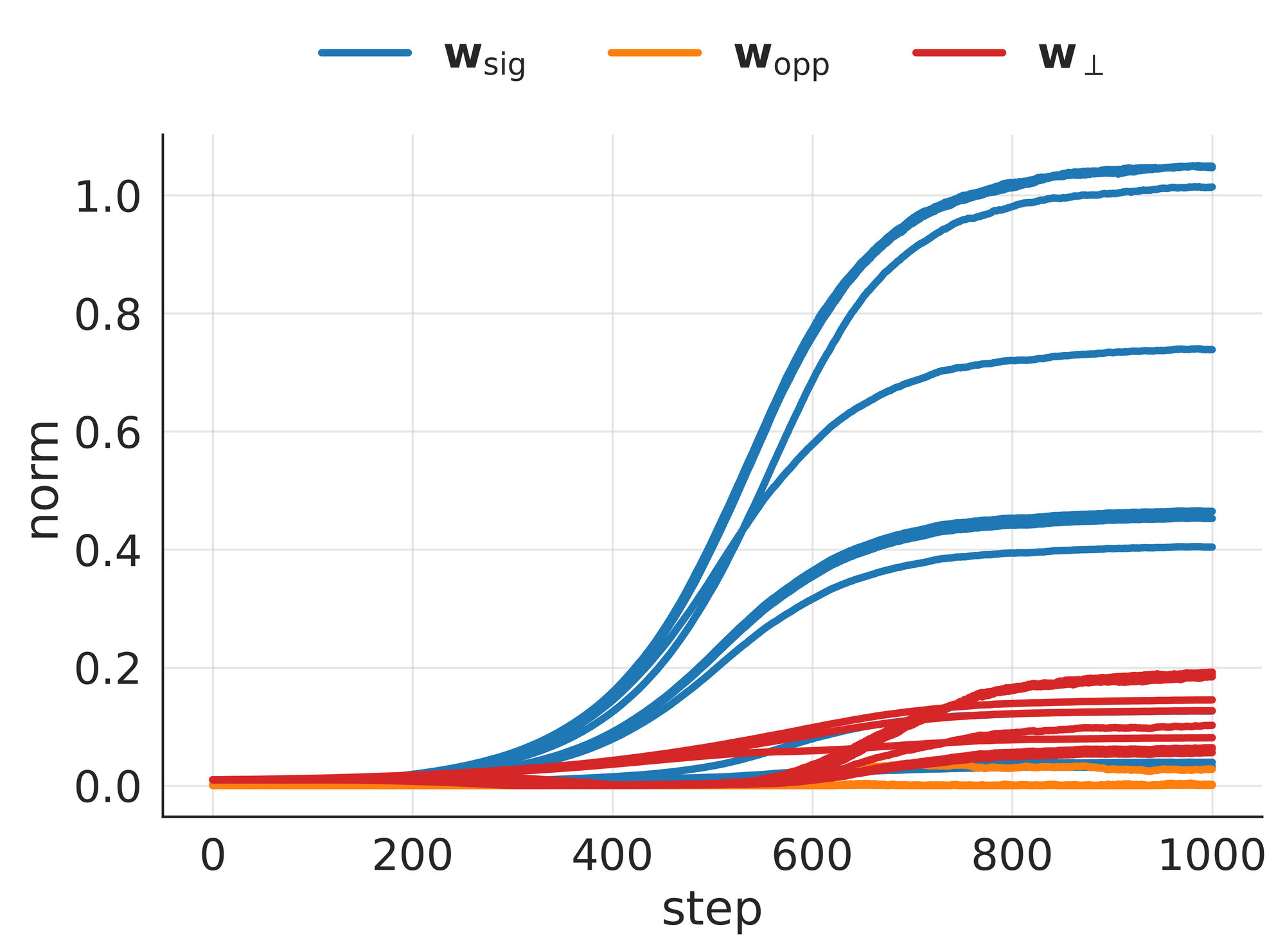}
        \caption{No spurious feature}
        \label{fig:glasgow}
    \end{subfigure}
    \begin{subfigure}[b]{0.245\textwidth}
        \centering
        \includegraphics[width=\textwidth, height=0.15\textheight]{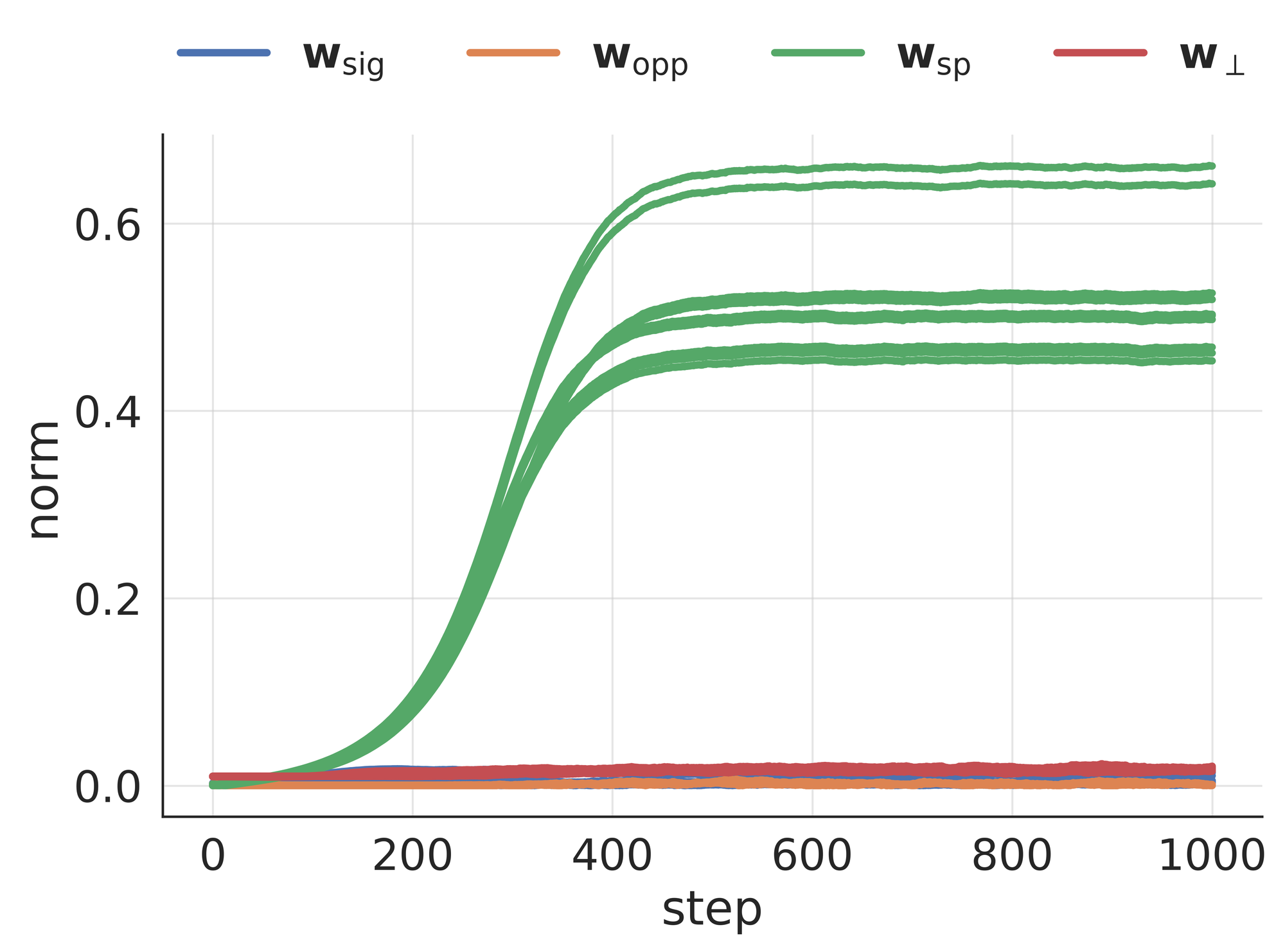}
        \caption{$\lambda=0.1$}
        \label{fig:lambda01}
    \end{subfigure}
    \begin{subfigure}[b]{0.245\textwidth}
        \centering
        \includegraphics[width=\textwidth, height=0.15\textheight]{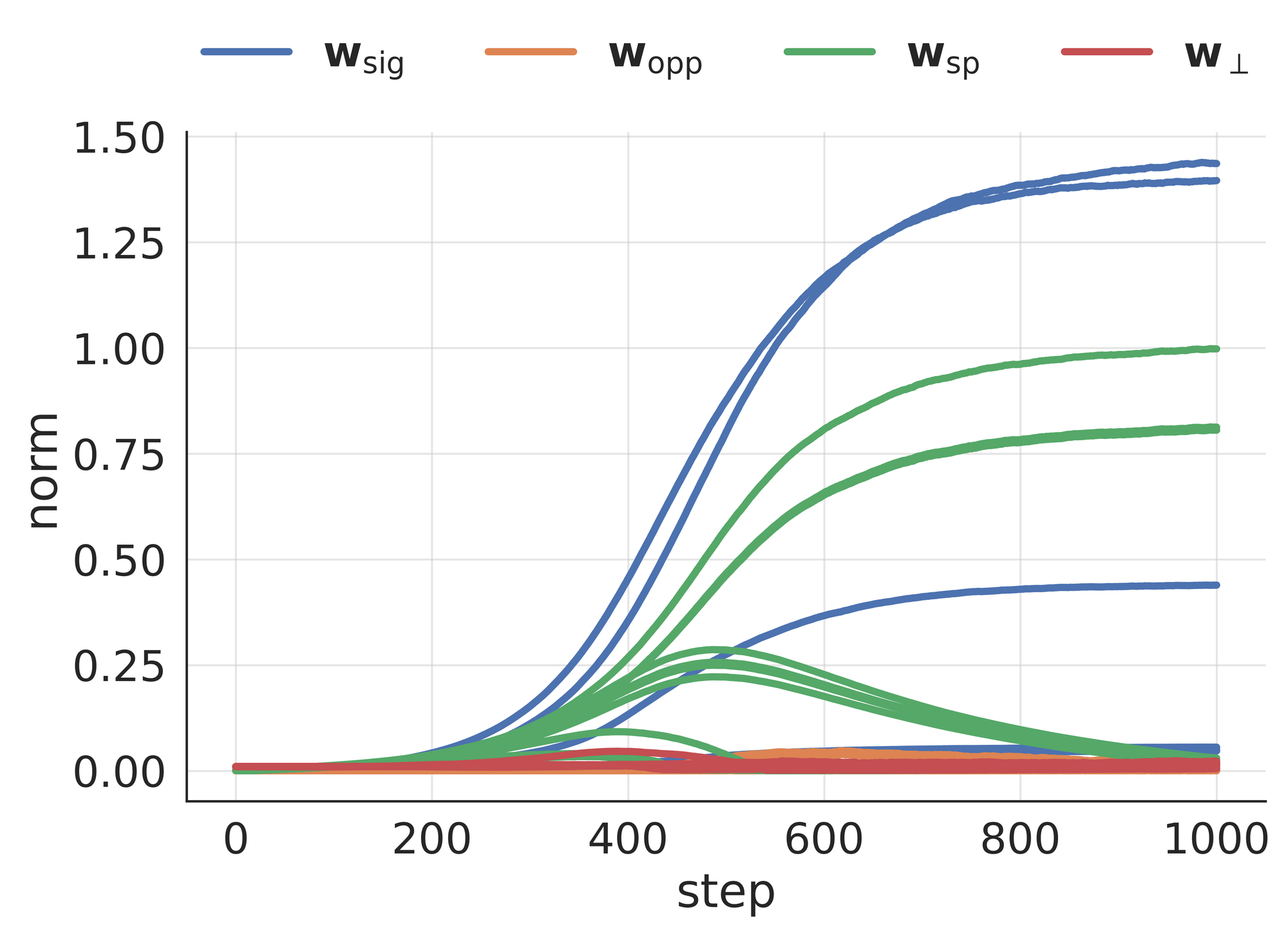}
        \caption{$\lambda=0.15$}
        \label{fig:lambda015}
    \end{subfigure}
    \begin{subfigure}[b]{0.245\textwidth}
        \centering
        \includegraphics[width=\textwidth, height=0.15\textheight]{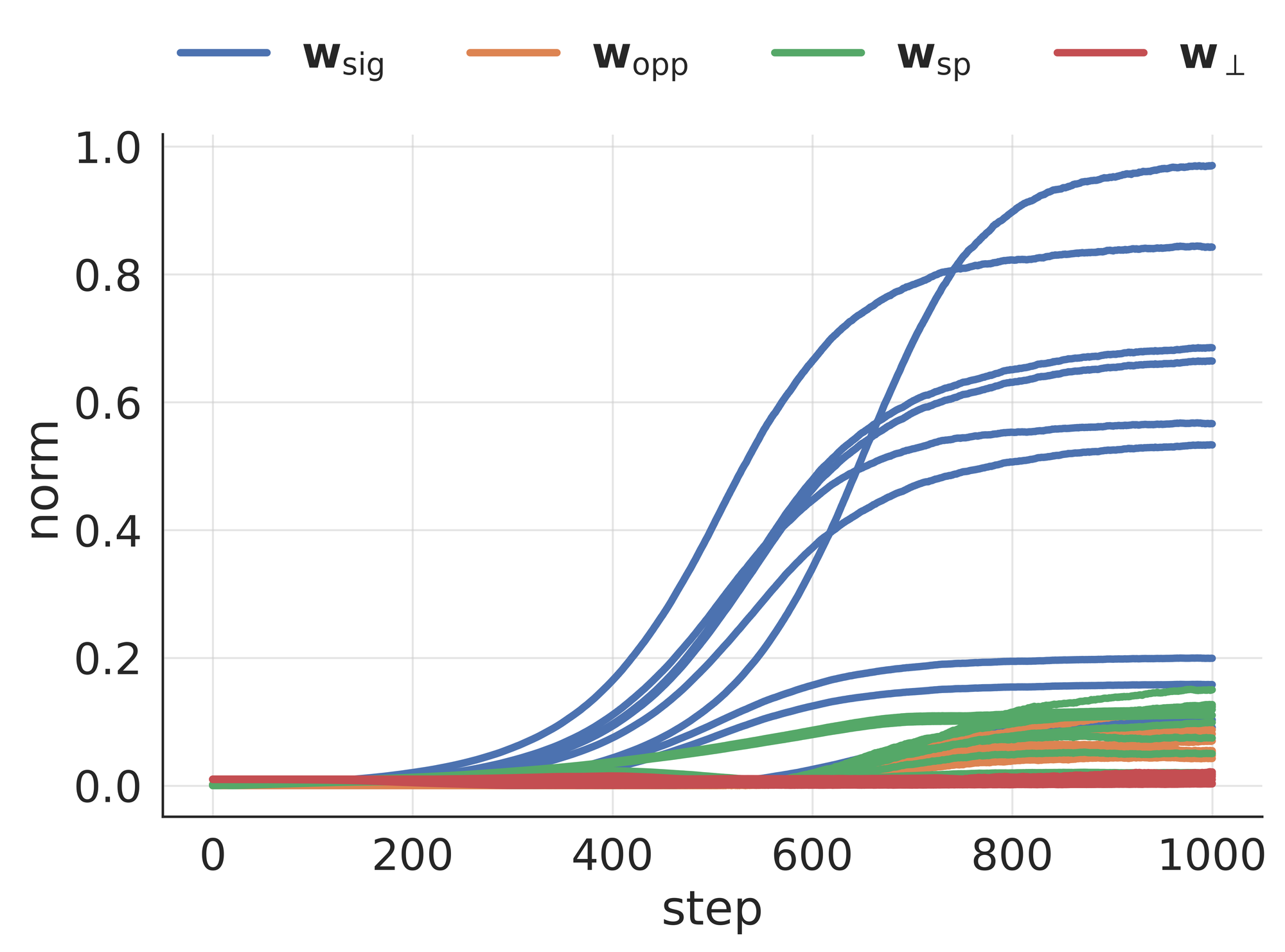}
        \caption{$\lambda=0.2$}
        \label{fig:lambda02}
    \end{subfigure}
    \caption{\textbf{XOR signal can be learned for large enough $\lambda$.} We display the results of training runs with dimension $d=100$, learning rate $\eta=0.05$, width $p=10$, initialization scale $\theta=0.01$, and batch size $m=5000$. For each of the $p=10$ neurons, we plot $\norm{\bwsig}$, $\norm{\bwopp}$, $\norm{\bwsp}$, and $\norm{\bwperp}$ (defined in \Secref{feature_learning_analysis}). We find that for large enough $\lambda$ --- corresponding to $\lambda\gg \log^{-1}(d)$ according to our theory --- the XOR signal can overtake the spurious feature. The network decomposes into disjoint signal and spurious subnetworks, with the size of the signal subnetwork increasing in $\lambda$.}
    \label{fig:lambdas}
\end{figure}

\section{The Constant $\lambda$ Case: When Can the XOR Signal be Learned?}

In this section, we address whether the XOR signal can be learned when the ``extreme'' correlation condition $\lambda\ll \log^{-1}(d)$ is \emph{not} satisfied, \eg the spurious correlation strength $\lambda$ is a constant.

\subsection{Intuition and Simulations}
\Thmref{main} requires $\lambda\ll \log^{-1}(d)$ to ensure monotonic growth of $\norm{\bwsp}$ in Phase II, the critical property for the suppression of the signal feature.
To see what happens when this condition is not satisfied, observe that the Phase II spurious feature recurrence (\Eqref{wsp_recurrence}) has a fixed point $\psi(\bar{\gamma})=1-\lambda$.
Since $\bar{\gamma}$ initially grows monotonically, if $\lambda$ were constant we would eventually reach a value which overshoots this fixed point. %
Then, the right-hand side of the recurrence would become negative, \ie $\wsptplus - \wspt \approx -\eta\wspt$, such that $\norm{\bwsp}$ would begin to decrease geometrically!
If it decreased enough, the exponential suppression of the signal feature would weaken, which may provide an opportunity for the XOR signal to be learned later in training.

In \Figref{lambdas}, we validate this hypothesis through simulation.
We contrast signal learning in the standard XOR model without a spurious feature~\citep{glasgow2024sgd} (\Figref{glasgow}) with our setting, varying the strength of the spurious correlation.
We find that the XOR signal can be learned for large enough $\lambda$, as the non-monotonic evolution of $\bwsp$ reduces signal suppression, enabling several neurons to pick up $\bwsig$ around step $500$.
\Figref{lambda015} shows that $\norm{\bwsp}$ increases geometrically --- in line with our theory --- then decreases geometrically.
An important observation about \Figref{lambda015} and \Figref{lambda02} is that the neurons decompose into disjoint \emph{subnetworks} which emphasize only one of the signal or spurious feature --- that is, the large $\norm{\bwsig}$ lines in \Figref{lambda015} correspond precisely to neurons with small $\norm{\bwsp}$, and vice versa.
Moreover, the spurious feature is not fully ``forgotten'' even for large $\lambda$.

The dynamics of the constant correlation case have a substantial impact on margin behavior and worst-group accuracy.
In \Figref{margins}, we plot the average margin over the majority and minority groups in each minibatch: while the minority group has highly negative margin at small values of $\lambda$ (implying low minority group accuracy), the picture is more benign at large values of $\lambda$ (implying high minority group accuracy).
\Figref{margins} also suggests that a theoretical analysis of this case will require a different approach --- margin concentration is the critical structural fact of our Phase II analysis, and it clearly does not hold for large $\lambda$.
We provide an idealized theoretical analysis, and discuss possible approaches for a full theory, in the next section.

\begin{figure}[t]
    \centering
    \begin{subfigure}[b]{0.325\textwidth}
        \centering
        \includegraphics[width=\textwidth, height=0.18\textheight]{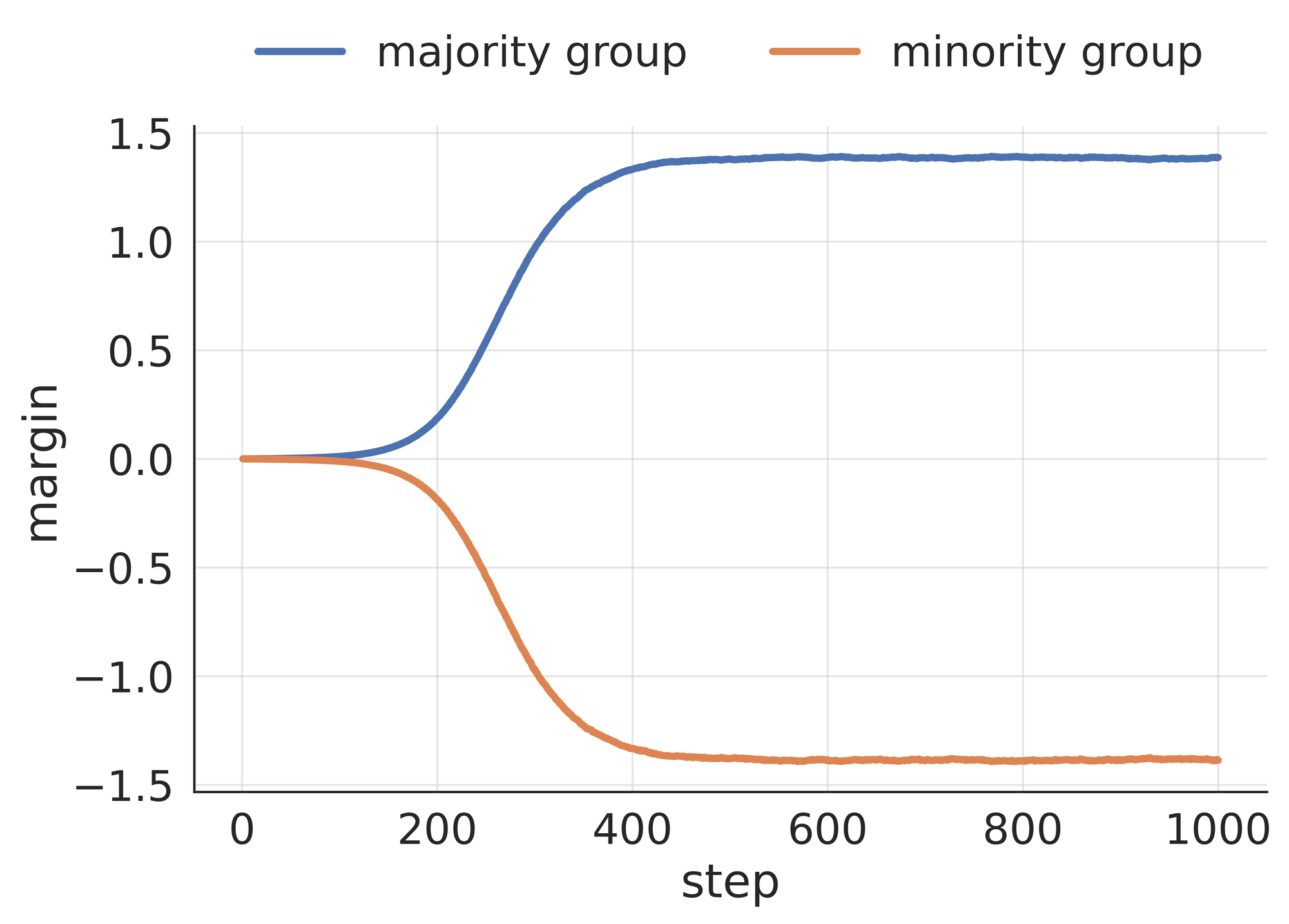}
        \caption{$\lambda=0.1$}
        \label{fig:lambda01_margin}
    \end{subfigure}
    \begin{subfigure}[b]{0.325\textwidth}
        \centering
        \includegraphics[width=\textwidth, height=0.18\textheight]{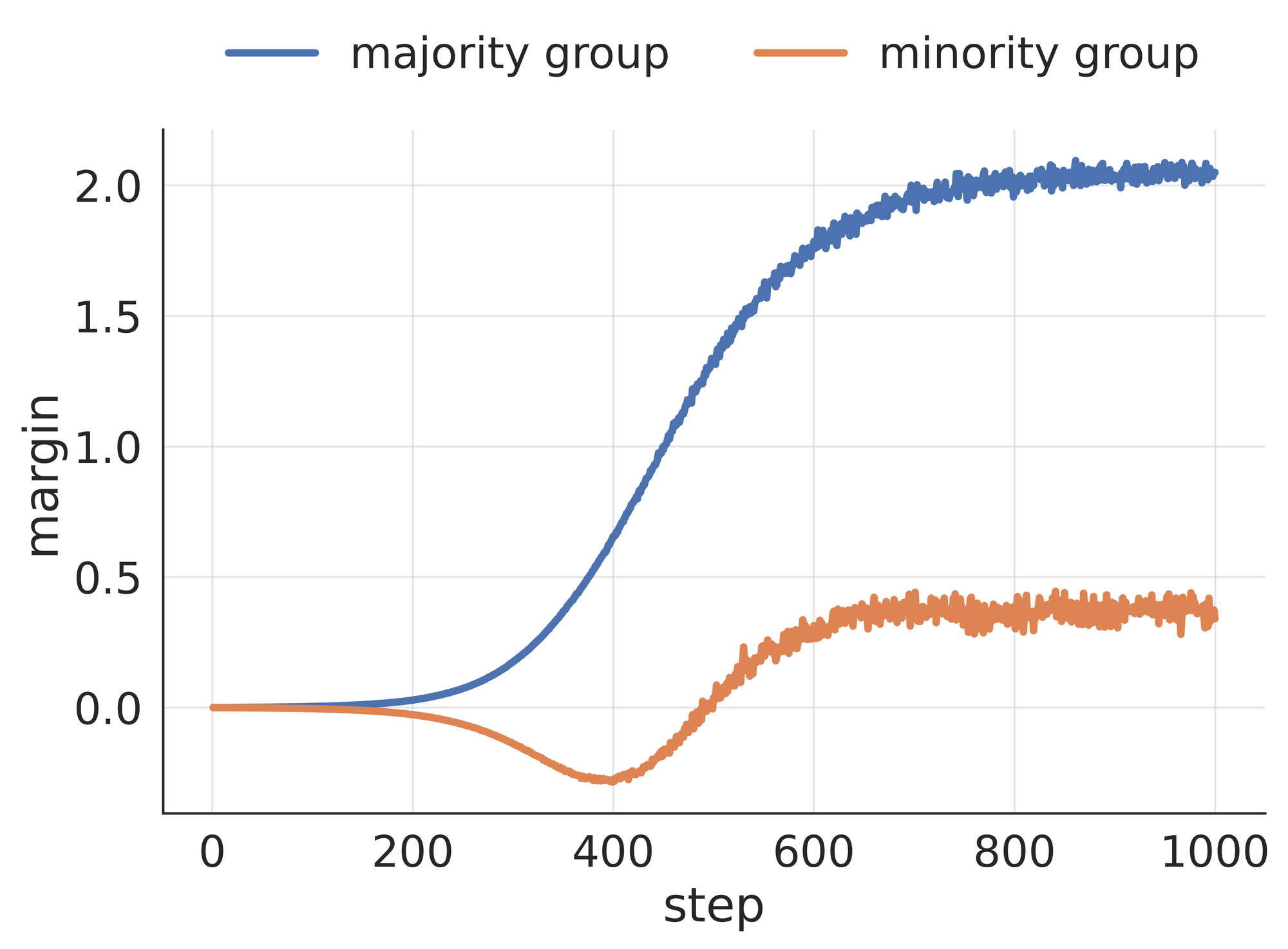}
        \caption{$\lambda=0.15$}
        \label{fig:lambda015_margin}
    \end{subfigure}
    \begin{subfigure}[b]{0.325\textwidth}
        \centering
        \includegraphics[width=\textwidth, height=0.18\textheight]{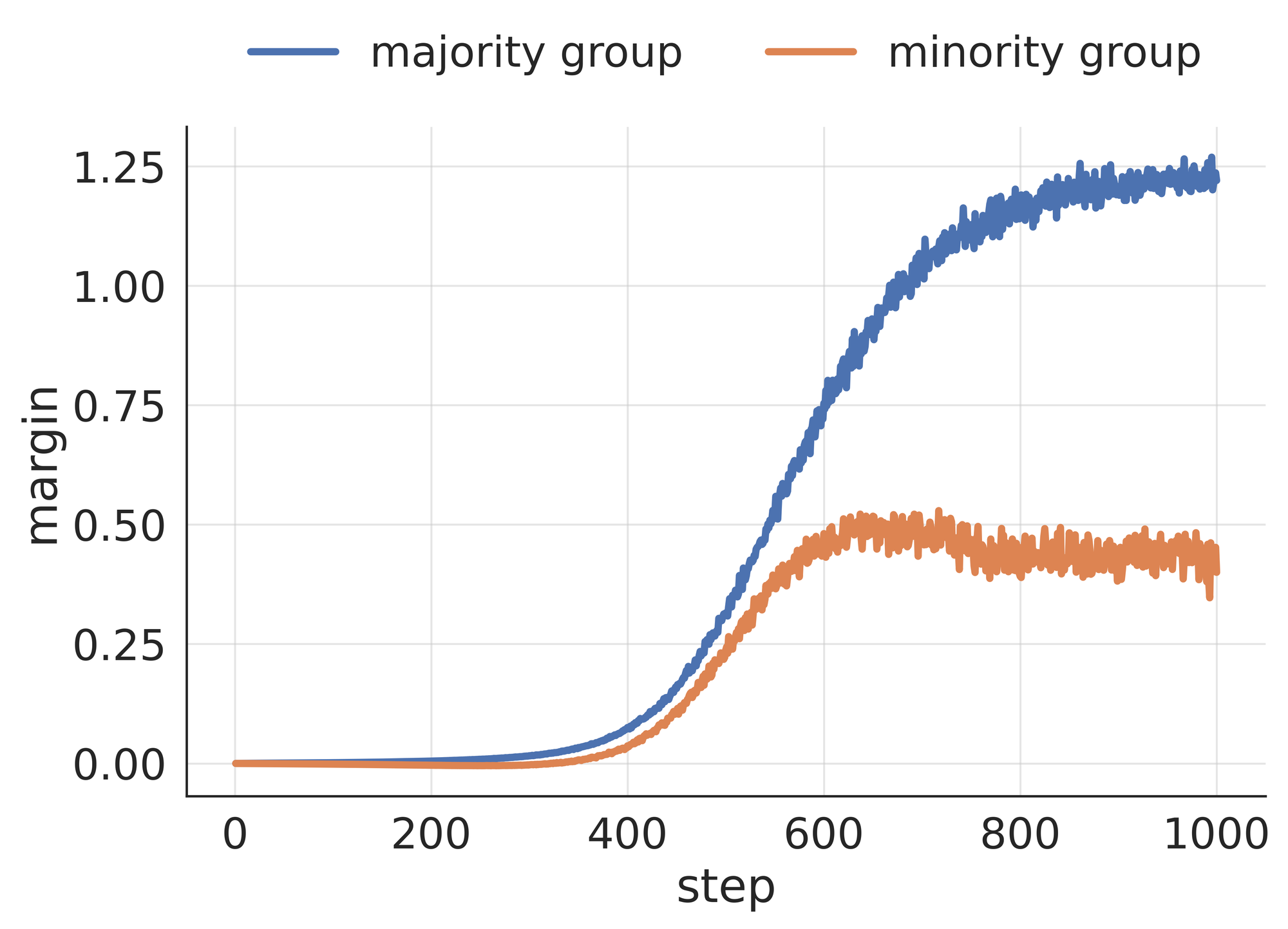}
        \caption{$\lambda=0.2$}
        \label{fig:lambda02_margin}
    \end{subfigure}
    \caption{\textbf{``Equal and opposite'' margins only hold for small $\lambda$.} We display the results of training runs with dimension $d=100$, learning rate $\eta=0.05$, width $p=50$, initialization scale $\theta=0.01$, and batch size $m=5000$.
    We plot the average per-group margin in each minibatch.
    We find that for small enough $\lambda$ --- corresponding to $\lambda \ll \log^{-1}(d)$ according to our theory --- the margins concentrate about certain averages of $\pm \wsp^2$, where the sign is positive for the majority group and negative for the minority group.
    For larger $\lambda$, the minority group margin exhibits improved behavior, eventually leading to perfect classification on both the majority and minority groups.}
    \label{fig:margins}
\end{figure}

\subsection{Theory for an Idealized Case} \label{sec:constant_lambda_idealized}
Using the insights developed in our work, we can show that the XOR signal feature overtakes the spurious feature in an idealized case for a single representative neuron.
A full formal analysis of the network dynamics is left open (and will likely require new techniques, as outlined at the end of this section).
Since our Phase I analysis already holds for constant $\lambda$, we need only provide an idealized version of the Phase II dynamics.
Let us consider the following two-variable reduction based on an ODE formulation of \Lemref{wsp_phase2_induction} and \Lemref{sigopp_phase2_induction}: fix $\lambda\in(0,\tfrac{1}{2})$ and let $\psi(u)\coloneqq 1/(1+e^{-u})$ denote the sigmoid.
For a single neuron $(a,\bw)$ define the majority and minority group margins
\begin{equation*}
    \gamma_{\textnormal{maj}} \coloneqq |a|(\norm{\bwsig}+\norm{\bwsp}) \qquad \gamma_{\textnormal{min}} \coloneqq |a|(\norm{\bwsig}-\norm{\bwsp}),
\end{equation*}
and the loss weights
\begin{equation*}
    \Gamma_{\textnormal{sig}} \coloneqq (1-\lambda)\psi(-\gamma_{\textnormal{maj}}) + \lambda\psi(-\gamma_{\textnormal{min}}) \qquad \Gamma_{\textnormal{sp}} \coloneqq (1-\lambda)\psi(-\gamma_{\textnormal{maj}}) - \lambda\psi(-\gamma_{\textnormal{min}}).
\end{equation*}
Then, the idealized Phase II dynamics are
\begin{equation*}
    \norm{\bwsigtplus}-\norm{\bwsigt} = \eta (\at)^2 \Gamma_{\textnormal{sig}}^{(t)} \norm{\bwsigt} \qquad \norm{\bwsptplus} - \norm{\bwspt} = \eta |\at| \Gamma_{\textnormal{sp}}^{(t)},
\end{equation*}
with balancedness condition $|\at|=\norm{\bwspt}$.
Each component of the reduction corresponds to a statement from our Phase II analysis, with lower-order terms dropped (details in Appendix~\ref{sec:crossing_time}).

Writing $\tau=\eta t$ and letting $\eta\to 0$, we find that these idealized dynamics yield a ``crossing time'' where the XOR signal feature overtakes the linear spurious feature, similar to our real XOR experiments.
\begin{lemma} \label{lem:constant_lambda_idealized}
    Fix $\lambda\in(0,\tfrac{1}{2})$ and let $\gamma^\star\coloneqq\log\big(\tfrac{1-\lambda}{\lambda}\big)$ and $R(\lambda)\coloneqq 2\lambda(1-\lambda)\gamma^\star$.
    Under the idealized dynamics above,
    \begin{enumerate}[label=(\roman*)]
        \item In the absence of signal, $\norm{\bwsp}$ converges monotonically to the equilibrium $\sqrt{\gamma^\star}$.
        \item Suppose $\norm{\bwsp^{(0)}}=\sqrt{\gamma^\star}$ and $\norm{\bwsig^{(0)}}\ll 1$, and neglect the feedback of $\norm{\bwsig}$ onto the margins. Then $\norm{\bwsig}$ grows geometrically at rate $R(\lambda)$, and $\norm{\bwsig^{(\tau^\star)}}=\norm{\bwsp^{(\tau^\star)}}$ at
        \begin{equation*}
            \tau^\star = \frac{\log(\sqrt{\gamma^\star}/\norm{\bwsig^{(0)}})}{R(\lambda)}.
        \end{equation*}
    \end{enumerate}
\end{lemma}

We prove \Lemref{constant_lambda_idealized} in Appendix~\ref{sec:crossing_time}.
Neglecting the feedback in (ii) is accurate while $\norm{\bwsig}\ll\norm{\bwsp}$, which holds for all but $O(1)$ time before the crossing, so the full idealized dynamics cross at $\tau^\star+O(1)$.
Providing such idealized dynamics is more challenging in the case of multiple neurons, though we expect it to be possible in the limit as the width of the network goes to infinity.
In this case, we propose to take the mean-field limit and study the Wasserstein gradient flow~\citep{mei2018mean, chizat2020implicit}, \ie derive dynamics on the distribution of the signal and spurious features, majority and minority group margins, etc.
Our analysis techniques would be a natural starting point for this future work, but we expect adapting them to be nontrivial --- in particular, one would need to derive analogous lemmas to those in Appendices \ref{sec:l0_analysis} and \ref{sec:lp_analysis} to characterize the Wasserstein gradients.

\section{Discussion} \label{sec:discussion}
In this paper, we presented the first end-to-end theoretical characterization of spurious feature learning for two-layer ReLU neural networks trained by online minibatch SGD on the logistic loss.
Our results provide a foundation for rigorous understanding of spurious feature learning.
Of particular note, our insights formalize and/or theoretically corroborate the four main findings of~\cite{qiu2024complexity} as follows:
\begin{itemize}
    \item \emph{Easier spurious features lead to slower signal feature emergence}: We formally prove this statement, with precise rates, for a linear spurious feature and quadratic signal feature.
    \item \emph{Learning may not separate into ``spurious'' and ``signal'' phases}: Our results imply that spurious and signal learning are entangled in Phase Ia (and in the constant $\lambda$ case).
    \item \emph{Spurious features are retained after the signal feature is learned}: We show empirically that the spurious feature is preserved even when $\lambda$ is large and constant.
    \item \emph{Spurious and signal features are learned by disjoint subnetworks}: We show empirically that the model decomposes into signal and spurious subnetworks in the constant $\lambda$ case.
\end{itemize}

Overall, we believe our theory could eventually inform the design of robustness-advancing interventions.
For example, the length of our Phase I constitutes a lower bound on the number of iterations necessary before the early-stopped checkpoint of \emph{Just Train Twice}~\citep{liu2021just} is saved --- otherwise, the misclassifications may not correlate well with the minority group.
Likewise, our observation that disjoint signal and spurious subnetworks emerge in the constant-correlation case justifies the assumption of \emph{Deep Feature Reweighting}~\citep{kirichenko2023last}, and suggests that it works by identifying and upweighting the signal subnetwork.
Finally, \emph{group/class-balancing}, or throwing away majority group/class data, has also been observed to improve minority group accuracy~\citep{idrissi2022simple, chaudhuri2023why, labonte2024group}. This mechanism increases the effective value of $\lambda$, which our Phase I theory predicts would lead to slower initial acquisition of the spurious feature.

\paragraph{Acknowledgements.} We thank Jacob Abernethy for compute assistance, Surbhi Goel and Steve Mussmann for insightful conversations, Xinchen Zhang for contributions to exploratory simulations, and Peter Chen for helpful editing suggestions.
VM gratefully acknowledges the support of the NSF (through award CCF-2239151 and award IIS-2212182), an Adobe Data Science Research Award, and an Amazon Research Award.

\bibliography{refs}

\clearpage

\appendix

\renewcommand{\partname}{}
\renewcommand{\thepart}{}

\part*{Appendix}
\addcontentsline{toc}{part}{Appendix}

\etocsettocstyle{\section*{Table of Contents}}{}  %
\etocsettocdepth{subsection}   %

\localtableofcontents           %

\clearpage

\section{Notation}
\label{sec:notation}

We summarize the notation used throughout this paper.
Unless otherwise stated, all asymptotic
notation is with respect to $d \to \infty$.

\subsection*{General Mathematical Notation}

\begin{longtable}{p{0.28\textwidth} p{0.66\textwidth}}
\toprule
\textbf{Symbol} & \textbf{Description} \\
\midrule
\endhead
\bottomrule
\endfoot

$\mathbf{X}$, $\bx$, $x$
    & Uppercase bold: matrix; lowercase bold: vector; italic: scalar \\[3pt]

$\be_k$
    & $k$-th standard basis vector (dimension clear from context) \\[3pt]

$\bzero$
    & All-zeros vector (dimension clear from context) \\[3pt]

$\mathbf{I}_d$ %
    & $d\times d$ identity matrix \\[3pt]

$\|\cdot\|$
    & Euclidean ($\ell_2$) norm of a vector \\[3pt]

$\|\cdot\|_\infty$
    & $\ell_\infty$ norm of a vector \\[3pt]

$\|\cdot\|_p$
    & $\ell_p$ norm of a vector for $1\leq p < \infty$ \\[3pt]

$\mathbb{P}_{\bx}$, $\mathbb{E}_{\bx}$
    & Probability and expectation with respect to random vector $\bx$ \\[3pt]

$\ind(\cdot)$
    & Indicator function of an event \\[3pt]

$\sgn(\cdot)$
    & Sign function \\[3pt]

$[t]$
    & $\{1, 2, \ldots, t\}$ for $t \in \mathbb{N}$ \\[3pt]

$A\Delta B$
    & Symmetric difference of sets: $A\Delta B \coloneqq (A\setminus B)\cup (B\setminus A)$ \\[3pt]

$A^{\textsf{C}}$ %
    & Complement of the event or set $A$ \\[3pt]

$\mathrm{Proj}_{\by}(\bx)$
    & Projection of $\bx$ onto $\by$:
      $\mathrm{Proj}_{\by}(\bx)\coloneqq \frac{\by^\top \bx}{\|\by\|^2}\,\by$ \\[3pt]

$\text{Ei}(z)$ %
    & Exponential integral: $\text{Ei}(z)\coloneqq \int_{-\infty}^z \frac{e^t}{t}\,dt$ \\[3pt]

\midrule
\multicolumn{2}{l}{\textit{Asymptotic notation (all with respect to $d \to \infty$)}} \\[3pt]

$x \ll y$
    & $x = o(y)$ \\[3pt]

$x \lesssim y$
    & $x = O(y)$ \\[3pt]

$x \gg y$
    & $x = \omega(y)$ \\[3pt]

$x \gtrsim y$
    & $x = \Omega(y)$ \\[3pt]

$x \asymp y$
    & $x = \Theta(y)$ \\[3pt]

$x = (1\pm o(1))\cdot y$ %
    & $x\in [(1-o(1))\cdot y, (1+o(1))\cdot y]$ \\[3pt]

\midrule
\multicolumn{2}{l}{\textit{Probability distributions}} \\[3pt]

$\Unif(\{\pm 1\})$
    & Univariate Rademacher distribution \\[3pt]

$\Unif(\{\pm 1\}^{d})$
    & Multivariate Rademacher distribution \\[3pt]

$\mathrm{Unif}(\mathbb{S}^{d-1}(\theta))$
    & Uniform distribution on the $\ell_2$-sphere in $\mathbb{R}^d$ of radius $\theta$ \\[3pt]

$\Unif(\X)$ %
    & Uniform distribution over a finite set $\X$ \\[3pt]

$\mathcal{N}(\bmu, \mathbf{\Sigma})$
    & Multivariate Gaussian with mean $\bmu$ and covariance $\mathbf{\Sigma}$ \\[3pt]

$\mathcal{N}(\mu, \sigma^2)$
    & Univariate Gaussian with mean $\mu$ and variance $\sigma^2$ \\[3pt]

$\Phi(\cdot)$, $\phi(\cdot)$
    & Standard Gaussian CDF and PDF, respectively \\[3pt]

\bottomrule
\end{longtable}

\subsection*{Problem Setting}

\begin{longtable}{p{0.28\textwidth} p{0.66\textwidth}}
\toprule
\textbf{Symbol} & \textbf{Description} \\
\midrule
\endhead
\bottomrule
\endfoot

$d$
    & Data dimension \\[3pt]

$\bmu_1 = \be_1 - \be_2$
    & Positive signal direction for class $+1$ \\[3pt]

$\bmu_2 = \be_1 + \be_2$
    & Positive signal direction for class $-1$ \\[3pt]

$\lambda \in (0, \Lambda]$, $\Lambda<\tfrac{1}{2}$
    & Spurious correlation strength; smaller $\lambda$ indicates stronger spurious correlation; $\lambda\to \tfrac{1}{2}$ not allowed \\[3pt]

$P_d(\lambda)$
    & Data distribution parameterized by $\lambda$ (\Eqref{XOR-spurious-correlation}) \\[3pt]

$\bxi \sim \mathrm{Unif}(\bzero^3 \times \{\pm 1\}^{d-3})$
    & Noise vector, independent of $\bmu_1, \bmu_2, \be_3$ \\[3pt]

$\bz \coloneqq x_1 \be_1 + x_2 \be_2$
    & Signal component of $\bx$ \\[3pt]

$\bs \coloneqq x_3 \be_3$
    & Spurious component of $\bx$ \\[3pt]

$\bx \coloneqq \bz + \bs + \bxi$
    & Shorthand decomposition of a data point \\[3pt]

$y(\bx) \coloneqq -x_1 x_2$
    & XOR target function (depends only on $\bz$) \\[3pt]

$\Xmaj$
    & Majority group: $\{\bx \in \{\pm 1\}^d : y(\bx) = x_3\}$ \\[3pt]

$\Xmin$
    & Minority group: $\{\bx \in \{\pm 1\}^d : y(\bx) = -x_3\}$ \\[3pt]

$\bxmaj$, $\bxmin$ %
    & Data points drawn from $\Unif(\Xmaj)$ and $\Unif(\Xmin)$, respectively \\[3pt]

$\textnormal{Acc}_{\X}(f)$ %
    & Accuracy of predictor $f$ on $\Unif(\X)$:
      $\P_{\bx\sim \Unif(\X)}\big(\sgn(f(\bx))=y(\bx)\big)$ \\[3pt]

$f^\star(\bx)$
    & Optimal (max-margin) neural network classifier \\[3pt]

$f^{\mathrm{sp}}(\bx) \coloneqq x_3$
    & Spurious predictor; achieves $1-\lambda$ average accuracy and $0$ worst-group accuracy \\[3pt]

\bottomrule
\end{longtable}

\subsection*{Neural Network and Training}

\begin{longtable}{p{0.28\textwidth} p{0.66\textwidth}}
\toprule
\textbf{Symbol} & \textbf{Description} \\
\midrule
\endhead
\bottomrule
\endfoot

$p$
    & Number of neurons (network width) \\[3pt]

$(a_j, \bw_j)$
    & $j$-th neuron: scalar output weight $a_j \in \mathbb{R}$, vector input weight $\bw_j \in \mathbb{R}^d$; we often drop the $j$ indexing \\[3pt]

$(\at, \bwt)$ %
    & A neuron at iteration $t$ \\[3pt]

$\rho$
    & Empirical distribution over neurons $\{(a_j, \bw_j)\}_{j=1}^p$;
      written $\rho^{(t)}$ at iteration $t$ \\[3pt]

$f_\rho(\bx)$
    & Network output:
      $\mathbb{E}_{(a,\bw)\sim\rho}[a\,\sigma(\bw^\top \bx)]
       = \frac{1}{p}\sum_{j=1}^p a_j\,\sigma(\bw_j^\top \bx)$ \\[3pt]

$\sigma(\alpha) \coloneqq \max(0,\alpha)$
    & ReLU activation function \\[3pt]

$\eta$
    & SGD step size/learning rate \\[3pt]

$m$
    & Minibatch size \\[3pt]

$r \sim \Unif(\{\pm 1\})$
    & Rademacher variable used only for initializations \\[3pt]

$\theta$
    & Initialization scale: $\|\bw^{(0)}\| = \theta$ and $|a^{(0)}| = \theta$ \\[3pt]

$M^{(t)} \sim P^m_d(\lambda)$
    & $t$-th minibatch of size $m$ \\[3pt]

$T$ %
    & Total number of SGD iterations \\[3pt]

$\psi(u) \coloneqq 1/(1+e^{-u})$
    & Sigmoid function \\[3pt]

$\gamma(\bx) \coloneqq y(\bx)\,f_\rho(\bx)$
    & Classification margin of datum $\bx$ \\[3pt]

$\gammaplust$
    & Margin on the positive neurons at iteration $t$ \\[3pt]

$\gammaminust$
    & Margin on the negative neurons at iteration $t$ \\[3pt]

$\bgammat$
    & Average margin: $\tfrac{1}{2}(\gammaplust + \gammaminust)$ \\[3pt]

$h(\gamma) \coloneqq -2\log(\psi(\gamma))$
    & Logistic loss as a function of margin \\[3pt]

$\ell_\rho(\bx) \coloneqq h(\gamma(\bx))$
    & Per-sample logistic loss \\[3pt]

$\ell_\rho^{(1)}(\bx) \coloneqq h'(\gamma(\bx))$
    & First derivative of the per-sample loss \\[3pt]

$L_\rho \coloneqq \mathbb{E}_{\bx}[\ell_\rho(\bx)]$
    & Population logistic loss \\[3pt]

$\widehat{L}_{\rhot}$
    & Empirical loss on the $t$-th minibatch $M^{(t)}$ \\[3pt]

$\nabla_{\bw} L$
    & $p$-scaled gradient of $L$ with respect to $\bw$:
      $\nabla_{\bw} L\coloneqq p\cdot\tfrac{\partial L}{\partial \bw}$ \\[3pt]

$\partial_u L$
    & $p$-scaled partial derivative of $L$ with respect to scalar $u \in \{a, w_i\}$:
      $\partial_u L \coloneqq p \cdot \tfrac{\partial L}{\partial u}$ \\[3pt]

\bottomrule
\end{longtable}

\subsection*{Weight Decomposition}

Each neuron weight vector $\bw$ is decomposed into four orthogonal components:
$\bw =\bwsig + \bwopp + \bwsp + \bwperp$,
where $\bw_{1:2} := \bwsig + \bwopp$.

\begin{longtable}{p{0.28\textwidth} p{0.66\textwidth}}
\toprule
\textbf{Symbol} & \textbf{Description} \\
\midrule
\endhead
\bottomrule
\endfoot

$\bwsig$
    & Core (signal) feature component, aligned with the XOR-correct
      classification direction; $\bmu_1 \parallel \bwsig$
      for $S^+$ neurons and $\bmu_2 \parallel \bwsig$
      for $S^-$ neurons \\[3pt]

$\bwopp$
    & Opposing feature component, aligned opposite to the signal direction \\[3pt]

$\bwsp \coloneqq w_3\be_3$
    & Spurious feature component (vector) \\[3pt]

$\wsp \coloneqq w_3$
    & Spurious feature component (scalar); projection onto $\be_3$ \\[3pt]

$\bwperp$
    & Orthogonal noise component; perpendicular to
      $\mathrm{span}\{\bmu_1, \bmu_2, \be_3\}$ \\[3pt]

$\bw_{\perp\setminus i}$ %
    & $\bwperp$ with its $i$-th coordinate zeroed \\[3pt]

$\bw_{1:2} \coloneqq \bwsig + \bwopp$
    & Combined projection onto $\mathrm{span}\{\bmu_1, \bmu_2\}$ \\[3pt]

$S^+$
    & Positive neurons: $S^+ \coloneqq \{(a, \bw) : \sgn(a) = +1\}$ \\[3pt]

$S^-$
    & Negative neurons: $S^- \coloneqq \{(a, \bw) : \sgn(a) = -1\}$ \\[3pt]

$\bsmaj$
    & Spurious direction aligned with the majority group:
      $\bsmaj = \be_3$ if $(a,\bw) \in S^+$,
      and $\bsmaj = -\be_3$ if $(a,\bw) \in S^-$ \\[3pt]

\bottomrule
\end{longtable}

\subsection*{Analysis-Specific Notation}

\begin{longtable}{p{0.28\textwidth} p{0.66\textwidth}}
\toprule
\textbf{Symbol} & \textbf{Description} \\
\midrule
\endhead
\bottomrule
\endfoot

$C> 0$
    & Positive constant chosen large enough, which does not vary from line to line \\[3pt]

$c<1, C'<C-3$
    & Other constants, which do not vary from line to line \\[3pt]

$\ell_0(\bx)$
    & First-order Taylor approximation of $\ell_\rho(\bx)$ about $f_\rho = 0$; $\ell_0(\bx)=-2\log(\frac{1}{2})-y(\bx)f_\rho(\bx)$ \\[3pt]

$L_0 \coloneqq \E_{\bx}[\ell_0(\bx)]$
    & Population loss under the $\ell_0$ approximation \\[3pt]

$\ell_0^{(1)}(\bx) = -1$
    & Shorthand for the (constant) first derivative of $\ell_0$ \\[3pt]

$\epsw$
    & Asymmetry term in the spurious feature gradient (\Lemref{sp_l0});
      decays exponentially in $\|\bwsp\|$ \\[3pt]

\midrule
\multicolumn{2}{l}{\textit{High-probability events}} \\[3pt]

$\Einitj$
    & The initialization event for the $j$-th neuron \\[3pt]

$\Einit$ %
    & The initialization event over all neurons ($p$ clear from context) \\[3pt]

$\Ebatcht$
    & The concentration event for the $t$-th minibatch \\[3pt]

$\Ebatch\coloneqq \bigcap_{t=1}^T \Ebatcht$
    & The concentration event over all minibatches ($T$ clear from context) \\[3pt]

$\Egammat$
    & The margin concentration event on the $t$-th minibatch \\[3pt]

$\Egamma\coloneqq\bigcap_{t=\TI}^{\TI+\TII} \Egammat$
    & The margin concentration event over Phase II \\[3pt]

$\Etrain \coloneqq \Einit \cap \Ebatch \cap \Egamma$
    & The train event over all neuron initializations and minibatches \\[3pt]

$\Etestj$
    & The concentration event for the $j$-th neuron \\[3pt]

$\Etest \coloneqq \bigcap_{j=1}^p \Etestj$
    & The test event over all neurons ($p$ clear from context) \\[3pt]

\midrule
\multicolumn{2}{l}{\textit{Training phases}} \\[3pt]

$\TI = \TIa + \TIb$
    & Total number of iterations in Phase~I \\[3pt]

$\TIa \asymp \log^{1/2}(d)d^{-1/2}\eta^{-1}$
    & Duration of Phase~Ia (sign alignment of $\wsp$ and $a$) \\[3pt]

$\TIb \asymp \log\log(d)\,\eta^{-1}$
    & Duration of Phase~Ib (growth until $\wsp$ dominates) \\[3pt]

$\scriptstyle \TII\asymp \tfrac{1}{\eta}\big(\tfrac{\log(d)}{\log\log(d)}+\log(\tfrac{1}{\theta})\big)$
    & Total number of iterations in Phase~II (until $\bar{\gamma}\asymp \log\log(d)$) \\[3pt]

$G_{t_1 \to t_2}$
    & Growth factor from iteration $t_1$ to $t_2$ \\[3pt]

\midrule %
\multicolumn{2}{l}{\textit{Idealized dynamics}} \\[3pt]

$\tau \coloneqq \eta t$ %
    & Continuous time in the limit $\eta\to 0$ \\[3pt]

$\gamma_{\textnormal{maj}}$, $\gamma_{\textnormal{min}}$ %
    & Majority and minority group margins of a single neuron \\[3pt]

$\Gamma_{\textnormal{sig}}$, $\Gamma_{\textnormal{sp}}$ %
    & Loss weights governing the signal and spurious updates \\[3pt]

$\gamma^\star \coloneqq \log\big(\tfrac{1-\lambda}{\lambda}\big)$ %
    & Equilibrium margin; $\norm{\bwsp}\to\sqrt{\gamma^\star}$ in the absence of signal \\[3pt]

$R(\lambda) \coloneqq 2\lambda(1-\lambda)\gamma^\star$ %
    & Geometric growth rate of $\norm{\bwsig}$ at the spurious equilibrium \\[3pt]

$\tau^\star$ %
    & Crossing time at which $\norm{\bwsig}=\norm{\bwsp}$ \\[3pt]

\midrule
\multicolumn{2}{l}{\textit{Technical lemma notation}} \\[3pt]

$B \coloneqq \{i : |\Delta_i| \geq \theta d^{-1/2}\}$
    & ``Bad'' index set for Berry-Esseen approximation arguments \\[3pt]

$\bv_B$, $\bv_{\setminus B}$
    & Restrictions of vector $\mathbf{v}$ to indices in $B$ and not in $B$, respectively \\[3pt]

$\Delta(\bv, k)$
    & Truncated moment approximation error (\Lemref{berry_esseen_tail}) \\[3pt]

$\Psi(k)$
    & $\E_{G \sim \mathcal{N}(0,\|\bwperp\|^2)}[|G|\,\ind(|G| \geq k)]$ \\[3pt]

$\bx_{\setminus i}$
    & $\bx - x_i \be_i$ (vector with $i$-th coordinate zeroed) \\[3pt]

$\bxi_{\setminus i}$ %
    & $\bxi - x_i \be_i$ (noise vector with $i$-th coordinate zeroed) \\[3pt]

$\bu_{1:3}$, $\bu_{4:}$ %
    & Shorthands for $\bu\sim\Nc(\bzero,\mathbf{I}_d)$: $\bu_{1:3} = (u_1, u_2, u_3, \bzero)$
      and $\bu_{4:} = (\bzero, u_4, u_5, \dots)$ \\[3pt]

\bottomrule
\end{longtable}

\clearpage

\section{\texorpdfstring{$L_0$}{L0} Analysis} \label{sec:l0_analysis}
In this section, we provide analysis when the neural network is small and hence $\ell_\rho$ is well-approximated by a first-order Taylor expansion about $f_\rho=0$, \ie
\begin{equation*}
    \ell_0(\bx) \coloneqq -2\log(\tfrac{1}{2})-y(\bx)f_\rho(\bx).
\end{equation*}
In~\Secref{l0_computation}, we derive the $L_0$ population partial derivatives for any neuron $(a,\bw)$.
Then, in~\Secref{error_term_bound}, we provide a bound for a certain asymmetry term in the $\bwsp$ partial derivative.

\subsection{\texorpdfstring{$\nabla L_0$}{L0 Gradient} Computation} \label{sec:l0_computation}

Let us begin with computation of the $L_0$ population partial derivatives (\cf Lemma C.4 of~\cite{glasgow2024sgd}).
Recall we define the positive neurons by $S^+\coloneqq \{(a,\bw):\sgn(a)=1\}$ and the ``negative neurons'' by $S^-\coloneqq \{(a,\bw):\sgn(a)=-1\}$.
Moreover, for any of the $p$ neurons $(a,\bw)$, we have
\begin{equation*}
    \nabla_{\bw} L_0 \coloneqq \E_{\bx} \left[\frac{\partial}{\partial\bw}p\ell_0(\bx)\right],
\end{equation*}
and that $\bmu_1 \parallel \bwsig$ if $(a,\bw)\in S^+$ while $\bmu_2 \parallel \bwsig$ if $(a,\bw)\in S^-$.
\begin{lemma}\label{lem:sigopp_l0}
    For any neuron $(a,\bw)$, we have
    \begin{align*}
        -\bwsig^\top\nabla_{\bw} L_0 &= \frac{\sqrt{2}}{4}|a|\norm{\bwsig}\P_{\bxi} \left(|\bw^\top\bxi+\bw^\top\be_3|\leq\sqrt{2}\norm{\bwsig}\right) \\
        -\bwopp^\top\nabla_{\bw} L_0 &= -\frac{\sqrt{2}}{4}|a|\norm{\bwopp}\P_{\bxi} \left(|\bw^\top\bxi+\bw^\top\be_3|\leq\sqrt{2}\norm{\bwopp} \right).
    \end{align*}
\end{lemma}

\begin{proof}
    For the $\bwsig$ component, we have by definition of $L_0$ and $\ell_0$ that
    \begin{align*}
        -\bwsig^\top\nabla_{\bw} L_0 &= -\bwsig^\top\E_{\bx} \left[\frac{\partial}{\partial\bw}p\ell_0(\bx)\right] \\
        &= \bwsig^\top \E_{\bx}\left[ay(\bx)\sigma'(\bw^\top\bx)\bx \right].
    \end{align*}
    where $y(\bx) \coloneqq -x_1 x_2$ is the XOR function.
    Recall that $\bz := x_1 \be_1 + x_2 \be_2 \in \{\pm \bmu_1, \pm \bmu_2\}$.
    Then, since $y(\bx)=y(\bz)$ for any $\bz\in\{\pm\bmu_1,\pm\bmu_2\}$ and $\bwsig\perp \bx-\bz$, we have
    \begin{equation} \label{eq:grad_decomp}
        -\bwsig^\top\nabla_{\bw} L_0 = \E_{\bx, \bz}\left[ay(\bz)\sigma'(\bw^\top\bx)\bwsig^\top \bz \right].
    \end{equation}
    Depending on $\sgn(a)$, we either have ($\bmu_1\parallel\bwsig$ and $\bmu_2\parallel\bwopp$) or ($ \bmu_1\parallel\bwopp$ and $\bmu_2\parallel \bwsig$).
    Moreover, by definition $\bmu_1 \perp \bmu_2$.
    Hence, $\bwsig^\top \bz = 0$ unless $\bz = \pm \bmu_1$ and $\bmu_1 \parallel \bwsig$ or $\bz = \pm \bmu_2$ and $\bmu_2 \parallel \bwsig$ --- in either case, $\bwsig^\top \bz \neq 0$ only if $\bz \parallel \bwsig$.
    This gives us
    \begin{equation*}
        -\bwsig^\top\nabla_{\bw} L_0 = \frac{1}{2}\E_{\substack{\bx, \bz \\\bz\parallel\bwsig}}\left[ay(\bz)\sigma'(\bw^\top\bx)\bwsig^\top \bz\right].
    \end{equation*}
    Now, recall that we defined $\bx \coloneqq \bz + \bs + \bxi$ and that $\bxi$ is independent of $\bz$ and $\bs$.
    Therefore, we can take an iterated expectation over the pair $(\bz,\bs)$ and then over $\bxi$.
    Since we are conditioning on the case where $\bz \parallel \bwsig$, the only cases we have to consider are $\bz = \pm \bmu_1$ for a neuron where $\bwsig \parallel \bmu_1$, and $\bz = \pm \bmu_2$ for a neuron where $\bwsig \parallel \bmu_2$.
    Define $\bsmaj \coloneqq \be_3$ if $(a,\bw)\in S^+$ and $\bsmaj \coloneqq - \be_3$ if $(a,\bw)\in S^-$.
    Then, by Equation~\eqref{eq:XOR-spurious-correlation}, it is easy to see that, conditioned on $\bz \parallel \bwsig$,
    \begin{equation*}
        (\bz, \bs) = \begin{cases}
            (\bz_0, \bsmaj) \text{ w.p. } (1 - \lambda)/2 \\
            (\bz_0, - \bsmaj) \text{ w.p. } \lambda/2 \\
            (- \bz_0, \bsmaj)\text{ w.p. } (1 - \lambda)/2 \\
            (- \bz_0, - \bsmaj) \text{ w.p. } \lambda/2.
        \end{cases}
    \end{equation*}
    In the above, we defined $\bz_0 \in \{\bmu_1,\bmu_2\}$ depending on which case the neuron lies in.
    Substituting this into our expression gives us
    \begin{align*}
        -\bwsig^\top\nabla_{\bw} L_0 &= \frac{1}{2}a\bigg( \\
        &\qquad \left(\frac{1}{2}-\frac{\lambda}{2}\right) \cdot \E_{\bxi} \left[y(\bz_0)\cdot \bwsig^\top\bz_0 \cdot \sigma'\left(\bwsig^\top\bz_0 +\bw^\top\bsmaj+\bw^\top\bxi \right) \right] \\
        &\qquad +\frac{\lambda}{2} \cdot \E_{\bxi}\left[y(\bz_0)\cdot \bwsig^\top\bz_0 \cdot \sigma'\left(\bwsig^\top\bz_0+\bw^\top(-\bsmaj)+\bw^\top\bxi\right)\right] \\
        &\qquad +\left(\frac{1}{2}-\frac{\lambda}{2}\right) \cdot \E_{\bxi}\left[y(-\bz_0)\cdot \bwsig^\top(-\bz_0) \cdot \sigma'\left(\bwsig^\top(-\bz_0)+\bw^\top\bsmaj+\bw^\top\bxi\right) \right] \\
        &\qquad +\frac{\lambda}{2} \cdot\E_{\bxi}\left[y(-\bz_0)\cdot \bwsig^\top(-\bz_0) \cdot \sigma'\left(\bwsig^\top(-\bz_0)+\bw^\top(-\bsmaj)+\bw^\top\bxi \right)\right]\bigg).
    \end{align*}
    Since $y(\bz_0)=y(-\bz_0)$, we can take $y(\bz_0)$ out as a common factor. Moreover, we have $ay(\bz_0)>0$ as $\bz_0\parallel\bwsig$.
    This yields
    \begin{align*}
        -\bwsig^\top\nabla_{\bw} L_0 &= \frac{1}{4}|a|\cdot \bwsig^\top\bz_0 \cdot \Big( \\
        &\qquad(1-\lambda) \cdot \E_{\bxi} \left[\sigma'\left(\bwsig^\top\bz_0+\bw^\top\bsmaj+\bw^\top\bxi\right)\right] +\lambda \cdot \E_{\bxi}\left[\sigma'\left(\bwsig^\top\bz_0-\bw^\top\bsmaj+\bw^\top\bxi \right) \right] \\
        &\qquad-(1-\lambda) \cdot \E_{\bxi}\left[\sigma'\left(-\bwsig^\top\bz_0+\bw^\top\bsmaj+\bw^\top\bxi \right) \right] -\lambda \cdot  \E_{\bxi}\left[\sigma'\left(-\bwsig^\top\bz_0-\bw^\top\bsmaj+\bw^\top\bxi \right) \right] \Big).
    \end{align*}
    Collecting the terms that are multiplied by $\lambda$, we have
    \begin{align*}
        -\bwsig^\top\nabla_{\bw} L_0 &= \frac{1}{4}|a|\cdot \bwsig^\top\bz_0 \cdot \Big( \\
        &\qquad \E_{\bxi} \left[\sigma'\left(\bwsig^\top\bz_0+\bw^\top\bsmaj+\bw^\top\bxi\right)-\sigma'\left(-\bwsig^\top\bz_0+\bw^\top\bsmaj+\bw^\top\bxi \right) \right] \\
        &\qquad+\lambda\cdot \E_{\bxi}\left[\sigma'\left(-\bwsig^\top\bz_0+\bw^\top\bsmaj+\bw^\top\bxi\right)-\sigma'\left(\bwsig^\top\bz_0+\bw^\top\bsmaj+\bw^\top\bxi \right) \right] \\
        &\qquad+\lambda\cdot \E_{\bxi}\left[\sigma'\left(\bwsig^\top\bz_0-\bw^\top\bsmaj+\bw^\top\bxi\right)-\sigma'\left(-\bwsig^\top\bz_0-\bw^\top\bsmaj+\bw^\top\bxi\right) \right] \Big).
    \end{align*}
    Performing casework on $\sigma'$ gives
    \begin{align*}
        -\bwsig^\top\nabla_{\bw} L_0 &= \frac{1}{4}|a|\cdot \bwsig^\top\bz_0 \cdot \Big( \\
        &\qquad(1-\lambda)\cdot \P_{\bxi} \left(|\bw^\top\bxi+\bw^\top\bsmaj|\leq|\bwsig^\top\bz_0|\right) \cdot \sgn(\bwsig^\top\bz_0)\\
        &\qquad+\lambda\cdot \P_{\bxi}\left(|\bw^\top\bxi-\bw^\top\bsmaj|\leq|\bwsig^\top\bz_0|\right) \cdot \sgn(\bwsig^\top\bz_0) \Big).
    \end{align*}
    Using $\sgn(\bwsig^\top\bz_0)\bwsig^\top\bz_0=|\bwsig^\top\bz_0|$, then $|\bwsig^\top\bz_0|= \norm{\bz_0}\norm{\bwsig} = \sqrt{2} \norm{\bwsig}$ for $\bz_0\parallel\bwsig$ and $\bz_0 \in \{\bmu_1,\bmu_2\}$, we obtain
    \begin{align*}
        -\bwsig^\top\nabla_{\bw} L_0 &= \frac{\sqrt{2}}{4}|a|\norm{\bwsig}\Big( \\
        &\qquad(1-\lambda)\cdot \P_{\bxi} \left(|\bw^\top\bxi+\bw^\top\bsmaj|\leq\sqrt{2}\norm{\bwsig} \right) \\
        &\qquad+\lambda \cdot \P_{\bxi}\left(|\bw^\top\bxi-\bw^\top\bsmaj|\leq\sqrt{2}\norm{\bwsig} \right) \Big).
    \end{align*}
    
    Note that if $\lambda=\tfrac{1}{2}$ we already recover the result of~\cite[Lemma C.4]{glasgow2024sgd}.
    Moreover, since $\bxi$ is symmetric about $0$, we have
    \begin{equation*}
        \P_{\bxi}\left(|\bw^\top\bxi-\bw^\top\bsmaj|\leq\sqrt{2}\norm{\bwsig}\right)=\P_{\bxi}\left(|\bw^\top\bxi+\bw^\top\bsmaj|\leq\sqrt{2}\norm{\bwsig}\right)=\P_{\bxi}\left(|\bw^\top\bxi+\bw^\top\be_3|\leq\sqrt{2}\norm{\bwsig} \right).
    \end{equation*}
    Therefore,
    \begin{equation*}
        -\bwsig^\top\nabla_{\bw} L_0 = \frac{\sqrt{2}}{4}|a|\norm{\bwsig}\P_{\bxi} \left(|\bw^\top\bxi+\bw^\top\be_3|\leq\sqrt{2}\norm{\bwsig} \right),
    \end{equation*}
    which proves the first equation.
    
    For the $\bwopp$ component, the same calculations with $ay(\bz)<0$ for $\bz\parallel\bwopp$ give
    \begin{equation*}
        -\bwopp^\top\nabla_{\bw} L_0 = -\frac{\sqrt{2}}{4}|a|\norm{\bwopp}\P_{\bxi} \left(|\bw^\top\bxi+\bw^\top\be_3|\leq\sqrt{2}\norm{\bwopp} \right),
    \end{equation*}
    which proves the second equation.
    This completes the proof of the lemma.
\end{proof}

Now we will show the lemma for the $\bwsp$ component.
\begin{lemma}\label{lem:sp_l0}
    For any neuron $(a,\bw)$, define
    \begin{align*}
        A &= \min(\{\pm\bw^\top\bmu_1,\pm\bw^\top\bmu_2\}) \\
        B &= \min(\{\pm\bw^\top\bmu_1,\pm\bw^\top\bmu_2\} \setminus \{A\})
    \end{align*}
    and
    \begin{equation*}
        \epsw =
        \begin{cases}
            \P_{\bxi}\left(\bw^\top\bxi-\wsp \in [A,B]\right)-\P_{\bxi}\left(\bw^\top\bxi+\wsp \in [A,B]\right) & \textnormal{if } \sgn(w_1)=\sgn(w_2) \\
            \P_{\bxi}\left(\bw^\top\bxi+\wsp \in [A,B]\right)-\P_{\bxi}\left(\bw^\top\bxi-\wsp \in [A,B]\right) & \textnormal{if } \sgn(w_1)\neq\sgn(w_2)
        \end{cases}.
    \end{equation*}
    We then have
    \begin{equation*}
        -\bwsp^\top\nabla_{\bw} L_0 = a\wsp\left(\frac{1}{2}+\frac{\epsw}{4}-\lambda\right).
    \end{equation*}
\end{lemma}

\begin{remark} \normalfont
    The term $\epsw$ is an intermediate term resulting from a slight asymmetry in the spurious feature gradient.
    In \Secref{error_term_bound}, we show that it is exponentially decreasing in $\norm{\bwsp}$.
\end{remark}

\begin{proof}
    We have (\cf\Eqref{grad_decomp})
    \begin{equation}\label{eq:grad_decomp_sp}
        -\bwsp^\top\nabla_{\bw} L_0 = \E_{\bx, \bz, \bs}\left[ay(\bz)\sigma'(\bw^\top\bx)\bwsp^\top \bs \right].
    \end{equation}
    Observe from Equation~\eqref{eq:XOR-spurious-correlation} that $\bs = y(\bz) \be_3$ with probability $1 - \lambda$ and $\bs = -y(\bz) \be_3$ with probability $\lambda$.
    Recall our definitions of the majority group as $\Xmaj \coloneqq \{\bx\in \{\pm 1\}^d:y(\bx)=x_3\}$ and the minority group as $\Xmin \coloneqq \{\bx \in \{\pm 1\}^d:y(\bx)=-x_3\}$.
    Equivalently, $\Xmaj \coloneqq \{\bx\in \{\pm 1\}^d:\bs = y(\bz) \be_3\}$ and $\Xmin \coloneqq \{\bx \in \{\pm 1\}^d: \bs = -y(\bz) \be_3\}$.
    Taking expectations only over $\bs$, we have
    \begin{align*}
        -\bwsp^\top\nabla_{\bw} L_0 &= a\Big( \\
        &\qquad (1-\lambda) \cdot \E_{\bz,\bxi \mid \bx\in\Xmaj}\left[y(\bz)\cdot \sigma'\left(\bw^\top\bz+\bw^\top(y(\bz)\be_3)+\bw^\top\bxi \right) \cdot \bwsp^\top(y(\bz)\be_3) \right] \\
        &\qquad + \lambda \cdot \E_{\bz,\bxi\mid \bx\in\Xmin}\left[y(\bz)\cdot \sigma'\left(\bw^\top\bz+\bw^\top(-y(\bz)\be_3)+\bw^\top\bxi\right)\cdot \bwsp^\top(-y(\bz)\be_3) \right] \Big).
    \end{align*}
    Since $y(\bz)^2=1$, we have
    \begin{align*}
        -\bwsp^\top\nabla_{\bw} L_0 &= a\wsp ( \\
        &\qquad (1-\lambda) \E_{\bz,\bxi \mid \bx\in\Xmaj }[\sigma'(\bw^\top\bz+y(\bz)\wsp+\bw^\top\bxi)] \\
        &\qquad - \lambda \E_{\bz,\bxi \mid \bx\in\Xmin }[\sigma'(\bw^\top\bz-y(\bz)\wsp+\bw^\top\bxi)]).
    \end{align*}
    Using that the conditional distribution of $\bz,\bxi$ is the same for the majority and minority groups, and noting that $\sigma'(u) = \ind(u \geq 0)$, we have
    \begin{align*}
        -\bwsp^\top\nabla_{\bw} L_0 &= a\wsp \Big( \\
        &\qquad (1-\lambda) \cdot \P_{\bz,\bxi}\left(\bw^\top\bz+\bw^\top\bxi>-y(\bz)\wsp\right) \\
        &\qquad - \lambda \cdot \P_{\bz,\bxi}\left(\bw^\top\bz+\bw^\top\bxi>y(\bz)\wsp\right)\Big).
    \end{align*}
    Recall that $\bxi$ is a symmetric random variable and independent of $\bz$. 
    Therefore, for a fixed $\bz$ we have
    \begin{align}
        \P_{\bxi}\big(\bw^\top\bz+\bw^\top\bxi>&-y(\bz)\wsp\big)+\P_{\bxi}\left(\bw^\top(-\bz)+\bw^\top\bxi>y(\bz)\wsp\right) \nonumber \\
        &= \P_{\bxi}\left(\bw^\top\bz+\bw^\top\bxi>-y(\bz)\wsp\right)+\P_{\bxi}\left(\bw^\top(-\bz)+\bw^\top(-\bxi)>y(\bz)\wsp\right) \nonumber \\
        &= \P_{\bxi}\left(\bw^\top\bz+\bw^\top\bxi>-y(\bz)\wsp\right)+\P_{\bxi}\left(\bw^\top\bz+\bw^\top\bxi\leq -y(\bz)\wsp\right) \nonumber \\
        &= 1. \label{eq:wsp_probs_cancel}
    \end{align}
    Moreover, it is immediate from Equation~\eqref{eq:XOR-spurious-correlation} that $\bz$ is also a symmetric random variable.
    Using this along with the fact that $y(\bz) = y(-\bz)$, the above gives us
    \begin{equation*}
        \P_{\bz, \bxi}\left(\bw^\top\bz+\bw^\top\bxi>-y(\bz)\wsp\right)+\P_{\bxi}\left(\bw^\top\bz+\bw^\top\bxi>y(\bz)\wsp \right) = 1.
    \end{equation*}
    This implies that
    \begin{equation*}
        -\bwsp^\top\nabla_{\bw} L_0 = a\wsp\left(\P_{\bz,\bxi}\left(\bw^\top\bz+\bw^\top\bxi>-y(\bz)\wsp\right)-\lambda\right).
    \end{equation*}
    It remains to show that
    \begin{equation}
        \P_{\bz,\bxi}\left(\bw^\top\bz+\bw^\top\bxi>-y(\bz)\wsp \right) = \frac{1}{2} + \frac{\epsw}{4}. \label{eq:half_plus_epsilon}
    \end{equation}
    We begin by writing out the probabilities over $\bz$.
    It is immediate from Equation~\eqref{eq:XOR-spurious-correlation} that $\bz$ is uniformly distributed over $\{\pm \bmu_1, \pm \bmu_2\}$.
    Therefore, we have
    \begin{align*}
        \P_{\bz,\bxi}\left(\bw^\top\bz+\bw^\top\bxi>-y(\bz)\wsp\right) &= \frac{1}{4}\Big( \\
        &\qquad \P_{\bxi}\left(\bw^\top\bxi+\wsp > -\bw^\top\bmu_1\right) + \P_{\bxi}\left(\bw^\top\bxi+\wsp > \bw^\top\bmu_1\right) \\
        &\qquad+ \P_{\bxi}\left(\bw^\top\bxi-\wsp > -\bw^\top\bmu_2\right) + \P_{\bxi}\left(\bw^\top\bxi-\wsp > \bw^\top\bmu_2\right) \Big).
    \end{align*}
    Taking complements and using that $\bxi$ is symmetric about $0$, we have
    \begin{align*}
        \P_{\bz,\bxi}\left(\bw^\top\bz+\bw^\top\bxi>-y(\bz)\wsp\right) &= \frac{1}{2} + \frac{1}{4} \Big( \\
        &\qquad- \P_{\bxi}\left(\bw^\top\bxi+\wsp \leq -\bw^\top\bmu_1\right) - \P_{\bxi}\left(\bw^\top\bxi+\wsp \leq \bw^\top\bmu_1\right) \\
        &\qquad+ \P_{\bxi}\left(\bw^\top\bxi+\wsp \leq -\bw^\top\bmu_2\right) + \P_{\bxi}\left(\bw^\top\bxi+\wsp \leq \bw^\top\bmu_2\right)\Big).
    \end{align*}
    For notational convenience, let us define
    \begin{align*}
        \epsw &\coloneqq - \P_{\bxi}\left(\bw^\top\bxi+\wsp \leq -\bw^\top\bmu_1 \right) - \P_{\bxi}\left(\bw^\top\bxi+\wsp \leq \bw^\top\bmu_1 \right) \\
        &\qquad+ \P_{\bxi}\left(\bw^\top\bxi+\wsp \leq -\bw^\top\bmu_2\right)+ \P_{\bxi}\left(\bw^\top\bxi+\wsp \leq \bw^\top\bmu_2\right).
    \end{align*}
    We will show that $\epsw$ defined here is equal to the formula in the statement of \Lemref{sp_l0}.
    Let us write $\epsw$ in interval notation to find that certain segments cancel.
    Let $A\leq B \leq C \leq D$ be the ordering on $\{\pm \bw^\top\bmu_1,\pm\bw^\top\bmu_2\}$ that is a function of $\bw$ (\ie $A$ is the smallest element, $B$ is the second smallest, and so on).
    Note that $A=-D$ and $B=-C$.
    Moreover, if $\sgn(w_1)=\sgn(w_2)$, then $A\in\{\pm \bw^\top\bmu_2\}$, whereas if $\sgn(w_1)\neq\sgn(w_2)$, then $A\in\{\pm\bw^\top\bmu_1\}$.
    We now handle these cases separately.
    If $\sgn(w_1)=\sgn(w_2)$, we have
    \begin{align*}
        \epsw &= 2\P_{\bxi}\left(\bw^\top\bxi+\wsp\leq A\right) + \P_{\bxi}\left(A\leq \bw^\top\bxi+\wsp\leq D\right) - 2\P_{\bxi}\left(\bw^\top\bxi+\wsp\leq A\right) \\
        &\qquad - \P_{\bxi}\left(A\leq \bw^\top\bxi + \wsp \leq B\right) - \P_{\bxi}\left(A\leq \bw^\top\bxi + \wsp \leq C \right) \\
        &= \P_{\bxi}\left(C\leq \bw^\top\bxi+\wsp\leq D \right) - \P_{\bxi}\left(A\leq \bw^\top\bxi+\wsp\leq B \right).
    \end{align*}
    Finally, since $A=-D$ and $B=-C$ and $\bxi$ is symmetric around $0$, we have
    \begin{equation*}
        \epsw = \P_{\bxi}\left(A\leq \bw^\top\bxi -\wsp \leq B\right) - \P_{\bxi}\left(A\leq \bw^\top\bxi + \wsp \leq B\right).
    \end{equation*}
    Likewise, if $\sgn(w_1)\neq\sgn(w_2)$, we have
    \begin{align*}
        \epsw &= 2\P_{\bxi}\left(\bw^\top\bxi+\wsp\leq A\right) + \P_{\bxi}\left(A\leq \bw^\top\bxi + \wsp \leq B\right) + \P_{\bxi}\left(A\leq \bw^\top\bxi + \wsp \leq C \right) \\
        &\qquad - 2\P_{\bxi}\left(\bw^\top\bxi+\wsp\leq A \right) - \P_{\bxi}\left(A\leq \bw^\top\bxi+\wsp\leq D \right) \\
        &= \P_{\bxi}\left(A\leq \bw^\top\bxi+\wsp\leq B \right) - \P_{\bxi}\left(C\leq \bw^\top\bxi+\wsp\leq D \right).
    \end{align*}
    Finally, since $\bxi$ is symmetric about $0$ with $A=-D$ and $B=-C$, we have
    \begin{equation*}
        \epsw = \P_{\bxi}\left(A\leq \bw^\top\bxi +\wsp \leq B \right) - \P_{\bxi}\left(A\leq \bw^\top\bxi - \wsp \leq B \right).
    \end{equation*}
    This completes the proof of the lemma.
\end{proof}

Next, we will show the lemma for the $\bwperp$ component.
\begin{lemma} \label{lem:perp_l0}
    For any neuron $(a,\bw)$, we have
    \begin{align*}
        -\bwperp^\top\nabla_{\bw} L_0 &= \frac{1}{8}|a|\E_{\bxi}\Big[|\bw^\top\bxi| \Big( \\
        &\qquad \ind \left(|\bw^\top\bxi| \geq |\sqrt{2}\norm{\bwsig}+\wsp| \right) +\ind\left(|\bw^\top\bxi| \geq |\sqrt{2}\norm{\bwsig}-\wsp| \right) \\
        &\qquad -\ind\left(|\bw^\top\bxi| \geq |\sqrt{2}\norm{\bwopp}+\wsp| \right) -\ind\left(|\bw^\top\bxi| \geq |\sqrt{2}\norm{\bwopp}-\wsp| \right) \Big) \Big].
    \end{align*}
\end{lemma}
\begin{proof}
    We have (\cf\Eqref{grad_decomp})
    \begin{equation} \label{eq:grad_decomp_perp}
        -\bwperp^\top\nabla_{\bw} L_0 = \E_{\bx, \bz, \bxi}\left[ay(\bz)\sigma'(\bw^\top\bx)\bwperp^\top \bxi \right].
    \end{equation}
    Recall that $\bx = \bz + \bs + \bxi$.
    Noting that $\bxi$ is symmetric around $0$, we perform a symmetrization argument to get
    \begin{align*}
        -\bwperp^\top\nabla_{\bw} L_0 &= \frac{1}{2}a\E_{\bz, \bs, \bxi}\left[y(\bz) \cdot \sigma'\left(\bw^\top\bz+\bw^\top\bs+\bw^\top\bxi\right) \cdot \bwperp^\top \bxi \right] \\
        &\qquad + \frac{1}{2}a\E_{\bz, \bs, \bxi}\left[y(\bz)\cdot \sigma'\left(\bw^\top\bz+\bw^\top\bs+\bw^\top(-\bxi) \right) \cdot\bwperp^\top (-\bxi) \right] \\
        &= \frac{1}{2}a\E_{\bz, \bs, \bxi}\left[y(\bz)\cdot \bwperp^\top \bxi \cdot \left(\sigma'\left(\bw^\top\bz+\bw^\top\bs+\bw^\top\bxi\right )-\sigma'\left(\bw^\top\bz+\bw^\top\bs-\bw^\top\bxi\right)\right)\right].
    \end{align*}
    Noting that $\sigma'(u) = \ind(u \geq 0)$ and performing casework on $\sigma'(\cdot)$ gives
    \begin{equation*}
        -\bwperp^\top\nabla_{\bw} L_0 = \frac{1}{2}a\E_{\bz,\bs, \bxi}\left[y(\bz)\cdot \bwperp^\top \bxi \cdot \ind\left(|\bw^\top\bxi| \geq|\bw^\top\bz+\bw^\top\bs|\right)\cdot \sgn(\bwperp^\top\bxi) \right].
    \end{equation*}
    Since $\sgn(\bwperp^\top\bxi)\bwperp^\top\bxi=|\bwperp^\top\bxi|$, we obtain
    \begin{equation*}
        -\bwperp^\top\nabla_{\bw} L_0 = \frac{1}{2}a\E_{\bz, \bs, \bxi}\left[y(\bz) \cdot |\bwperp^\top \bxi| \cdot \ind\left(|\bw^\top\bxi|\geq|\bw^\top\bz+\bw^\top\bs|\right)\right].
    \end{equation*}
    Since $\bxi$ is independent of $(\bz,\bs)$, we can expand the expectation over $(\bz,\bs)$ using the joint distribution defined in Equation~\eqref{eq:XOR-spurious-correlation} to obtain
    \begin{align*}
        -\bwperp^\top\nabla_{\bw} L_0 &= \frac{1}{8}a\E_{\bxi}\Big[|\bw^\top\bxi| \Big( \\
        & \qquad \ind\left(|\bw^\top\bxi| \geq |\bw^\top\bmu_1+\wsp|\right)+\ind\left(|\bw^\top\bxi| \geq |\bw^\top\bmu_1-\wsp|\right) \\
        & \qquad -\ind\left(|\bw^\top\bxi| \geq |\bw^\top\bmu_2+\wsp| \right) -\ind\left(|\bw^\top\bxi| \geq |\bw^\top\bmu_2-\wsp|\right) \Big) \Big].
    \end{align*}
    If $(a,\bw)\in S^+$ we have $\bmu_1\parallel \bwsig$, $\bmu_2\parallel\bwopp$, and $a>0$.
    Since $\bw^\top \bmu_1 = \bwsig^\top\bmu_1=\sqrt{2}\norm{\bwsig}$ and $\bw^\top \bmu_2 = \bwopp^\top\bmu_2=\sqrt{2}\norm{\bwopp}$ in this case, we have
    \begin{align} 
        -\bwperp^\top\nabla_{\bw} L_0 &= \frac{1}{8}|a|\E_{\bxi}\Big[|\bw^\top\bxi| \Big( \nonumber \\
        & \qquad \ind\left(|\bw^\top\bxi| \geq |\sqrt{2}\norm{\bwsig}+\wsp|\right) +\ind\left(|\bw^\top\bxi| \geq |\sqrt{2}\norm{\bwsig}-\wsp|\right) \nonumber \\
        & \qquad -\ind\left(|\bw^\top\bxi| \geq |\sqrt{2}\norm{\bwopp}+\wsp| \right) -\ind\left(|\bw^\top\bxi| \geq |\sqrt{2}\norm{\bwopp}-\wsp| \right) \Big) \Big]. \label{eq:perp_norms} 
    \end{align}
    On the other hand, if $(a,\bw)\in S^-$ we have $\bmu_2\parallel \bwsig$, $\bmu_1\parallel\bwopp$, and $a<0$.
    Since $\bw^\top \bmu_2 = \bwsig^\top\bmu_2=\sqrt{2}\norm{\bwsig}$ and $\bw^\top \bmu_1 = \bwopp^\top\bmu_1=\sqrt{2}\norm{\bwopp}$ in this case, simplification results in the same expression as \Eqref{perp_norms}.
    This completes the proof.
\end{proof}

Finally, we repeat the analysis for any individual coordinate of $\bwperp$.
\begin{lemma} \label{lem:wi_l0}
    For any neuron $(a,\bw)$ and $i>3$, we have
    \begin{align*}
        -w_i\partial_{w_i} L_0 &= \frac{1}{8}|a||w_i| \Big( \nonumber \\
        &\qquad \P_{\bxi}\left(|w_i|\geq |\sqrt{2}\norm{\bwsig}+\wsp+\bw^\top\bxi_{\setminus i}|\right)+\P_{\bxi}\left(|w_i|\geq |\sqrt{2}\norm{\bwsig}-\wsp+\bw^\top\bxi_{\setminus i}| \right) \nonumber \\
        &\qquad -\P_{\bxi}\left(|w_i|\geq |\sqrt{2}\norm{\bwopp}+\wsp+\bw^\top\bxi_{\setminus i}| \right)-\P_{\bxi}\left(|w_i|\geq |\sqrt{2}\norm{\bwopp}-\wsp+\bw^\top\bxi_{\setminus i}| \right) \Big).
    \end{align*}
\end{lemma}
\begin{proof}
    We have (\cf\Eqref{grad_decomp})
    \begin{equation} \label{eq:grad_decomp_wi}
        -w_i \partial_{w_i} L_0 = \E_{\bx}\left[ay(\bx)\sigma'(\bw^\top\bx)w_ix_i \right].
    \end{equation}
    Noting that $x_i$ is symmetric around $0$ (it is a single coordinate of $\bxi$), we perform a similar symmetrization to \Lemref{perp_l0} to get
    \begin{equation*}
        -w_i \partial_{w_i} L_0 = \frac{1}{2}a \E_{\bx}\left[y(\bx) \cdot w_ix_i \cdot \left(\sigma'\left(\bw^\top\bx_{\setminus i}+w_ix_i\right)-\sigma'\left(\bw^\top\bx_{\setminus i}-w_ix_i\right)\right)\right]
    \end{equation*}
    Noting that $\sigma'(u) = \ind(u \geq 0)$ and performing casework on $\sigma'(\cdot)$ gives us
    \begin{equation*}
        -w_i\partial_{w_i} L_0 = \frac{1}{2}a \E_{\bx}\left[y(\bx)\cdot w_ix_i \cdot \ind\left(|w_i| \geq |\bw^\top\bx_{\setminus i}|\right) \cdot \sgn(w_ix_i) \right].
    \end{equation*}
    Since $\sgn(w_ix_i)w_ix_i=|w_ix_i|=|w_i|$, we obtain
    \begin{equation*}
        -w_i\partial_{w_i} L_0 = \frac{1}{2}a |w_i| \E_{\bx}\left[y(\bx) \cdot \ind\left(|w_i| \geq |\bw^\top\bx_{\setminus i}|\right)\right].
    \end{equation*}
    Since $\bxi$ is independent of $(\bz,\bs)$ and symmetric, we can expand the expectation over $(\bz,\bs)$ using the joint distribution defined in Equation~\eqref{eq:XOR-spurious-correlation} to obtain
    \begin{align*}
        -w_i\partial_{w_i} L_0 &= \frac{1}{8}a|w_i| \Big( \\
        &\qquad \P_{\bxi}\left(|w_i|\geq |\bw^\top\bmu_1+\wsp+\bw^\top\bxi_{\setminus i}|\right)+\P_{\bxi}\left(|w_i|\geq |\bw^\top\bmu_1-\wsp+\bw^\top\bxi_{\setminus i}| \right) \\
        &\qquad -\P_{\bxi}\left(|w_i|\geq |\bw^\top\bmu_2+\wsp+\bw^\top\bxi_{\setminus i}|\right)-\P_{\bxi}\left(|w_i|\geq |\bw^\top\bmu_2-\wsp+\bw^\top\bxi_{\setminus i}| \right) \Big).
    \end{align*}
    If $(a,\bw)\in S^+$ we have $\bmu_1\parallel \bwsig$, $\bmu_2\parallel\bwopp$, and $a>0$.
    Since $\bw^\top \bmu_1 = \bwsig^\top\bmu_1=\sqrt{2}\norm{\bwsig}$ and $\bw^\top \bmu_2 = \bwopp^\top\bmu_2=\sqrt{2}\norm{\bwopp}$ in this case, we have
    \begin{align} 
        -w_i\partial_{w_i} L_0 &= \frac{1}{8}|a||w_i| \Big( \nonumber \\
        &\qquad \P_{\bxi}\left(|w_i|\geq |\sqrt{2}\norm{\bwsig}+\wsp+\bw^\top\bxi_{\setminus i}|\right)+\P_{\bxi}\left(|w_i|\geq |\sqrt{2}\norm{\bwsig}-\wsp+\bw^\top\bxi_{\setminus i}|\right) \nonumber \\
        &\qquad -\P_{\bxi}(|w_i|\geq |\sqrt{2}\norm{\bwopp}+\wsp+\bw^\top\bxi_{\setminus i}|)-\P_{\bxi}\left(|w_i|\geq |\sqrt{2}\norm{\bwopp}-\wsp+\bw^\top\bxi_{\setminus i}|\right) \Big). \label{eq:wi_norms}
    \end{align}
    On the other hand, if $(a,\bw)\in S^-$ we have $\bmu_2\parallel \bwsig$, $\bmu_1\parallel\bwopp$, and $a<0$.
    Since $\bw^\top \bmu_2 = \bwsig^\top\bmu_2=\sqrt{2}\norm{\bwsig}$ and $\bw^\top \bmu_1 = \bwopp^\top\bmu_1=\sqrt{2}\norm{\bwopp}$ in this case, simplification results in the same expression as \Eqref{wi_norms}.
    This completes the proof of the lemma.
\end{proof}

\subsection{\texorpdfstring{$\epsw$}{} Asymmetry Term Bound} \label{sec:error_term_bound}

In this section, we will show that the asymmetry term $\epsw$ from \Lemref{sp_l0} decays exponentially in $\norm{\bwsp}$.
Recall that
\begin{align*}
    A &= \min(\{\pm\bw^\top\bmu_1,\pm\bw^\top\bmu_2\}) \\
    B &= \min(\{\pm\bw^\top\bmu_1,\pm\bw^\top\bmu_2\} \setminus \{A\})
\end{align*}
and
\begin{equation*}
    \epsw \coloneqq
    \begin{cases}
        \P_{\bxi}\left(\bw^\top\bxi-\wsp \in [A,B]\right)-\P_{\bxi}\left(\bw^\top\bxi+\wsp \in [A,B]\right) & \textnormal{if } \sgn(w_1)=\sgn(w_2) \\
        \P_{\bxi}\left(\bw^\top\bxi+\wsp \in [A,B]\right)-\P_{\bxi}\left(\bw^\top\bxi-\wsp \in [A,B]\right) & \textnormal{if } \sgn(w_1)\neq\sgn(w_2)
    \end{cases}.
\end{equation*}

We begin with a tight bound on $|\epsw|$ using the Berry-Esseen theorem.
\begin{lemma} \label{lem:v_clt}
    For any neuron $(a,\bw)$, write $\phi(x)\coloneqq\frac{1}{\norm{\bwperp}\sqrt{2\pi}}\exp\left(\frac{-x^2}{2\norm{\bwperp}^2}\right)$ for the probability density function of the Gaussian distribution $\Nc(0, \norm{\bwperp}^2)$.
    Then, we have
    \begin{equation*}
        |\epsw| \lesssim \phi(\wsp) \int_A^B \exp\left(\frac{-u^2}{2\norm{\bwperp}^2}\right)\left|\sinh\left(\frac{u\wsp}{\norm{\bwperp}^2}\right)\right|du + \frac{\norm{\bwperp}_3^3}{\norm{\bwperp}^3_2}.
    \end{equation*}
\end{lemma}
\begin{proof}
    Without loss of generality let $\sgn(w_1)=\sgn(w_2)$.
    (Note that the case $\sgn(w_1) \neq \sgn(w_2)$ proceeds identically because the expression for $|\epsw|$ is the same in both cases and $\bw^\top \bxi$ does not depend on $w_1$ and $w_2$.)
    Write $\Phi(x)\coloneqq \P_{G\sim\Nc(0,\norm{\bwperp}^2)}(G\leq x)$ for the cumulative distribution function of the centered Gaussian with variance $\norm{\bwperp}^2$.
    We apply \Thmref{berry_esseen} with $\bv \coloneqq \bwperp$ to obtain
    \begin{align*}
        \epsw &\coloneqq \P_{\bxi}\left(\bw^\top\bxi-\wsp \in [A,B]\right)-\P_{\bxi}\left(\bw^\top\bxi+\wsp \in [A,B]\right) \\
        &= \P_{\bxi}\left(\bwperp^\top \bxi \leq B + \wsp\right) - \P_{\bxi}\left(\bwperp^\top \bxi \leq A + \wsp\right) \\
        &\qquad - \P_{\bxi}\left(\bwperp^\top \bxi \leq B - \wsp\right) + \P_{\bxi}\left(\bwperp^\top \bxi \leq A - \wsp \right) \\
        &\lesssim \Phi(B+\wsp)-\Phi(A+\wsp)-\Phi(B-\wsp)+\Phi(A-\wsp) + \frac{\norm{\bwperp}_3^3}{\norm{\bwperp}_2^3}.
    \end{align*}
    In particular,
    \begin{align*}
        \epsw &\lesssim \int_{A+\wsp}^{B+\wsp}\phi(u)du-\int_{A-\wsp}^{B-\wsp}\phi(u)du + \frac{\norm{\bwperp}_3^3}{\norm{\bwperp}_2^3} \\
        &=\int_A^B \phi(u+\wsp)-\phi(u-\wsp)du + \frac{\norm{\bwperp}_3^3}{\norm{\bwperp}_2^3},
    \end{align*}
    where in the last step we changed variables.
    This expression is convenient to see that $\epsw\approx 0$ when $A,B\approx 0$ or $\wsp \approx 0$.
    To formalize this intuition, we expand and factorize to obtain
    \begin{align*}
        \int_A^B \phi(u+\wsp)-\phi(u-\wsp)du &=\int_A^B \left(\frac{1}{\norm{\bwperp}\sqrt{2\pi}}e^{\frac{-(u+\wsp)^2}{2\norm{\bwperp}^2}}-\frac{1}{\norm{\bwperp}\sqrt{2\pi}}e^{\frac{-(u-\wsp)^2}{2\norm{\bwperp}^2}}\right) du \\
        &=\int_A^B \frac{1}{\norm{\bwperp}\sqrt{2\pi}}e^{\frac{-u^2}{2\norm{\bwperp}^2}}e^{\frac{-\wsp^2}{2\norm{\bwperp}^2}}\left( e^{\frac{-u\wsp}{\norm{\bwperp}^2}} - e^{\frac{u\wsp}{\norm{\bwperp}^2}} \right) du \\
        &= -2\phi(\wsp)\int_A^B e^{\frac{-u^2}{2\norm{\bwperp}^2}} \sinh\left(\frac{u\wsp}{\norm{\bwperp}^2}\right)du.
    \end{align*}
    Taking the absolute value and using the triangle inequality, we have
    \begin{equation*}
        |\epsw| \lesssim \phi(\wsp) \int_A^B \exp\left(\frac{-u^2}{2\norm{\bwperp}^2}\right)\left|\sinh\left(\frac{u\wsp}{\norm{\bwperp}^2}\right)\right|du + \frac{\norm{\bwperp}_3^3}{\norm{\bwperp}_2^3},
    \end{equation*}
    as desired.
    This completes the proof of the lemma.
\end{proof}

Now, let us show a more sophisticated upper bound on $|\epsw|$ which implies exponential decay in $\norm{\bwsp}$.
\begin{lemma} \label{lem:v_ub}
    For any neuron $(a,\bw)$ we have
    \begin{equation*}
        |\epsw| \lesssim  \frac{\norm{\bw_{1:2}}}{\norm{\bwperp}}\exp\left(\frac{-\norm{\bwsp}^2+2\sqrt{2}\norm{\bw_{1:2}}\norm{\bwsp}}{2\norm{\bwperp}^2}\right) + \frac{\norm{\bwperp}_3^3}{\norm{\bwperp}_2^3}.
    \end{equation*}
\end{lemma}

\begin{proof}
    Beginning from the statement of \Lemref{v_clt}, we have
    \begin{equation*}
        |\epsw|\lesssim \phi(\wsp)\int_A^B \exp\left(\frac{-u^2}{2\norm{\bwperp}^2}\right) \left|\sinh\left(\frac{u\wsp}{\norm{\bwperp}^2}\right)\right|du + \frac{\norm{\bwperp}_3^3}{\norm{\bwperp}_2^3}.
    \end{equation*}
    Then using the identity $|\sinh(x)|\leq\cosh(x)\leq \exp(|x|)$ and taking the uniform upper bound we have
    \begin{align*}
        |\epsw|&\lesssim \phi(\wsp)\int_A^B \exp\left(\frac{-u^2}{2\norm{\bwperp}^2}\right) \exp\left(\frac{|u\wsp|}{\norm{\bwperp}^2}\right) du + \frac{\norm{\bwperp}_3^3}{\norm{\bwperp}_2^3} \\
        &\leq \phi(\wsp)(B-A)\sup_{u\in[A,B]}\exp\left(\frac{-u^2}{2\norm{\bwperp}^2}\right) \exp\left(\frac{|u\wsp|}{\norm{\bwperp}^2}\right) + \frac{\norm{\bwperp}_3^3}{\norm{\bwperp}_2^3}.
    \end{align*}
    Since $\exp\left(\frac{-u^2}{2\norm{\bwperp}^2}\right)\leq 1$ and $\sup_{u\in[A,B]}\exp\left(\frac{|u\wsp|}{\norm{\bwperp}^2}\right)= \exp\left(\frac{|A|\norm{\bwsp}}{\norm{\bwperp}^2}\right)$, we have
    \begin{align*}
        |\epsw|&\lesssim \phi(\wsp) (B-A) \exp\left(\frac{|A|\norm{\bwsp}}{\norm{\bwperp}^2}\right) + \frac{\norm{\bwperp}_3^3}{\norm{\bwperp}_2^3} \\
        &\lesssim \frac{B-A}{\norm{\bwperp}}\exp\left(\frac{-\norm{\bwsp}^2+2|A|\norm{\bwsp}}{2\norm{\bwperp}^2}\right) + \frac{\norm{\bwperp}_3^3}{\norm{\bwperp}_2^3}.
    \end{align*}
    Finally since $B-A\leq |A|\leq \sqrt{2}\norm{\bw_{1:2}}$ we obtain
    \begin{equation*}
        |\epsw| \lesssim \frac{\norm{\bw_{1:2}}}{\norm{\bwperp}}\exp\left(\frac{-\norm{\bwsp}^2+2\sqrt{2}\norm{\bw_{1:2}}\norm{\bwsp}}{2\norm{\bwperp}^2}\right) + \frac{\norm{\bwperp}_3^3}{\norm{\bwperp}_2^3},
    \end{equation*}
    which completes the proof of the lemma.
\end{proof}

\clearpage

\section{\texorpdfstring{$L_\rho$}{Lp} Analysis} \label{sec:lp_analysis}
In this section, we directly analyze the dynamics of the neural network trained under the standard logistic loss $\ell_\rho$, \ie
\begin{equation*}
    \ell_\rho(\bx) = -2\log\left(\frac{1}{1+\exp(-y(\bx)f_\rho(\bx))}\right).
\end{equation*}
In \Secref{l0_lp_error_analysis}, we bound the deviation between the $L_0$ and $L_\rho$ population gradients.
In~\Secref{grad_lrho}, we directly characterize the $L_\rho$ population gradients assuming conditions on the empirical margin of the data that will hold later in our analysis.

\subsection{\texorpdfstring{$\nabla L_\rho - \nabla L_0$}{L0 vs Lp Gradient} Error Analysis} \label{sec:l0_lp_error_analysis}
In this section, we will bound the absolute values of the coordinates of $\nabla L_\rho-\nabla L_0$.
We will find this difference to be small throughout Phase I of our analysis.
However, in Phase II the neural network is too large for this $L_0$-approximation to be useful, and we must directly analyze the $L_\rho$ gradient (see \Secref{grad_lrho}).

We define $\gamma(\bx) \coloneqq y(\bx) f_\rho(\bx)$ as the margin of data point $\bx$, and the composite notation $\ell_\rho(\bx) \coloneqq h(\gamma(\bx))$ where $h(\gamma) \coloneqq -2 \log (\psi(\gamma))$ (where $\psi(u)\coloneqq 1/(1+e^{-u})$ denotes the sigmoid).
Accordingly, we define $\ell^{(1)}_\rho(\bx) \coloneqq h'(\gamma(\bx))$.
Similarly, note that $\ell_0(\bx) = -2\log(\tfrac{1}{2}) - \gamma(\bx)$ and so we can define $\ell_0^{(1)}(\bx) \coloneqq -1$ as shorthand.

First, we require the following generalization of~\cite[Lemma C.5]{glasgow2024sgd}.
\begin{lemma} \label{lem:c5_extension}
    Define $g_{ij}:\{\pm1\}^d\to\{\pm1\}^d$ to flip the $i^{th}$ and $j^{th}$ bits of the input $\bx$ and indices $(i,j)$ are chosen such that $y(g(\bx)) = y(\bx)$.\footnote{Note that this is the case as long as $i, j \in \{1,2\}$ \emph{or} both $i,j \notin \{1,2\}$.}
    For any $\bx\in\{\pm 1\}^d$, we have
    \begin{equation*}
        |\ell^{(1)}_\rho(\bx)-\ell^{(1)}_\rho(g_{ij}(\bx))|\leq \E_{(a,\bw) \sim \rho}[|aw_i|+|aw_j|].
    \end{equation*}
\end{lemma}
\begin{proof}
    Recall that $\ell^{(1)}_\rho(\bx) \coloneqq h'(\gamma(\bx))$, and note that $h'(\gamma)$ is $\frac{1}{2}$-Lipschitz with respect to $\gamma$.\footnote{\cite{glasgow2024sgd} has $2$-Lipschitz but $\tfrac{1}{2}$ is actually the tight constant; it follows by noting that $h''(\gamma) = \frac{2e^{-\gamma}}{(1 + e^{-\gamma})^2} \leq \tfrac{1}{2}$.}
    Since we have assumed that $y(\bx) = y(g_{ij}(\bx))$, this gives us
    \begin{equation}\label{eq:gij_lipschitz}
        |\ell^{(1)}_\rho(\bx)-\ell^{(1)}_\rho(g_{ij}(\bx))|\leq\frac{1}{2}|f_\rho(\bx)-f_\rho(g_{ij}(\bx))|.
    \end{equation}
    Applying the triangle inequality and $1$-Lipschitzness of the ReLU, we have
    \begin{align*}
        |f_\rho(\bx)-f_\rho(g_{ij}(\bx))|&= \Big|\E_{(a,\bw) \sim \rho}\left[a\cdot \left(\sigma(\bw^\top\bx)-\sigma(\bw^\top g_{ij}(\bx)) \right) \right] \Big| \\
        &\leq \E_{(a,\bw) \sim \rho}\left[|a|\cdot |\bw^\top(\bx-g_{ij}(\bx))|\right].
    \end{align*}
    By definition of $g_{ij}$ we have $|\bw^\top(\bx-g_{ij}(\bx))|=|2w_i x_i + 2w_j x_j|$, and using the triangle inequality once again gives
    \begin{equation*}
        \E_{(a,\bw) \sim \rho}\left[|a|\cdot|\bw^\top(\bx-g_{ij}(\bx))|\right]\leq2\E_{(a,\bw) \sim \rho}[|aw_i|+|aw_j|].
    \end{equation*}
    Substituting this into \Eqref{gij_lipschitz} completes the proof of the lemma.
\end{proof}

We will first show a bound on all components of $\nabla L_\rho-\nabla L_0$ which is effective when the network is small.
\begin{lemma} \label{lem:l0_lp_norm_small}
    Suppose $\E_{(a,\bw) \sim \rho}[\norm{a\bw}]\leq d^{C}$ for a constant $C>0$.
    Then, for any neuron $(a,\bw)$ and any $i \in [d]$, we have
    \begin{equation*}
        \frac{2}{|a|}|\partial_{w_i} L_\rho-\partial_{w_i} L_0| \leq \log(d) \cdot\E_{(a,\bw) \sim \rho}[\norm{a\bw}]+d^{-C},
    \end{equation*}
    where $\partial_{w_i} L$ denotes the $p$-scaled partial derivative of $L$ with respect to the $i$-th coordinate of $\bw$.
\end{lemma}
\begin{proof}
    Define $\Delta_{\bx}\coloneqq (\ell^{(1)}_\rho(\bx)-\ell^{(1)}_0(\bx))\cdot y(\bx)\sigma'(\bw^\top\bx)$ as shorthand.
    We have by the triangle inequality and $|x_i|=1$ that
    \begin{equation*}
        \frac{1}{|a|}|\partial_{w_i}L_\rho-\partial_{w_i}L_0| = |\E_{\bx}[\Delta_{\bx}x_i]| \leq \E_{\bx}[|\Delta_{\bx}|].
    \end{equation*}
    Now, note that $\ell^{(1)}_\rho(\bx) = h'(\gamma(\bx))$ and $\ell^{(1)}_0(\bx) = -1 = h'(0)$.
    Therefore, again using the $\frac{1}{2}$-Lipschitzness of $h'(\gamma)$ and $\sigma'(u)\leq 1$ for the ReLU activation, we have
    \begin{equation*}
        \E_{\bx}[|\Delta_{\bx}|] \leq \frac{1}{2}\E_{\bx}[|\gamma(\bx)|] \leq \frac{1}{2}\E_{\bx}[|f_\rho(\bx)|].
    \end{equation*}
    Splitting on the event that $|f_\rho(\bx)|\leq \log(d) \E_{(a,\bw) \sim \rho}[\norm{a\bw}]$ and using \Lemref{frho_bounds_iii}, we have for a fixed $\bx$ that
    \begin{align*}
        |f_\rho(\bx)| &\leq \log(d)\cdot \E_{(a,\bw) \sim \rho}[\norm{a\bw}] \cdot \ind\left(|f_\rho(\bx)|\leq \log(d) \E_{(a,\bw) \sim \rho}[\norm{a\bw}]\right)\\
        &\qquad +\sqrt{d}\cdot \E_{(a,\bw) \sim \rho}[\norm{a\bw}] \cdot \ind\left(|f_\rho(\bx)|\geq \log(d) \E_{(a,\bw) \sim \rho}[\norm{a\bw}]\right) \\
        &\leq \log(d) \cdot\E_{(a,\bw) \sim \rho}[\norm{a\bw}]+\sqrt{d}\cdot \E_{(a,\bw) \sim \rho}[\norm{a\bw}] \cdot \ind\left(|f_\rho(\bx)|\geq \log(d) \E_{(a,\bw) \sim \rho}[\norm{a\bw}]\right).
    \end{align*}
    Bringing back the expectation over $\bx$, we have
    \begin{align*}
        \frac{1}{2}\E_{\bx}[|f_\rho(\bx)|] &\leq \frac{1}{2}\log(d)\cdot\E_{(a,\bw) \sim \rho}[\norm{a\bw}]+\frac{1}{2}\sqrt{d}\cdot\E_{(a,\bw) \sim \rho}[\norm{a\bw}] \cdot \P_{\bx}\left(|f_\rho(\bx)|\geq \log(d) \E_{(a,\bw) \sim \rho}[\norm{a\bw}]\right) \\
        &\leq \frac{1}{2}\log(d)\cdot\E_{(a,\bw)\sim\rho}[\norm{a\bw}] + d^{-C},
    \end{align*}
    where the last step follows from $\E_{(a,\bw) \sim \rho}[\norm{a\bw}]\leq d^{C}$ and \Lemref{frho_bounds_iv} (choosing the constant in \Lemref{frho_bounds_iv} greater than $2C+1$).
    This completes the proof of the lemma.
\end{proof}

Let us show the key lemma of this section.
It is a modified version of~\cite[Lemma C.3]{glasgow2024sgd} which accounts for the spurious correlation.
In contrast to~\cite{glasgow2024sgd}, for which a leave-one-out analysis is sufficient, we must proceed with a slightly more complicated leave-\emph{two}-out technique.
\begin{lemma} \label{lem:l0_lp_norm}
    Recall that we write $\bx\coloneqq \bz+\bs+\bxi$.
    Define $\bxloo\coloneqq \bx-x_i\be_i$ and suppose $\E_{(a,\bw) \sim \rho}[\norm{a\bw}]\leq d^{C}$ for a constant $C>0$.
    Then, for any neuron $(a,\bw)$, we have
    \begin{align*}
    &\frac{2}{|a|}|\partial_{w_i}L_\rho-\partial_{w_i}L_0| \\
        &\leq
        \begin{cases}
            \E_{(a,\bw) \sim \rho}[|aw_1|+|aw_2|]+\log(d) \cdot \P_{\bx}\left(|\bw^\top(\bs+\bxi)|<|\bw^\top\bz|\right) \cdot \E_{(a,\bw) \sim \rho}[\norm{a\bw}] + d^{-C} & i \in\{1,2\} \\
            \E_{(a,\bw) \sim \rho}[|aw_i|]+\log(d) \cdot \P_{\bx}\left(|\bw^\top\bxloo|<|w_i|\right) \cdot\E_{(a,\bw) \sim \rho}[\norm{a\bw}] + d^{-C} & i > 3.
        \end{cases} \nonumber
    \end{align*}
\end{lemma}

\begin{proof}
    Let us begin with the case $i \in \{1,2\}$.
    The symmetry required by~\cite{glasgow2024sgd} is that flipping one bit does not change the marginal, \ie $\P(\bxloo +x_i\be_i)=\P(\bxloo -x_i\be_i)$.
    While this no longer holds due to the spurious correlation, it is in fact the case that flipping \textit{two} specific bits does not change the marginal.
    Define $g \coloneqq g_{12}$ to flip the first two bits of $\bx$, \eg $g(\bz+\bs+\bxi)=-\bz+\bs+\bxi$.
    It is immediate from Equation~\eqref{eq:XOR-spurious-correlation} that $\P(\bx)=\P(g(\bx))$.
    Once again, define $\Delta_{\bx}\coloneqq (\ell^{(1)}_\rho(\bx)-\ell^{(1)}_0(\bx))\cdot y(\bx) \sigma'(\bw^\top\bx)$ as shorthand.
    We have
    \begin{align}
        \frac{1}{|a|}|\partial_{w_i}L_\rho-\partial_{w_i}L_0| &= |\E_{\bx}[\Delta_{\bx}x_i]| \nonumber \\
        &=\frac{1}{2}|\E_{\bx}[\Delta_{\bx}x_i]+\E_{\bx}[\Delta_{g(\bx)}g(\bx)_i]| \nonumber \\
        &=\frac{1}{2}|\E_{\bx}[(\Delta_{\bx}-\Delta_{g(\bx)})x_i]|, \label{eq:delta_diff}
    \end{align}
    where we used $g(\bx)_i=-x_i$ for $i \in \{1,2\}$ by definition.
    Notice that if $|\bw^\top(\bs+\bxi)|\geq|\bw^\top\bz|$ then $\sigma'(\bw^\top\bx)=\sigma'(\bw^\top g(\bx))$ as the sign of $\bw^\top \bx$ is specified by the sign of $\bw^\top (\bs+\bxi)$.
    Under this assumption, we have
    \begin{equation*}
        |\Delta_{\bx}-\Delta_{g(\bx)}|=\big|\big(\ell^{(1)}_\rho(\bx)-\ell^{(1)}_\rho(g(\bx))\big)-\big(\ell_0^{(1)}(\bx)-\ell_0^{(1)}(g(\bx))\big)\big|=|\ell^{(1)}_\rho(\bx)-\ell^{(1)}_\rho(g(\bx))|,
    \end{equation*}
    where in the last step we used $\ell_0^{(1)}(\bx)=-1$ and $y(\bx)=y(g(\bx))$ by definition.
    On the other hand, we have a coarser bound using the triangle inequality and $|x_i|=1$ as follows:
    \begin{equation*}
        \frac{1}{2}|\E_{\bx}[(\Delta_{\bx}-\Delta_{g(\bx)})x_i]|\leq \frac{1}{2}\E_{\bx}[|\Delta_{\bx}|+|\Delta_{g(\bx)}|]=\E_{\bx}[|\Delta_{\bx}|],
    \end{equation*}
    where the last step holds since each $\bx$ is counted exactly twice.
    Returning to \Eqref{delta_diff} and using the triangle inequality once again, we have
    \begin{align*}
        \frac{1}{2}|\E_{\bx}[(\Delta_{\bx}-\Delta_{g(\bx)})x_i]| &\leq \frac{1}{2}\Big|\E_{\bx}\left[\ind\left(|\bw^\top(\bs+\bxi)|\geq|\bw^\top\bz|\right) \cdot(\Delta_{\bx}-\Delta_{g(\bx)})x_i \right] \Big| \\
        &\qquad +\frac{1}{2}\Big|\E_{\bx}\left[\ind\left(|\bw^\top(\bs+\bxi)|<|\bw^\top\bz|\right)\cdot(\Delta_{\bx}-\Delta_{g(\bx)})x_i\right]\Big| \\
        &\leq\frac{1}{2}\sup_{\bx\in\{\pm1\}^d}|\ell^{(1)}_\rho(\bx)-\ell^{(1)}_\rho(g(\bx))| \\
        &\qquad+\E_{\bx}\left[\ind\left(|\bw^\top(\bs+\bxi)|<|\bw^\top\bz|\right)\cdot |\Delta_{\bx}|\right].
    \end{align*}
    By \Lemref{c5_extension} we have $\sup_{\bx\in\{\pm1\}^d}|\ell^{(1)}_\rho(\bx)-\ell^{(1)}_\rho(g(\bx))|\leq\E_{(a,\bw) \sim \rho}[|aw_1|+|aw_2|]$.
    The remainder of the proof follows in the same way as \Lemref{l0_lp_norm_small}.
    For the $i>3$ case we can use the same leave-one-out symmetrization as~\cite{glasgow2024sgd} instead of the $g$ symmetrization.
    This completes the proof of the lemma.
\end{proof}

\subsection{\texorpdfstring{$\nabla L_\rho$}{Lp Gradient} Computation} \label{sec:grad_lrho}
In this section, we will derive some $L_\rho$ population gradients assuming conditions on the empirical margin of the data which will hold during our Phase II.
Let us first recall some notation.
We defined the $p$-scaled $L_\rho$ gradient as
\begin{equation*}
    \nabla_{\bw} L_\rho \coloneqq p\E_{\bx} \left[\frac{\partial}{\partial\bw}\ell_\rho(\bx)\right].
\end{equation*}
Recall that we denote the majority group as $\Xmaj \coloneqq \{\bx\in\{\pm 1\}^d:y(\bx)=x_3\}$ and the minority group as $\Xmin \coloneqq \{\bx\in\{\pm 1\}^d:y(\bx)=-x_3\}$.
Additionally, $\psi(u)\coloneqq 1/(1+e^{-u})$ denotes the sigmoid and $\gamma(\bx)\coloneqq y(\bx)f_\rho(\bx)$ denotes the margin.
Moreover, we let $C>0$ denote a sufficiently large constant which does not change from line to line.
In the remainder of this section, we will often reference the event $\Etest$ defined in \Defref{events_test}; for a point $\bx=(\bz,\bs,\bxi)$, we will equivalently write that $\bx$ satisfies $\Etest$ and $\bxi$ satisfies $\Etest$, as the event is solely a property of $\bxi$. 

Our key decomposition of the $L_\rho$ gradient is as follows:
\begin{align}
    \nabla_{\bw} L_\rho &= p\E_{\bx}\left[\ell^{(1)}_\rho(\bx)y(\bx)\nabla_{\bw} f_\rho(\bx)\right] \nonumber \\
    &= p\E_{\bx}\left[ \frac{-2y(\bx)\exp(-\gamma(\bx))}{1+\exp(-\gamma(\bx))} \nabla_{\bw} f_\rho(\bx)\right] \nonumber \\
    &= 2p\E_{\bx} \left[\psi(-\gamma(\bx)) \cdot (-y(\bx)\nabla_{\bw} f_\rho(\bx)) \right] \nonumber \\
    &= 2\E_{\bx} \left[\psi(-\gamma(\bx)) \cdot \nabla_{\bw} p\ell_0(\bx) \right]. \label{eq:lrho_gradient_decomp}
\end{align}
Hence, the $L_\rho$ gradient may be analyzed as a product of a margin term and an $\nabla L_0$ term.
We first show an upper bound on the $L_\rho$ gradient for the $\bwsig$ and $\bwopp$ components for neural network parameters under the condition that the margins for the majority and minority groups are approximately ``equal and opposite''.
\begin{lemma} \label{lem:sigopp_lp}
    Assume there exists $\gamma>0$ such that
    \begin{equation} \label{eq:equal_and_opposite_sigopp}
        \gamma(\bx) =
        \begin{cases*}
            (1\pm o(1))\cdot \gamma & $\forall \bx\in\Xmaj$ satisfying $\Etest$ \\
            (-1\pm o(1))\cdot \gamma & $\forall \bx\in \Xmin$ satisfying $\Etest$,
        \end{cases*}
    \end{equation}
    where the event $\Etest$ is defined in \Defref{events_test}.
    If all neurons $(a,\bw)$ satisfy $\frac{1}{C}\norm{\bwsp}\geq \norm{\bwsig}+\norm{\bwopp}+\norm{\bwperp}\log^{1/2}(d)$, then for any neuron $(a,\bw)$ and any $c<1$ we have
    \begin{align*}
        -\bwsig^\top\nabla_{\bw} L_\rho &\lesssim |a|\norm{\bwsig}\Big(e^{-c\gamma}\Big(\max_{(a,\bw)}|a|\norm{\bwsig}\Big) +\P_{\bxi}\left(|\bw^\top\bxi+\bw^\top\be_3| \leq \sqrt{2}\norm{\bwsig}\right) + d^{-C}\Big) \\
        -\bwopp^\top\nabla_{\bw} L_\rho &\lesssim |a|\norm{\bwopp}\Big(e^{-c\gamma}\Big(\max_{(a,\bw)}|a|\norm{\bwopp} \Big)+\P_{\bxi}\left(|\bw^\top\bxi+\bw^\top\be_3| \leq \sqrt{2}\norm{\bwopp}\right) + d^{-C}\Big).
    \end{align*}
\end{lemma}
\begin{proof}
    Combining \Eqref{grad_decomp} and \Eqref{lrho_gradient_decomp}, we have
    \begin{equation*}
        -\bwsig^\top\nabla_{\bw} L_\rho = 2\E_{\bx, \bz}\left[\psi(-\gamma(\bx)) \cdot ay(\bz)\sigma'(\bw^\top\bx)\bwsig^\top \bz\right]
    \end{equation*}
    Following \Lemref{sigopp_l0}, we have $\bwsig^\top\bz\neq 0$ only if $\bz\parallel \bwsig$, which yields
    \begin{equation*}
        -\bwsig^\top\nabla_{\bw} L_\rho = \E_{\substack{\bx, \bz \\\bz\parallel\bwsig}}\left[\psi(-\gamma(\bx)) \cdot ay(\bz)\sigma'(\bw^\top\bx)\bwsig^\top \bz \right].
    \end{equation*}
    Following \Lemref{sigopp_l0} again, define $\bz_0\in\{\bmu_1,\bmu_2\}$ depending on whether $\bwsig \parallel \bmu_1$ or $\bwsig \parallel \bmu_2$.
    Using the distribution of $(\bz,\bs)$, we have $y(\bz_0)=y(-\bz_0)$ and $ay(\bz_0)>0$ as $\bz_0\parallel \bwsig$.
    Recall that we define $\bsmaj \coloneqq \be_3$ if $(a,\bw)\in S^+$ and $\bsmaj \coloneqq - \be_3$ if $(a,\bw)\in S^-$ (all as in \Lemref{sigopp_l0}).
    Therefore, we have
    \begin{align} \label{eq:wsig_lp_expansion}
        -\bwsig^\top\nabla_{\bw} L_\rho &= \frac{1}{2}|a|\cdot \bwsig^\top\bz_0 \cdot \Big( \nonumber \\
        & \qquad (1-\lambda) \E_{\bxi} \left[\psi(-\gamma(\bz_0+\bsmaj+\bxi))\cdot \sigma'(\bwsig^\top\bz_0 +\bw^\top\bsmaj+\bw^\top\bxi) \right] \nonumber \\
        &\qquad +\lambda \E_{\bxi}\left[\psi(-\gamma(\bz_0-\bsmaj+\bxi))\cdot \sigma'(\bwsig^\top\bz_0-\bw^\top\bsmaj+\bw^\top\bxi) \right] \nonumber \\
        &\qquad -(1-\lambda) \E_{\bxi}\left[\psi(-\gamma(-\bz_0+\bsmaj+\bxi))\cdot \sigma'(-\bwsig^\top\bz_0+\bw^\top\bsmaj+\bw^\top\bxi)\right] \nonumber \\
        &\qquad -\lambda \E_{\bxi}\left[\psi(-\gamma(-\bz_0-\bsmaj+\bxi))\cdot \sigma'(-\bwsig^\top\bz_0-\bw^\top\bsmaj+\bw^\top\bxi)\right]\Big).
    \end{align}
    Recall from \Lemref{sigopp_l0} that by casework on $\sigma'$ and symmetry of $\bxi$ about zero, we have
    \begin{align*}
        & \E_{\bxi} \left[|\sigma'(\bwsig^\top\bz_0+\bw^\top\bsmaj+\bw^\top\bxi)-\sigma'(-\bwsig^\top\bz_0+\bw^\top\bsmaj+\bw^\top\bxi)| \right] \\
        &\qquad  =\P_{\bxi}\left(|\bw^\top\bxi+\bw^\top\bsmaj| \leq \sqrt{2}\norm{\bwsig} \right) =\P_{\bxi}\left(|\bw^\top\bxi+\bw^\top\be_3| \leq \sqrt{2}\norm{\bwsig} \right), \\
        & \E_{\bxi}\left[|\sigma'(\bwsig^\top\bz_0-\bw^\top\bsmaj+\bw^\top\bxi)-\sigma'(-\bwsig^\top\bz_0-\bw^\top\bsmaj+\bw^\top\bxi)|\right] \\
        &\qquad = \P_{\bxi}\left(|\bw^\top\bxi-\bw^\top\bsmaj| \leq \sqrt{2}\norm{\bwsig} \right) = \P_{\bxi}\left(|\bw^\top\bxi+\bw^\top\be_3| \leq \sqrt{2}\norm{\bwsig} \right).
    \end{align*}
    Now let us analyze the difference between the margin terms.
    For the majority group terms, using \Eqref{equal_and_opposite_sigopp} with \Lemref{sigmoid_lipschitz} yields
    \begin{align*}
        &|\psi(-\gamma(\bz_0+\bsmaj+\bxi)) - \psi(-\gamma(-\bz_0+\bsmaj+\bxi))| \\
        &\qquad \lesssim \exp\big(-\min(|\gamma(\bz_0+\bsmaj+\bxi)|,|\gamma(-\bz_0+\bsmaj+\bxi)|) \big) \cdot |\gamma(\bz_0+\bsmaj+\bxi) - \gamma(-\bz_0+\bsmaj+\bxi)| \\
        &\qquad \lesssim \exp\big(-\min((1\pm o(1))\cdot\gamma, (1\pm o(1))\cdot\gamma)\big) \cdot |f(\bz_0+\bsmaj+\bxi)-f(-\bz_0+\bsmaj+\bxi)| \\
        &\qquad \lesssim \exp(-c\gamma)\cdot |f(\bz_0+\bsmaj+\bxi)-f(-\bz_0+\bsmaj+\bxi)|,
    \end{align*}
    for any $c<1$ and $\bxi$ satisfying the event $\Etest$.
    Under this same event, using the norm bound in the statement of the lemma, we have $\sgn(\bw^\top\bx)=\sgn(\wsp x_3)$ for all neurons $(a,\bw)$ following \Lemref{end_of_phase1}.
    Hence, the set $S$ of active ReLUs are the same for any two points in $\Xmaj$ and any two points in $\Xmin$ that satisfy the event $\Etest$.
    Therefore, for any $\bxi$ satisfying the event $\Etest$, we have
    \begin{align*}
        &|f(\bz_0+\bsmaj+\bxi)-f(-\bz_0+\bsmaj+\bxi)| \\
        &\qquad = \bigg|\frac{1}{p}\sum_{(a,\bw)} a \sigma(\bw^\top(\bz_0+\bsmaj+\bxi))-\frac{1}{p}\sum_{(a,\bw)} a \sigma(\bw^\top(-\bz_0+\bsmaj+\bxi)) \bigg| \\
        &\qquad = \bigg|\frac{1}{p}\sum_{(a,\bw)\in S} a \bw^\top(\bz_0+\bsmaj+\bxi)-\frac{1}{p}\sum_{(a,\bw)\in S} a\bw^\top(-\bz_0+\bsmaj+\bxi) \bigg| \\
        &\qquad = \bigg|\frac{2}{p}\sum_{(a,\bw)\in S} a \bw^\top\bz_0 \bigg| \\
        &\qquad \lesssim \max_{(a,\bw)} |a|\norm{\bwsig}.
    \end{align*}
    Using $|ac-bd|\leq |ac - bc| + |bc - bd| \leq |a-b| + |c-d|$ for $0\leq a,b,c,d\leq 1$, we then have
    \begin{align*}
        &\E_{\bxi}\Big[\psi(-\gamma(\bz_0+\bsmaj+\bxi))\cdot \sigma'(\bwsig^\top\bz_0 +\bw^\top\bsmaj+\bw^\top\bxi) \\
        &\qquad -\psi(-\gamma(-\bz_0+\bsmaj+\bxi))\cdot \sigma'(-\bwsig^\top\bz_0+\bw^\top\bsmaj+\bw^\top\bxi)\Big] \\
        &\lesssim e^{-c\gamma}\Big(\max_{(a,\bw)}|a|\norm{\bwsig}\Big)+\P_{\bxi}\left(|\bw^\top\bxi+\bw^\top\be_3| \leq \sqrt{2}\norm{\bwsig}\right)+ d^{-C},
    \end{align*}
    where the final $d^{-C}$ follows in the case that $\bxi$ does not satisfy the event $\Etest$, using $\P_{\bxi}[\Etest^{\textsf{C}}]\leq d^{-C}$ by \Lemref{events}.
    An identical argument and inequality results for the minority group terms in \Eqref{wsig_lp_expansion}.
    Combining the majority group and minority group terms in \Eqref{wsig_lp_expansion} with $|\bwsig^\top\bz_0|=\sqrt{2}\norm{\bwsig}$, we obtain
    \begin{equation*}
        -\bwsig^\top\nabla_{\bw} L_\rho \lesssim |a|\norm{\bwsig}\Big(e^{-c\gamma}\Big(\max_{(a,\bw)}|a|\norm{\bwsig}\Big) +\P_{\bxi}\left(|\bw^\top\bxi+\bw^\top\be_3| \leq \sqrt{2}\norm{\bwsig}\right) + d^{-C}\Big),
    \end{equation*}
    as desired.
    Repeating a similar analysis for $\bwopp$ gives
    \begin{equation*}
        -\bwopp^\top\nabla_{\bw} L_\rho \lesssim |a|\norm{\bwopp}\Big(e^{-c\gamma}\Big(\max_{(a,\bw)}|a|\norm{\bwopp}\Big)+\P_{\bxi}\left(|\bw^\top\bxi+\bw^\top\be_3| \leq \sqrt{2}\norm{\bwopp}\right) + d^{-C}\Big).
    \end{equation*}
    This completes the proof of the lemma.
\end{proof}

Next, we provide the lemma for the $\wsp$ component under the same margin concentration condition.
This lemma recovers \Lemref{sp_l0} when $\psi(\gamma)=\tfrac{1}{2}$, and can be thought of as its generalization to the $L_\rho$ gradient.
\begin{lemma} \label{lem:sp_lp}
    Assume there exists $\gamma,\epsilon>0$ such that
    \begin{equation} \label{eq:equal_and_opposite_sp}
        \gamma(\bx) =
        \begin{cases*}
            (1\pm \epsilon)\cdot \gamma & $\forall \bx\in \Xmaj$ satisfying $\Etest$\\
            (-1\pm \epsilon)\cdot \gamma & $\forall \bx\in \Xmin$ satisfying $\Etest$,
        \end{cases*}
    \end{equation}
    where the event $\Etest$ is defined in \Defref{events_test}.
    We then have
    \begin{equation*}
        \Big|-\bwsp^\top\nabla_{\bw} L_\rho -
        a\wsp \left( (1-\lambda-\psi(\gamma)) \cdot \left( 1 + \frac{\epsw}{2}\right) + \lambda\psi(\gamma)\epsw \right)\Big| \lesssim |a||\wsp|(\epsilon\gamma +d^{-C}).
    \end{equation*}
\end{lemma}
\begin{proof}
    Combining \Eqref{grad_decomp_sp} and \Eqref{lrho_gradient_decomp}, we have
    \begin{equation*}
        -\bwsp^\top\nabla_{\bw} L_\rho = 2\E_{\bx,\bz,\bs}\left[\psi(-\gamma(\bx))\cdot ay(\bz)\sigma'(\bw^\top\bx)\bwsp^\top\bs \right].
    \end{equation*}
    By definition of $\bs$, and noting that $\bx \in \Xmaj$ with probability $1 - \lambda$ and $\bx \in \Xmin$ with probability $\lambda$, we have
    \begin{align*}
        -\bwsp^\top\nabla_{\bw} L_\rho &= 2a\Big( \\
        &\qquad (1-\lambda) \E_{\bz,\bxi \mid \bx\in\Xmaj}\left[\psi(-\gamma(\bx))\cdot y(\bz)\cdot \sigma'(\bw^\top\bz+\bw^\top(y(\bz)\be_3)+\bw^\top\bxi)\cdot \bwsp^\top(y(\bz)\be_3) \right] \\
        &\qquad  + \lambda \E_{\bz,\bxi \mid \bx\in\Xmin}\left[\psi(-\gamma(\bx))\cdot y(\bz)\cdot \sigma'(\bw^\top\bz+\bw^\top(-y(\bz)\be_3)+\bw^\top\bxi)\cdot \bwsp^\top(-y(\bz)\be_3) \right] \Big).
    \end{align*}
    Above, we used that the conditional distribution of $(\bz,\bxi)$ is the same for majority and minority groups (already noted in the proof of~\Lemref{sp_l0}).
    Since $y(\bz)^2=1$, we obtain
    \begin{align} \label{eq:wsp_lp_popgrad}
        -\bwsp^\top\nabla_{\bw} L_\rho &= 2a\wsp \Big( \nonumber \\
        &\qquad (1-\lambda) \E_{\bz,\bxi \mid \bx\in\Xmaj}\left[\psi(-\gamma(\bx))\cdot \sigma'(\bw^\top\bz+y(\bz)\wsp+\bw^\top\bxi) \right] \nonumber \\
        &\qquad - \lambda \E_{\bz,\bxi \mid \bx\in\Xmin}\left[\psi(-\gamma(\bx))\cdot\sigma'(\bw^\top\bz-y(\bz)\wsp+\bw^\top\bxi) \right] \Big).
    \end{align}
    \paragraph{Notation.}
    Let us use the shorthand $\bxi\in \Etest$ to denote that $\bxi$ satisfies the event $\Etest$. (Note that $\Etest$ is used to denote an event and a set interchangeably; this is an overloading of notation which is valid since $\bxi$ is uniformly distributed).
    Recall that we defined $\bx := \bz + \bs + \bxi$, and that the event $\bx \in \Xmaj$ is defined only on $\bz, \bs$ (in particular, $\bx \in \Xmaj \iff x_3 = - x_1 x_2$, \ie $s_3 = -z_1 z_2$).
    We write $(\bz,\bs)\in\Xmaj$ as shorthand.
    
    Recall that $\bxi\sim\Unif(0^3\times\{\pm 1\}^{d-3})$ and $\bxi$ is independent of $(\bz,\bs)$.
    Considering the majority group conditional expectation term in Equation~\eqref{eq:wsp_lp_popgrad}, we can therefore write
    \begin{align}
        &\E_{\bz,\bxi \mid \bx\in\Xmaj}\left[\psi(-\gamma(\bx))\cdot \sigma'(\bw^\top\bz+y(\bz)\wsp+\bw^\top\bxi) \right] \nonumber\\
        &\qquad= \frac{1}{2^{d-3}} \left[\sum_{\bxi \in 0^3\times\{\pm 1\}^{d-3}} \E_{(\bz, \bs) \in \Xmaj} \left[ \psi(-\gamma(\bz + \bs + \bxi)) \cdot \sigma'(\bw^\top \bz + y(\bz) \wsp + \bw^\top \bxi) \right] \right] \nonumber\\
        &\qquad= \frac{1}{2^{d-3}} \left[\sum_{\bxi \in \Etest} \E_{(\bz, \bs) \in \Xmaj} \left[ \psi(-\gamma(\bz + \bs + \bxi)) \cdot \sigma'(\bw^\top \bz + y(\bz) \wsp + \bw^\top \bxi) \right] \right] \nonumber\\
        &\qquad\qquad + \frac{1}{2^{d-3}} \left[\sum_{\bxi \notin \Etest} \E_{(\bz, \bs) \in \Xmaj} \left[ \psi(-\gamma(\bz + \bs + \bxi)) \cdot \sigma'(\bw^\top \bz + y(\bz) \wsp + \bw^\top \bxi) \right]\right].\label{eq:wsp_lp_popgrad_intermediate}
    \end{align}
    Now, since $\bxi\sim\Unif(0^3\times\{\pm 1\}^{d-3})$ and $\bxi$ is independent of $(\bz,\bs)$, the conditional distribution of $\bxi\mid \bxi \in \Etest$ is uniform over the set $\Etest$; moreover, $\P_{\bxi}(\bxi \in \Etest) = \frac{|\Etest|}{2^{d-3}}$.
    This means that 
    \begin{align*}
        &\P_{\bxi}(\bxi \in \Etest) \cdot \E_{\bz,\bxi \mid \bx\in\Xmaj, \bxi\in\Etest}\left[\psi(-\gamma(\bx))\cdot \sigma'(\bw^\top\bz+y(\bz)\wsp+\bw^\top\bxi)\right] \\
        &\qquad = \frac{|\Etest|}{2^{d-3}} \cdot \frac{1}{|\Etest|} \cdot \left[\sum_{\bxi \in \Etest}\E_{(\bz, \bs) \in \Xmaj} \left[ \psi(-\gamma(\bz + \bs + \bxi)) \cdot \sigma'(\bw^\top \bz + y(\bz) \wsp + \bw^\top \bxi) \right] \right],
    \end{align*}
    and similarly,
    \begin{align*}
        &\P_{\bxi}(\bxi \notin \Etest) \cdot \E_{\bz,\bxi \mid \bx\in\Xmaj, \bxi\notin\Etest}\left[\psi(-\gamma(\bx))\cdot \sigma'(\bw^\top\bz+y(\bz)\wsp+\bw^\top\bxi)\right] \\
        &\qquad = \frac{|\Etest^{\textsf{C}}|}{2^{d-3}} \cdot \frac{1}{|\Etest^{\textsf{C}}|} \cdot \left[\sum_{\bxi \notin \Etest}\E_{(\bz, \bs) \in \Xmaj} \left[ \psi(-\gamma(\bz + \bs + \bxi)) \cdot \sigma'(\bw^\top \bz + y(\bz) \wsp + \bw^\top \bxi) \right] \right].
    \end{align*}
    Substituting these into the display in Equation~\eqref{eq:wsp_lp_popgrad_intermediate} yields
    \begin{align} \label{eq:etest_conditional}
        & \E_{\bz,\bxi \mid \bx\in\Xmaj}\left[\psi(-\gamma(\bx))\cdot \sigma'(\bw^\top\bz+y(\bz)\wsp+\bw^\top\bxi)\right] \nonumber \\
        &\qquad =\P_{\bxi}(\bxi\in \Etest) \cdot \E_{\bz,\bxi \mid \bx\in\Xmaj, \bxi\in\Etest}\left[\psi(-\gamma(\bx))\cdot \sigma'(\bw^\top\bz+y(\bz)\wsp+\bw^\top\bxi)\right] \nonumber \\
        &\qquad\qquad + \P_{\bxi}(\bxi\notin \Etest) \cdot \E_{\bz,\bxi \mid \bx\in\Xmaj, \bxi\notin \Etest}\left[\psi(-\gamma(\bx))\cdot \sigma'(\bw^\top\bz+y(\bz)\wsp+\bw^\top\bxi)\right].
    \end{align}
    For the first term in \Eqref{etest_conditional}, we use the condition in \Eqref{equal_and_opposite_sp} and the fact that $\sigma'(\cdot) \geq 0$ to obtain the sandwich inequality
    \begin{align*}
        &\psi(-(1 + \epsilon)\gamma)\cdot \E_{\bz,\bxi | \bx\in\Xmaj,\bxi\in\Etest}\left[ \sigma'(\bw^\top\bz+y(\bz)\wsp+\bw^\top\bxi) \right] \\
        &\qquad \leq \E_{\bz,\bxi \mid \bx\in\Xmaj, \bxi\in\Etest}\left[\psi(-\gamma(\bx))\cdot \sigma'(\bw^\top\bz+y(\bz)\wsp+\bw^\top\bxi)\right] \\
        &\qquad\leq \psi(-(1 -\epsilon)\gamma)\cdot \E_{\bz,\bxi | \bx\in\Xmaj,\bxi\in\Etest}\left[ \sigma'(\bw^\top\bz+y(\bz)\wsp+\bw^\top\bxi) \right].
    \end{align*}
    By $\frac{1}{4}$-Lipschitzness of $\psi$, we then have $|\psi(-\gamma)-\psi(-(1\pm \epsilon)\gamma)| \lesssim \epsilon \gamma$.
    Combined with $\sigma'(\cdot)\leq 1$ and the above display, we obtain the two-sided inequality
    \begin{align} \label{eq:etest_conditional_term1}
        & \Big| \E_{\bz,\bxi \mid \bx\in\Xmaj, \bxi\in\Etest}\left[\psi(-\gamma(\bx))\cdot \sigma'(\bw^\top\bz+y(\bz)\wsp+\bw^\top\bxi)\right] \nonumber \\
        &\qquad - \psi(-\gamma) \cdot \E_{\bz,\bxi \mid \bx\in\Xmaj, \bxi\in\Etest}\left[\sigma'(\bw^\top\bz+y(\bz)\wsp+\bw^\top\bxi)\right] \Big| \lesssim \epsilon \gamma.
    \end{align}
    For the second term in \Eqref{etest_conditional}, we use $\psi(\cdot)\leq 1$ and $\sigma'(\cdot)\leq 1$ to obtain
    \begin{equation} \label{eq:etest_conditional_term2}
        \E_{\bz,\bxi \mid \bx\in\Xmaj, \bxi\notin \Etest}\left[\psi(-\gamma(\bx))\cdot \sigma'(\bw^\top\bz+y(\bz)\wsp+\bw^\top\bxi)\right] \leq 1.
    \end{equation}
    We will now substitute \Eqref{etest_conditional_term1} and \Eqref{etest_conditional_term2} into \Eqref{etest_conditional}.
    Using the triangle inequality and $\P_{\bxi}[\bxi\notin \Etest]\leq d^{-C}$ by \Lemref{events}, we obtain
    \begin{align} \label{eq:test_conditional_xmaj}
        & \Big| \E_{\bz,\bxi \mid \bx\in\Xmaj}\left[\psi(-\gamma(\bx))\cdot \sigma'(\bw^\top\bz+y(\bz)\wsp+\bw^\top\bxi)\right] \nonumber \\
        &\qquad - \psi(-\gamma) \cdot \E_{\bz,\bxi \mid \bx\in\Xmaj, \bxi\in\Etest}\left[\sigma'(\bw^\top\bz+y(\bz)\wsp+\bw^\top\bxi)\right] \Big| \lesssim \epsilon\gamma + d^{-C}.
    \end{align}
    An identical series of steps to the above follows for the minority group.
    Substituting the result into \Eqref{wsp_lp_popgrad} yields
    \begin{align}\label{eq:gradient-sp-marginsremoved}
        \Big|-\bwsp^\top\nabla_{\bw} L_\rho &- 2a\wsp \Big( \nonumber \\
        & (1-\lambda)\cdot \psi(-\gamma)\cdot \E_{\bz,\bxi | \bx\in\Xmaj,\bxi\in\Etest}\left[ \sigma'(\bw^\top\bz+y(\bz)\wsp+\bw^\top\bxi) \right] \nonumber\\
        & - \lambda \cdot \psi(\gamma)\cdot \E_{\bz,\bxi | \bx\in \Xmin,\bxi\in\Etest}\left[\sigma'(\bw^\top\bz-y(\bz)\wsp+\bw^\top\bxi) \right] \Big)\Big| \lesssim |a| |\wsp| (\epsilon\gamma + d^{-C}).
    \end{align}
    The remainder of the proof will simplify the expression on the left-hand side of \Eqref{gradient-sp-marginsremoved} to obtain the desired statement of the lemma.
    These steps will resemble corresponding steps in the proof of~\Lemref{sp_l0}.
    Using $\psi(-\gamma) = 1 - \psi(\gamma)$ and the definition of $\sigma$, we obtain
    \begin{align*}
        &(1-\lambda) \cdot \psi(-\gamma) \cdot \E_{\bz,\bxi | \bx\in\Xmaj,\bxi\in\Etest}\left[\sigma'(\bw^\top\bz+y(\bz)\wsp+\bw^\top\bxi)\right] \\
        & \qquad - \lambda \cdot \psi(\gamma)\cdot \E_{\bz,\bxi | \bx\in\Xmin,\bxi\in\Etest}\left[\sigma'(\bw^\top\bz-y(\bz)\wsp+\bw^\top\bxi)\right]\\
        &= (1-\lambda) \cdot (1-\psi(\gamma)) \cdot \P_{\bz,\bxi | \bx\in\Xmaj,\bxi\in\Etest}\left(\bw^\top\bz+\bw^\top\bxi>-y(\bz)\wsp\right) \\
        & \qquad - \lambda \cdot \psi(\gamma) \cdot \P_{\bz,\bxi | \bx\in\Xmin,\bxi\in\Etest}\left(\bw^\top\bz+\bw^\top\bxi>y(\bz)\wsp\right).
    \end{align*}
    Note that $\bxi\mid\bxi\in \Etest$ is symmetric about $0$ since $\bxi$ is symmetric and enters $\Etest$ only through the term $|\bw^\top\bxi|$, which is invariant under the mapping $\bxi\mapsto-\bxi$.
    Then, following \Eqref{wsp_probs_cancel} and using that both $\bz$ and $\bxi\mid\bxi\in\Etest$ are symmetric about $0$, we have
    \begin{align*}
        & \P_{\bz,\bxi|\bxi\in\Etest}\left(\bw^\top\bz+\bw^\top\bxi>-y(\bz)\wsp \right) - \P_{\bz,\bxi|\bxi\in\Etest}\left(\bw^\top\bz+\bw^\top\bxi>y(\bz)\wsp\right) \\
        &\qquad = 2\P_{\bz,\bxi|\bxi\in\Etest}\left(\bw^\top\bz+\bw^\top\bxi>-y(\bz)\wsp \right) - 1.
    \end{align*}
    Using that the conditional and unconditional distributions of $(\bz,\bxi)$ coincide according to \Eqref{XOR-spurious-correlation}, the left-hand side of \Eqref{gradient-sp-marginsremoved} simplifies to
    \begin{align} 
        &2a\wsp \Big( \nonumber \\
        &\qquad (1-\lambda-\psi(\gamma)) \cdot \P_{\bz,\bxi|\bxi\in\Etest}\left(\bw^\top\bz+\bw^\top\bxi>-y(\bz)\wsp \right) \nonumber \\
        &\qquad + \lambda \cdot \psi(\gamma) \cdot (2\P_{\bz,\bxi|\bxi\in\Etest}\left(\bw^\top\bz+\bw^\top\bxi>-y(\bz)\wsp \right) - 1)\Big) \nonumber \\
        &= 2a\wsp \Big( (1-\lambda-\psi(\gamma)+2\lambda\psi(\gamma)) \cdot \P_{\bz,\bxi|\bxi\in\Etest}\left(\bw^\top\bz+\bw^\top\bxi>-y(\bz)\wsp \right) - \lambda\psi(\gamma) \Big) \label{eq:wsp_lp_almostfinal}
    \end{align}
    By the law of total probability, we have for any event $A$ that
    \begin{align*}
        \P_{\bxi}(A) &= \P_{\bxi}(\bxi\in \Etest) \cdot \P_{\bxi}(A\mid \bxi \in \Etest) + \P_{\bxi}(\bxi\notin \Etest) \cdot \P_{\bxi}(A\mid \bxi\notin \Etest)\\
        &= \P_{\bxi}(A \mid \bxi \in \Etest) - \P_{\bxi}(\bxi \notin \Etest) \cdot \P_{\bxi}(A | \bxi \in \Etest) + \P_{\bxi}(\bxi \notin \Etest) \cdot \P_{\bxi}(A \mid \bxi \notin \Etest).
    \end{align*}
    Rearranging and taking absolute values on both sides then gives us
    \begin{align}\label{eq:etest_diff}
        \Big|\P_{\bxi}(A) - \P_{\bxi}(A\mid \bxi\in \Etest)\Big| \leq \P_{\bxi}(\bxi \notin \Etest) \leq d^{-C}.
    \end{align}
    By \Eqref{half_plus_epsilon}, we have $\P_{\bz,\bxi}\left(\bw^\top\bz+\bw^\top\bxi>-y(\bz)\wsp\right)=\tfrac{1}{2}+\frac{\epsw}{4}$.
    Hence, by \Eqref{etest_diff} we have
    \begin{equation*}
        \bigg| \P_{\bz,\bxi|\bxi\in\Etest}\left(\bw^\top\bz+\bw^\top\bxi>-y(\bz)\wsp\right) - \left(\frac{1}{2}+\frac{\epsw}{4}\right)\bigg| \leq d^{-C}.
    \end{equation*}
    Applying this to \Eqref{wsp_lp_almostfinal}, we finally have
    \begin{align*}
        & \bigg| 2a\wsp \Big( (1-\lambda-\psi(\gamma)+2\lambda\psi(\gamma)) \cdot \P_{\bz,\bxi|\bxi\in\Etest}\left(\bw^\top\bz+\bw^\top\bxi>-y(\bz)\wsp \right) - \lambda\psi(\gamma) \Big) \\
        &\qquad - a\wsp \left( (1-\lambda-\psi(\gamma)) \cdot \left(1+\frac{\epsw}{2}\right) + \lambda\psi(\gamma)\epsw \right)\bigg| \lesssim |a||\wsp| d^{-C},
    \end{align*}
    where the error term can be combined with the right-hand side of \Eqref{gradient-sp-marginsremoved}.
    This completes the proof of the lemma.
\end{proof}

Next, we will show the lemma for the $\bwperp$ component, again under the margin concentration condition.

\begin{lemma} \label{lem:wperp_lp_2}
    Assume there exists $\gamma>0$ such that
    \begin{equation} \label{eq:equal_and_opposite_wperp_2}
        \gamma(\bx) =
        \begin{cases*}
            (1\pm o(1))\cdot \gamma & $\forall \bx\in\Xmaj$ satisfying $\Etest$\\
            (-1\pm o(1))\cdot \gamma & $\forall \bx\in \Xmin$ satisfying $\Etest$,
        \end{cases*}
    \end{equation}
    where the event $\Etest$ is defined in \Defref{events_test}.
    If all neurons $(a,\bw)$ satisfy $\frac{1}{C}\norm{\bwsp}\geq \norm{\bwsig}+\norm{\bwopp}+\norm{\bwperp}\log^{1/2}(d)$, then for any neuron $(a,\bw)$ and any $c<1$ we have
    \begin{equation*}
        \frac{1}{|a|}\norm{\nabla_{\bwperp} L_\rho} \lesssim e^{-c\gamma}|a|\norm{\bwperp} + \sqrt{\P_{\bxi}\left(|\bw^\top\bxi|\geq|\bw^\top\bz+\bw^\top\bs|\right)} + d^{1/2-C}.
    \end{equation*}
\end{lemma}
\begin{proof}
    By the dual norm characterization we have
    \begin{equation*}
        \norm{\nabla_{\bwperp} L_\rho} = \sup_{\norm{\bv}=1}|\bv^\top\nabla_{\bwperp} L_\rho |.
    \end{equation*}
    Combining \Eqref{grad_decomp_perp} and \Eqref{lrho_gradient_decomp}, we have
    \begin{equation*}
        \norm{\nabla_{\bwperp} L_\rho} = 2 \sup_{\norm{\bv}=1} \Big| \E_{\bx,\bxi}\left[\psi(-\gamma(\bx))\cdot a y(\bx) \sigma'(\bw^\top\bx) \bv^\top\bxi \right] \Big|.
    \end{equation*}
    Using that $\bxi$ is symmetric about zero, we have
    \begin{align*}
        \norm{\nabla_{\bwperp} L_\rho} &= \sup_{\norm{\bv}=1} \Big| \E_{\bz,\bs,\bxi}\Big[ay(\bz)\bv^\top\bxi \Big( \\
        &\qquad \psi(-\gamma(\bz+\bs+\bxi)) \cdot \sigma'(\bw^\top(\bz+\bs+\bxi)) \\
        &\qquad - \psi(-\gamma(\bz+\bs-\bxi)) \cdot \sigma'(\bw^\top(\bz+\bs-\bxi)) \Big) \Big] \Big|.
    \end{align*}
    Bringing the absolute value inside by Jensen's inequality, and using $|y(\bz)|=1$ and $|ac-bd|\leq |a-b|+|c-d|$ for $0\leq a,b,c,d\leq 1$, we have
    \begin{align*}
        \norm{\nabla_{\bwperp} L_\rho} &\leq |a| \sup_{\norm{\bv}=1} \E_{\bz,\bs,\bxi}\Big[|\bv^\top\bxi| \cdot \Big( \\
        &\qquad |\psi(-\gamma(\bz+\bs+\bxi)) - \psi(-\gamma(\bz+\bs-\bxi))| \\
        &\qquad + |\sigma'(\bw^\top(\bz+\bs+\bxi))-\sigma'(\bw^\top(\bz+\bs-\bxi))| \Big) \Big].
    \end{align*}
    Recall from \Lemref{perp_l0} that by casework on $\sigma'$ we have
    \begin{equation*}
        |\sigma'(\bw^\top\bz+\bw^\top\bs+\bw^\top\bxi)-\sigma'(\bw^\top\bz+\bw^\top\bs-\bw^\top\bxi)|=\ind\left(|\bw^\top\bxi|\geq|\bw^\top\bz+\bw^\top\bs| \right).
    \end{equation*}
    Using \Eqref{equal_and_opposite_wperp_2} with \Lemref{sigmoid_lipschitz}, we have
    \begin{align*}
        & |\psi(-\gamma(\bz+\bs+\bxi))-\psi(-\gamma(\bz+\bs-\bxi))| \\
        &\lesssim \exp\big(-\min(|\gamma(\bz+\bs+\bxi)|, |\gamma(\bz+\bs-\bxi)|) \big) \cdot |\gamma(\bz+\bs+\bxi)-\gamma(\bz+\bs-\bxi)| \\
        &\lesssim \exp\big(-\min((1\pm o(1))\cdot\gamma,(1\pm o(1))\cdot\gamma) \big) \cdot |f(\bz+\bs+\bxi)-f(\bz+\bs-\bxi)| \\
        &\lesssim \exp(-c\gamma) \cdot |f(\bz+\bs+\bxi)-f(\bz+\bs-\bxi)|,
    \end{align*}
    for any $c<1$ and $\bxi$ satisfying the event $\Etest$.
    Under this same event, using the norm bound in the statement of the lemma, we have $\sgn(\bw^\top\bx)=\sgn(\wsp x_3)$ for all neurons $(a,\bw)$ following \Lemref{end_of_phase1}.
    Hence, the set $S$ of active ReLUs are the same for any two points in $\Xmaj$ and any two points in $\Xmin$ that satisfy the event $\Etest$.
    Therefore, for any $\bxi$ satisfying the event $\Etest$, we have
    \begin{align*}
        & |f(\bz+\bs+\bxi)-f(\bz+\bs-\bxi)| \\
        &\qquad = \bigg|\frac{1}{p}\sum_{(a,\bw)} a \sigma(\bw^\top(\bz+\bs+\bxi))-\frac{1}{p}\sum_{(a,\bw)} a \sigma(\bw^\top(\bz+\bs-\bxi)) \bigg| \\
        &\qquad = \bigg|\frac{1}{p}\sum_{(a,\bw)\in S} a \bw^\top(\bz+\bs+\bxi)-\frac{1}{p}\sum_{(a,\bw)\in S} a\bw^\top(\bz+\bs-\bxi) \bigg| \\
        &\qquad = \frac{2}{p}\sum_{(a,\bw)\in S} |a| |\bw^\top\bxi|.
    \end{align*}
    On the other hand, if $\bxi$ does not satisfy the event $\Etest$ then we can use $|\bv^\top\bxi|\leq \sqrt{d}$ and $0\leq \psi(\cdot)\leq 1$ to find
    \begin{equation*}
        |\bv^\top\bxi|\cdot |\psi(-\gamma(\bz+\bs+\bxi))-\psi(-\gamma(\bz+\bs-\bxi))| \leq \sqrt{d}.
    \end{equation*}
    As in \Lemref{sp_lp}, let us use the shorthand $\bxi\in\Etest$ to denote that $\bxi$ satisfies the event $\Etest$.
    Then we have
    \begin{align}
        \frac{1}{|a|}\norm{\nabla_{\bwperp} L_\rho} &\lesssim \sup_{\norm{\bv}=1} \frac{2e^{-c\gamma}}{p}\sum_{(a,\bw)\in S} |a|\cdot \E_{\bxi\mid\bxi\in\Etest}\left[|\bv^\top\bxi| \cdot|\bw^\top\bxi|\right] \nonumber \\
        &\qquad + \E_{\bxi}\left[\left|\bv^\top\bxi \right|\cdot \ind\left(|\bw^\top\bxi|\geq|\bw^\top\bz+\bw^\top\bs|\right)\right] + d^{1/2-C}, \label{eq:wperp_lp}
    \end{align}
    where the final $d^{1/2-C}$ follows in the case that $\bxi$ does not satisfy the event $\Etest$, using $\P_{\bxi}[\Etest^{\textsf{C}}]\leq d^{-C}$ by \Lemref{events}.
    For any $\bq\in\R^{d-3}$ we have $\E_{\bxi}[(\bq^\top\bxi)^2]=\norm{\bq}^2$ by isotropy of $\bxi$.
    When we condition on $\Etest$, this equality may not hold precisely, but we can use the following fact (similarly to \Lemref{sp_lp}).
    Because $\P_{\bxi}(\Etest) \geq 1 - d^{-C}$ by \Lemref{events} and the distribution of $\bxi$ is uniform, we have
    \begin{align*}
    \E_{\bxi}[(\bq^\top\bxi)^2] &= \frac{1}{2^{d-3}} \sum_{\bxi \in \{\pm1\}^{d-3}} (\bq^\top\bxi)^2 \\
    &= \frac{1}{2^{d-3}} \sum_{\bxi \in \Etest} (\bq^\top\bxi)^2 + \frac{1}{2^{d-3}} \sum_{\bxi \notin \Etest} (\bq^\top\bxi)^2 \\
    &\geq \P(\bxi \in \Etest) \E_{\bxi \mid \Etest}[(\bq^\top\bxi)^2] \\
    &\geq (1 - d^{-C}) \E_{\bxi \mid \Etest}[(\bq^\top\bxi)^2],
    \end{align*}
    which then gives us
    \begin{equation*}
        \E_{\bxi \mid \Etest}[(\bq^\top\bxi)^2] \leq \frac{\E_{\bxi}[(\bq^\top\bxi)^2]}{1 - d^{-C}} \lesssim \E_{\bxi}[(\bq^\top\bxi)^2] = \norm{\bq}^2.
    \end{equation*}
    For the first term in \Eqref{wperp_lp}, we have by the Cauchy-Schwarz inequality that
    \begin{equation*}
        \E_{\bxi\mid\bxi\in\Etest}\left[|\bv^\top\bxi| \cdot|\bw^\top\bxi|\right] \leq \sqrt{\E_{\bxi\mid\bxi\in\Etest}\left[(\bv^\top\bxi)^2 \right]} \cdot \sqrt{\E_{\bxi\mid\bxi\in\Etest}\left[(\bw^\top\bxi)^2 \right]} \lesssim \norm{\bv}\norm{\bwperp}.
    \end{equation*}
    Therefore, using $|S|\leq p$ we have
    \begin{equation*}
        \sup_{\norm{\bv}=1} \frac{2e^{-c\gamma}}{p}\sum_{(a,\bw)\in S} |a|\cdot \E_{\bxi\mid\bxi\in\Etest}\left[|\bv^\top\bxi| \cdot|\bw^\top\bxi|\right] \lesssim \sup_{\norm{\bv}=1} \frac{2e^{-c\gamma}}{p}\sum_{(a,\bw)\in S} |a| \norm{\bv}\norm{\bwperp} \lesssim e^{-c\gamma} |a|\norm{\bwperp}.
    \end{equation*}
    Then, for the second term in \Eqref{wperp_lp}, we have by the Cauchy-Schwarz inequality that 
    \begin{align*}
        \sup_{\norm{\bv}=1} \E_{\bxi}\left[|\bv^\top\bxi |\cdot \ind\left(|\bw^\top\bxi|\geq|\bw^\top\bz+\bw^\top\bs|\right)\right] &\leq \sup_{\norm{\bv}=1} \sqrt{\E_{\bxi}\left[(\bv^\top\bxi)^2 \right]}\cdot \sqrt{\E_{\bxi}\left[\ind\left(|\bw^\top\bxi|\geq|\bw^\top\bz+\bw^\top\bs|\right) \right]} \\
        &= \sup_{\norm{\bv}=1} \norm{\bv} \cdot \sqrt{\P_{\bxi}\left(|\bw^\top\bxi|\geq|\bw^\top\bz+\bw^\top\bs|\right)} \\
        &= \sqrt{\P_{\bxi}\left(|\bw^\top\bxi|\geq|\bw^\top\bz+\bw^\top\bs|\right)}.
    \end{align*}
    Putting it together with subadditivity of the supremum we have
    \begin{equation*}
        \frac{1}{|a|}\norm{\nabla_{\bwperp} L_\rho} \lesssim e^{-c\gamma}|a|\norm{\bwperp}+\sqrt{\P_{\bxi}\left(|\bw^\top\bxi|\geq|\bw^\top\bz+\bw^\top\bs|\right)} + d^{1/2-C}.
    \end{equation*}
    This completes the proof of the lemma.
\end{proof}

Finally, we repeat the analysis for any individual coordinate of $\bwperp$ under the same margin concentration condition.
\begin{lemma} \label{lem:wi_lp}
    Assume there exists $\gamma>0$ such that
    \begin{equation} \label{eq:equal_and_opposite_wi}
        \gamma(\bx) =
        \begin{cases*}
            (1\pm o(1))\cdot \gamma & $\forall \bx\in\Xmaj$ satisfying $\Etest$ \\
            (-1\pm o(1))\cdot \gamma & $\forall \bx\in \Xmin$ satisfying $\Etest$,
        \end{cases*}
    \end{equation}
    where the event $\Etest$ is defined in \Defref{events_test}.
    If all neurons $(a,\bw)$ satisfy $\frac{1}{C}\norm{\bwsp}\geq \norm{\bwsig}+\norm{\bwopp}+\norm{\bwperp}\log^{1/2}(d)$, then for any neuron $(a,\bw)$, any $c<1$, and any $i>3$ we have
    \begin{align*}
        -w_i\partial_{w_i} L_\rho &\lesssim |a||w_i|\Big(e^{-c\gamma}\Big(\max_{(a,\bw)}|a||w_i|\Big) + d^{-C} + \Big( \nonumber \\
        &\qquad \P_{\bxi}\left(|w_i|\geq |\sqrt{2}\norm{\bwsig}+\wsp+\bw^\top\bxi_{\setminus i}|\right)+\P_{\bxi}\left(|w_i|\geq |\sqrt{2}\norm{\bwsig}-\wsp+\bw^\top\bxi_{\setminus i}|\right) \nonumber \\
        &\qquad +\P_{\bxi}\left(|w_i|\geq |\sqrt{2}\norm{\bwopp}+\wsp+\bw^\top\bxi_{\setminus i}|\right)+\P_{\bxi}\left(|w_i|\geq |\sqrt{2}\norm{\bwopp}-\wsp+\bw^\top\bxi_{\setminus i}|\right)\Big)\Big).
    \end{align*}
\end{lemma}
\begin{proof}
    Combining \Eqref{grad_decomp_wi} and \Eqref{lrho_gradient_decomp}, we have
    \begin{equation*}
        -w_i \partial_{w_i} L_\rho = 2\E_{\bx}\left[\psi(-\gamma(\bx))\cdot ay(\bx)\sigma'(\bw^\top\bx)w_ix_i \right] = 2aw_i\E_{\bx}\left[\psi(-\gamma(\bx))\cdot y(\bx)\sigma'(\bw^\top\bx)x_i \right].
    \end{equation*}
    Using that $x_i$ is symmetric about zero (it is a single coordinate of $\bxi$), we have
    \begin{align*}
        -w_i \partial_{w_i} L_\rho &= aw_i \Big( \\
        &\qquad \E_{\bx}\left[\psi(-\gamma(\bx_{\setminus i}+\bx_i))\cdot y(\bx)\cdot \sigma'(\bw^\top\bx_{\setminus i}+w_ix_i) \right] \\
        &\qquad -\E_{\bx}\left[\psi(-\gamma(\bx_{\setminus i}-\bx_i))\cdot y(\bx)\cdot \sigma'(\bw^\top\bx_{\setminus i}-w_ix_i) \right] \Big).
    \end{align*}
    Recall from \Lemref{wi_l0} that by casework on $\sigma'$ we have
    \begin{equation*}
        \E_{\bx}\left[|\sigma'(\bw^\top\bx_{\setminus i}+w_ix_i)-\sigma'(\bw^\top\bx_{\setminus i}-w_ix_i)|\right] = \P_{\bx}\left(|w_i|\geq |\bw^\top\bx_{\setminus i}| \right).
    \end{equation*}
    Using \Eqref{equal_and_opposite_wi} with \Lemref{sigmoid_lipschitz}, we have
    \begin{align*}
        &|\psi(-\gamma(\bx_{\setminus i}+\bx_i))-\psi(-\gamma(\bx_{\setminus i}-\bx_i))| \\
        &\lesssim \exp\big(-\min(|\gamma(\bx_{\setminus i}+\bx_i)|, |\gamma(\bx_{\setminus i}-\bx_i)|) \big) \cdot |\gamma(\bx_{\setminus i}+\bx_i)-\gamma(\bx_{\setminus i}-\bx_i)| \\
        &\lesssim \exp\big(-\min((1\pm o(1))\cdot\gamma, (1\pm o(1))\cdot \gamma) \big) \cdot |f(\bx_{\setminus i}+\bx_i)-f(\bx_{\setminus i}-\bx_i)| \\
        &\lesssim \exp(-c\gamma)\cdot |f(\bx_{\setminus i}+\bx_i)-f(\bx_{\setminus i}-\bx_i)|,
    \end{align*}
    for any $c<1$ and $\bxi$ satisfying the event $\Etest$ (by definition, if $\bxi$ satisfies $\Etest$ then any one-bit-flip of $\bxi$ also satisfies $\Etest$).
    Under this same event, using the norm bound in the statement of the lemma, we have $\sgn(\bw^\top\bx)=\sgn(\wsp x_3)$ for all neurons $(a,\bw)$ following \Lemref{end_of_phase1}.
    Hence, the set $S$ of active ReLUs are the same for any two points in $\Xmaj$ and any two points in $\Xmin$ that satisfy the event $\Etest$.
    Therefore, for any $\bxi$ satisfying the event $\Etest$, we have
    \begin{align*}
        & |f(\bx_{\setminus i}+\bx_i)-f(\bx_{\setminus i}-\bx_i)| \\
        &\qquad = \bigg|\frac{1}{p}\sum_{(a,\bw)} a \sigma(\bw^\top(\bx_{\setminus i}+\bx_i))-\frac{1}{p}\sum_{(a,\bw)} a \sigma(\bw^\top(\bx_{\setminus i}-\bx_i)) \bigg| \\
        &\qquad = \bigg|\frac{1}{p}\sum_{(a,\bw)\in S} a \bw^\top(\bx_{\setminus i}+\bx_i)-\frac{1}{p}\sum_{(a,\bw)\in S} a\bw^\top(\bx_{\setminus i}-\bx_i) \bigg| \\
        &\qquad = \bigg|\frac{2}{p}\sum_{(a,\bw)\in S} a w_i x_i \bigg| \\
        &\qquad \lesssim \max_{(a,\bw)} |a||w_i|.
    \end{align*}
    Using $|ac-bd|\leq |a-b| + |c-d|$ for $0\leq a,b,c,d\leq 1$, we then have
    \begin{align*}
        & \E_{\bx}\big[\psi(-\gamma(\bx_{\setminus i}+\bx_i))\cdot y(\bx)\sigma'(\bw^\top\bx_{\setminus i}+w_ix_i) -\psi(-\gamma(\bx_{\setminus i}-\bx_i))\cdot y(\bx)\sigma'(\bw^\top\bx_{\setminus i}-w_ix_i) \big] \\
        &\qquad \lesssim e^{-c\gamma}\Big(\max_{(a,\bw)}|a||w_i|\Big)+\P_{\bxi}\left(|w_i|\geq |\bw^\top\bx_{\setminus i}|\right) + d^{-C},
    \end{align*}
    where the final $d^{-C}$ follows in the case that $\bxi$ does not satisfy the event $\Etest$, using $\P_{\bxi}[\Etest^{\textsf{C}}]\leq d^{-C}$ by \Lemref{events}.
    In particular,
    \begin{equation*}
        -w_i \partial_{w_i} L_\rho \lesssim |a||w_i|\Big(e^{-c\gamma}\Big(\max_{(a,\bw)}|a||w_i| \Big)+\P_{\bxi}\left(|w_i|\geq |\bw^\top\bx_{\setminus i}| \right) + d^{-C}\Big).
    \end{equation*}
    We already simplified $\P_{\bxi}\left(|w_i|\geq |\bw^\top\bx_{\setminus i}| \right)$ in \Lemref{wi_l0}; substituting the result (but with all terms positive since there is no $y(\bx)$ term), we obtain
    \begin{align*}
        -w_i\partial_{w_i} L_\rho &\lesssim |a||w_i|\Big(e^{-c\gamma}\Big(\max_{(a,\bw)}|a||w_i|\Big)+ d^{-C} + \Big( \nonumber \\
        &\qquad \P_{\bxi}\left(|w_i|\geq |\sqrt{2}\norm{\bwsig}+\wsp+\bw^\top\bxi_{\setminus i}|\right)+\P_{\bxi}\left(|w_i|\geq |\sqrt{2}\norm{\bwsig}-\wsp+\bw^\top\bxi_{\setminus i}|\right) \nonumber \\
        &\qquad +\P_{\bxi}\left(|w_i|\geq |\sqrt{2}\norm{\bwopp}+\wsp+\bw^\top\bxi_{\setminus i}|\right)+\P_{\bxi}\left(|w_i|\geq |\sqrt{2}\norm{\bwopp}-\wsp+\bw^\top\bxi_{\setminus i}| \right) \Big)\Big).
    \end{align*}
    This completes the proof of the lemma.
\end{proof}

\clearpage

\section{Phase I Induction} \label{sec:phase1_induction}
Recall that in Phase I, the neural network is small and hence the logistic loss function $\ell_\rho$ is well-approximated by a first-order Taylor expansion about $f_\rho=0$, \ie
\begin{equation*}
    \ell_0(\bx) \coloneqq -2\log(\tfrac{1}{2})-y(\bx)f_\rho(\bx).
\end{equation*}
In this section, we perform an induction that holds throughout Phase I to characterize SGD feature learning in the first $\TI\asymp \log\log(d)\eta^{-1}$ iterations, where $\eta\ll 1$ is the learning rate.

We prove some preliminary lemmas, including our definition of the inductive hypothesis, in \Secref{phase1_induction_prelim}.
We use these properties to bound deviation between the gradient of the minibatch loss ($\widehat{L}_\rho$) and the gradient of the population first-order approximation of the loss ($L_0$) in \Secref{phase1_l0_lp_error_analysis}.
Then, we show the inductive step in \Secref{phase1_induction_step}.

Throughout this section, we write $\rho\coloneqq\rho^{(t)}$ as shorthand where appropriate, and we let $C > 0$ denote a sufficiently large constant which does not change from line to line.
We will work in high probability under the events $\Einit$, $\Etrain$, and $\Etest$ established in \Defref{events} and \Lemref{events}.

\subsection{Technical Preliminaries} \label{sec:phase1_induction_prelim}
In Phase I, we will show for all neurons $(a,\bw)$ that the norm of the spurious component $\norm{\bwsp}$ grows exponentially fast, while the norm of the signal component $\norm{\bwsig}$, opposing component $\norm{\bwopp}$, and orthogonal component $\norm{\bwperp}$ remain close to their initialization.
In order to lower bound the growth of $\norm{\bwsp}$ we require \emph{sign alignment}, \ie $\sgn(a)=\sgn(\wsp)$.
To achieve this condition, we divide Phase I into two sub-phases (we omit the superscripts $(t)$ for brevity, but include them in formal lemmas and proofs).
During Phase Ia, sign alignment may or may not be satisfied, but by the end of Phase Ia we can guarantee that $\sgn(a) = \sgn(\wsp)$ for all neurons $(a,\bw)$ with high probability.
We show that Phase Ia lasts $\TIa\asymp \log^{1/2}(d)d^{-1/2}\eta^{-1}$ iterations.
Then, Phase Ib lasts $\TIb \asymp \log\log(d)\eta^{-1}$ iterations, during which the spurious component grows very fast.
By the end of Phase Ib, dominance of the spurious component $\bwsp$ of every neuron is guaranteed (in the sense described informally above, and made formal by \Lemref{end_of_phase1}).

The ultimate goal of Phase I is to show that $\sgn(f_\rho(\bx))=x_3$ with high probability; that is, the prediction is uniformly dominated by the shortcut spurious correlation.
Our first lemma establishes a sufficient condition for this phenomenon.
Recall that we define the ``positive'' neurons by $S^+\coloneqq \{(a,\bw):\sgn(a)=1\}$ and the ``negative'' neurons by $S^-\coloneqq \{(a,\bw):\sgn(a)=-1\}$.
\begin{lemma} \label{lem:end_of_phase1}
    Suppose $|S^+|,|S^-|>0$.
    Under the event $\Etest$ (\Defref{events_test}), if all neurons $(a,\bw)$ satisfy $\sgn(\wsp)=\sgn(a)$ and $\frac{1}{C}\norm{\bwsp}\geq \norm{\bwsig}+\norm{\bwopp}+\norm{\bwperp}\log^{1/2}(d)$, then $\sgn(f_\rho(\bx))=x_3$.
\end{lemma}
\begin{proof}
    The definition of the event $\Etest$ implies that $|\bwperp^\top\bx| < C \norm{\bwperp}\log^{1/2}(d)$ for all neurons $(a,\bw)$.
    Under this event, the norm bound in the statement of the lemma gives
    \begin{equation*}
        |\wsp x_3| = \|\bwsp\| \geq C \|\bwsig\| + C \|\bwopp\| + C \|\bwperp\| \log^{1/2}(d) > \frac{C}{\sqrt{2}} |w_1x_1+w_2x_2| + |\bwperp^\top \bx|,
    \end{equation*}
    wherein choosing $C>\sqrt{2}$ implies $\sgn(\bw^\top\bx)=\sgn(\wsp x_3)$.
    Moreover, the sign alignment condition $\sgn(\wsp)=\sgn(a)$ gives
    \begin{equation} \label{eq:wsp_dominates}
        \sgn(a\sigma(\bw^\top\bx)) =
        \begin{cases}
            x_3 & \sgn(\wsp)=x_3 \\
            0 & \sgn(\wsp)\neq x_3.
        \end{cases}
    \end{equation}
    Since $|S^+|,|S^-|>0$ we may write
    \begin{equation*}
        pf_\rho(\bx)=\sum_{(a,\bw)\in S^+} a \sigma(\bw^\top\bx) + \sum_{(a,\bw)\in S^-} a \sigma(\bw^\top\bx).
    \end{equation*}
    By \Eqref{wsp_dominates}, if $x_3=1$ then each term in the $S^+$ summation is positive while the entire $S^-$ summation is zero, telling us that $\sgn(f_\rho(\bx)) = 1$ in this case.
    On the other hand if $x_3=-1$ then each term in the $S^-$ summation is negative while the entire $S^+$ summation is zero, telling us that $\sgn(f_\rho(\bx)) = -1$ in this case.
    This completes the proof of the lemma.
\end{proof}

Throughout Phase I, we will require certain scaling factors to be satisfied.
\begin{definition} \label{def:phase1_scalings}
    We say the \emph{Phase I scalings} are satisfied if the following conditions are met:
    \crefalias{enumi}{definition}
    \begin{enumerate}[label=(\roman*), ref=\thedefinition(\roman*)]
        \item \label{def:phase1_scalings_lr} The learning rate $\log\log(d)d^{-C}\ll \eta\ll \log^{-3}(d)$.
        \item \label{def:phase1_scalings_iter} The iteration $t\lesssim \log\log(d)\eta^{-1}$.
        \item \label{def:phase1_scalings_width} The width $\log(d)\ll p \ll d^{C}$.
        \item \label{def:phase1_scalings_init} The initialization scale $\theta\ll \log^{-5C}(d)$.
        \item \label{def:phase1_scalings_batch} The batch size $m\gg d\log^{7C}(d)$.
    \end{enumerate}
\end{definition}

\begin{remark} \normalfont
    For clarity, we list the limiting usage of each scaling (\ie the result which requires the tightest application of each item in \Defref{phase1_scalings}):
    \begin{itemize}
        \item \Defref{phase1_scalings_lr} is limited by \Lemref{phase1_norms_a}. The upper bound helps with the balancedness condition $|a|\approx \norm{\bw}$, while the lower bound ensures $t\lesssim d^C$.
        \item \Defref{phase1_scalings_iter} is limited by \Lemref{phase1_norms} and \Lemref{phase1_norms_scalar}. It is essentially the duration of Phase I.
        \item \Defref{phase1_scalings_width} is limited by \Lemref{events} --- it ensures the event $\Einit$ occurs with high probability.
        \item \Defref{phase1_scalings_init} is limited by \Lemref{wperp_phase1_induction}. Throughout this section, the initialization scale $\theta$ acts as a free parameter which we can make sufficiently small.
        \item \Defref{phase1_scalings_batch} is limited by \Lemref{phase1_l0_lp_final}, where it is used in Hoeffding's inequality to concentrate the empirical gradient.
    \end{itemize}
\end{remark}

We can now define the Phase I inductive hypothesis.
Note that the Phase I ``base case'' is essentially given by the event $\Einit$ (\Defref{events_init}).
\begin{definition} \label{def:phase1_hypothesis}
    A neural network $f_\rho$ is said to obey the \emph{Phase I inductive hypothesis at iteration $t\geq 1$} if the following conditions are met for all neurons $(\at,\bwt)$:
    \crefalias{enumi}{definition}
    \begin{enumerate}[label=(\roman*), ref=\thedefinition(\roman*)]
        \item \label{def:phase1_hypothesis_wsp} $\wspt-\wsptminus = \eta \left( 1\pm o(\log^{-2}(d)) \right) \cdot \left(\tfrac{1}{2}-\lambda \right) \cdot\sgn(\azero)(|\wsptminus|+\theta)$.
        \item \label{def:phase1_hypothesis_wsig} $\norm{\bwsigt-\bwsigtminus} \lesssim \eta (\norm{\bwsigtminus} + \theta \log^C(d)d^{-1/2} )$.
        \item \label{def:phase1_hypothesis_wopp} $\norm{\bwoppt-\bwopptminus} \lesssim \eta (\norm{\bwopptminus}+\theta\log^C(d)d^{-1/2})$.
        \item \label{def:phase1_hypothesis_wperp} $\norm{\bwperpt-\bwperptminus} \lesssim \eta \theta\log^{-2C}(d)$.
        \item \label{def:phase1_hypothesis_wperp_inf} $\norm{\bwperpt-\bwperptminus}_\infty \lesssim \eta\theta\log^{3C}(d)d^{-1/2}$.
        \item \label{def:phase1_hypothesis_a} $|\at|\leq\norm{\bwt}$.
    \end{enumerate}
\end{definition}

\begin{remark} \normalfont
    The inductive hypothesis defined above implies that each of $\norm{\bwsp}$, $\norm{\bwsig}$, and $\norm{\bwopp}$ grow (at most) geometrically with additive factor, while $\norm{\bwperp}$ only grows linearly.
    The additive factor for $\norm{\bwsp}$, which is on the order $\theta$, is extremely large --- polynomially larger in $d$ than the additive factor of $\norm{\bwsig}$, which is on the order $\theta d^{-1/2}$.
    Moreover, the sign of $\wsp$ growth is modulated by $\sgn(\azero)$.
    Given these recurrences, we can sketch the proof that this inductive hypothesis holds throughout Phase I (\Propref{phase1}): Phase Ia lasts $\TIa \asymp \log^{1/2}(d)d^{-1/2}\eta^{-1}$ iterations, after which $\sgn(\azero)=\sgn(\wsp)$ for all neurons $(a,\bw)$.
    Phase Ib lasts $\TIb \asymp \log\log(d)\eta^{-1}$ iterations, wherein the spurious feature grows quickly and ultimately ensures $\norm{\bwsp}\asymp \theta \log^C(d)$.
    The non-spurious weights grow by only a polylogarithmic factor during this time.
\end{remark}

The next lemma controls the norms of all the neuron components during Phase I relative to the initialization scale $\theta$ assuming that the inductive hypothesis holds up until that point.
Notice that the bounds hold for any $t\leq \TI$; for example, we have $\norm{\bwzero}\asymp\norm{\bwperpzero}$ but $\norm{\bw^{(\TI)}}\asymp \norm{\bwsp^{(\TI)}}$, and \Lemref{phase1_norms_w} holds in either case.
A second remark is that \Lemref{phase1_norms_wsp} will eventually become tight such that $\norm{\bwspt}\asymp \theta\log^C(d)$, as we show in \Propref{phase1}.
\begin{lemma} \label{lem:phase1_norms}
    Suppose the Phase I scalings are satisfied (\Defref{phase1_scalings}).
    Under the event $\Einit$ (\Defref{events_init}), if the neural network $f_\rho$ obeys the Phase I inductive hypothesis for all iterations $k\leq t$ (\Defref{phase1_hypothesis}), then the following conditions hold on all neurons $(\at,\bwt)$:
    \crefalias{enumi}{lemma}
    \begin{enumerate}[label=(\roman*), ref=\thelemma(\roman*)]
        \item \label{lem:phase1_norms_wsp} $\norm{\bwspt} \lesssim \theta \log^C(d)$.
        \item \label{lem:phase1_norms_wsig} $\norm{\bwsigt},\norm{\bwoppt} \lesssim \theta\log^{2C}(d)d^{-1/2}$.
        \item \label{lem:phase1_norms_wperp} $\norm{\bwperpt} \asymp \theta$ with $\norm{\bwperpt-\bwperpzero}\lesssim \theta\log^{-C}(d)$.\footnote{The latter condition is important: it corresponds to the requirement on $\bDelta$ in \Lemref{boolean_to_gaussian_delta}.}
        \item \label{lem:phase1_norms_wperp_inf} $\norm{\bwperpt}_\infty \lesssim \theta\log^{4C}(d)d^{-1/2}$.
        \item \label{lem:phase1_norms_w} $\norm{\bwt} \asymp \norm{\bwspt} + \norm{\bwperpt} \lesssim \theta \log^C(d)$.
    \end{enumerate}
\end{lemma}
\begin{proof}
    For \Lemref{phase1_norms_wsp}, suppose $\sgn(\azero)=\sgn(\wsp^{(k)})$ for all $k \leq t$ to obtain an upper bound, as otherwise $\norm{\bwspt}$ will be strictly smaller.\footnote{In \Propref{phase1} we will show that this alignment occurs naturally by the end of Phase Ia.}
    In this case, since \Defref{phase1_hypothesis_wsp} holds for all $k\leq t$, we can apply \Lemref{phase1_recursion} with $w=\wsp$, $z=\theta$, $\mu=\tfrac{1}{2}-\lambda$, and $\delta=\log^{-2}(d)$.
    This gives
    \begin{equation*}
        \norm{\bwspt}= \left(1\pm 3\eta t \log^{-2}(d)\right)\cdot \left(1+\eta\left(\frac{1}{2}-\lambda\right)\right) ^t\cdot(\norm{\bwspzero}+\theta)-\theta.
    \end{equation*}
    Using $t\lesssim \log\log(d)\eta^{-1}$ we have $\eta t \log^{-2}(d))\ll1$.
    Moreover, $\left(1+\eta\left(\tfrac{1}{2}-\lambda\right)\right)^t\leq \log^C(d)$ by the inequality $e^z \geq 1 + z$.
    Under the event $\Einit$ (see Definition~\ref{def:events}) we have $\norm{\bwspzero} \leq \norm{\bwzero}_{\infty} \lesssim \theta \log^{1/2}(d)d^{-1/2}$.
    Hence, $\norm{\bwspt}\lesssim \theta(\log^C(d)-1)\lesssim \theta\log^C(d)$.
    
    For \Lemref{phase1_norms_wsig}, since \Defref{phase1_hypothesis_wsig} holds for all $k\leq t$, we can apply the upper bound of \Lemref{phase1_recursion} with $w=\bwsig$, $z=\theta\log^C(d)d^{-1/2}$, $\delta=0$, and $\mu=c_0$ for a constant $c_0>0$ resulting from the $\lesssim$ in \Defref{phase1_hypothesis_wsig}.
    This gives
    \begin{equation*}
        \norm{\bwsigt} \leq (1+c_0\eta)^t (\norm{\bwsigzero}+\theta\log^C(d)d^{-1/2} )-\theta\log^C(d)d^{-1/2}.
    \end{equation*}
    Using $t\lesssim \log\log(d)\eta^{-1}$, we have $(1+c_0\eta)^t\leq \log^C(d)$.
    Under the event $\Einit$ we have $\norm{\bwsigzero} \lesssim \theta \log^{1/2}(d)d^{-1/2}$.
    Hence, $\norm{\bwsigt}\lesssim \theta\log^C(d)d^{-1/2} (\log^C(d)-1 )\lesssim \theta\log^{2C}(d)d^{-1/2}$.
    The result for $\bwopp$ follows in the exact same way.
    
    For \Lemref{phase1_norms_wperp}, since \Defref{phase1_hypothesis_wperp} holds for all $k\leq t$, we have both $\norm{\bwperpt}\lesssim \norm{\bwperpzero} + t\eta\theta\log^{-2C}(d)$ and $\norm{\bwperpt}\gtrsim \norm{\bwperpzero} - t\eta\theta \log^{-2C}(d)$.
    Using $t\lesssim \log\log(d)\eta^{-1}$ we find $t\eta\theta \log^{-2C}(d)\lesssim \theta\log^{-C}(d)$, and under the event $\Einit$ we have $\norm{\bwperpzero}\asymp \theta$.
    Hence, we have both $\norm{\bwperpt}\lesssim \theta + \theta\log^{-C}(d)$ and $\norm{\bwperpt}\gtrsim \theta - \theta\log^{-C}(d)$; it follows that $\norm{\bwperpt}\asymp \theta$ with $\norm{\bwperpt-\bwperpzero}\lesssim \theta\log^{-C}(d)$.
    
    For \Lemref{phase1_norms_wperp_inf}, since \Defref{phase1_hypothesis_wperp_inf} holds for all $k\leq t$, we can apply the triangle inequality $t$ times under the event $\Einit$ to obtain $\norm{\bwperpt}_\infty \lesssim \norm{\bwperpzero}_\infty+t\eta\theta\log^{3C}(d)d^{-1/2} \lesssim \theta\log^{4C}(d)d^{-1/2}$.
    
    For \Lemref{phase1_norms_w}, by the triangle inequality we have $\norm{\bwt}\leq \norm{\bwsigt}+\norm{\bwoppt}+\norm{\bwspt}+\norm{\bwperpt}$.
    By \Lemref{phase1_norms_wsig} and \Lemref{phase1_norms_wperp} we have $\norm{\bwsigt}, \norm{\bwoppt}\ll \norm{\bwperpt}$, which implies the upper bound $\norm{\bwt}\lesssim \norm{\bwspt}+\norm{\bwperpt}$.
    We then have the lower bound by $\norm{\bwt}\geq \norm{\bwsp+\bwperp}\geq \tfrac{1}{\sqrt{2}}(\norm{\bwsp}+\norm{\bwperp})$.
    
    Thus, we have proved all the parts of the lemma.
\end{proof}

Next, we present the partner lemma to \Lemref{phase1_norms} which controls behavior of the scalar weight $a$ for any neuron $(a,\bw)$.
It requires minibatch concentration and thus holds over the event $\Etrain$.
\begin{lemma} \label{lem:phase1_norms_scalar}
    Suppose the Phase I scalings are satisfied (\Defref{phase1_scalings}).
    Under the event $\Etrain$ (\Defref{events_train}), if the neural network $f_\rho$ obeys the Phase I inductive hypothesis for all iterations $k\leq t$ (\Defref{phase1_hypothesis}), then the following conditions hold on all neurons $(\at,\bwt)$:
    \crefalias{enumi}{lemma}
    \begin{enumerate}[label=(\roman*), ref=\thelemma(\roman*)]
        \item \label{lem:phase1_norms_a} $|\at|=\left(1\pm o(\log^{-2}(d))\right)\cdot\norm{\bwt} = \left(1\pm o(\log^{-2}(d))\right)\cdot (\norm{\bwspt} + \norm{\bwperpt}) \lesssim \theta \log^C(d)$.
        \item \label{lem:phase1_norms_sgn} $\sgn(\at)=\sgn(\azero)$.
        \item \label{lem:phase1_norms_aw} $\E_{(a,\bw)\sim\rho}[\norm{\at\bwt}]\lesssim \theta^2\log^{2C}(d)$.
    \end{enumerate}
\end{lemma}
\begin{proof}
    For \Lemref{phase1_norms_a}, for the upper bound we have $|\at|\leq \norm{\bwt}$ by \Defref{phase1_hypothesis_a}.
    For the lower bound, we denote as shorthand $E(t)\coloneqq\norm{\bwt}^2 - (\at)^2$ and note that $E(0)=0$ by definition.
    Under the event $\Etrain$, by \Lemref{balanced_iv} we have $E(t)\leq 10\eta^2 (\at)^2 + E(t-1)$ for all $t\ll d^C$ iterations.
    Using \Defref{phase1_hypothesis_a}, \Lemref{phase1_norms_wperp}, and \Lemref{phase1_norms_w}, we have by strong induction that $(a^{(k)})^2\lesssim \norm{\bwsp^{(k)}}^2+\theta^2$ for all $k \leq t$.
    In particular,
    \begin{equation*}
        E(t) \leq \sum_{k=1}^{t-1} 10\eta^2 (a^{(k)})^2 \lesssim t\eta^2 \cdot (\max_{k< t} \norm{\bwsp^{(k)}}^2+\theta^2). %
    \end{equation*}
    where the last inequality uses the inductive step.
    Under the event $\Etrain$, we have $\norm{\bwspzero}\lesssim \theta\log^{1/2}(d)d^{-1/2}$.
    Moreover, by \Defref{phase1_hypothesis_wsp}, we have that $\norm{\bwsp^{(k)}}$ is monotonically increasing if $\sgn(\wsp^{(k)})=\sgn(\azero)$ and monotonically decreasing otherwise (recall that $\sgn(\wsp)$ may flip).
    Thus, we can write
    \begin{equation} \label{eq:max_wsp}
        \max_{k<t}\norm{\bwsp^{(k)}}^2\lesssim \max(\theta^2 \log(d) d^{-1}, \norm{\bwspt}^2).
    \end{equation}
    If the maximum in \Eqref{max_wsp} is achieved by $\theta^2 \log(d) d^{-1}$ we have $E(t) \ll t\eta^2 \theta^2 \lesssim t\eta^2 \norm{\bwt}^2$ by \Lemref{phase1_norms_wperp}.
    On the other hand, if the maximum in \Eqref{max_wsp} is achieved by $\norm{\bwspt}^2$ we have $E(t) \ll t\eta^2(\norm{\bwspt}^2+\theta^2) \lesssim t\eta^2\norm{\bwt}^2$ by \Lemref{phase1_norms_wperp} and orthogonality of $\bwsp$ and $\bwperp$.
    In either case, applying $t\lesssim \log\log(d)\eta^{-1}$ and $\eta\ll \log^{-3}(d)$ we have
    \begin{equation*}
        (\at)^2=\norm{\bwt}^2-E(t)=\big(1-O(t\eta^2)\big)\cdot\norm{\bwt}^2=\big(1-o(\log^{-2}(d))\big)\cdot\norm{\bwt}^2,
    \end{equation*}
    so we have ultimately shown the desired lower bound $|\at|=\left(1-o(\log^{-2}(d))\right)\cdot\norm{\bwt}$.\footnote{Here we used the Taylor expansion $\sqrt{1-x}=1-x/2-O(x^2)$ for $x\ll 1$ such that $\sqrt{1-o(\log^{-2}(d))}=1-o(\log^{-2}(d))$.}
    The remainder of the result follows from \Lemref{phase1_norms_w}.

    For \Lemref{phase1_norms_sgn}, assume for a contradiction that $\sgn(\at)\neq\sgn(\azero)$.
    Then there exists some iteration $k< t$ where its sign flips, \ie
    \begin{equation} \label{eq:phase1_a_contradict_term1}
        |a^{(k)}-a^{(k+1)}|=|a^{(k)}|+|a^{(k+1)}|\gtrsim \norm{\bw^{(k)}} + \norm{\bw^{(k+1)}} \geq \norm{\bw^{(k)}},
    \end{equation}
    where the second-to-last step follows by \Lemref{phase1_norms_a}.
    On the other hand, by the definition of gradient descent and the triangle inequality we have
    \begin{equation*} 
        |a^{(k)}-a^{(k+1)}| \leq \eta |\partial_{a^{(k)}}\widehat{L}_\rho| \leq \eta |\partial_{a^{(k)}} L_\rho | + \eta |\partial_{a^{(k)}} \widehat{L}_\rho - \partial_{a^{(k)}} L_\rho |.
    \end{equation*}
    By \Lemref{balanced_ii} we have $|\partial_{a^{(k)}} L_\rho |\lesssim \norm{\bw^{(k)}}$.
    Under the event $\Etrain$, by \Lemref{empirical_concentration_iii} with $m\gg d\log^2(d)$, we have $|\partial_{a^{(k)}} \widehat{L}_\rho - \partial_{a^{(k)}} L_\rho | \ll \norm{\bw^{(k)}}$.
    Together with $\eta\ll 1$, we have
    \begin{equation} \label{eq:phase1_a_contradict_term2}
        |a^{(k)}-a^{(k+1)}|\lesssim \eta\norm{\bw^{(k)}} \ll \norm{\bw^{(k)}}.
    \end{equation}
    Therefore, \Eqref{phase1_a_contradict_term1} and \Eqref{phase1_a_contradict_term2} are contradictory.
    This implies $\sgn(\at)=\sgn(\azero)$.

    For \Lemref{phase1_norms_aw} we have by \Lemref{phase1_norms_w} and \Lemref{phase1_norms_a} that \begin{equation*}
        \E_{(a,\bw)\sim\rho}[\norm{\at\bwt}]=\frac{1}{p}\sum_{j=1}^p \norm{a_j \bw_j}\lesssim \theta^2\log^{2C}(d).
    \end{equation*}

    Thus, we have proved all parts of the lemma.
\end{proof}

\subsection{Phase I \texorpdfstring{$\nabla \widehat{L}_\rho - \nabla L_0$}{L0 vs Lp Gradient} Error Analysis} \label{sec:phase1_l0_lp_error_analysis}
In this section, we will leverage Phase I properties to show that $\norm{\nabla \widehat{L}_\rho - \nabla L_0}$ is sufficiently small, justifying the $L_0$ approximation.
First, we will show an extension of \Lemref{l0_lp_norm} that uses the central limit theorem and the Phase I inductive hypothesis to control the approximation error of the population gradient, $\norm{\nabla L_\rho-\nabla L_0}$.
This result is related to~\cite[Lemma C.17]{glasgow2024sgd}, though it depends on the \emph{leave-two-out} analysis we developed in \Lemref{l0_lp_norm}.

\begin{lemma} \label{lem:phase1_l0_lp_norm}
    Suppose the Phase I scalings are satisfied (\Defref{phase1_scalings}).
    Under the event $\Etrain$ (\Defref{events_train}), if $f_\rho$ obeys the Phase I inductive hypothesis for all iterations $k\leq t$ (\Defref{phase1_hypothesis}), then we have
    \begin{equation*}
        |\partial_{\wit}L_\rho-\partial_{\wit}L_0| \lesssim
        \begin{cases*}
            \theta^3 \log^{6C}(d)d^{-1/2} + \theta\log^C(d)d^{-C} & $i<3$ \\
            \theta^3 \log^{4C}(d) + \theta\log^C(d)d^{-C} & $i= 3$ \\
            \theta^3\log^{8C}(d)d^{-1/2} + \theta\log^C(d) d^{-C} & $i>3$
        \end{cases*}
    \end{equation*}
    for all neurons $(\at,\bwt)$.
\end{lemma}
\begin{proof}
    Let us omit $(t)$ superscripts for clarity and begin with the case $i<3$.
    Recall that we write $\bx\coloneqq \bz+\bs+\bxi$ and $\bw_{1:2}\coloneqq \bwsig+\bwopp$.
    Using the law of total probability and $|\bw^\top\bz|=\norm{\bz}\norm{\bwsig}=\sqrt{2}\norm{\bw_{1:2}}$, we have
    \begin{equation*}
        \P_{\bx}\left(|\bw^\top\bxi+\bw^\top\bs|<|\bw^\top\bz|\right) \leq \E_{\bs}\left[\P_{\bxi}\left(|\bw^\top\bxi+\bw^\top\bs|<\sqrt{2}\norm{\bw_{1:2}} \mid \bs\right)\right].
    \end{equation*}
    Let us apply \Lemref{boolean_to_gaussian_delta} to the right-hand side with $\bvzero=\bwperpzero$, $\bDelta=\bwperpt-\bwperpzero$, $\mu=-\bw^\top\bs$, and $k=\sqrt{2}\norm{\bw_{1:2}}$.
    Under the event $\Etrain$, the condition on $\bvzero$ is satisfied; we use \Lemref{phase1_norms_wperp} and $C>\tfrac{1}{2}$ for the condition on $\bDelta$, and \Lemref{phase1_norms_wsig} for the condition on $k$.
    Upper bounding $\exp\left(\tfrac{-\mu^2}{C\theta^2}\right)\leq 1$, we have
    \begin{equation*}
        \P_{\bx}\left(|\bw^\top\bxi+\bw^\top\bs|<|\bw^\top\bz|\right) \lesssim \frac{\norm{\bw_{1:2}}}{\theta} + \frac{1}{\sqrt{d}}.
    \end{equation*}
    By \Lemref{phase1_norms_wsig}, we have $\norm{\bw_{1:2}}\lesssim \theta \log^{2C}(d) d^{-1/2}$.
    Hence,
    \begin{equation} \label{eq:lp_l0_wsig}
        \P_{\bx}\left(|\bw^\top\bxi+\bw^\top\bs|<|\bw^\top\bz|\right) \lesssim \log^{2C}(d)d^{-1/2} + d^{-1/2} \lesssim \log^{2C}(d)d^{-1/2}.
    \end{equation}
    By \Lemref{phase1_norms_a} and \Lemref{phase1_norms_aw} we have $|a|\lesssim \theta\log^C(d)$ and $\E_{(a,\bw) \sim \rho}[\norm{a\bw}]\lesssim \theta^2\log^{2C}(d)$.
    Together with \Eqref{lp_l0_wsig}, we can apply \Lemref{l0_lp_norm} as follows:
    \begin{align*}
        & |\partial_{w_i}L_\rho-\partial_{w_i}L_0 | \\
        &\lesssim |a|\left(\E_{(a,\bw) \sim \rho}[|a w_1|+|a w_2|]+\log(d)\cdot\P_{\bx}\left(|\bw^\top\bxi+\bw^\top\bs|<|\bw^\top\bz|\right)\cdot\E_{(a,\bw) \sim \rho}[\norm{a\bw}]+ d^{-C}\right) \\
        &\lesssim \theta \log^C(d) \left( \theta^2\log^{3C}(d)d^{-1/2} + \theta^2\log^{4C+1}(d)d^{-1/2} + d^{-C} \right) \\
        &\lesssim \theta^3\log^{6C}(d)d^{-1/2} + \theta\log^C(d)d^{-C},
    \end{align*}
    where we used $C>1$.
    This completes the proof for the case $i<3$.
    
    Now, consider the case $i=3$.
    By \Lemref{phase1_norms_a} and \Lemref{phase1_norms_aw} we have $|a|\lesssim \theta\log^C(d)$ and $\E_{(a,\bw) \sim \rho}[\norm{a\bw}]\lesssim \theta^2\log^{2C}(d)$.
    Hence, we can apply \Lemref{l0_lp_norm_small} as follows:
    \begin{align*}
        |\partial_{w_i}L_\rho-\partial_{w_i}L_0| &\lesssim |a|\left(\log(d)\cdot \E_{(a,\bw) \sim \rho}[\norm{a\bw}] + d^{-C}\right) \\
        &\lesssim \theta \log^C(d) \left( \theta^2 \log^{2C+1}(d) + d^{-C} \right) \\
        &\lesssim \theta^3 \log^{4C}(d) + \theta\log^C(d)d^{-C},
    \end{align*}
    where we used $C>1$.
    This completes the proof for $i=3$.

    Finally, consider the case $i>3$ and recall $\bxloo\coloneqq \bx-x_i\be_i$.
    Using the law of total probability, we have
    \begin{equation*}
        \P_{\bx}\left(|\bw^\top\bxloo|<|w_i|\right) = \E_{\bz,\bs}\left[\P_{\bxi}\left(|\bw^\top\bxi_{\setminus i}+\bw^\top\bz+\bw^\top\bs|<|w_i| \mid \bz,\bs\right)\right].
    \end{equation*}
    Let us apply \Lemref{boolean_to_gaussian_delta} to the right-hand side with $\bvzero=\bwzero_{\perp\setminus i}$, $\bDelta=\bwt_{\perp\setminus i}-\bwzero_{\perp\setminus i}$, $\mu=-\bw^\top\bz-\bw^\top\bs$, and $k=|w_i|$.
    Under the event $\Etrain$, the condition on $\bvzero$ is satisfied; we use 
    \Lemref{phase1_norms_wperp} and $C>\tfrac{1}{2}$ for the condition on $\bDelta$, and \Lemref{phase1_norms_wperp_inf} for the condition on $k$.\footnote{While \Lemref{phase1_basecase} and \Lemref{phase1_norms_wperp} are stated for $\bwperp$, it is straightforward to see they hold for $\bw_{\perp\setminus i}$.}
    Upper bounding $\exp\left(\tfrac{-\mu^2}{C\theta^2}\right)\leq 1$, we have
    \begin{equation*}
        \P_{\bx}\left(|\bw^\top\bxloo|<|w_i|\right) \lesssim \frac{|w_i|}{\theta} + \frac{1}{\sqrt{d}}.
    \end{equation*}
    By \Lemref{phase1_norms_wperp_inf}, we have $\norm{\bwperp}_\infty\lesssim \theta\log^{4C}(d)d^{-1/2}$.
    Hence,
    \begin{equation} \label{eq:lp_l0_wi}
        \P_{\bx}\left(|\bw^\top\bxloo|<|w_i|\right) \lesssim \log^{4C}(d)d^{-1/2} + d^{-1/2} \lesssim \log^{4C}(d)d^{-1/2}. 
    \end{equation}
    By \Lemref{phase1_norms_a} and \Lemref{phase1_norms_aw} we have $|a|\lesssim \theta\log^C(d)$ and $\E_{(a,\bw) \sim \rho}[\norm{a\bw}]\lesssim \theta^2\log^{2C}(d)$.
    Together with \Eqref{lp_l0_wi}, we can apply \Lemref{l0_lp_norm} as follows:
    \begin{align*}
        |\partial_{w_i}L_\rho-\partial_{w_i}L_0| &\lesssim |a|\left(\E_{(a,\bw) \sim \rho}[|a w_i|]+\log(d)\cdot \P_{\bx}\left(|\bw^\top\bxloo|<|w_i|\right)\cdot \E_{(a,\bw) \sim \rho}[\norm{a\bw}]+ d^{-C}\right) \\
        &\lesssim \theta\log^C(d) \left(\theta^2\log^{5C}(d)d^{-1/2} + \theta^2\log^{6C+1}(d)d^{-1/2} + d^{-C} \right) \\
        &\lesssim \theta^3\log^{8C}(d)d^{-1/2} + \theta \log^C(d) d^{-C},
    \end{align*}
    where we used $C>1$.
    This completes the proof of the lemma.
\end{proof}

We can now obtain general bounds on the magnitude of each component of $\nabla_{\bw}\widehat{L}_\rho-\nabla_{\bw}L_0$, which is what we will ultimately need.
\begin{lemma} \label{lem:phase1_l0_lp_final}
    Suppose the Phase I scalings are satisfied (\Defref{phase1_scalings}).
    Under the event $\Etrain$ (\Defref{events_train}), if $f_\rho$ obeys the Phase I inductive hypothesis for all iterations $k\leq t$ (\Defref{phase1_hypothesis}), then we have
    \begin{equation*}
        |\partial_{\wit}\widehat{L}_\rho-\partial_{\wit}L_0| \lesssim
        \begin{cases*}
            \theta^3\log^{6C}(d)d^{-1/2}+\theta\log^{-2C}(d)d^{-1/2} & $i <3$ \\
            \theta^3\log^{4C}(d) + \theta\log^{-2C}(d)d^{-1/2} & $i= 3$ \\
            \theta^3\log^{8C}(d)d^{-1/2} + \theta\log^{-2C}(d)d^{-1/2} & $i>3$.
        \end{cases*}
    \end{equation*}
    for all neurons $(\at,\bwt)$.
\end{lemma}
\begin{proof}
    Let us omit $(t)$ superscripts for clarity.
    By the triangle inequality we have for any $i\in [d]$ that
    \begin{equation*}
        |\partial_{w_i}\widehat{L}_\rho-\partial_{w_i}L_0| \leq |\partial_{w_i}\widehat{L}_\rho-\partial_{w_i}L_\rho| + |\partial_{w_i} L_\rho-\partial_{w_i}L_0|.
    \end{equation*}
    By \Lemref{phase1_norms_a} we have $|a|\lesssim \theta\log^C(d)$.
    Under the event $\Etrain$, by \Lemref{empirical_concentration_i} with $m\gg d\log^{7C}(d)$ and $C>2$, we have
    \begin{equation*}
        |\partial_{w_i}\widehat{L}_\rho-\partial_{w_i}L_\rho| \ll |\at| \log^{-3C}(d)d^{-1/2} \lesssim \theta \log^{-2C}(d) d^{-1/2}.
    \end{equation*}
    The result then follows directly from \Lemref{phase1_l0_lp_norm}.
\end{proof}

\subsection{Phase I Inductive Step} \label{sec:phase1_induction_step}
In this section, we show the inductive steps for \Defref{phase1_hypothesis} and ultimately obtain the Phase I result in \Propref{phase1}.
We first characterize the growth of $\bwsp$, constituting the inductive step for \Defref{phase1_hypothesis_wsp}.
\begin{lemma} \label{lem:wsp_phase1_induction}
    Suppose the Phase I scalings are satisfied (\Defref{phase1_scalings}).
    Under the event $\Etrain$, if $f_\rho$ obeys the Phase I inductive hypothesis for all iterations $k\leq t$ (\Defref{phase1_hypothesis}), then we have
    \begin{equation*}
        \wsptplus -\wspt = \eta \left( 1\pm o(\log^{-2}(d))\right) \cdot \left(\frac{1}{2}-\lambda \right) \cdot\sgn(\azero)(|\wspt|+\theta).
    \end{equation*}
    for all neurons $(\atplus,\bwtplus)$.
\end{lemma}
\begin{proof}
    We first compute the order of the intermediate term $\epswt$ defined in \Lemref{sp_l0}.
    By \Lemref{v_ub}, we have
    \begin{align*}
        |\epswt| &\lesssim \frac{\norm{\bwt_{1:2}}}{\norm{\bwperpt}}\exp\left(\frac{-\norm{\bwspt}^2+2\sqrt{2}\norm{\bwt_{1:2}}\norm{\bwspt}}{2\norm{\bwperpt}^2}\right) + \frac{\norm{\bwperp}_3^3}{\norm{\bwperp}^3} \\
        &\leq \frac{\norm{\bwt_{1:2}}}{\norm{\bwperpt}}\exp\left(\frac{2\sqrt{2}\norm{\bwt_{1:2}}\norm{\bwspt}}{2\norm{\bwperpt}^2}\right) + \frac{\norm{\bwperp}_\infty}{\norm{\bwperp}},
    \end{align*}
    where we used the basic inequalities $-\norm{\bwspt}^2\leq 0$ and $\norm{\bwperpt}_3^3\leq \norm{\bwperpt}_\infty \norm{\bwperpt}^2$.
    By \Lemref{phase1_norms_wsp}, \Lemref{phase1_norms_wsig}, \Lemref{phase1_norms_wperp}, and \Lemref{phase1_norms_wperp_inf} we have $\norm{\bwspt}\lesssim \theta\log^C(d)$, $\norm{\bwt_{1:2}}\lesssim \theta\log^{2C}(d)d^{-1/2}$, $\norm{\bwperpt}\asymp \theta$, and $\norm{\bwperpt}_\infty\lesssim \theta\log^{4C}(d)d^{-1/2}$.
    Using $e^u\lesssim 1+u$ for $0<u\ll 1$,\footnote{This can be seen via the following Taylor expansion: $e^u\leq 1+u+Cu^2\lesssim 1+u$ for $0<u\ll 1$.} we have
    \begin{align}
        |\epswt| &\lesssim \log^{2C}(d)d^{-1/2}\cdot \exp(C\log^{3C}(d)d^{-1/2}) + \log^{4C}(d)d^{-1/2} \nonumber \\
        &\lesssim \log^{2C}(d)d^{-1/2}\cdot (1+\log^{3C}(d)d^{-1/2})+\log^{4C}(d)d^{-1/2} \nonumber \\
        &\lesssim \log^{5C}(d)d^{-1/2}. \label{eq:wsp_phase1_eps}
    \end{align}
    By the definition of gradient descent, we have
    \begin{align}
        \wsptplus &= \wspt - \eta\partial_{\wspt}\widehat{L}_\rho \nonumber\\
        &= \wspt - \eta \frac{\bwsp^{(t)\top}\nabla_{\bwt} \widehat{L}_\rho}{\wspt} \nonumber\\
        &= \wspt - \eta \frac{\bwsp^{(t)\top}\nabla_{\bwt} L_0 + \bwsp^{(t)\top}(\nabla_{\bwt} \widehat{L}_\rho-\nabla_{\bwt} L_0)}{\wspt}. \label{eq:wsp_phase1_proj}
    \end{align}
    By \Lemref{sp_l0} we have
    \begin{equation*}
        -\bwsp^{(t)\top}\nabla_{\bwt}L_0 =  \frac{1}{2} \at\wspt \cdot \left(1+\frac{\epswt}{2}-2\lambda\right).
    \end{equation*}
    Using $|\epswt|\lesssim \log^{5C}(d)d^{-1/2}$ by \Eqref{wsp_phase1_eps} and recalling that we assumed $1/2 - \lambda > 0$, we obtain\footnote{This is also where we use the upper bound $\lambda\in (0,\Lambda]$ for a constant $\Lambda<\tfrac{1}{2}$, \ie we cannot have $\lambda\to \tfrac{1}{2}$.}
    \begin{equation*}
        -\bwsp^{(t)\top}\nabla_{\bwt}L_0 = \left( 1\pm o(\log^{-2}(d))\right) \cdot \left(\frac{1}{2}-\lambda \right) \cdot \at\wspt.
    \end{equation*}
    Moreover, by \Lemref{phase1_norms_wperp} and \Lemref{phase1_norms_a} we have $|\at|= \left(1\pm o(\log^{-2}(d))\right)\cdot  (|\wspt| + \theta)$, implying that
    \begin{equation} \label{eq:wsp_phase1_term_1}
        -\bwsp^{(t)\top}\nabla_{\bwt}L_0 = \left( 1\pm o(\log^{-2}(d))\right) \cdot \left(\frac{1}{2}-\lambda \right) \cdot\sgn(\at)\sgn(\wspt) \cdot (|\wspt|+\theta)|\wspt|.
    \end{equation}
    Next, by the Cauchy-Schwarz inequality and \Lemref{phase1_l0_lp_final} with $i=3$, we have
    \begin{align*} 
        |\bwsp^{(t)\top}(\nabla_{\bwt}\widehat{L}_\rho-\nabla_{\bwt}L_0)| &\leq |\wspt|\norm{\nabla_{\bwspt}\widehat{L}_\rho-\nabla_{\bwspt}L_0} \\
        &\lesssim |\wspt| (\theta^3 \log^{4C}(d)+\theta\log^{-2C}(d)d^{-1/2} ). 
    \end{align*}
    Using $\theta\ll \log^{-3C}(d)$ we have $\theta^3 \log^{4C}(d) \ll \theta \log^{-2}(d)$.
    We also have $\theta \log^{-2C}(d)d^{-1/2}\ll \theta \log^{-2}(d)$.
    Hence,
    \begin{equation} \label{eq:wsp_phase1_term_2}
        |\bwsp^{(t)\top}(\nabla_{\bwt}\widehat{L}_\rho-\nabla_{\bwt}L_0)| \ll \theta|\wspt|,
    \end{equation}
    in particular, \Eqref{wsp_phase1_term_2} is dominated by \Eqref{wsp_phase1_term_1}.
    Returning to \Eqref{wsp_phase1_proj}, we then have
    \begin{align*}
        \wsptplus -\wspt &= \frac{\eta}{\wspt} \left( 1\pm o(\log^{-2}(d))\right) \cdot \left(\frac{1}{2}-\lambda \right) \cdot\sgn(\at)\sgn(\wspt) \cdot (|\wspt|+\theta)|\wspt| \\
        &= \eta \left( 1\pm o(\log^{-2}(d))\right) \cdot \left(\frac{1}{2}-\lambda \right) \cdot\sgn(\at)(|\wspt|+\theta)
    \end{align*}
    Finally, by \Lemref{phase1_norms_sgn} we have $\sgn(\at)=\sgn(\azero)$.
    Hence,
    \begin{equation*}
        \wsptplus -\wspt = \eta \left( 1\pm o(\log^{-2}(d))\right) \cdot \left(\frac{1}{2}-\lambda \right) \cdot\sgn(\azero)(|\wspt|+\theta).
    \end{equation*}
    This completes the proof of the lemma.
\end{proof}

We will now upper bound the growth of $\bwsig$, constituting the inductive step for \Defref{phase1_hypothesis_wsig}.
\begin{lemma} \label{lem:wsig_phase1_induction}
    Suppose the Phase I scalings are satisfied (\Defref{phase1_scalings}).
    Under the event $\Etrain$ (\Defref{events_train}), if $f_\rho$ obeys the Phase I inductive hypothesis for all iterations $k\leq t$ (\Defref{phase1_hypothesis}), then we have
    \begin{equation*}
        \norm{\bwsigtplus-\bwsigt} \lesssim \eta \norm{\bwsigt} + \eta\theta\log^C(d)d^{-1/2}
    \end{equation*}
    for all neurons $(\atplus,\bwtplus)$.
\end{lemma}
\begin{proof}
    By the definition of gradient descent, we have
    \begin{align}
        \bwsigtplus &= \bwsigt - \eta\nabla_{\bwsigt}\widehat{L}_\rho \nonumber\\
        &= \bwsigt - \eta \frac{\bwsig^{(t)\top}\nabla_{\bwt} \widehat{L}_\rho}{\norm{\bwsigt}^2} \bwsigt \nonumber\\
        &= \bwsigt - \eta \frac{\bwsig^{(t)\top}\nabla_{\bwt} L_0 + \bwsig^{(t)\top}(\nabla_{\bwt} \widehat{L}_\rho-\nabla_{\bwt} L_0)}{\norm{\bwsigt}^2} \bwsigt . \label{eq:wsig_phase1_proj}
    \end{align}
    By \Lemref{sigopp_l0}, we have
    \begin{equation*}
        |\bwsig^{(t)\top} \nabla_{\bwt} L_0 | = \frac{\sqrt{2}}{4}|\at|\norm{\bwsigt} \cdot \P_{\bxi}\left(|\bw^{(t)\top}\bxi+\bw^{(t)\top}\be_3|\leq\sqrt{2}\norm{\bwsigt} \right).
    \end{equation*}
    Let us apply \Lemref{boolean_to_gaussian_delta} with $\bvzero=\bwperpzero$, $\bDelta=\bwperpt-\bwperpzero$, $\mu=-\bw^{(t)\top} \be_3$, and $k= \sqrt{2}\norm{\bwsigt}$.
    Under the event $\Etrain$, the condition on $\bvzero$ is satisfied; we use \Lemref{phase1_norms_wperp} and $C>\tfrac{1}{2}$ for the condition on $\bDelta$, and \Lemref{phase1_norms_wsig} for the condition on $k$.
    Using $|\bw^{(t)\top} \be_3|=\norm{\bwspt}$, we have
    \begin{equation} \label{eq:wsig_berry_esseen}
        |\bwsig^{(t)\top} \nabla_{\bwt} L_0| 
        \lesssim |\at|\norm{\bwsigt} \left( \frac{\norm{\bwsigt}}{\theta}\exp\left(\frac{-\norm{\bwspt}^2}{C\theta^2} \right) + \frac{1}{\sqrt{d}} \right).
    \end{equation}
    Now, note that $\max_{y,z\in \R}\frac{y}{z}\exp\Big( \tfrac{-y^2}{Cz^2} \Big)=\sqrt{\tfrac{C}{2e}}$ which is a constant.
    In particular,
    \begin{equation} \label{eq:wsig_exp_term1}
        \norm{\bwspt}\frac{\norm{\bwsigt}}{\theta} \exp\left(\frac{-\norm{\bwspt}^2}{C\theta^2}\right) \lesssim \norm{\bwsigt}.
    \end{equation}
    Moreover,
     \begin{equation} \label{eq:wsig_exp_term2}
        \norm{\bwperpt}\frac{\norm{\bwsigt}}{\theta} \exp\left(\frac{-\norm{\bwspt}^2}{C\theta^2}\right) \lesssim \norm{\bwsigt},
    \end{equation}
    where we used $\exp\Big(\frac{-\norm{\bwspt}^2}{C\theta^2}\Big)\leq 1$.
    By \Lemref{phase1_norms_a}, we have $|\at|\asymp\norm{\bwspt}+\norm{\bwperpt}\lesssim \theta\log^C(d)$.
    Combining \Eqref{wsig_berry_esseen}, \Eqref{wsig_exp_term1}, and \Eqref{wsig_exp_term2} we have
    \begin{equation} \label{eq:wsig_phase1_term_1}
        |\bwsig^{(t)\top} \nabla_{\bwt} L_0| \lesssim \norm{\bwsigt}^2+\norm{\bwsigt}\theta\log^C(d)d^{-1/2}.
    \end{equation}
    Next, by the Cauchy-Schwarz inequality and \Lemref{phase1_l0_lp_final} with $i<3$, we have
    \begin{align}
        |\bwsig^{(t)\top}(\nabla_{\bwt}\widehat{L}_\rho-\nabla_{\bwt}L_0)| &\leq \norm{\bwsigt} \cdot \norm{\nabla_{\bwsigt}\widehat{L}_\rho-\nabla_{\bwsigt}L_0} \nonumber \\
        &\lesssim \norm{\bwsigt} \cdot (\theta^3 \log^{6C}(d)d^{-1/2} + \theta\log^{-2C}(d) d^{-1/2} ) \nonumber \\
        &\lesssim \norm{\bwsigt} \cdot \theta d^{-1/2}, \label{eq:wsig_phase1_term_2}
    \end{align}
    where we used $\theta \ll \log^{-3C}(d)$.
    Returning to \Eqref{wsig_phase1_proj} and combining \Eqref{wsig_phase1_term_1} and \Eqref{wsig_phase1_term_2}, we have
    \begin{align*}
        \norm{\bwsigtplus -\bwsigt} &\lesssim \eta\frac{\norm{\bwsigt} + \theta \log^C(d) d^{-1/2}+ \theta d^{-1/2}}{\norm{\bwsigt}} \norm{\bwsigt} \\
        &\lesssim \eta\norm{\bwsigt}+\eta\theta\log^{C}(d) d^{-1/2}.
    \end{align*}
    This completes the proof of the lemma.
\end{proof}

Similarly, we can upper bound the growth of $\bwopp$, constituting the inductive step for \Defref{phase1_hypothesis_wopp}.
\begin{lemma} \label{lem:wopp_phase1_induction}
    Suppose the Phase I scalings are satisfied (\Defref{phase1_scalings}).
    Under the event $\Etrain$ (\Defref{events_train}), if $f_\rho$ obeys the Phase I inductive hypothesis for all iterations $k\leq t$ (\Defref{phase1_hypothesis}), then we have
    \begin{align*}
        \norm{\bwopptplus - \bwoppt} \lesssim \eta\norm{\bwoppt}+\eta\theta\log^C(d)d^{-1/2}
    \end{align*}
    for all neurons $(\atplus, \bwtplus)$.
\end{lemma}
\begin{proof}
    By the definition of gradient descent, we have
    \begin{align}
        \bwopptplus &= \bwoppt - \eta\nabla_{\bwoppt}\widehat{L}_\rho \nonumber\\
        &= \bwoppt - \eta \frac{\bwopp^{(t)\top}\nabla_{\bwt} \widehat{L}_\rho}{\norm{\bwoppt}^2} \bwoppt \nonumber\\
        &= \bwoppt - \eta \frac{\bwopp^{(t)\top}\nabla_{\bwt} L_0 + \bwopp^{(t)\top}(\nabla_{\bwt} \widehat{L}_\rho-\nabla_{\bwt} L_0)}{\norm{\bwoppt}^2} \bwoppt . \label{eq:wopp_phase1_proj}
    \end{align}
    By \Lemref{sigopp_l0}, we have
    \begin{equation*}
        |\bwopp^{(t)\top} \nabla_{\bwt} L_0 | = \frac{\sqrt{2}}{4}|\at|\norm{\bwoppt} \cdot \P_{\bxi}\left(|\bw^{(t)\top}\bxi+\bw^{(t)\top}\be_3|\leq\sqrt{2}\norm{\bwoppt} \right).
    \end{equation*}
    Following the same analysis of \Lemref{wsig_phase1_induction}, we obtain
    \begin{equation} \label{eq:wopp_phase1_term_1}
        |\bwopp^{(t)\top} \nabla_{\bwt} L_0 | \lesssim \norm{\bwoppt}^2 + \norm{\bwoppt}\theta\log^C(d)d^{-1/2}.
    \end{equation}
    Next, by the Cauchy-Schwarz inequality and \Lemref{phase1_l0_lp_final} with $i<3$, we have
    \begin{align}
        |\bwopp^{(t)\top}(\nabla_{\bwt}\widehat{L}_\rho-\nabla_{\bwt}L_0)| &\leq \norm{\bwoppt}\cdot \norm{\nabla_{\bwoppt}\widehat{L}_\rho-\nabla_{\bwoppt}L_0} \nonumber \\
        &\lesssim \norm{\bwoppt} \cdot (\theta^3 \log^{6C}(d)d^{-1/2} + \theta\log^{-2C}(d)d^{-1/2}) \nonumber \\
        &\lesssim \norm{\bwoppt} \cdot \theta d^{-1/2}, \label{eq:wopp_phase1_term_2}
    \end{align}
    where we used $\theta\ll \log^{-3C}(d)$.
    Returning to \Eqref{wopp_phase1_proj} and combining \Eqref{wopp_phase1_term_1} and \Eqref{wopp_phase1_term_2}, we have
    \begin{align*}
        \norm{\bwopptplus - \bwoppt} &\lesssim \eta \frac{\norm{\bwoppt}+\theta\log^C(d)d^{-1/2}+\theta d^{-1/2}}{\norm{\bwoppt}}\norm{\bwoppt} \\
        &\lesssim \eta\norm{\bwoppt}+\eta\theta\log^C(d)d^{-1/2}.
    \end{align*}
    This completes the proof of the lemma.
\end{proof}

Next, we can upper bound the growth of $\bwperp$, constituting the inductive step for \Defref{phase1_hypothesis_wperp}.
\begin{lemma} \label{lem:wperp_phase1_induction}
    Suppose the Phase I scalings are satisfied (\Defref{phase1_scalings}).
    Under the event $\Etrain$, if $f_\rho$ obeys the Phase I inductive hypothesis for all iterations $k\leq t$ (\Defref{phase1_hypothesis}), then we have 
    \begin{equation*}
        \norm{\bwperptplus-\bwperpt} \lesssim \eta \theta\log^{-2C}(d).
    \end{equation*}
    for all neurons $(\atplus,\bwtplus)$.
\end{lemma}
\begin{proof}
    By the definition of gradient descent, we have
    \begin{align}
        \bwperptplus &= \bwperpt - \eta\nabla_{\bwperpt}\widehat{L}_\rho \nonumber\\
        &= \bwperpt - \eta \nabla_{\bwperpt}L_0 -\eta(\nabla_{\bwperpt} \widehat{L}_\rho-\nabla_{\bwperpt} L_0) \label{eq:wperp_phase1_proj}.
    \end{align}
    Note that
    \begin{equation*}
        \nabla_{\bwperpt} L_0 = \Proj_{\bwperpt} \nabla L_0 + \Proj_{V^{(t)}} \nabla L_0,
    \end{equation*}
    where $V^{(t)}\coloneqq \{\bm{v}\in\R^d:\bm{v}^\top\bwperpt=0\}$ is the orthogonal complement of $\bwperpt$.
    By \Lemref{perp_l0} we have\footnote{This is not the exact form of \Lemref{perp_l0}, but is straightforward to see by following the proof without the $\bwperp$ multiplied on the outside.}
    \begin{align}
       \nabla_{\bwperpt} L_0 &= -\frac{1}{8}|\at|\E_{\bxi}\Big[\bxi\cdot \sgn(\bw^{(t)\top}\bxi)\Big( \nonumber \\
        &\qquad \ind\Big(|\bw^{(t)\top}\bxi| \geq |\sqrt{2}\norm{\bwsigt}+\wspt|\Big) +\ind\Big(|\bw^{(t)\top}\bxi| \geq |\sqrt{2}\norm{\bwsigt}-\wspt|\Big) \nonumber \\
        &\qquad -\ind\Big(|\bw^{(t)\top}\bxi| \geq |\sqrt{2}\norm{\bwoppt}+\wspt|\Big) -\ind\Big(|\bw^{(t)\top}\bxi| \geq |\sqrt{2}\norm{\bwoppt}-\wspt|\Big) \Big)\Big]. \label{eq:phase1_wperp_l0_grad}
    \end{align}
    Consider the term
    \begin{equation*}
        \bm{q}\coloneqq\E_{\bxi}\left[\bxi\cdot \sgn(\bw^{(t)\top}\bxi)\cdot \ind\left(|\bw^{(t)\top}\bxi| \geq |\sqrt{2}\norm{\bwsigt}+\wspt|\right)\right].
    \end{equation*}
    Let us apply \Lemref{lindeberg} to $\bm{q}$ with $\kappa=|\sqrt{2}\norm{\bwsigt}+\wspt|$ to find that $\norm{\Proj_{V^{(t)}}\bm{q}}\lesssim \left( \frac{\norm{\bwperpt}_\infty}{\norm{\bwperpt}}\right)^{1/7}$.
    Repeating this procedure for each term in \Eqref{phase1_wperp_l0_grad} and combining via the triangle inequality, we have
    \begin{equation*}
        \norm{\Proj_{V^{(t)}} \nabla L_0} \lesssim |\at| \left(\frac{\norm{\bwperpt}_{\infty}}{\norm{\bwperpt}}\right)^{1/7}.
    \end{equation*}
    By \Lemref{phase1_norms_wperp}, \Lemref{phase1_norms_wperp_inf}, and \Lemref{phase1_norms_a}, we have $\norm{\bwperpt}\asymp \theta$, $\norm{\bwperpt}_\infty\lesssim \theta\log^{4C}(d)d^{-1/2}$, and $|\at|\lesssim \theta\log^C(d)$.
    Thus, we obtain
    \begin{equation} \label{eq:wperp_lindeberg}
         \norm{\Proj_{V^{(t)}} \nabla L_0} \lesssim \theta \log^{11C/7}(d) d^{-1/14}.
    \end{equation}
    Now let us analyze the term
    \begin{equation*}
        \Proj_{\bwperpt} \nabla L_0 = \frac{\bwperp^{(t)\top}\nabla_{\bwt} L_0}{\norm{\bwperpt}^2} \bwperpt.
    \end{equation*}
    By \Lemref{perp_l0} we have
    \begin{align}
        -\bwperp^{(t)\top}\nabla_{\bwt} L_0 &= \frac{1}{8}|\at|\E_{\bxi}\Big[|\bw^{(t)\top}\bxi|\Big( \nonumber \\
        & \phantom{+}\ind\Big(|\bw^{(t)\top}\bxi| \geq |\sqrt{2}\norm{\bwsigt}+\wspt|\Big) +\ind\Big(|\bw^{(t)\top}\bxi| \geq |\sqrt{2}\norm{\bwsigt}-\wspt|\Big) \nonumber \\
        & -\ind\Big(|\bw^{(t)\top}\bxi| \geq |\sqrt{2}\norm{\bwoppt}+\wspt|\Big) -\ind\Big(|\bw^{(t)\top}\bxi| \geq |\sqrt{2}\norm{\bwoppt}-\wspt|\Big) \Big)\Big]. \label{eq:wperp_l0_grad}
    \end{align}
    Recalling the application of \Lemref{phase1_norms_wperp} and \Lemref{phase1_norms_wperp_inf} above, we have
    \begin{align}
        \norm{\bwperpt}_\infty \log^{1/2}\left( \frac{\norm{\bwperpt}}{\norm{\bwperpt}_\infty} \right) &\lesssim \theta\log^{4C}(d)d^{-1/2} \cdot \log^{1/2} \left( \frac{d^{1/2}}{\log^{4C}(d)} \right) \nonumber \\
        &\lesssim \theta \log^{5C}(d) d^{-1/2}, \label{eq:wperp_berry_esseen_helper}
    \end{align}
    where we used that the function $x\log^{1/2}(\tfrac{y}{x})$ is increasing in $x$ for all $x\lesssim y$.
    Define
    \begin{equation*}
        \Psi(k)\coloneqq\E_{G\sim \Nc(0,\norm{\bwperpt}^2)}[|G|\ind(|G|\geq k)]
    \end{equation*}
    for $k\geq 0$, and consider the term
    \begin{equation*}
        \bm{r}\coloneqq \E_{\bxi}\left[|\bw^{(t)\top}\bxi|\cdot \ind\left(|\bw^{(t)\top}\bxi| \geq |\sqrt{2}\norm{\bwsigt}+\wspt|\right)\right].
    \end{equation*}
    Let us apply \Lemref{berry_esseen_tail} to $\bm{r}$ with $\bv=\bwt$ and $k=|\sqrt{2}\norm{\bwsigt}+\wspt|$ to find that 
    \begin{equation} \label{eq:wperp_berry_esseen}
        \left| \bm{r} - \Psi\left(|\sqrt{2}\norm{\bwsigt} +\wspt|\right) \right| \lesssim \theta \log^{5C}(d) d^{-1/2},
    \end{equation}
    where we used \Eqref{wperp_berry_esseen_helper}.
    Applying this procedure to each term in \Eqref{wperp_l0_grad} gives us $\Psi$-approximations for each truncated moment.
    Now, note that
    \begin{equation*}
        \Psi''(k)=\sqrt{\frac{2}{\pi}}\left(\frac{k^2-\norm{\bwperpt}^2}{\norm{\bwperpt}^3}\right) \exp\left(\frac{-k^2}{2\norm{\bwperpt}^2}\right),
    \end{equation*}
    such that
    \begin{equation*}
        \sup_{k\geq 0} |\Psi''(k)| \lesssim \frac{1}{\norm{\bwperpt}} \lesssim \frac{1}{\theta},
    \end{equation*}
    where the last inequality follows by \Lemref{phase1_norms_wperp} again.
    Using Taylor's theorem to linearize the $\Psi$-approximations of each term in \Eqref{wperp_l0_grad}, we obtain
    \begin{align*}
        \left| \Psi\left(|\sqrt{2}\norm{\bwsigt}+\wspt|\right)-\Psi(|\wspt|) + \sqrt{2}\norm{\bwsigt}\cdot \sgn(\wspt)\cdot \Psi'(|\wspt|) \right| &\lesssim \norm{\bwsigt}^2 \theta^{-1} \\
        \left| \Psi\left(|\sqrt{2}\norm{\bwsigt}-\wspt|\right)-\Psi(|\wspt|) - \sqrt{2}\norm{\bwsigt}\cdot \sgn(\wspt)\cdot \Psi'(|\wspt|) \right| &\lesssim \norm{\bwsigt}^2 \theta^{-1} \\
        \left| \Psi\left(|\sqrt{2}\norm{\bwoppt}+\wspt|\right)-\Psi(|\wspt|) + \sqrt{2}\norm{\bwoppt}\cdot \sgn(\wspt)\cdot \Psi'(|\wspt|) \right| &\lesssim \norm{\bwoppt}^2 \theta^{-1} \\
        \left| \Psi\left(|\sqrt{2}\norm{\bwoppt}-\wspt|\right)-\Psi(|\wspt|) - \sqrt{2}\norm{\bwoppt}\cdot \sgn(\wspt)\cdot \Psi'(|\wspt|) \right| &\lesssim \norm{\bwoppt}^2 \theta^{-1}.
    \end{align*}
    We will use these inequalities to show that the sum of the $\Psi$-approximations of each term in \Eqref{wperp_l0_grad} is extremely small.
    We have
    \begin{align*}
        & \Psi\left(|\sqrt{2}\norm{\bwsigt}+\wspt|\right)+\Psi\left(|\sqrt{2}\norm{\bwsigt} -\wspt|\right) - \Psi\left(|\sqrt{2}\norm{\bwoppt}+\wspt|\right)-\Psi\left(|\sqrt{2}\norm{\bwoppt}-\wspt|\right) \nonumber \\
        &= \left[\Psi\left(|\sqrt{2}\norm{\bwsigt}+\wspt|\right)+\Psi\left(|\sqrt{2}\norm{\bwsigt}-\wspt|\right) - 2\Psi(|\wspt|)\right] \nonumber \\
        &\qquad - \left[\Psi\left(|\sqrt{2}\norm{\bwoppt}+\wspt|\right)+\Psi\left(|\sqrt{2}\norm{\bwoppt}-\wspt|\right)-2\Psi(|\wspt|)\right] \nonumber \\
        &\leq \left|\Psi\left(|\sqrt{2}\norm{\bwsigt}+\wspt|\right)+\Psi\left(|\sqrt{2}\norm{\bwsigt}-\wspt|\right)-2\Psi(|\wspt|)\right| \\
        &\qquad + \left|\Psi\left(|\sqrt{2}\norm{\bwoppt}+\wspt|\right)+\Psi\left(|\sqrt{2}\norm{\bwoppt}-\wspt|\right)-2\Psi(|\wspt|)\right|.
    \end{align*}
    We then have by the linearization inequalities above that
    \begin{align*}
        \left|\Psi\left(|\sqrt{2}\norm{\bwsigt}+\wspt|\right)+\Psi\left(|\sqrt{2}\norm{\bwsigt}-\wspt|\right)-2\Psi(|\wspt|)\right| &\lesssim \norm{\bwsigt}^2 \theta^{-1} \lesssim \theta\log^{4C}(d)d^{-1} \\
        \left|\Psi\left(|\sqrt{2}\norm{\bwoppt}+\wspt|\right)+\Psi\left(|\sqrt{2}\norm{\bwoppt}-\wspt|\right)-2\Psi(|\wspt|)\right| &\lesssim \norm{\bwoppt}^2 \theta^{-1} \lesssim \theta\log^{4C}(d)d^{-1},
    \end{align*}
    where the first-order terms precisely cancel due to opposing signs, and we used $\norm{\bwsigt},\norm{\bwoppt}\lesssim \theta\log^{2C}(d)d^{-1/2}$ by \Lemref{phase1_norms_wsig}.
    Combining this result with \Eqref{wperp_l0_grad} and \Eqref{wperp_berry_esseen} we finally obtain
    \begin{align}
        -\bwperp^{(t)\top}\nabla_{\bwt} L_0 &\lesssim |\at| \cdot (\theta\log^{4C}(d)d^{-1}+\theta\log^{5C}(d)d^{-1/2}) \nonumber \\
        &\lesssim \theta^2\log^{6C}(d)d^{-1/2}, \label{eq:wperp_phase1_term_1}
    \end{align}
    where the last inequality follows by \Lemref{phase1_norms_a} again.
    Next, by the Cauchy-Schwarz inequality and \Lemref{phase1_l0_lp_final} with $i>3$, we have
    \begin{align} 
        |\bwperp^{(t)\top}(\nabla_{\bwt}\widehat{L}_\rho-\nabla_{\bwt}L_0)| &\leq \norm{\bwperpt} \cdot \norm{\nabla_{\bwperpt}\widehat{L}_\rho-\nabla_{\bwperpt}L_0} \nonumber \\
        &\lesssim \norm{\bwperpt} \cdot \sqrt{d} \cdot \max_{i>3}|\partial_{\wit}\widehat{L}_\rho-\partial_{\wit} L_0| \nonumber \\
        &\lesssim \norm{\bwperpt} \cdot (\theta^3\log^{8C}(d)+\theta\log^{-2C}(d)) \nonumber \\
        &\lesssim \norm{\bwperpt} \cdot \theta\log^{-2C}(d), \label{eq:wperp_phase1_term_2}
    \end{align}
    where we used $\theta \ll \log^{-5C}(d)$ and the basic inequality $\norm{\cdot}\leq \sqrt{d}\norm{\cdot}_\infty$.
    Returning to \Eqref{wperp_phase1_proj} and combining \Eqref{wperp_lindeberg}, \Eqref{wperp_phase1_term_1}, and \Eqref{wperp_phase1_term_2}, we have
    \begin{align*}
        \norm{\bwperptplus - \bwperpt} &\lesssim \eta \frac{\theta^2\log^{6C}(d)d^{-1/2}+\norm{\bwperpt}\theta\log^{-2C}(d)}{\norm{\bwperpt}^2}\norm{\bwperpt} + \eta\theta\log^{11C/7}(d)d^{-1/14}\\
        &\lesssim \eta \theta\log^{-2C}(d),
    \end{align*}
    where the last inequality follows by \Lemref{phase1_norms_wperp} again.
    This completes the proof of the lemma.
\end{proof}

Finally, we can upper bound the growth of $\norm{\bwperp}_\infty$, constituting the inductive step for \Defref{phase1_hypothesis_wperp_inf}.
\begin{lemma} \label{lem:wperp_inf_phase1_induction}
    Suppose the Phase I scalings are satisfied (\Defref{phase1_scalings}).
    Under the event $\Etrain$, if $f_\rho$ obeys the Phase I inductive hypothesis for all iterations $k\leq t$ (\Defref{phase1_hypothesis}), then we have
    \begin{equation*}
        \norm{\bwperptplus - \bwperpt}_\infty \lesssim \eta\theta \log^{3C}(d)d^{-1/2}
    \end{equation*}
    for all neurons $(\atplus,\bwtplus)$.
\end{lemma}
\begin{proof}
    Let $i>3$; we will examine the growth of $|\wi|$ to obtain the desired bound.
    By the definition of gradient descent, we have
    \begin{align}
        \witplus &= \wit - \eta\partial_{\wit}\widehat{L}_\rho \nonumber\\
        &= \wit - \eta \frac{\wit\partial_{\wit} \widehat{L}_\rho}{\wit} \nonumber\\
        &= \wit - \eta \frac{\wit\partial_{\wit} L_0 + \wit(\partial_{\wit} \widehat{L}_\rho-\partial_{\wit} L_0)}{\wit}. \label{eq:wi_phase1_proj}
    \end{align}
    Recall we denote $\bxi_{\setminus i}\coloneqq \bxi - x_i\be_i$.
    By \Lemref{wi_l0} we have
    \begin{align*}
        -\wit\partial_{\wit} L_0 &= \frac{1}{8}|\at||\wit|\Big( \\
        &\qquad \P_{\bxi}\left(|\wit|\geq |\sqrt{2}\norm{\bwsigt}+\wspt+\bw^{(t)\top}\bxi_{\setminus i}|\right) +\P_{\bxi}\left(|\wit|\geq |\sqrt{2}\norm{\bwsigt}-\wspt+\bw^{(t)\top}\bxi_{\setminus i}|\right) \\
        &\qquad -\P_{\bxi}\left(|\wit|\geq |\sqrt{2}\norm{\bwoppt}+\wspt+\bw^{(t)\top}\bxi_{\setminus i}|\right) -\P_{\bxi}\left(|\wit|\geq |\sqrt{2}\norm{\bwoppt}-\wspt+\bw^{(t)\top}\bxi_{\setminus i}|\right)\Big).
    \end{align*}
    We will rewrite the sum of four probabilities in interval notation to find that some segments cancel.
    Define $X\coloneqq \bw^{(t)\top}\bxi_{\setminus i}$ and
    \begin{align*}
        I_1 &\coloneqq \left[-|\wit|-\sqrt{2}\norm{\bwsigt}-\wspt, |\wit|-\sqrt{2}\norm{\bwsigt}-\wspt\right] \\
        I_2 &\coloneqq \left[-|\wit|-\sqrt{2}\norm{\bwsigt}+\wspt, |\wit|-\sqrt{2}\norm{\bwsigt}+\wspt\right] \\
        I_3 &\coloneqq \left[-|\wit|-\sqrt{2}\norm{\bwoppt}-\wspt, |\wit|-\sqrt{2}\norm{\bwoppt}-\wspt\right] \\
        I_4 &\coloneqq \left[-|\wit|-\sqrt{2}\norm{\bwoppt}+\wspt, |\wit|-\sqrt{2}\norm{\bwoppt}+\wspt\right].
    \end{align*}
    Then, we can write
    \begin{align*}
        -\wit\partial_{\wit} L_0 &= \frac{1}{8}|\at||\wit|\left(\P_{\bxi}(X\in I_1) + \P_{\bxi}(X\in I_2) - \P_{\bxi}(X\in I_3) - \P_{\bxi}(X\in I_4)\right) \\
        &\lesssim |\at||\wit|\left(|\P_{\bxi}(X\in I_1) - \P_{\bxi}(X\in I_3)| + |\P_{\bxi}(X\in I_2) - \P_{\bxi}(X\in I_4)|\right) \\
        &\leq |\at||\wit|\left(\P_{\bxi}(X\in I_1\Delta I_3) + \P_{\bxi}(X\in I_2\Delta I_4)\right),
    \end{align*}
    where we used the triangle inequality, and $I\Delta J \coloneqq (I\setminus J) \cup (J \setminus I)$ denotes the symmetric difference of intervals $I$ and $J$.
    Note $I_1$ and $I_3$ (respectively $I_2$ and $I_4$) are intervals of width $2|\wit|$ whose centers differ by only $\sqrt{2}(\norm{\bwsig}-\norm{\bwopp})\lesssim \theta\log^{2C}(d)d^{-1/2}$, where the inequality follows by \Lemref{phase1_norms_wsig}.
    So, $I_1\Delta I_3$ (respectively $I_2\Delta I_4$) comprises two intervals of width at most $\theta\log^{2C}(d)d^{-1/2}$.
    Let us apply \Lemref{boolean_to_gaussian_delta} to each of the four intervals, with $\bvzero=\bwzero_{\perp\setminus i}$, $\bDelta=\bw_{\perp\setminus i}-\bwzero_{\perp\setminus i}$, and $k\lesssim \theta\log^{2C}(d)d^{-1/2}$.\footnote{Here, $\mu$ is the center of each interval, but its value doesn't matter as we upper bound $\exp\Big(\tfrac{-\mu^2}{C\theta^2}\Big)\leq 1$ anyway.}
    Under the event $\Etrain$, the condition on $\bvzero$ is satisfied; we use \Lemref{phase1_norms_wperp} for the condition on $\bDelta$.\footnote{While \Lemref{phase1_basecase} and \Lemref{phase1_norms_wperp} are stated for $\bwperp$, it is straightforward to see they hold for $\bw_{\perp\setminus i}$.}
    Hence, we have
    \begin{align*}
        \P_{\bxi}(X\in I_1\Delta I_3) &\lesssim \frac{\theta\log^{2C}(d)d^{-1/2}}{\theta}+\frac{1}{\sqrt{d}} \\
        \P_{\bxi}(X\in I_2\Delta I_4) &\lesssim \frac{\theta\log^{2C}(d)d^{-1/2}}{\theta}+\frac{1}{\sqrt{d}},
    \end{align*}
    and therefore
    \begin{equation} \label{eq:wi_phase1_term_1}
        -\wit\partial_{\wit} L_0 \lesssim |\at||\wit| \cdot (\log^{2C}(d)d^{-1/2}+d^{-1/2})\lesssim |\wit| \cdot \theta\log^{3C}(d)d^{-1/2},
    \end{equation}
    where we used $|\at|\lesssim \theta\log^C(d)$ by \Lemref{phase1_norms_a}.
    Next, by the Cauchy-Schwarz inequality and \Lemref{phase1_l0_lp_final} with $i>3$, we have
    \begin{align}
        |\wit(\partial_{\wit}\widehat{L}_\rho-\partial_{\wit}L_0)| &\leq |\wit|\cdot |\partial_{\wit}\widehat{L}_\rho-\partial_{\wit} L_0| \nonumber \\
        &\lesssim |\wit| \cdot (\theta^3 \log^{8C}(d)d^{-1/2} + \theta\log^{-2C}(d) d^{-1/2}) \nonumber \\
        &\lesssim |\wit| \cdot \theta\log^{2C}(d)d^{-1/2}, \label{eq:wi_phase1_term_2}
    \end{align}
    where we used $\theta\ll \log^{-3C}(d)$.
    Returning to \Eqref{wi_phase1_proj} and combining \Eqref{wi_phase1_term_1} and \Eqref{wi_phase1_term_2}, we have
    \begin{align*}
        |\witplus-\wit| &\lesssim \eta \frac{\theta\log^{3C}(d)d^{-1/2}+\theta\log^{2C}(d)d^{-1/2}}{|\wit|}|\wit| \\
        &\lesssim \eta\theta\log^{3C}(d)d^{-1/2}.
    \end{align*}
    Since this holds for any $i> 3$, we have
    \begin{equation*}
        \norm{\bwperptplus-\bwperpt}_\infty \lesssim \eta\theta\log^{3C}(d)d^{-1/2}.
    \end{equation*}
    This completes the proof of the lemma.
\end{proof}

We are now ready to show the Phase I result.
\begin{proposition} \label{prop:phase1}
    Suppose the Phase I scalings are satisfied (\Defref{phase1_scalings}).
    Under the event $\Etrain$ (\Defref{events_train}), upon $\TI\asymp\log\log(d)\eta^{-1}$ iterations of online minibatch SGD under the $\ell_\rho$ loss, all neurons $(a^{(\TI)},\bw^{(\TI)})$ obey the Phase I inductive hypothesis (\Defref{phase1_hypothesis}) with $\sgn(\wsp^{(\TI)})=\sgn(\azero)$.
    Moreover, for any test point $\bx$ satisfying the event $\Etest$ (\Defref{events_test}), we have $\sgn(f_{\rho^{(\TI)}}(\bx))=x_3$.
\end{proposition}
\begin{proof}
    Under the event $\Etrain$, the inductive steps for \Defref{phase1_hypothesis} are given as follows:
    \begin{itemize}
        \item The \Defref{phase1_hypothesis_wsp} inductive step is given by \Lemref{wsp_phase1_induction}.
        \item The \Defref{phase1_hypothesis_wsig} inductive step is given by \Lemref{wsig_phase1_induction}.
        \item The \Defref{phase1_hypothesis_wopp} inductive step is given by \Lemref{wopp_phase1_induction}.
        \item The \Defref{phase1_hypothesis_wperp} inductive step is given by \Lemref{wperp_phase1_induction}.
        \item The \Defref{phase1_hypothesis_wperp_inf} inductive step is given by \Lemref{wperp_inf_phase1_induction}.
        \item The \Defref{phase1_hypothesis_a} inductive step is given by \Lemref{balanced_iii}.
    \end{itemize}
    By \Lemref{phase1_norms_sgn}, we have for all neurons $(a^{(\TI)},\bw^{(\TI)})$ that $\sgn(a^{(\TI)})=\sgn(\azero)$, and by definition of the event $\Etrain$ we have $|S^+|,|S^-|>0$.
    Hence, by \Lemref{end_of_phase1}, we have $\sgn(f_\rho(\bx))=x_3$ for any test point $\bx$ that satisfies the event $\Etest$ if for all neurons $(a^{(\TI)},\bw^{(\TI)})$ we have $\sgn(\wsp^{(\TI)})=\sgn(a^{(\TI)})$ and $\frac{1}{C}\norm{\bwsp^{(\TI)}}\geq \norm{\bwsig^{(\TI)}}+\norm{\bwopp^{(\TI)}}+\norm{\bwperp^{(\TI)}}\log^{1/2}(d)$.
    We will now show these conditions hold.
    
    The condition $\sgn(\wsp^{(\TI)})=\sgn(a^{(\TI)})$ is satisfied after at most $\TIa\asymp \log^{1/2}(d)d^{-1/2}\eta^{-1}$ iterations; this constitutes Phase Ia.
    The intuition is that $\wsp\to 0$ monotonically if $\sgn(\wsp)\neq \sgn(\azero)$ and vice versa; in the former case $\wsp$ takes at most $\TIa$ iterations to flip signs.
    We now formalize this intuition.
    If $\sgn(\azero)\neq \sgn(\wspzero)$, we have by \Defref{phase1_hypothesis_wsp} that
    \begin{equation*}
        \wspt - \wsptminus = -\eta\left(1\pm o(\log^{-2}(d))\right)\cdot \left(\frac{1}{2}-\lambda\right) \cdot \sgn(\wsptminus) (|\wsptminus| + \theta).
    \end{equation*}
    Following \Lemref{phase1_recursion} with $w=\wsp$, $z=\theta$, $\mu=\tfrac{1}{2}-\lambda$, and $\delta=\log^{-2}(d)$, we find
    \begin{equation*}
        |\wsp^{(\TIa)}| = \left(1\pm 3\eta \TIa \log^{-2}(d) \right) \cdot\left( 1-\eta\left(\frac{1}{2}-\lambda\right)\right)^{\TIa}\cdot (|\wspzero|+\theta)-\theta.
    \end{equation*}
    Using $\TIa \lesssim \log^{1/2}(d)d^{-1/2}\eta^{-1}$ we have $\eta \TIa \log^{-2}(d) \ll 1$; hence for a constant $c_1>0$ we have
    \begin{equation*}
         |\wsp^{(\TIa)}| \leq (1-c_1\eta)^{\TIa} \cdot (|\wspzero|+\theta)-\theta.
    \end{equation*}
    Setting the upper bound to zero to find the transition point, we find
    \begin{equation*}
        \TIa \log(1-c_1\eta) = -\log\bigg( 1+ \frac{|\wspzero|}{\theta} \bigg),
    \end{equation*}
    wherein applying $\log(1+x)\leq x$ with $\eta\ll 1$ and $|\wspzero|\lesssim \theta\log^{1/2}(d)d^{-1/2}$ by \Lemref{phase1_basecase} gives $\TIa \cdot c_1\eta \lesssim \log^{1/2}(d)d^{-1/2}$, that is $\TIa \lesssim \log^{1/2}(d)d^{-1/2}\eta^{-1}$.
    Moreover, by observation of \Defref{phase1_hypothesis_wsp}, $\sgn(\wsp)$ does not change once $\sgn(\wsp)=\sgn(\azero)$.

    The condition $\frac{1}{C}\norm{\bwsp^{(\TI)}}\geq \norm{\bwsig^{(\TI)}}+\norm{\bwopp^{(\TI)}}+\norm{\bwperp^{(\TI)}}\log^{1/2}(d)$ is satisfied after at most $\TIb\asymp \log\log(d)\eta^{-1}$ additional iterations; this constitutes Phase Ib.
    Specifically, by \Lemref{phase1_norms_wsig} and \Lemref{phase1_norms_wperp} we have $\norm{\bwsig^{(\TI)}},\norm{\bwopp^{(\TI)}}\lesssim \theta\log^{2C}(d)d^{-1/2}$ and $\norm{\bwperp^{(\TI)}}\lesssim \theta$.
    On the other hand, since $\sgn(\wsp)=\sgn(\azero)$, by \Lemref{phase1_recursion} with $w=\wsp$, $z=\theta$, $\mu=\tfrac{1}{2}-\lambda$, and $\delta=\log^{-2}(d)$, we find
    \begin{equation*}
        |\wsp^{(\TI)}| \geq \left(1 - 3\eta \TIb \log^{-2}(d)\right)\cdot \left(1+\eta\left(\frac{1}{2}-\lambda\right)\right)^{\TIb}(|\wspzero|+\theta)-\theta.
    \end{equation*}
    Using $\TIb\lesssim \log\log(d)\eta^{-1}$ we have $\eta \TIb \log^{-2}(d) \ll 1$; hence for a constant $c_2>0$ (that depends on $\lambda \in (0, \tfrac{1}{2})$), we have
    \begin{equation*}
        |\wsp^{(\TI)}| \geq (1+c_2\eta)^{\TIb}(|\wspzero|+\theta)-\theta.
    \end{equation*}
    Choose $\TIb=c_3\log\log(d)\eta^{-1}$ for some large enough constant $c_3>0$.
    We then have
    \begin{equation*}
        (1+c_2\eta)^{\TIb} = \exp\left(\frac{c_3\log\log(d)}{\eta}\log(1+c_2\eta)\right),
    \end{equation*}
    wherein using $\log(1+c_2\eta)\geq \frac{c_2\eta}{2}$ for $\eta\ll 1$ gives
    \begin{equation*}
        \frac{c_3\log\log(d)}{\eta}\log(1+c_2\eta) \geq \frac{c_2c_3\log\log(d)}{2} \geq C \log\log(d).
    \end{equation*}
    Therefore, we have $(1+c_2\eta)^{\TIb} \geq \exp(C\log\log(d)) = \log^C(d)$.
    Since $|\wspzero|\geq 0$, this gives $|\wsp^{(\TI)}|\geq \theta(\log^C(d)-1)\gtrsim \theta\log^C(d)$.
    Choosing $c_3$ to be large enough to ensure that $C>\tfrac{1}{2}$ yields the result.\footnote{To make the key point explicit: combined with \Lemref{phase1_norms_wsp} we have shown $\norm{\bwsp^{(\TI)}}\asymp \theta \log^C(d)$ for a constant $C>\tfrac{1}{2}$.}

    Finally, we have
    \begin{equation*}
        \TI=\TIa+\TIb\asymp \log^{1/2}(d)d^{-1/2}\eta^{-1}+\log\log(d)\eta^{-1}\asymp \log\log(d)\eta^{-1},
    \end{equation*}
    \ie the length of Phase I is dominated by the length of Phase Ib.
    This completes the proof of Phase I.
\end{proof}

\clearpage

\section{Phase II Induction} \label{sec:phase2_induction}
Recall from \Secref{phase1_induction} that upon $\TI\asymp\log\log(d)\eta^{-1}$ iterations (\ie the end of Phase I), we have $\sgn(f_{\rho^{(\TI)}}(\bx))=x_3$ when $\bx$ satisfies the event $\Etest$ (\Defref{events_test}).
(Moreover, we know that $\bx$ satisfies the event $\Etest$ with probability at least $1 - d^{-C}$; therefore, with high probability the neural network's prediction entirely relies on the spurious feature at test time.)
In Phase II, we will show that this condition continues to hold for roughly $\log(d)\eta^{-1}$ iterations, which is the sample complexity to learn the quadratic feature in the counterfactual situation where the linear spurious correlation did not exist~\citep{glasgow2024sgd}.
The key characteristic of the Phase II analysis is that the $L_0$ approximation of the population gradients is no longer sufficient as the network weights have become too large by this point. Hence, we must analyze $\nabla L_\rho$ directly via characterizing the training data margin induced by the network.
Despite this additional complexity, we show that the neurons continue to grow identically under the ``extreme correlation'' condition $\lambda \ll \log^{-1}(d)$.

We prove some preliminary lemmas, including our definition of the inductive hypothesis, in \Secref{phase2_preliminaries}.
We then show the inductive step in \Secref{phase2_inductive_step}.
Throughout this section, we write $\rho\coloneqq\rho^{(t)}$ as shorthand where appropriate, and we let $C > 0$ denote a sufficiently large constant which does not change from line to line.
We will also use constants $c<1$ and $C'<C-3$ which do not change from line to line.
We will work in high probability under the events $\Einit$, $\Etrain$, and $\Etest$ established in \Defref{events} and \Lemref{events}.

\subsection{Technical Preliminaries} \label{sec:phase2_preliminaries}
Our Phase II analysis requires the ``extreme'' correlation condition $\lambda\ll \log^{-1}(d)$.
In Phase II, the spurious feature grows very slowly, but the extreme correlation ensures it is still monotonically increasing.
The key object we track in Phase II is the margin $\gamma(\bx)\coloneqq y(\bx)f_\rho(\bx)$.
We define Phase II to last until $|\gamma(\bx)|\approx \log\log(d)$ (see \Eqref{margin_def} for the formal statement), %
which we show takes $\TII\asymp \tfrac{1}{\eta}\big(\tfrac{\log(d)}{\log\log(d)}+\log(\tfrac{1}{\theta})\big)$ iterations.

Throughout Phase II, we will require certain scaling factors to be satisfied.
Note that the Phase II scalings are stricter than the Phase I scalings (\Defref{phase1_scalings}).
In particular, note that if $\theta\asymp \text{poly}^{-1}(d)$ we require a large batch size $m\asymp \text{poly}(d)$.
\begin{definition} \label{def:phase2_scalings}
    We say the \emph{Phase II scalings} are satisfied if the following conditions are met:
    \crefalias{enumi}{definition}
    \begin{enumerate}[label=(\roman*), ref=\thedefinition(\roman*)]
        \item \label{def:phase2_scalings_lr} The learning rate $\log(d)d^{-C}\ll \eta\ll \log^{-3}(d)$.
        \item \label{def:phase2_scalings_iter} The iteration $t\leq \TI+\TII \lesssim \log(d)\eta^{-1}$.
        \item \label{def:phase2_scalings_width} The width $\log^5(d)\ll p \ll d^{C}$.
        \item \label{def:phase2_scalings_init} The initialization scale $d^{-C/2}\ll \theta\ll \log^{-5C}(d)$.
        \item \label{def:phase2_scalings_batch} The batch size $m\gg d\log^6(d)\theta^{-2}$.
        \item \label{def:phase2_scalings_strength} The spurious correlation strength $\lambda \ll \log^{-1}(d)$.
    \end{enumerate}
\end{definition}

\begin{remark} \normalfont
    For clarity, we list the limiting usage of each scaling (\ie the result which requires the tightest application of each item in \Defref{phase2_scalings}):
    \begin{itemize}
        \item \Defref{phase2_scalings_lr} is limited by \Lemref{phase2_norms_a}. The upper bound helps with the balancedness condition $|a|\approx \norm{\bw}$, while the lower bound ensures that $t \ll d^C$ throughout Phase II (in particular, $\TI+\TII \ll d^C$, which is sufficient for this).
        \item \Defref{phase2_scalings_iter} is limited by \Lemref{phase2_norms} and \Lemref{phase2_norms_scalar}. It is essentially the duration of Phase II.
        \item \Defref{phase2_scalings_width} is limited by \Lemref{phase2_margin}, where the exponent on the lower bound ensures $\gamma_+,\gamma_-$ concentrate sufficiently closely about $\bar{\gamma}$.
        \item \Defref{phase2_scalings_init} is limited by \Lemref{sigopp_phase2_induction}, \Lemref{wperp_phase2_induction}, and \Lemref{wi_phase2_induction}. The lower bound is loose but suffices for analysis, and the upper bound follows by the Phase I scalings. Similarly to Phase I, the initialization scale $\theta$ acts as a free parameter which we can make sufficiently small.
        \item \Defref{phase2_scalings_batch} is limited by \Lemref{sigopp_phase2_induction}, \Lemref{wperp_phase2_induction}, and \Lemref{wi_phase2_induction}, where it is used in Hoeffding's inequality to concentrate the empirical gradient.
        \item \Defref{phase2_scalings_strength} is limited by the stopping condition $\bgammat\asymp \log\log(d)$ such that $e^{-\bgammat}-\lambda \geq 0$, \ie the ``average margin'' over neurons (which we will define shortly) grows monotonically.
    \end{itemize}
\end{remark}

Recall that $\TI=\TIa+\TIb\asymp\log\log(d) \eta^{-1}$ and we define $\psi(x)\coloneqq 1/(1+e^{-x})$ to be the sigmoid.
Recall also that we define the ``positive'' neurons by $S^+\coloneqq \{(\at,\bwt):\sgn(\at)=1\}$ and the ``negative'' neurons by $S^-\coloneqq \{(\at,\bwt):\sgn(\at)=-1\}$.
In Phase II, we will primarily be interested in the unnormalized margins induced by the positive and negative neurons on a test sample, which we will show concentrate (with high probability over the test data) about
\begin{equation} \label{eq:margin_def}
    \gammaplust \coloneqq \frac{1}{p}\sum_{(\at,\bwt)\in S^+} (\wspt)^2 \qquad \gammaminust \coloneqq \frac{1}{p}\sum_{(\at,\bwt)\in S^-} (\wspt)^2.
\end{equation}
We track the average margin $\bgammat\coloneqq \tfrac{1}{2}(\gamma_+^{(t)}+\gamma_-^{(t)})$ and define Phase II to end when $\bgammat\asymp \log\log(d)$.
We can now define our Phase II inductive hypothesis.
\begin{definition} \label{def:phase2_hypothesis}
    A neural network $f_\rho$ is said to obey the \emph{Phase II inductive hypothesis at iteration} $t\geq \TI$ if there exists $c<1$ such that the following conditions are met for all neurons $(\at,\bwt)$:
    \crefalias{enumi}{definition}
    \begin{enumerate}[label=(\roman*), ref=\thedefinition(\roman*)]
        \item \label{def:phase2_hypothesis_wsp} $\wspt - \wsptminus = \eta (1 \pm o(1)) \cdot \wsptminus \left(1-\lambda-\psi(\bgammatminus)\right)$.
        \item \label{def:phase2_hypothesis_wsig} $\norm{\bwsigt-\bwsigtminus}\lesssim \eta e^{-c\bgammatminus}\Big(\max_{(a,\bw)}(\wsptminus)^2\norm{\bwsigtminus}\Big) + \eta \theta\log^{-1}(d)d^{-1/2}$.
        \item \label{def:phase2_hypothesis_wopp} $\norm{\bwoppt-\bwopptminus} \lesssim \eta e^{-c\bgammatminus}\Big(\max_{(a,\bw)}(\wsptminus)^2\norm{\bwopptminus}\Big) + \eta \theta\log^{-1}(d)d^{-1/2}$.
        \item \label{def:phase2_hypothesis_wperp} $\norm{\bwperpt-\bwperptminus} \lesssim \eta e^{-c\bgammatminus}(\wsptminus)^2\norm{\bwperptminus} + \eta \theta\log^{-1}(d)$.
        \item \label{def:phase2_hypothesis_wperp_inf} $\norm{\bwperpt-\bwperptminus}_\infty \lesssim \eta e^{-c\bgammatminus}\Big(\max_{(a,\bw)}(\wsptminus)^2\norm{\bwperptminus}_\infty\Big) + \eta \theta\log^{-1}(d)d^{-1/2}$.
        \item \label{def:phase2_hypothesis_a} $|\at|\leq\norm{\bwt}$.
    \end{enumerate}
\end{definition}

\begin{remark} \normalfont
    The inductive hypothesis defined above implies that the growth of $\wsp$ decelerates exponentially.
    Specifically, since $1-\lambda-\psi(\bgammat)\asymp e^{-\bgammat}$ when $\lambda \ll \log^{-1}(d)$, the first recurrence (\Defref{phase2_hypothesis_wsp}) implies that
    \begin{equation} \label{eq:wsp_phase2_exp}
        \wspt - \wsptminus \asymp \eta \wspt e^{-\bgammat}.
    \end{equation}
    Yet, the condition $\lambda \ll \log^{-1}(d)$ is sufficient to show that $\wsp$ growth remains monotonic.
    The geometric growth factor of $\bwsig$, $\bwopp$, and $\bwperp$ also decelerates exponentially, leading to domination by small additive factors (the second terms in Definitions~\ref{def:phase2_hypothesis_wsig},~\ref{def:phase2_hypothesis_wopp},~\ref{def:phase2_hypothesis_wperp} and~\ref{def:phase2_hypothesis_wperp_inf}).
    In \Propref{phase2}, we use a continuous approximation of \Eqref{wsp_phase2_exp} to show that Phase II lasts $\TII\asymp \tfrac{1}{\eta}\big(\tfrac{\log(d)}{\log\log(d)}+\log(\tfrac{1}{\theta})\big)$ iterations, whereupon $\norm{\bwsp}\asymp (\log\log(d))^{1/2}$.
    The non-spurious weights grow by only a polylogarithmic factor during this time, and remain an order of magnitude smaller.
\end{remark}

Let us recall some properties of the neural network at the end of Phase I, which will effectively function as our Phase II ``base case''.
\begin{remark} \label{rem:phase2_basecase}
    Recall that, under the event $\Etrain$ (\Defref{events_train}) and with the Phase I scalings satisfied (\Defref{phase1_scalings}), the neural network $f_\rho$ satisfies the following properties at iteration $\TI$ for all neurons $(a^{(\TI)}, \bw^{(\TI)})$:
    \begin{enumerate}[label=(\roman*), ref=\thelemma(\roman*)]
        \item $\norm{\bwsp^{(\TI)}}\asymp \theta \log^C(d)$. %
        \item $\norm{\bwsig^{(\TI)}}, \norm{\bwopp^{(\TI)}}\lesssim \theta \log^{2C}(d)d^{-1/2}$.
        \item $\norm{\bwperp^{(\TI)}}\asymp \theta$.
        \item $\norm{\bwperp^{(\TI)}}_\infty\lesssim \theta\log^{4C}(d)d^{-1/2}$.
        \item $|a^{(\TI)}|= \left(1\pm o(\log^{-2}(d))\right) \cdot \norm{\bwsp^{(\TI)}}$.
        \item $\sgn(a^{(\TI)})=\sgn(\azero)=\sgn(\wsp^{(\TI)})$.
    \end{enumerate}
\end{remark}

The next lemma controls the norms of all the neuron components during Phase II relative to the initialization scale $\theta$ assuming that the Phase II inductive hypothesis holds up until that point.
Notice that the bounds hold for any $\TI \leq t \leq \TI+\TII$; two special cases are $\norm{\bwsp^{(\TI)}}\asymp \theta\log^C(d)$ but $\norm{\bwsp^{(\TI+\TII)}}\asymp (\log\log(d))^{1/2}$.
A second remark is that the upper bounds of \Lemref{phase2_norms_margin} and \Lemref{phase2_norms_wsp} will eventually become tight, such that $\bgammat\asymp \log\log(d)$ and $\norm{\bwspt}\asymp (\log\log(d))^{1/2}$, as we show in \Propref{phase2}.
\begin{lemma} \label{lem:phase2_norms}
    Suppose the Phase II scalings are satisfied (\Defref{phase2_scalings}).
    Under the event $\Etrain$ (\Defref{events_train}), if $f_\rho$ obeys the Phase II inductive hypothesis for all iterations $\TI\leq k\leq t$ (\Defref{phase2_hypothesis}), then the following conditions hold on all neurons $(\at,\bwt)$:
    \crefalias{enumi}{lemma}
    \begin{enumerate}[label=(\roman*), ref=\thelemma(\roman*)]
        \item \label{lem:phase2_norms_margin} $\theta^2\log^{2C}(d) \lesssim \bgammat \lesssim \log\log(d).$
        \item \label{lem:phase2_norms_wsp} $\theta\log^C(d)\lesssim \norm{\bwspt} \lesssim (\log\log(d))^{1/2}$.
        \item \label{lem:phase2_norms_wsig} $\norm{\bwsigt}, \norm{\bwoppt} \lesssim \theta \log^{3C}(d)d^{-1/2}$.
        \item \label{lem:phase2_norms_wperp} $\theta\log^{-{C'}}(d)\lesssim \norm{\bwperpt} \lesssim \theta\log^{C'}(d)$ where $C'<C-3$. %
        \item \label{lem:phase2_norms_wperp_inf} $\norm{\bwperpt}_\infty \lesssim \theta \log^{5C}(d)d^{-1/2}$.
        \item \label{lem:phase2_norms_w} $\norm{\bwt} = \left(1\pm o(\log^{-2}(d)) \right) \cdot \norm{\bwspt} \lesssim (\log\log(d))^{1/2}$.
    \end{enumerate}
\end{lemma}
\begin{proof}
    For \Lemref{phase2_norms_margin}, the upper bound holds by definition: Phase II lasts until $\bgammat \asymp\log\log(d)$.
    For the \Lemref{phase2_norms_margin} lower bound, note that the growth of $\wspt$ in \Defref{phase2_hypothesis_wsp} is monotonic (increasing if $\wspt > 0$, decreasing if $\wspt < 0$), hence so is the growth of $\bgammat$.
    The monotonicity of $\wspt$ is formally shown below: using $1-\psi(z)\asymp e^{-z}$ for $z\geq 0$ with $\bgammat \lesssim \log\log(d)$ and $\lambda\ll \log^{-1}(d)$, we have
    \begin{equation*}
        1-\lambda-\psi(\bgammat)\asymp e^{-\bgammat}-\lambda \gtrsim \log^{-1}(d) \geq 0.
    \end{equation*}
    Regardless of the sign of $\wspt$, the above argument implies that $\norm{\bwspt}$ is monotonically increasing.
    Hence the lower bound holds by $\norm{\bwsp^{(\TI)}}\gtrsim \theta\log^C(d)$ from \Remref{phase2_basecase}, monotonicity, and concentration of $|S^+|,|S^-|$ under $\Etrain$.

    For the \Lemref{phase2_norms_wsp} upper bound, if $(\wspt)^2\gg \log\log(d)$, then we must have $\bgammat \gg \log\log(d)$ which is a contradiction with the stopping condition of Phase II.
    We already showed the lower bound in \Lemref{phase2_norms_margin}.
    
    For \Lemref{phase2_norms_wsig}, we may study the neuron that maximizes $(\wsp^{(\TI)})^2\norm{\bwsig^{(\TI)}}$ to obtain an upper bound.
    In particular, the maximizing neuron at time $t$ is at most as large as the maximizing neuron at time $t-1$ plus the maximal allowed growth from $t-1$ to $t$.
    Moreover, this maximal growth is the same for all neurons, as it depends only on $\bar{\gamma}^{(t-1)}$ which is a global property of the network.
    So, treating the recurrence on the maximizing neuron at time $\TI$ would give an upper bound over all neurons.
    Therefore, we consider the simpler form of \Defref{phase2_hypothesis_wsig} as
    \begin{equation*}
        \norm{\bwsigt-\bwsigtminus}\lesssim \eta e^{-c\bgammatminus}(\wsptminus)^2\norm{\bwsigtminus} + \eta \theta\log^{-1}(d)d^{-1/2}.
    \end{equation*}
    For $\TI\leq t_1<t_2\leq\TI+\TII$ define the \emph{growth factor} $G_{t_1\to t_2}\coloneqq \prod_{t=t_1}^{t_2-1} (1+\eta e^{-c\bgammat}(\wspt)^2)$.
    Taking logs, we obtain
    \begin{equation} \label{eq:log_of_growthfactor}
        \log G_{t_1\to t_2} = \sum_{t=t_1}^{t_2-1} \log\big(1+\eta e^{-c\bgammat}(\wspt)^2 \big) \leq \eta \sum_{t=t_1}^{t_2-1} e^{-c\bgammat}(\wspt)^2,
    \end{equation}
    where we used $\log(1+x)\leq x$.
    We will now upper bound the last term in \Eqref{log_of_growthfactor}.
    Recall that $\tfrac{1}{2}e^{-z}\leq 1-\psi(z)\leq e^{-z}$ for $z\geq 0$.
    Using $\bgammat \lesssim \log\log(d)$ and $\lambda \ll \log^{-1}(d)$, we have
    \begin{equation*}
        \frac{1}{2}e^{-c\bgammat} \leq 1-\lambda-\psi(c\bgammat) \leq e^{-c\bgammat}.
    \end{equation*}
    From \Defref{phase2_hypothesis_wsp} we then obtain
    \begin{equation*}
        \frac{\eta}{2}(1-o(1))\cdot \wspt e^{-c\bgammat} \leq \wsptplus - \wspt \leq \eta (1+o(1)) \cdot \wspt e^{-c\bgammat}.
    \end{equation*}
    Using the difference of squares formula,
    \begin{equation*}
        (\wsptplus)^2 - (\wspt)^2 = (\wsptplus+\wspt)(\wsptplus - \wspt) = \big(2\wspt + (\wsptplus-\wspt)\big) \cdot (\wsptplus-\wspt).
    \end{equation*}
    For the lower bound,
    \begin{equation} \label{eq:diff_of_squares_lb}
        (\wsptplus)^2 - (\wspt)^2 \geq 2\wspt ( \wsptplus-\wspt) \geq \eta(1-o(1))\cdot (\wspt)^2 e^{-c\bgammat}.
    \end{equation}
    For the upper bound,
    \begin{align*}
        (\wsptplus)^2 - (\wspt)^2 &= 2\wspt(\wsptplus-\wspt)+(\wsptplus-\wspt)^2 \\
        &\leq 2\eta (1+o(1))\cdot(\wspt)^2e^{-c\bgammat}+\eta^2(1+o(1))^2\cdot (\wspt)^2 e^{-2c\bgammat},
    \end{align*}
    wherein the latter term is lower-order by $\eta,e^{-\bgammat}\ll 1$ to obtain
    \begin{equation} \label{eq:diff_of_squares_ub}
        (\wsptplus)^2 - (\wspt)^2 \leq 2\eta(1+o(1))\cdot (\wspt)^2 e^{-c\bgammat}.
    \end{equation}
    Combining \Eqref{diff_of_squares_lb} and \Eqref{diff_of_squares_ub} gives
    \begin{equation} \label{eq:diff_of_squares}
        \eta(1-o(1))\cdot (\wspt)^2 e^{-c\bgammat} \leq (\wsptplus)^2 - (\wspt)^2 \leq 2\eta(1+o(1))\cdot (\wspt)^2 e^{-c\bgammat}.
    \end{equation}
    In particular, $\eta (\wspt)^2 e^{-c\bgammat}  \leq 2((\wsptplus)^2 - (\wspt)^2)$.
    Returning now to \Eqref{log_of_growthfactor}, we have the telescoping sum
    \begin{align*}
        \eta \sum_{t=t_1}^{\TI+\TII-1} (\wspt)^2 e^{-c\bgammat} &\leq 2 \sum_{t=t_1}^{\TI+\TII-1}(\wsptplus)^2 - (\wspt)^2 = 2 ((\wsp^{(\TI+\TII)})^2-(\wsp^{(t_1)})^2),
    \end{align*}
    in other words $\log G_{t_1\to\TI+\TII}\leq 2((\wsp^{(\TI+\TII)})^2-(\wsp^{(t_1)})^2)$.
    Using $(\wsp^{(\TI+\TII)})^2\lesssim \log\log(d)$ we have $G_{t_1\to \TI+\TII}\lesssim \exp\big(C'((\wsp^{(\TI+\TII)})^2-(\wsp^{(t_1)})^2)\big) = \log^{C'}(d)$ (recalling that $C>0$ is chosen sufficiently large).
    Therefore,
    \begin{align*}
        \norm{\bwsig^{(\TI+\TII)}}&\lesssim \norm{\bwsig^{(\TI)}} G_{\TI\to \TI+\TII} + \sum_{t=\TI}^{\TI+\TII-1} \eta \theta\log^{-1}(d)d^{-1/2} G_{t\to \TI+\TII} \\
        &\lesssim \log^{C'}(d)\norm{\bwsig^{(\TI)}}+ \TII\eta \theta \log^{C'-1}(d) d^{-1/2} \\
        &\lesssim \log^{C'}(d)\norm{\bwsig^{(\TI)}} + \log^{C'}(d)\theta d^{-1/2}.
    \end{align*}
    Substituting $\norm{\bwsig^{(\TI)}}\lesssim \theta\log^{2C}(d)d^{-1/2}$ from \Remref{phase2_basecase}, we find $\norm{\bwsig^{(\TI+\TII)}} \lesssim \theta \log^{3C}(d)d^{-1/2}$, as desired.
    The $\bwopp$ result follows by a similar analysis.

    \Lemref{phase2_norms_wperp} and \Lemref{phase2_norms_wperp_inf} follow in the same way as \Lemref{phase2_norms_wsig}, except with initial conditions $\norm{\bwperp^{(\TI)}}\lesssim \theta$ and $\norm{\bwperp^{(\TI)}}_\infty\lesssim \theta\log^{4C}(d)d^{-1/2}$, respectively, from \Remref{phase2_basecase}.
    We obtain the lower bound in \Lemref{phase2_norms_wperp} using \Defref{phase2_hypothesis_wperp} and the initial condition $\norm{\bwperp^{(\TI)}}\gtrsim \theta$ from \Remref{phase2_basecase}.

    For \Lemref{phase2_norms_w}, we infer from \Lemref{phase2_norms_wsig} and \Lemref{phase2_norms_wperp} that $\norm{\bwsigt}, \norm{\bwoppt} \ll \norm{\bwspt}$. Therefore, we have
    \begin{align*}
        |\norm{\bwt}^2 - \norm{\bwspt}^2 | \lesssim \norm{\bwperpt}^2 \lesssim \theta^2\log^{2C'}(d).
    \end{align*}
    But $\norm{\bwspt}^2\gtrsim \theta^2\log^{2C}(d)$ with $C'<C-3$ such that
    \begin{equation*}
        \frac{\theta^2\log^{2C'}(d)}{\norm{\bwspt}^2} \lesssim \log^{2(C'-C)}(d) \ll \log^{-6}(d).
    \end{equation*}
    In particular, $\norm{\bwt} = \left(1\pm o(\log^{-2}(d)) \right)\cdot \norm{\bwspt}$ as desired.

    Thus, we have proved all parts of the lemma.
\end{proof}

Next, we present the partner lemma to \Lemref{phase2_norms} which controls behavior of the scalar weight $a$ for any neuron $(a,\bw)$.
\begin{lemma} \label{lem:phase2_norms_scalar}
    Suppose the Phase II scalings are satisfied (\Defref{phase2_scalings}).
    Under the event $\Etrain$ (\Defref{events_train}), if $f_\rho$ obeys the Phase II inductive hypothesis for all iterations $\TI\leq k\leq t$ (\Defref{phase2_hypothesis}), then the following conditions hold on all neurons $(\at,\bwt)$:
    \crefalias{enumi}{lemma}
    \begin{enumerate}[label=(\roman*), ref=\thelemma(\roman*)]
        \item \label{lem:phase2_norms_a} $|\at| = \left( 1\pm o(\log^{-2}(d))\right) \cdot \norm{\bwspt} \lesssim (\log\log(d))^{1/2}. $%
        \item \label{lem:phase2_norms_sgn} $\sgn(\at)=\sgn(\azero)=\sgn(\wspt)$.
    \end{enumerate}
\end{lemma}
\begin{proof}
    For \Lemref{phase2_norms_a}, for the upper bound we have $|\at|\leq \norm{\bwt} = (1\pm o(\log^{-2}(d))\cdot \norm{\bwspt}$ by \Defref{phase2_hypothesis_a} and \Lemref{phase2_norms_w}.
    For the lower bound, we denote as shorthand $E(t)\coloneqq\norm{\bwt}^2 - (\at)^2$ and note that $E(\TI)\ll \log^{-2}(d)\cdot \norm{\bw^{(\TI)}}^2$ by \Remref{phase2_basecase} (see \Lemref{phase1_norms_a}).
    Under the event $\Etrain$, by \Lemref{balanced_iv} we have $E(t)\leq 10\eta^2 (\at)^2 + E(t-1)$ for all iterations $t\ll d^C$.
    Moreover, by \Defref{phase2_hypothesis_a} we have $(a^{(k)})^2\leq \norm{\bw^{(k)}}^2$ for all $k \leq t$.
    In particular,
    \begin{equation} \label{eq:phase2_et_helper}
        E(t) \leq E(\TI) + \sum_{k=\TI}^{t-1} 10\eta^2 (a^{(k)})^2 \lesssim \log^{-2}(d)\cdot \norm{\bw^{(\TI)}}^2+(t-\TI)\eta^2 \norm{\bwt}^2.
    \end{equation}
    where the last inequality uses the inductive step.
    By $t\lesssim \log(d)\eta^{-1}$ and $\eta\ll \log^{-3}(d)$ we have $t\eta^2\ll \log^{-2}(d)$.
    Using \Lemref{phase2_norms_w} again with the property that $\norm{\wspt}$ is increasing during Phase II (established in the proof of \Lemref{phase2_norms}), we have $\norm{\bw^{(\TI)}} \asymp \norm{\bwsp^{(\TI)}} \leq \norm{\bwspt} \asymp \norm{\bwt}$ such that $E(t)\ll \log^{-2}(d)\cdot \norm{\bwt}^2$.
    Hence,
    \begin{equation*}
        (\at)^2 = \norm{\bwt}^2 - E(t) = (1-o(\log^{-2}(d))\cdot \norm{\bwt}^2,
    \end{equation*}
    so we have shown $|\at|=\left(1-o(\log^{-2}(d))\right) \cdot \norm{\bwt}$,\footnote{Here we used the Taylor expansion $\sqrt{1-x}=1-x/2-O(x^2)$ for $x\ll 1$ such that $\sqrt{1-o(\log^{-2}(d))}=1-o(\log^{-2}(d))$.} and applying \Lemref{phase2_norms_w} we obtain the desired bound $|\at|=\left(1-o(\log^{-2}(d))\right) \cdot \norm{\bwspt}$.
    The remainder of the result follows from \Lemref{phase2_norms_w}.

    For \Lemref{phase2_norms_sgn}, the condition $\sgn(\at)=\sgn(\azero)$ follows identically to \Lemref{phase1_norms_sgn} (using only $m\gg d\log^2(d)$ and $\eta\ll 1$).
    Moreover, we have $\sgn(a^{(\TI)})=\sgn(\wsp^{(\TI)})$ by \Remref{phase2_basecase} and monotonic growth of $\wsp$ in Phase II --- increasing if $\wsp > 0$ and decreasing if $\wsp < 0$ --- implies that $\sgn(\wsp)$ does not change, giving the second part of \Lemref{phase2_norms_sgn}.
    
    Thus, we have proved all the parts of the lemma.
\end{proof}

Recall the majority group is denoted $\Xmaj \coloneqq \{\bx\in\{\pm 1\}^d:y(\bx)=x_3\}$ and the minority group is denoted $\Xmin \coloneqq \{\bx\in\{\pm 1\}^d:y(\bx)=-x_3\}$.
We will now show that the unnormalized margin $\gamma(\bx)\coloneqq y(\bx)f_\rho(\bx)$ concentrates tightly about certain average values of $\pm \wsp^2$, where the sign is positive for the majority group and negative for the minority group.
Recall the margins $\gamma_+^{(t)}$, $\gamma_-^{(t)}$, and $\bar{\gamma}^{(t)}$ are defined in \Eqref{margin_def}.
\begin{lemma} \label{lem:phase2_margin}
    Suppose the Phase II scalings are satisfied (\Defref{phase2_scalings}).
    Under the events $\Etrain$ and $\Etest$ (\Defref{events_train} and \Defref{events_test}), if $f_\rho$ obeys the Phase II inductive hypothesis for all iterations $\TI\leq k \leq t$ (\Defref{phase2_hypothesis}), then
    \begin{equation*}
        \gamma(\bx) =
        \begin{cases}
            \left(1 \pm o(\log^{-2}(d))\right) \cdot \bgammat & \text{ if } \bx\in \Xmaj \text{ and satisfies $\Etest$ }\\
            \left(-1\pm o(\log^{-2}(d))\right) \cdot \bgammat & \text{ if } \bx\in \Xmin \text{ and satisfies $\Etest$.}
        \end{cases}
    \end{equation*}
\end{lemma}
\begin{proof}
    Let us omit $(t)$ superscripts for clarity.
    By \Lemref{phase2_norms} we have $\norm{\bwsp}\gtrsim \theta\log^C(d)$, while \\$\norm{\bwsig},\norm{\bwopp},\norm{\bwperp}\lesssim \theta\log^{C'}(d)$ where $C'<C-3$.
    Hence, following \Lemref{end_of_phase1} we have $\sgn(\bw^\top\bx)=\sgn(\wsp x_3)$ when $\bx$ satisfies $\Etest$.
    Moreover, by definition of the event $\Etest$, we have
    \begin{equation*} 
        |\bw^\top\bx - \wsp x_3| \lesssim |\bwperp^\top\bx| < C\norm{\bwperp}\log^{1/2}(d) \lesssim \theta\log^{C-5/2}(d),
    \end{equation*}
    for all neurons $(a,\bw)$, where we used $\norm{\bw_{1:2}}\ll \norm{\bwperp}$ by \Lemref{phase2_norms}.
    In particular, for $r \in \{\pm 1\}$, applying the triangle inequality gives us
    \begin{equation*}
        \bigg|\frac{1}{p}\sum_{(a,\bw)\in S^r} \wsp (\bw^\top\bx-\wsp x_3)\bigg| \lesssim \frac{1}{p} \sum_{(a,\bw)\in S^r} |\wsp|\cdot \theta\log^{C-5/2}(d).
    \end{equation*}
    Using again the fact that $|\wsp|\gtrsim \theta\log^C(d)$, we have $|\wsp|\cdot \theta\log^{C-5/2}(d) \ll (\wsp)^2\cdot \log^{-2}(d)$.
    Hence
    \begin{equation} \label{eq:margin_ub}
        \bigg|\frac{1}{p}\sum_{(a,\bw)\in S^r} \wsp (\bw^\top\bx-\wsp x_3)\bigg| \ll \gamma_{r}\cdot\log^{-2}(d).
    \end{equation}
    
    Next, since the Phase II dynamics of $|\wsp|$ depend only on $|\wspzero|$ and the shared margins $\{\bar{\gamma}^{(s)}\}_{s\leq t}$, they are independent of $\sgn(\azero)$.
    Treating the $\{\wsp\}$ over $1\leq j \leq p$ as fixed, the only randomness is over the partition $\{S^+,S^-\}$; $\bar{\gamma}$ is deterministic and $\gamma_+,\gamma_-$ are sums of independent bounded random variables with $\E[\gamma_+],\E[\gamma_-]=\bar{\gamma}$.
    Under $\Etrain$, using \Defref{events_gamma_t}, we therefore have
    \begin{equation*}
        |\gamma_+ - \bar{\gamma}| \lesssim \bar{\gamma} \sqrt{\frac{\log(d)}{p}} \ll \bar{\gamma} \cdot \log^{-2}(d),
    \end{equation*}
    where we used $p\gg \log^5(d)$.
    The same holds for $\gamma_-$ such that
    \begin{align}\label{eq:gammat_concentration}
        \gamma_+,\gamma_- &= \left(1\pm o(\log^{-2}(d))\right) \cdot \bar{\gamma}.
    \end{align}
    
    Let us now consider each possible pairing of the label and spurious feature.
    \begin{enumerate}
        \item In the first case, we have $y(\bx)=x_3=1$.
        Since \Lemref{phase2_norms_sgn} implies that $\sgn(a)=\sgn(\wsp)$, we have $\sigma(\bw^\top\bx)=0$ if and only if $(a,\bw)\in S^-$.
        In particular,
        \begin{equation*}
            \gamma(\bx)=\frac{1}{p}\sum_{(a,\bw)\in S^+} a\bw^\top\bx.
        \end{equation*}
        By \Lemref{phase2_norms_a} we have $a= \left(1\pm o(\log^{-2}(d))\right)\cdot \wsp$ such that
        \begin{equation*}
             \gamma(\bx)= \left(1\pm o(\log^{-2}(d))\right) \cdot \frac{1}{p}\sum_{(a,\bw)\in S^+} \wspt \bw^\top\bx.
        \end{equation*}
        Using $x_3=1$ and \Eqref{margin_ub}, we have
        \begin{equation*}
            \gamma(\bx) = \left(1\pm o(\log^{-2}(d))\right) \cdot \left(\gamma_+ + \frac{1}{p}\sum_{(a,\bw)\in S^+} \wsp (\bw^\top\bx-\wsp x_3)\right) = \left(1\pm o(\log^{-2}(d))\right) \cdot \gamma_+.
        \end{equation*}
        Finally using \Eqref{gammat_concentration} we have $\gamma(\bx)=\left(1\pm o(\log^{-2}(d))\right) \cdot \bar{\gamma}$ as desired.
        \item In the second case, we have $y(\bx)=x_3=-1$.
        Since \Lemref{phase2_norms_sgn} implies that $\sgn(a)=\sgn(\wsp)$, we have $\sigma(\bw^\top\bx)=0$ if and only if $(\at,\bwt)\in S^+$.
        In particular,
        \begin{equation*}
            \gamma(\bx)=-\frac{1}{p}\sum_{(a,\bw)\in S^-} a \bw^\top\bx.
        \end{equation*}
        By \Lemref{phase2_norms_a} we have $a = \left(1\pm o(\log^{-2}(d))\right)\cdot \wsp$ such that
        \begin{equation*}
             \gamma(\bx)=-\left(1\pm o(\log^{-2}(d))\right) \cdot \frac{1}{p}\sum_{(a,\bw)\in S^-} \wsp \bw^\top\bx.
        \end{equation*}
        Using $x_3=-1$ and \Eqref{margin_ub}, we have
        \begin{equation*}
            \gamma(\bx) = \left(1\pm o(\log^{-2}(d))\right) \cdot \left(\gamma_- +\frac{1}{p}\sum_{(a,\bw)\in S^-} \wsp (\bw^\top\bx-\wsp x_3)\right) = \left(1\pm o(\log^{-2}(d))\right) \cdot \gamma_-.
        \end{equation*}
        Finally using \Eqref{gammat_concentration} we have $\gamma(\bx)=\left(1\pm o(\log^{-2}(d))\right) \cdot \bar{\gamma}$ as desired.
        \item In the third case, we have $y(\bx)=-x_3=1$.
        Since \Lemref{phase2_norms_sgn} implies that $\sgn(a)=\sgn(\wsp)$, we have $\sigma(\bw^\top\bx)=0$ if and only if $(a,\bw)\in S^+$.
        In particular,
        \begin{equation*}
            \gamma(\bx)=\frac{1}{p}\sum_{(a,\bw)\in S^-} a \bw^\top\bx.
        \end{equation*}
        By \Lemref{phase2_norms_a} we have $a = \left(1\pm o(\log^{-2}(d))\right)\cdot \wsp$ such that
        \begin{equation*}
             \gamma(\bx)=\left(1\pm o(\log^{-2}(d))\right) \cdot \frac{1}{p}\sum_{(a,\bw)\in S^-} \wsp \bw^\top\bx.
        \end{equation*}
        Using $x_3=-1$ and \Eqref{margin_ub}, we have
        \begin{equation*}
            \gamma(\bx) = \left(1\pm o(\log^{-2}(d))\right) \cdot \left(-\gamma_- +\frac{1}{p}\sum_{(a,\bw)\in S^-} \wsp (\bw^\top\bx-\wsp x_3)\right) = \left(-1\pm o(\log^{-2}(d))\right)\cdot \gamma_-.
        \end{equation*}
        Finally using \Eqref{gammat_concentration} we have $\gamma(\bx)=\left(-1\pm o(\log^{-2}(d))\right) \cdot \bar{\gamma}$ as desired.
        \item In the fourth case, we have $y(\bx)=-x_3=-1$.
        Since \Lemref{phase2_norms_sgn} implies that $\sgn(a)=\sgn(\wsp)$, we have $\sigma(\bw^\top\bx)=0$ if and only if $(a,\bw)\in S^-$.
        In particular,
        \begin{equation*}
            \gamma(\bx)=-\frac{1}{p}\sum_{(a,\bw)\in S^+} a \bw^\top\bx.
        \end{equation*}
        By \Lemref{phase2_norms_a} we have $a = \left(1\pm o(\log^{-2}(d))\right)\cdot \wsp$ such that
        \begin{equation*}
             \gamma(\bx)=-\left(1\pm o(\log^{-2}(d))\right) \cdot \frac{1}{p}\sum_{(a,\bw)\in S^+} \wsp \bw^\top\bx.
        \end{equation*}
        Using $x_3=1$ and \Eqref{margin_ub}, we have
        \begin{equation*}
            \gamma(\bx) = -\left(1\pm o(\log^{-2}(d))\right) \cdot \left(\gamma_+ +\frac{1}{p}\sum_{(a,\bw)\in S^+} \wsp (\bw^\top\bx-\wsp x_3)\right) = \left(-1\pm o(\log^{-2}(d))\right) \cdot \gamma_+.
        \end{equation*}
        Finally using \Eqref{gammat_concentration} we have $\gamma(\bx)=\left(-1\pm o(\log^{-2}(d))\right) \cdot \bar{\gamma}$ as desired.
    \end{enumerate}
    This completes the proof of the lemma.
\end{proof}

\subsection{Phase II Inductive Step} \label{sec:phase2_inductive_step}
In this section, we show the Phase II inductive step and ultimately obtain the Phase II result in \Propref{phase2}.
We first characterize the growth of $\bwsp$, constituting the inductive step for \Defref{phase2_hypothesis_wsp}.
\begin{lemma} \label{lem:wsp_phase2_induction}
    Suppose the Phase II scalings are satisfied (\Defref{phase2_scalings}).
    Under the event $\Etrain$ (\Defref{events_train}), if $f_\rho$ obeys the Phase II inductive hypothesis for all iterations $\TI\leq k\leq t$ (\Defref{phase2_hypothesis}), then we have
    \begin{equation*}
        \wsptplus - \wspt = \eta(1 \pm o(1)) \cdot \wspt \left(1-\lambda - \psi(\bgammat)\right)
    \end{equation*}
    for all neurons $(\atplus,\bwtplus)$.
\end{lemma}
\begin{proof}
    We first compute the order of the intermediate term $\epswt$ defined in \Lemref{sp_l0}.
    By \Lemref{v_ub}, we have
    \begin{align*}
        |\epswt| &\lesssim \frac{\norm{\bwt_{1:2}}}{\norm{\bwperpt}}\exp\left(\frac{-\norm{\bwspt}^2+2\sqrt{2}\norm{\bwt_{1:2}}\norm{\bwspt}}{2\norm{\bwperpt}^2}\right) + \frac{\norm{\bwperp}_3^3}{\norm{\bwperp}^3} \\
        &\leq \frac{\norm{\bwt_{1:2}}}{\norm{\bwperpt}}\exp\left(\frac{2\sqrt{2}\norm{\bwt_{1:2}}\norm{\bwspt}}{2\norm{\bwperpt}^2}\right) + \frac{\norm{\bwperp}_\infty}{\norm{\bwperp}},
    \end{align*}
    where we used the basic inequalities $-\norm{\bwspt}^2\leq 0$ and $\norm{\bwperpt}_3^3\leq \norm{\bwperpt}_\infty \norm{\bwperpt}^2$.
    By \Lemref{phase2_norms_wsp}, \Lemref{phase2_norms_wsig}, \Lemref{phase2_norms_wperp}, and \Lemref{phase2_norms_wperp_inf} we have $\norm{\bwspt}\lesssim (\log\log(d))^{1/2}$, $\norm{\bwt_{1:2}}\lesssim \theta\log^{3C}(d)d^{-1/2}$, $\theta\log^{-C'}(d)\lesssim \norm{\bwperpt}\lesssim \theta\log^{C'}(d)$ for $C'<C-3$, and $\norm{\bwperpt}_\infty\lesssim \theta\log^{5C}(d)d^{-1/2}$.
    Using $e^u\lesssim 1+u$ for $0<u\ll 1$,\footnote{This can be seen via the following Taylor expansion: $e^u\leq 1+u+Cu^2\lesssim 1+u$ for $0<u\ll 1$.} we have
    \begin{align}
        |\epswt| &\lesssim \log^{4C}(d)d^{-1/2}\cdot \exp(C\log^{6C}(d)d^{-1/2}) + \log^{6C}(d)d^{-1/2} \nonumber \\
        &\lesssim \log^{4C}(d)d^{-1/2}\cdot (1+\log^{6C}(d)d^{-1/2})+\log^{6C}(d)d^{-1/2} \nonumber \\
        &\lesssim \log^{10C}(d)d^{-1/2} \nonumber \\
        &= o(1). \label{eq:wsp_phase2_eps}
    \end{align}
    By the definition of gradient descent, we have
    \begin{align}
        \wsptplus &= \wspt - \eta\partial_{\wspt}\widehat{L}_\rho \nonumber\\
        &= \wspt - \eta \frac{\bwsp^{(t)\top}\nabla_{\bwt} \widehat{L}_\rho}{\wspt} \nonumber\\
        &= \wspt - \eta \frac{\bwsp^{(t)\top}\nabla_{\bwt} L_\rho + \bwsp^{(t)\top}(\nabla_{\bwt}\widehat{L}_\rho - \nabla_{\bwt} L_\rho)}{\wspt}. \label{eq:wsp_phase2_proj}
    \end{align}
    By \Lemref{phase2_margin}, the assumption of \Lemref{sp_lp} is satisfied with $\epsilon=\log^{-2}(d)$ and $\gamma=\bgammat$.
    Therefore, we have
    \begin{multline*}
        \Big| -\bwsp^{(t)\top} \nabla_{\bwt} L_\rho - \at\wspt \left((1-\lambda-\psi(\bgammat))\cdot\left( 1+\frac{\epswt}{2}\right) + \lambda\psi(\bgammat)\epswt \right) \Big| \\ \lesssim |\at||\wspt|(\log^{-2}(d)\bgammat+d^{-C}).
    \end{multline*}
    We will now use the above to show that 
    \begin{align}\label{eq:wsp_phase2_lrhofinal}
        -\bwsp^{(t)\top} \nabla_{\bwt} L_\rho = \at \wspt (1 - \lambda - \psi(\bgammat))(1 \pm o(1)).
    \end{align}
    To do this, it suffices to show that three terms are lower-order in the following sense:
    \begin{align*}
        |\epswt| &= o(1) \\
        \lambda \psi(\bgammat) |\epswt| &\ll 1 - \lambda - \psi(\bgammat) \\
        \log^{-2}(d) \bgammat + d^{-C} &\ll 1 - \lambda - \psi(\bgammat).
    \end{align*}
    The first relation above follows directly from~\Eqref{wsp_phase2_eps}.
    To show the second relation, we use the fact that $1-\psi(z)\asymp e^{-z}$ for $z\geq 0$ with $\bgammat\lesssim \log\log(d)$ by \Lemref{phase2_norms_margin} and $\lambda \ll \log^{-1}(d)$. From this, we have
    \begin{equation*}
        1-\lambda-\psi(\bgammat)\asymp e^{-\bgammat} - \lambda \gtrsim \log^{-1}(d).
    \end{equation*}
    On the other hand, $\lambda \psi(\bgammat) |\epswt| \ll \log^{-1}(d) \log^{10C}(d) d^{-1/2} \ll \log^{-1}(d)$; therefore, the second relation holds.
    Finally, to show the third relation, note that $\log^{-2}(d)\bgammat+d^{-C}\ll \log^{-1}(d)$.
    Putting these together directly yields~\Eqref{wsp_phase2_lrhofinal}. 
    Further, plugging in the expression in \Lemref{phase2_norms_a} for $\at$ yields
     \begin{equation*}
        -\bwsp^{(t)\top}\nabla_{\bwt} L_\rho = \at\wspt\cdot (1 \pm o(1))\cdot (1-\lambda-\psi(\bgammat)) = (\wspt)^2 \cdot (1 \pm o(1))\cdot (1-\lambda-\psi(\bgammat))
    \end{equation*}
    where we also used $\sgn(\at)=\sgn(\wspt)$ by \Lemref{phase2_norms_sgn}.
    Next, applying \Lemref{empirical_concentration_i} with $m\gg d\log^{6}(d)\theta^{-2}$ and $i=3$ we have
    \begin{equation*}
        |\partial_{\wspt} \widehat{L}_\rho - \partial_{\wspt} L_\rho| \ll |\at|\theta\log^{-2}(d)d^{-1/2}\lesssim \theta\log^{-1}(d)d^{-1/2},
    \end{equation*}
    where we used $|\at|\lesssim (\log\log(d))^{1/2}$ by \Lemref{phase2_norms_a}.
    Hence, we have
    \begin{equation*}
       |\bwsp^{(t)\top}(\nabla_{\bwt} \widehat{L}_\rho-\nabla_{\bwt} L_\rho)| \lesssim |\wspt| \theta\log^{-1}(d)d^{-1/2}.
    \end{equation*}
    We now justify that this term is lower-order compared to $(\wspt)^2(1-\lambda-\psi(\bgammat))$.
    By \Lemref{phase2_norms_wsp} we have that $|\wspt|\gtrsim \theta\log^C(d)$.
    Using $(1-\lambda-\psi(\bgammat))\gtrsim \log^{-1}(d)$ again, we compare $\theta\log^{C-1}(d)$ to $\theta\log^{-1}(d)d^{-1/2}$, and the latter is clearly lower-order.
    Returning to \Eqref{wsp_phase2_proj}, we have
    \begin{equation*}
        \wsptplus - \wspt = \eta \frac{(\wspt)^2\cdot (1 \pm o(1))\cdot (1-\lambda-\psi(\bgammat))}{\wspt} = \eta(1 \pm o(1))\cdot \wspt(1-\lambda-\psi(\bgammat)).
    \end{equation*}
    This completes the proof of the lemma.
\end{proof}

We will now upper bound the growth of $\bwsig$ and $\bwopp$, constituting the inductive step for \Defref{phase2_hypothesis_wsig} and \Defref{phase2_hypothesis_wopp}.
\begin{lemma} \label{lem:sigopp_phase2_induction}
    Suppose the Phase II scalings are satisfied (\Defref{phase2_scalings}).
    Under the event $\Etrain$ (\Defref{events_train}), if $f_\rho$ obeys the Phase II inductive hypothesis for all iterations $\TI\leq k\leq t$ (\Defref{phase2_hypothesis}), then for any constant $c<1$ we have
    \begin{align*}
        \norm{\bwsigtplus - \bwsigt} &\lesssim \eta e^{-c\bgammat}\Big(\max_{(a,\bw)}(\wspt)^2\norm{\bwsigt}\Big) + \eta \theta\log^{-1}(d)d^{-1/2} \\
        \norm{\bwopptplus - \bwoppt} &\lesssim \eta e^{-c\bgammat}\Big(\max_{(a,\bw)}(\wspt)^2\norm{\bwoppt}\Big) + \eta \theta\log^{-1}(d)d^{-1/2}
    \end{align*}
    for all neurons $(\atplus,\bwtplus)$.
\end{lemma}
\begin{proof}
    By the definition of gradient descent, we have
    \begin{align}
        \bwsigtplus &= \bwsigt - \eta\nabla_{\bwsigt}\widehat{L}_\rho \nonumber\\
        &= \bwsigt - \eta \frac{\bwsig^{(t)\top}\nabla_{\bwt} \widehat{L}_\rho}{\norm{\bwsigt}^2} \bwsigt \nonumber\\
        &= \bwsigt - \eta \frac{\bwsig^{(t)\top}\nabla_{\bwt} L_\rho + \bwsig^{(t)\top}(\nabla_{\bwt} \widehat{L}_\rho-\nabla_{\bwt} L_\rho)}{\norm{\bwsigt}^2} \bwsigt . \label{eq:wsig_phase2_proj}
    \end{align}
    By \Lemref{phase2_margin}, the assumption of \Lemref{sigopp_lp} is satisfied with $\gamma=\bgammat$.
    Hence, we have
    \begin{equation*}
        -\bwsig^{(t)\top}\nabla_{\bw}L_\rho \lesssim |\at|\norm{\bwsigt}\Big(e^{-c\bgammat} \Big(\max_{(a,\bw)}|\at|\norm{\bwsigt}\Big) +\P_{\bxi}\left(|\bw^{(t)\top}\bxi+\bw^{(t)\top}\be_3| \leq \sqrt{2}\norm{\bwsigt}\right) + d^{-C}\Big).
    \end{equation*}
    For the second term, by Hoeffding's inequality we have
    \begin{equation*}
        \P_{\bxi}\left(|\bw^{(t)\top}\bxi+\bw^{(t)\top}\be_3| \leq \sqrt{2}\norm{\bwsigt} \right) \lesssim \exp\left( \frac{-(\wspt)^2}{C\norm{\bwperpt}^2} \right) \lesssim \exp\left(\frac{-\log^{6}(d)}{C}\right),
    \end{equation*}
    where we used $|\wspt|\gtrsim \theta\log^C(d)$ by \Lemref{phase2_norms_wsp} and $\norm{\bwperpt}\lesssim \theta\log^{C'}(d)$ for $C'<C-3$ by \Lemref{phase2_norms_wperp}.
    Moreover, applying \Lemref{empirical_concentration_i} with $m\gg d\log^{6}(d)\theta^{-2}$ we have that for any $i$,
    \begin{equation*}
        |\partial_{\wit} \widehat{L}_\rho - \partial_{\wit} L_\rho| \ll |\at|\theta \log^{-2}(d)d^{-1/2}\lesssim \theta \log^{-1}(d)d^{-1/2},
    \end{equation*}
    where we used $|\at|\lesssim (\log\log(d))^{1/2}$ by \Lemref{phase2_norms_a}.
    By the Cauchy-Schwarz inequality, we obtain
    \begin{equation*}
       |\bwsig^{(t)\top}(\nabla_{\bwt} \widehat{L}_\rho-\nabla_{\bwt} L_\rho)| \lesssim \norm{\bwsigt} \theta \log^{-1}(d)d^{-1/2}.
    \end{equation*}
    The $\exp\left(\frac{-\log^{6}(d)}{C}\right)$ and $d^{-C}$ terms are both lower-order compared to $\theta\log^{-1}(d)d^{-1/2}$ as $\theta\gg d^{-C/2}$.
    Returning to \Eqref{wsig_phase2_proj}, and using $|\at|\asymp|\wspt|$ by \Lemref{phase2_norms_a}, we have
    \begin{align*}
        \norm{\bwsigtplus - \bwsigt} &\lesssim \eta \frac{\norm{\bwsigt}(e^{-c\bgammat}\Big(\max_{(a,\bw)}(\wspt)^2\norm{\bwsigt}\Big) + \theta \log^{-1}(d)d^{-1/2})}{\norm{\bwsigt}} \\
        &\lesssim \eta e^{-c\bgammat}\Big(\max_{(a,\bw)}(\wspt)^2\norm{\bwsigt}\Big) + \eta \theta\log^{-1}(d)d^{-1/2}.
    \end{align*}
    The result for $\bwopp$ follows in a similar way.
    This completes the proof of the lemma.
\end{proof}

Next, we can upper bound the growth of $\norm{\bwperp}$, constituting the inductive step for \Defref{phase2_hypothesis_wperp}.
\begin{lemma} \label{lem:wperp_phase2_induction}
    Suppose the Phase II scalings are satisfied (\Defref{phase2_scalings}).
    Under the event $\Etrain$ (\Defref{events_train}), if $f_\rho$ obeys the Phase II inductive hypothesis for all iterations $\TI\leq k\leq t$ (\Defref{phase2_hypothesis}), then for any constant $c<1$ we have
    \begin{align*}
        \norm{\bwperptplus - \bwperpt} &\lesssim \eta e^{-c\bgammat}(\wspt)^2\norm{\bwperpt}+ \eta \theta \log^{-1}(d)
    \end{align*}
    for all neurons $(\atplus,\bwtplus)$.
\end{lemma}
\begin{proof}
    By the definition of gradient descent, we have
    \begin{equation*}
        \bwperptplus = \bwperpt - \eta \nabla_{\bwperpt} \widehat{L}_\rho,
    \end{equation*}
    hence by the triangle inequality,
    \begin{equation} \label{eq:wperp_phase2}
        \norm{\bwperptplus - \bwperpt} \leq \eta \norm{\nabla_{\bwperpt} L_\rho} + \eta\norm{\nabla_{\bwperpt} \widehat{L}_\rho - \nabla_{\bwperpt} L_\rho}.
    \end{equation}
    By \Lemref{phase2_margin}, the assumption of \Lemref{wperp_lp_2} is satisfied with $\gamma=\bgammat$. %
    Hence,
    \begin{equation*}
        \frac{1}{|\at|}\norm{\nabla_{\bwperp} L_\rho} \lesssim e^{-c\bgammat}|\at|\norm{\bwperpt}+\sqrt{\P_{\bxi}\left(|\bw^\top\bxi|\geq|\bw^\top\bz+\bw^\top\bs|\right)} + d^{1/2-C}.
    \end{equation*}
    For the second term, by Hoeffding's inequality we have
    \begin{equation*}
        \P_{\bxi}\left(|\bw^\top\bxi|\geq|\bw^\top\bz+\bw^\top\bs| \right)\lesssim \exp\left( \frac{-(\wspt)^2}{C\norm{\bwperpt}^2} \right) \lesssim \exp\left(\frac{-\log^{6}(d)}{C}\right),
    \end{equation*}
    where we used $|\wspt|\gtrsim \theta\log^C(d)$ by \Lemref{phase2_norms_wsp} and $\norm{\bwperpt}\lesssim \theta\log^{C'}(d)$ for $C'<C-3$ by \Lemref{phase2_norms_wperp}.
    Moreover, applying \Lemref{empirical_concentration_ii} with $m\gg d\log^{6}(d)\theta^{-2}$ we have
    \begin{equation*}
        \norm{\nabla_{\bwperpt} \widehat{L}_\rho - \nabla_{\bwperpt} L_\rho} \lesssim |\at| \theta\log^{-2}(d)\lesssim \theta\log^{-1}(d),
    \end{equation*}
    where we used $|\at|\lesssim (\log\log(d))^{1/2}$ by \Lemref{phase2_norms_a}.
    The $\exp\left(\frac{-\log^{6}(d)}{C}\right)$ and $d^{1/2-C}$ terms are both lower-order compared to $\theta\log^{-1}(d)$ as $\theta\gg d^{-C/2}$.
    Returning to \Eqref{wperp_phase2}, and using $|\at|\asymp|\wspt|$ by \Lemref{phase2_norms_a}, we have
    \begin{equation*}
        \norm{\bwperptplus - \bwperpt} \lesssim \eta e^{-c\bgammat}(\wspt)^2\norm{\bwperpt}+ \eta\theta \log^{-1}(d).
    \end{equation*}
    This completes the proof of the lemma.
\end{proof}

Finally, we can upper bound the growth of $\norm{\bwperp}_\infty$, constituting the inductive step for \Defref{phase2_hypothesis_wperp_inf}.
\begin{lemma} \label{lem:wi_phase2_induction}
    Suppose the Phase II scalings are satisfied (\Defref{phase2_scalings}).
    Under the event $\Etrain$ (\Defref{events_train}), if $f_\rho$ obeys the Phase II inductive hypothesis for all iterations $\TI\leq k\leq t$ (\Defref{phase2_hypothesis}), then for a constant $c<1$ we have
    \begin{align*}
        \norm{\bwperptplus - \bwperpt}_\infty &\lesssim \eta e^{-c\bgammat}\Big(\max_{(a,\bw)}(\wspt)^2\norm{\bwperpt}_\infty\Big) + \eta \theta\log^{-1}(d)d^{-1/2}
    \end{align*}
    for all neurons $(\atplus, \bwtplus)$.
\end{lemma}
\begin{proof}
    Let $i>3$; we will examine the growth of $|w_i|$ to obtain the desired bound.
    By the definition of gradient descent, we have
    \begin{align}
        \witplus &= \wit - \eta\partial_{\wit}\widehat{L}_\rho \nonumber\\
        &= \wit - \eta \frac{\wit\partial_{\wit} \widehat{L}_\rho}{\wit} \nonumber\\
        &= \wit - \eta \frac{\wit\partial_{\wit} L_\rho + \wit(\partial_{\wit} \widehat{L}_\rho-\partial_{\wit} L_\rho)}{\wit}. \label{eq:wi_phase2_proj}
    \end{align}
    By \Lemref{phase2_margin}, the assumption of \Lemref{wi_lp} is satisfied with $\gamma=\bgammat$. %
    Hence we have
    \begin{align*}
        -w_i\partial_{w_i} L_\rho &\lesssim |\at||\wit|\Big(e^{-c\bgammat}\Big(\max_{(a,\bw)}|\at||\wit|\Big) + d^{-C} + \Big( \nonumber \\
        &\qquad \P_{\bxi}\left(|\wit|\geq |\sqrt{2}\norm{\bwsigt}+\wspt+\bw^{(t)\top}\bxi_{\setminus i}|\right)+\P_{\bxi}\left(|\wit|\geq |\sqrt{2}\norm{\bwsigt}-\wspt+\bw^{(t)\top}\bxi_{\setminus i}|\right) \nonumber \\
        &\qquad +\P_{\bxi}\left(|\wit|\geq |\sqrt{2}\norm{\bwoppt}+\wspt+\bw^{(t)\top}\bxi_{\setminus i}|\right)+\P_{\bxi}\left(|\wit|\geq |\sqrt{2}\norm{\bwoppt}-\wspt+\bw^{(t)\top}\bxi_{\setminus i}|\right) \Big)\Big).
    \end{align*}
    For the last term, let us rewrite the probabilities in interval notation.
    Define $X\coloneqq \bw^{(t)\top}\bxi_{\setminus i}$ and
    \begin{align*}
        I_1 &\coloneqq \left[-|\wit|-\sqrt{2}\norm{\bwsigt}-\wspt, |\wit|-\sqrt{2}\norm{\bwsigt}-\wspt \right] \\
        I_2 &\coloneqq \left[-|\wit|-\sqrt{2}\norm{\bwsigt}+\wspt, |\wit|-\sqrt{2}\norm{\bwsigt}+\wspt \right] \\
        I_3 &\coloneqq \left[-|\wit|-\sqrt{2}\norm{\bwoppt}-\wspt, |\wit|-\sqrt{2}\norm{\bwoppt}-\wspt \right] \\
        I_4 &\coloneqq \left[-|\wit|-\sqrt{2}\norm{\bwoppt}+\wspt, |\wit|-\sqrt{2}\norm{\bwoppt}+\wspt \right].
    \end{align*}
    Then, we can write the last term as
    \begin{equation*}
        \P_{\bxi}(X\in I_1)+\P_{\bxi}(X\in I_2)+\P_{\bxi}(X\in I_3)+\P_{\bxi}(X\in I_4).
    \end{equation*}
    Let us upper bound $\P_{\bxi}(X\in I_1)$, and the other terms will follow similarly.
    Note that $I_1$ is on the negative real line, as the dominating term is $-\wspt$.
    Thus by Hoeffding's inequality, we have
    \begin{equation*}
        \P_{\bxi}(X\in I_1) = \P_{\bxi}\left(\bw^{(t)\top}\bxi_{\setminus i}\leq -(\wspt-\sqrt{2}\norm{\bwsigt}-|\wit|)\right)\lesssim \exp\left( \frac{-(\wspt)^2}{C\norm{\bwt_{\perp\setminus i}}^2} \right). 
    \end{equation*}
    Then, using $|\wspt|\gtrsim \theta\log^C(d)$ by \Lemref{phase2_norms_wsp} and $\norm{\bwt_{\perp\setminus i}}\lesssim \theta\log^{C'}(d)$ for $C'<C-3$ by \Lemref{phase2_norms_wperp},\footnote{While \Lemref{phase2_norms_wperp} is stated for $\bwperp$, it is straightforward to see it holds for $\bw_{\perp\setminus i}$.} we have
    \begin{equation*}
        \P_{\bxi}(X\in I_1) \lesssim \exp\left(\frac{-\log^{6}(d)}{C}\right).
    \end{equation*}
    Applying this result to each of the four intervals, and using $|\at|\asymp |\wspt|\lesssim (\log\log(d))^{1/2}$ by \Lemref{phase2_norms_a}, we obtain
    \begin{equation*}
        -\wit \partial_{\wit} L_\rho \lesssim \Bigg(e^{-c\bgammat}\Big(\max_{(a,\bw)}(\wspt)^2|\wit|\Big)
        + (\log\log(d))^{1/2}\left(\exp\left(\frac{-\log^{6}(d)}{C}\right) + d^{-C} \right) \Bigg) |\wit|.
    \end{equation*}
    Moreover, applying \Lemref{empirical_concentration_i} with $m\gg d\log^{6}(d)\theta^{-2}$ we have
    \begin{equation*}
        |\partial_{\wit} \widehat{L}_\rho - \partial_{\wit} L_\rho| \ll |\at|\theta \log^{-2}(d)d^{-1/2}\lesssim \theta\log^{-1}(d)d^{-1/2},
    \end{equation*}
    such that
    \begin{equation*}
        \wit (\partial_{\wit} \widehat{L}_\rho - \partial_{\wit} L_\rho) \lesssim |\wit|\theta\log^{-1}(d)d^{-1/2}.
    \end{equation*}
    The $\exp\left(\frac{-\log^{6}(d)}{C}\right)$ and $d^{-C}$ terms are both lower-order compared to $\theta\log^{-1}(d)d^{-1/2}$ as $\theta\gg d^{-C/2}$.
    Returning to \Eqref{wi_phase2_proj}, we have
    \begin{align*}
        |\witplus - \wit| &\lesssim \eta \frac{|\wit|\Big(e^{-c\bgammat}\Big(\max_{(a,\bw)}(\wspt)^2|\wit| \Big) + \theta\log^{-1}(d)d^{-1/2}\Big)}{\wit} \\
        &\lesssim \eta e^{-c\bgammat}\Big(\max_{(a,\bw)}(\wspt)^2|\wit|\Big) + \eta\theta\log^{-1}(d)d^{-1/2}.
    \end{align*}
    Since this holds for any $i>3$, we have
    \begin{equation*}
        \norm{\bwperptplus-\bwperpt}_\infty \lesssim \eta e^{-c\bgammat}\Big(\max_{(a,\bw)}(\wspt)^2 \norm{\bwperpt}_\infty \Big) + \eta\theta\log^{-1}(d)d^{-1/2}.
    \end{equation*}
    This completes the proof of the lemma.
\end{proof}

We are now ready to show the Phase II result.
\begin{proposition} \label{prop:phase2}
    Suppose the Phase II scalings are satisfied (\Defref{phase2_scalings}).
    Under the event $\Etrain$ (\Defref{events_train}), upon
    \begin{equation*}
        \TII \asymp
        \begin{cases*}
            \log(d)(\log\log(d))^{-1}\eta^{-1} & $\theta \asymp \textnormal{polylog}^{-1}(d)$ \\
            \log(d)\eta^{-1} & $\theta \asymp \textnormal{poly}^{-1}(d)$
        \end{cases*}
    \end{equation*}
    further iterations of online minibatch SGD under the $\ell_\rho$ loss, all neurons $(a^{(\TI+\TII)}, \bw^{(\TI+\TII)})$ obey the Phase II inductive hypothesis (\Defref{phase2_hypothesis}) with $\sgn(f_{\rho^{(\TI+\TII)}}(\bx))=x_3$ under the event $\Etest$ (\Defref{events_test}).
\end{proposition}
\begin{proof}
    Recall that the initial conditions for Phase II are given by \Remref{phase2_basecase}.
    The inductive steps for \Defref{phase2_hypothesis} are given as follows:
    \begin{enumerate}
        \item The \Defref{phase2_hypothesis_wsp} inductive step is given by \Lemref{wsp_phase2_induction}.
        \item The \Defref{phase2_hypothesis_wsig} inductive step is given by \Lemref{sigopp_phase2_induction}.
        \item The \Defref{phase2_hypothesis_wopp} inductive step is given by \Lemref{sigopp_phase2_induction}.
        \item The \Defref{phase2_hypothesis_wperp} inductive step is given by \Lemref{wperp_phase2_induction}.
        \item The \Defref{phase2_hypothesis_wperp_inf} inductive step is given by \Lemref{wi_phase2_induction}.
        \item The \Defref{phase2_hypothesis_a} inductive step is given by \Lemref{balanced_iii}.
    \end{enumerate}
    Let us now analyze the length of Phase II (by definition the number $\TII$ of iterations until $\bgammat\asymp \log\log(d)$).
    By \Eqref{diff_of_squares}, we have $(\wsptplus)^2-(\wspt)^2 \asymp \eta(\wspt)^2 e^{-c\bgammat}$.
    Combining with the fact that $S^+,S^-$ do not vary with $t$, we can obtain a recurrence for $\bgammat$ by
    \begin{align}
        \bgammatplus - \bgammat &= \frac{1}{2p}\left(\left(\sum_{(a,\bw)\in S^+} (\wsptplus)^2 - (\wspt)^2 \right) + \left( \sum_{(a,\bw)\in S^-} (\wsptplus)^2 - (\wspt)^2 \right)\right) \nonumber \\
        &\asymp \frac{\eta e^{-c\bgammat}}{2p}\left( \sum_{(a,\bw)\in S^+} (\wspt)^2 + \sum_{(a,\bw)\in S^-} (\wspt)^2 \right) \nonumber \\
        &= \frac{\eta e^{-c\bgammat}}{2} (\gammaplust + \gammaminust) \nonumber \\
        &= \eta \bgammat e^{-c\bgammat}. \label{eq:gamma_recurrence}
    \end{align}
    Write $\zt$ for $c\bgammat$ as shorthand.
    Let us now perform a continuous approximation.
    The number of iterations is
    \begin{equation*}
        \TII = \sum_t 1 = \sum_t \frac{e^{\zt}}{\eta \zt} \cdot \eta \zt e^{-\zt}.
    \end{equation*}
    This is a Riemann sum with nonuniform step sizes $\eta \zt e^{-\zt}$, wherein the quadrature error bound gives
    \begin{align*}
        \left| \TII -  \int_{z^{(\TI)}}^{\log\log(d)} \frac{1}{\eta ze^{-z}}dz \right| &\leq  \frac{1}{2}\sum_t \left| \frac{d}{dz} \frac{e^{\zt}}{\eta \zt} \right| \cdot (z^{(t+1)} - \zt)^2\\
        &\lesssim \sum_t \left| \frac{d}{dz} \frac{e^{\zt}}{\eta \zt} \right| \cdot (\eta \zt e^{-\zt})^2,
    \end{align*}
    where we applied \Eqref{gamma_recurrence} for the update rule of $\zt$.
    Substituting $\left|\frac{d}{dz} \frac{e^{\zt}}{\eta \zt }\right|=\frac{e^{\zt}|\zt-1|}{\eta(\zt)^2}$ we obtain
    \begin{align*}
        \left| \TII -  \int_{z^{(\TI)}}^{\log\log(d)} \frac{1}{\eta ze^{-z}}dz \right| &\lesssim \sum_t\frac{e^{\zt}|\zt-1|}{\eta(\zt)^2} \cdot \eta^2 (\zt)^2 e^{-2\zt} \\
        &= \frac{\eta}{2}\sum_t |\zt -1| \cdot e^{-\zt} \\
        &\lesssim \eta \TII \log\log(d),
    \end{align*}
    where we used $e^{-\zt}\leq 1$ and $\zt \lesssim \log\log(d)$.
    Since $\eta\ll (\log\log(d))^{-1}$ we have that the approximation error is $o(\TII)$.
    Hence
    \begin{equation*}
        \TII \asymp \int_{z^{(\TI)}}^{\log\log(d)} \frac{1}{\eta ze^{-z}}dz.
    \end{equation*}
    Proceeding to solve the integral, we have
    \begin{equation*}
        \TII \asymp \frac{1}{\eta}\int_{z^{(\TI)}}^{\log\log(d)} \frac{e^z}{z}dz = \frac{1}{\eta}\Big(\text{Ei}(\log\log(d))-\text{Ei}(z^{(\TI)}) \Big)
    \end{equation*}
    where
    \begin{equation*}
        \text{Ei}(z)=\int_{-\infty}^z \frac{e^t}{t}dt
    \end{equation*}
    is the exponential integral.
    The asymptotics of this object are well-understood~\citep{NIST}; in particular $\text{Ei}(z)\asymp \tfrac{e^z}{z}$ for $z\gg 1$ and $\text{Ei}(z)\asymp \log(z)$ for $z\ll 1$.
    Note $z^{(\TI)}\ll 1$ since $z^{(\TI)}\lesssim \theta^2\log^{2C}(d)$ by \Remref{phase2_basecase} and $\theta\ll\log^{-C}(d)$.
    Hence, we have
    \begin{equation*}
        \TII \asymp \frac{1}{\eta}\left( \frac{\log(d)}{\log\log(d)} + \log\left(\frac{1}{z^{(\TI)}}\right)\right).%
    \end{equation*}
    Now, using $z^{(\TI)}\asymp \theta^2\log^{2C}(d)$ by \Remref{phase2_basecase}, we have
    \begin{equation*}
        \left|\log\left(\frac{1}{z^{(\TI)}}\right)-2\log\left(\frac{1}{\theta}\right)\right|\asymp \log\log(d),
    \end{equation*}
    which is dominated by the $\tfrac{\log(d)}{\log\log(d)}$ term such that
    \begin{equation*}
        \TII \asymp \frac{1}{\eta}\left( \frac{\log(d)}{\log\log(d)} + \log\left(\frac{1}{\theta}\right)\right).
    \end{equation*}
    Finally, under the event $\Etest$ we have $\sgn(f_\rho(\bx))=x_3$ by the same argument as \Lemref{end_of_phase1}.
    This completes the proof of Phase II.
\end{proof}

We now prove \Thmref{main}, which we restate for convenience.
\begin{theorem}[Restatement of \Thmref{main}]
    There exists a large enough constant $C>0$ such that the following holds.
    If \Asmref{main} is satisfied, then with probability at least $1-d^{-C}$, upon
    \begin{equation*}
        T \asymp
        \begin{cases*}
            \log(d)(\log\log(d))^{-1}\eta^{-1} & \text{if } $\theta \asymp \textnormal{polylog}^{-1}(d)$ \\
            \log(d)\eta^{-1} & \text{if } $\theta \asymp \textnormal{poly}^{-1}(d)$
        \end{cases*}
    \end{equation*}
    iterations of online minibatch SGD under the $\ell_\rho$ loss, we have
    \begin{equation*}
        \textnormal{Acc}_{\Xmaj}(f_{\rho^{(T)}}) \geq 1-d^{-C} \qquad \textnormal{Acc}_{\Xmin}(f_{\rho^{(T)}}) \leq d^{-C}.
    \end{equation*}
\end{theorem}
\begin{proof}
    First, note that \Asmref{main} implies both the Phase I scalings (\Defref{phase1_scalings}) and the Phase II scalings (\Defref{phase2_scalings}); in particular, $\theta\ll\log^{-5C}(d)$ gives $m\gg d\log^6(d)\theta^{-2}\gg d\log^{7C}(d)$.
    By \Lemref{events}, the event $\Etrain$ holds with probability at least $1-d^{-C}$ over the initialization and minibatches, and we work on this event for the remainder of the proof.
    By \Propref{phase1} and \Propref{phase2}, at the end of Phase II we have $\sgn(f_{\rho^{(T)}}(\bx))=x_3$ for every $\bx$ satisfying $\Etest$.
    Here $T=\TI+\TII\asymp\TII$, since Phase I lasts only $\TI\asymp\log\log(d)\eta^{-1}$ iterations, which is lower-order in both regimes of $\theta$.

    It remains to bound the probability that a test point fails $\Etest$.
    Since $\Etest$ depends only on $\bxi$ (\Defref{events_test}), and $\bxi\sim\Unif(0^3\times\{\pm 1\}^{d-3})$ independently of $(\bz,\bs)$ under both $\Unif(\Xmaj)$ and $\Unif(\Xmin)$, \Lemref{events} gives that $\bx\sim\Unif(\X)$ fails $\Etest$ with probability at most $d^{-C}$ for each $\X\in\{\Xmaj,\Xmin\}$.
    For $\bx\in\Xmaj$ satisfying $\Etest$ we have $\sgn(f_{\rho^{(T)}}(\bx))=x_3=y(\bx)$, so $\textnormal{Acc}_{\Xmaj}(f_{\rho^{(T)}})\geq 1-d^{-C}$.
    For $\bx\in\Xmin$ satisfying $\Etest$ we have $\sgn(f_{\rho^{(T)}}(\bx))=x_3=-y(\bx)$, so $\textnormal{Acc}_{\Xmin}(f_{\rho^{(T)}})\leq d^{-C}$.
\end{proof}

\clearpage

\section{Technical Lemmas} \label{sec:technical_lemmas}
This section provides some technical lemmas.
\Secref{events} details certain high-probability events and demarcates each source of randomness in our proofs.
\Secref{booleans_to_gaussians} regards concentration of uniform samples from the Boolean hypercube.
\Secref{simultaneous_training} provides norm bounds and concentration for training ReLU neural networks with Lipschitz loss functions.
\Secref{misc_lemmas} derives the solution to the linear non-homogeneous recurrences of \Defref{phase1_hypothesis}.
Finally, \Secref{crossing_time} details our idealized crossing-time analysis from \Secref{constant_lambda_idealized}.

\subsection{High-Probability Events} \label{sec:events}
In this section, we specify the high-probability events which will be used as building blocks in many of our proofs.
Recall we define the positive neurons by $S^+\coloneqq \{(a,\bw):\sgn(a)=1\}$ and the ``negative neurons'' by $S^-\coloneqq \{(a,\bw):\sgn(a)=-1\}$.
\begin{definition} \label{def:events}
    Define the following events for a constant $C>0$:
    \crefalias{enumi}{definition}
    \begin{enumerate}[label=(\roman*), ref=\thedefinition(\roman*)]
        \item \label{def:events_init_j} The initialization event for the $j$-th neuron $(\azero_j, \bwzero_j)$ with $\bwzero_j\sim\Unif(\mathbb{S}^{d-1}(\theta))$ and $a_j=r_j\theta$ where $r_j\sim\Unif(\{\pm 1\})$:
        \begin{equation*}
            \Einitj \coloneqq \big\{\norm{\bwzero_{j\perp}}\asymp \theta\big\} \cap \big\{\norm{\bwzero_{j\perp}}_3^3\lesssim \theta^3 d^{-1/2}\big\} \cap \big\{\norm{\bwzero_{j\perp}}_\infty \lesssim \theta\log^{1/2}(d)d^{-1/2} \big\}.
        \end{equation*}
        \item \label{def:events_init} The initialization event over all $p$ neurons (sampled \textit{i.i.d.}):
        \begin{equation*}
            \Einit \coloneqq \left\{|S^+|,|S^-| = \left( 1 \pm o(\log^{1/2}(d)p^{-1/2}) \right) \cdot \frac{p}{2}\right\} \cap \bigcap_{j=1}^p \Einitj.
        \end{equation*}
        \item \label{def:events_batch_t} The concentration event for the $t$-th minibatch $M^{(t)}\sim P_d^m(\lambda)$:
        \begin{multline*}
            \Ebatcht \coloneqq \bigg\{ (\partial_{\wit} L_{\rhot} - \partial_{\wit} \widehat{L}_{\rhot})^2 \leq \frac{\log^2(d)}{m} \quad \forall i\in [d] \bigg\} \\
            \cap \bigg\{ (\partial_{\at} L_{\rhot} - \partial_{\at} \widehat{L}_{\rhot})^2 \leq \frac{d\log^2(d)}{m}\norm{\bwt}^2\bigg\}.
        \end{multline*}
        \item \label{def:events_batch} The concentration event over all $T$ minibatches (sampled \textit{i.i.d.}): $\Ebatch\coloneqq \bigcap_{t=1}^T \Ebatcht$.
        \item \label{def:events_gamma_t} The concentration event over margins for the $t$-th minibatch, where $\TI\leq t\leq\TI+ \TII$ (Phase II only):
        \begin{equation*}
            \Egammat \coloneqq \left\{\max(|\gammaplust - \bgammat|,|\gammaminust - \bgammat|) \leq \bgammat C' \sqrt{\frac{\log d}{p}}\right\}.
        \end{equation*}
        \item \label{def:events_gamma} The concentration event over margins over all $\TII$ minibatches in Phase II only (sampled \iid): $\Egamma \coloneqq \bigcap_{t=\TI}^{\TI+\TII} \Egammat$.
        \item \label{def:events_train} The train event over all neuron initializations and minibatches: $\Etrain \coloneqq \Einit \cap \Ebatch \cap \Egamma$.
        \item \label{def:events_test_j} The concentration event for the $j$-th neuron on a test point $\bx\sim P_d(\lambda)$:
        \begin{equation*}
            \Etestj \coloneqq \big\{|\bw_{j\perp}^\top(\bx_{\setminus i}+\bx_i)|,|\bw_{j\perp}^\top(\bx_{\setminus i}-\bx_i)| < C\norm{\bw_{j\perp}}\log^{1/2}(d) \quad \forall i > 3\big\}.
        \end{equation*}
        \item \label{def:events_test} The concentration event over all $p$ neurons on a test point $\bx\sim P_d(\lambda)$: $\Etest \coloneqq \bigcap_{j=1}^p \Etestj$.
    \end{enumerate}
\end{definition}

\paragraph{Notation.} For a point $\bx=(\bz,\bs,\bxi)$, we will equivalently write that $\bx$ satisfies $\Etest$ and $\bxi$ satisfies $\Etest$, as the event is solely a property of $\bxi$, and similarly for $\Etestj$.

To show that each of these events occur with high probability, we will use the following lemmas.
The first lemma gives the analysis for \Defref{events_init_j}.
\begin{lemma} \label{lem:phase1_basecase}
    For any neuron $(\azero, \bwzero)$, we have $\norm{\bwperpzero} \asymp \theta$, $\norm{\bwperpzero}_3^3\lesssim \theta^3 d^{-1/2}$, and $\norm{\bwzero}_\infty \lesssim \theta\log^{1/2}(d)d^{-1/2}$ with probability at least $1-d^{-C}$ for any fixed $C>0$.
\end{lemma}
\begin{proof}
    Write $\bwzero\coloneqq \theta \frac{\bu}{\norm{\bu}}$ where $\bu\sim\Nc(\bzero,\mathbf{I}_d)$.
    Choose $v = (C+1)\log d$.
    By the standard Gaussian tail bound we have $\P(|u_i|\geq \sqrt{2v})\leq 2e^{-v} = 2d^{-(C+1)}$, so a union bound over $i\in [d]$ gives
    \begin{equation*}
        \P\!\left(\exists\, i : |u_i|\geq \sqrt{2v}\right)\leq 2d^{-C}.
    \end{equation*}
    Standard $\chi^2$ concentration bounds (\eg\cite{laurent2000adaptive}) give
    \begin{equation*}
        \P\left(\norm{\bu}^2\leq d - 2\sqrt{dv} - 2v\right)\leq e^{-v} = d^{-(C+1)}.
    \end{equation*}
    On the complement of these two events, which holds with probability at least $1 - 3d^{-C}$, we have
    \begin{equation*}
        |w_i^{(0)}|\leq \theta\frac{\sqrt{2v}}{\sqrt{d - 2\sqrt{dv} - 2v}}\lesssim \theta\sqrt{\frac{\log(d)}{d}},
    \end{equation*}
    establishing the $\ell^\infty$ bound.

    We now show the desired bound for $\norm{\bwperpzero}$.
    Recall that we write $\bu_{1:3}\coloneqq(u_1,u_2,u_3,\bzero)$ and $\bu_{4:}\coloneqq (\bzero,u_4,u_5,\dots)$ as shorthand, so that
    \begin{equation*}
        \norm{\bwperpzero}= \theta\sqrt{\frac{\norm{\bu_{4:}}^2}{\norm{\bu_{1:3}}^2+\norm{\bu_{4:}}^2}}.
    \end{equation*}
    Applying $\chi^2$ concentration again, each with failure probability at most $e^{-v}=d^{-(C+1)}$, yields
    \begin{equation*}
        \norm{\bu_{1:3}}^2\in \left[3 \pm \left(2\sqrt{3v}+2v\right)\right] \quad\text{and}\quad \norm{\bu_{4:}}^2\in \left[(d-3)\pm \left(2\sqrt{(d-3)v} + 2v \right)\right].
    \end{equation*}
    This gives the sandwich inequality
    \begin{equation*}
        \theta\sqrt{\frac{(d-3)-2\sqrt{(d-3)v} - 2v}{d-2\sqrt{(d-3)v}+2\sqrt{3v}}} \leq \norm{\bwperpzero} \leq \theta \sqrt{\frac{(d-3)+2\sqrt{(d-3)v} + 2v}{d+2\sqrt{(d-3)v}-2\sqrt{3v}}}.
    \end{equation*}
    Expanding to leading order with $v = (C+1)\log d$ yields
    \begin{equation*}
        \norm{\bwperpzero} \asymp \theta \sqrt{\frac{d-3}{d}}\left(1\pm \sqrt{\frac{v}{d}}\right)\asymp \theta\left(1\pm\sqrt{\frac{\log(d)}{d}}\right) \asymp \theta.
    \end{equation*}
    
    We now show the desired bound for $\norm{\bwperpzero}_3^3$.
    Clearly we have $\norm{\bwperpzero}_3^3\leq \norm{\bwzero}_3^3$, so we can show the result for the latter.
    Applying $\chi^2$ concentration again, with failure probability at most $e^{-v}=d^{-(C+1)}$, yields $\norm{\bu}^2 \geq \tfrac{d}{2}$.
    Now, $\norm{\bu}_3^3=\sum_{i=1}^d |u_i|^3$.
    Since $u_i\sim \Nc(0,1)$, by direct integration we have $\E[|u_i|^3]=\frac{2\sqrt{2}}{\sqrt{\pi}}$, so $\E[\norm{\bu}_3^3]\asymp d$.
    Moreover, $|u_i|^3$ is subexponential with constant $\psi_1$-norm, so Bernstein's inequality implies
    \begin{equation*}
        \P(\norm{\bu}_3^3 \gtrsim d) \leq e^{-v} = d^{-(C+1)}.
    \end{equation*}
    Hence,
    \begin{equation*}
        \norm{\bwperpzero}_3^3 \leq \norm{\bwzero}_3^3\leq \theta^3 \frac{\norm{\bu}_3^3}{\norm{\bu}^3}\lesssim \theta^3\frac{d}{d^{3/2}}=\theta^3 d^{-1/2},
    \end{equation*}
    as desired.
    
    A union bound over all failure events gives total failure probability at most $6d^{-(C+1)} \leq d^{-C}$ for $d$ large enough.
    This completes the proof of the lemma.
\end{proof}

Next, we consider empirical concentration of empirical $\widehat{L}_\rho$ gradients about the population $L_\rho$ gradients (\cf\cite[Lemma B.12]{glasgow2024sgd}).
This constitutes the analysis for \Defref{events_batch_t}.
\begin{lemma} \label{lem:empirical_concentration}
    Suppose we train via online SGD with batch size $m$ under the $\ell_\rho$ loss.
    Then, for any neuron $(a,\bw)$, with probability at least $1-d^{-C}$ for any fixed $C>0$, the following hold:
    \crefalias{enumi}{lemma}
    \begin{enumerate}[label=(\roman*), ref=\thelemma(\roman*)]
        \item \label{lem:empirical_concentration_i} $(\partial_{w_i} L_\rho-\partial_{w_i}\widehat{L}_\rho)^2\leq \frac{\log^2 (d)}{m}a^2$ for all $i\in[d]$.
        \item \label{lem:empirical_concentration_ii} $\norm{\nabla_{\bw} L_\rho-\nabla_{\bw}\widehat{L}_\rho}^2\leq \frac{d\log^2 (d)}{m}a^2$.
        \item \label{lem:empirical_concentration_iii} $(\partial_{a} L_\rho-\partial_{a}\widehat{L}_\rho)^2\leq \frac{d\log^2 (d)}{m}\norm{\bw}^2$.
    \end{enumerate}
\end{lemma}
\begin{proof}
    Recall that $L_\rho=\E_{M\sim P_d^m(\lambda)}[\widehat{L}_\rho]$ and we defined $\partial_{u}L\coloneqq p\frac{\partial L}{\partial u}$ for $u\in\{a,\bw\}$.
    We also defined $\gamma(\bx) \coloneqq y(\bx) f_\rho(\bx)$ as the margin of data point $\bx$, and the composite notation $\ell_\rho(\bx) \coloneqq h(\gamma(\bx))$ where $h(\gamma) \coloneqq -2 \log (\psi(\gamma))$ (where $\psi(u)\coloneqq 1/(1+e^{-u})$ denotes the sigmoid), with $\ell^{(1)}_\rho(\bx) \coloneqq h'(\gamma(\bx))$.
    Consider \Lemref{empirical_concentration_i} and note that
    \begin{equation*}
        p\frac{\partial}{\partial w_i}\ell_\rho(\bx)=\ell^{(1)}_\rho(\bx)y(\bx)a\sigma'(\bw^\top\bx)x_i.
    \end{equation*}
    Since $\ell_\rho$ is $2$-Lipschitz, $\sigma'(z)\leq 1$ for the ReLU, and $y(\bx),x_i\in\{\pm 1\}$, we have that $\ell^{(1)}_\rho(\bx)y(\bx)a\sigma'(\bw^\top\bx)x_i$ is bounded in $[-2|a|,2|a|]$.
    Since we are using online SGD the minibatches are independent, and thus Hoeffding's inequality gives that with probability at least $1-\delta$,
    \begin{equation*}
        |\partial_{w_i} L_\rho-\partial_{w_i}\widehat{L}_\rho| \leq \sqrt{\frac{8a^2\log(1/\delta)}{m}}.
    \end{equation*}
    Setting $\delta = e^{-\log^2(d)}$ so that $\log(1/\delta) = \log^2(d)$,  squaring both sides gives
    \begin{equation*}
        (\partial_{w_i} L_\rho-\partial_{w_i}\widehat{L}_\rho)^2\leq \frac{\log^2 (d)}{m}a^2
    \end{equation*}
    with failure probability $e^{-\log^2(d)}$ for each fixed $i$. 
    A union bound over $i\in [d]$ gives total failure probability  $d\cdot e^{-\log^2(d)} \leq d^{-C}$ for any fixed $C>0$ and $d$ large enough.
    This immediately gives \Lemref{empirical_concentration_ii}.
    
    For \Lemref{empirical_concentration_iii}, we have
    \begin{equation*}
        p\frac{\partial}{\partial a} \ell_\rho(\bx) = \ell^{(1)}_\rho(\bx)y(\bx)\sigma(\bw^\top\bx).
    \end{equation*}
    We have that $\ell_\rho$ is $2$-Lipschitz, $\sigma(z)\leq z$ for the ReLU, $\bw^\top\bx \leq \sqrt{d}\norm{\bw}$ by the Cauchy-Schwarz inequality, $y(\bx)\in\{\pm 1\}$, and $\bx\in\{\pm 1\}^d$.
    Hence $\ell^{(1)}_\rho(\bx)y(\bx)\sigma(\bw^\top\bx)$ is bounded in $[-2\sqrt{d}\norm{\bw},2\sqrt{d}\norm{\bw}]$, and we can again use Hoeffding's inequality to find that with probability at least $1-\delta$,
    \begin{equation*}
        |\partial_a L_\rho - \partial_a \widehat{L}_\rho| \leq \sqrt{\frac{8d\norm{\bw}^2\log(1/\delta)}{m}}.
    \end{equation*}
    Setting $\delta = e^{-\log^2 (d)}$ so that $\log(1/\delta) = \log^2(d)$, squaring both sides gives
    \begin{equation*}
        (\partial_a L_\rho - \partial_a \widehat{L}_\rho)^2\leq \frac{d\log^2(d)}{m}\norm{\bw}^2
    \end{equation*}
    with failure probability $e^{-\log^2(d)} \leq d^{-C}$ for any fixed $C>0$ and $d$ large enough.
    This completes the proof of the lemma.
\end{proof}

We are now ready to show that each event in \Defref{events} occurs with high probability.
\begin{lemma} \label{lem:events}
    For a constant $C>0$ chosen large enough, if $\log(d)\ll p \ll d^C$ and $T\ll d^{C}$, then each event in \Defref{events} occurs with probability at least $1-d^{-C}$.
\end{lemma}
\begin{proof}
    Since all lemmas hold for any fixed $C>0$, union bounds over at most 
    polynomially many events (in $d$) are handled by applying each lemma 
    with $C$ replaced by $C + C_0$ for an appropriate absolute constant 
    $C_0$; we suppress this adjustment throughout.

    \Defref{events_init_j} occurs with probability at least $1-d^{-C}$ 
    by \Lemref{phase1_basecase}.

    For \Defref{events_init}, recall that $\sgn(\azero)$ is an independent Rademacher variable for each neuron $(\azero,\bwzero)$.
    By a Chernoff bound we have with probability at least $1-d^{-C}$ that
    \begin{equation*}
        |S^+|,|S^-| = \frac{p}{2} \pm O\Big(\sqrt{p\log(d)}\Big),
    \end{equation*}
    wherein $p\gg \log(d)$ implies
    \begin{equation*}
        |S^+|,|S^-| = \left(1 + o(\log^{1/2}(d)p^{-1/2}) \right) \cdot \frac{p}{2}.
    \end{equation*}
    A union bound over $p\ll d^C$ applications of \Lemref{phase1_basecase} gives $\bigcap_{j=1}^p\Einitj$ with probability at least $1-d^{-C}$, and the result follows by another union bound.

    \Defref{events_batch_t} occurs with probability at least $1-d^{-C}$ 
    by \Lemref{empirical_concentration}.

    The condition on \Defref{events_batch} follows by a union bound over $T\ll d^C$  applications of \Lemref{empirical_concentration}.

    For \Defref{events_gamma_t}, restrict our attention to Phase II only.
    Recall that in Phase II, for fixed magnitudes of the spurious features $\{|\wspt|\}$, the margins
    \begin{align*}
        \gammaplust &\coloneqq \frac{1}{p} \sum_{(\at, \bwt)} \ind[(\at,\bwt) \in S^+] \cdot (\wspt)^2 \\
        \gammaminust &\coloneqq \frac{1}{p} \sum_{(\at, \bwt)} \ind[(\at,\bwt) \in S^-] \cdot (\wspt)^2
    \end{align*}
    are sums of independent random variables (as each of $\ind[(\at,\bwt) \in S^+] \sim \text{Bernoulli}(1/2)$).
    Note that the $\wspt$'s are within a constant factor of each other: by \Remref{phase2_basecase}, they begin Phase II within a constant factor of each other, and by \Defref{phase2_hypothesis_wsp}, they evolve together according to $\bar{\gamma}$ and a small $o(1)$ correction factor, both of which are shared across neurons.
    A straightforward application of Hoeffding's inequality then yields that the condition on \Defref{events_gamma_t} holds with probability at least $1 - d^{-C}$ (where the constant $C$ depends on $C'$ as well as the constant factor between the $\wspt$'s).
    The condition on \Defref{events_gamma} follows by a union bound over $\TII\ll d^C$ iterations.

    The condition on \Defref{events_train} follows by a union bound over \Defref{events_init}, \Defref{events_batch}, and \Defref{events_gamma}.

    For \Defref{events_test_j}, $\bwperp^\top\bxi$ is a sum of  independent bounded terms for any neuron $(a,\bw)$.
    By Hoeffding's inequality, with probability at least $1-d^{-C}$, we have $|\bwperp^\top\bxi| < C\norm{\bwperp}\log^{1/2}(d)$.
    We need this property to hold for all one-bit-flips of $\bxi$.
    Indeed, this holds with a slightly larger constant.
    Since flipping the $i$-th bit (say, from $\bxi$ to $\bxi'$) can change the Rademacher sum $|\bwperp^\top\bxi|$ by at most $2|w_i|$, we have
    \begin{equation*}
        |\bwperp^\top\bxi'|\leq|\bwperp^\top\bxi|+2|w_i|\leq C\norm{\bwperp}\log^{1/2}(d)+2|w_i|.
    \end{equation*}
    If $d\geq e$, applying the above with $C-2$ in place of $C$, every one-bit flip $\bxi'$ satisfies $|\bwperp^\top\bxi'|< C\norm{\bwperp}\log^{1/2}(d)$, as desired.
    
    The condition on \Defref{events_test} follows by a union bound over $p\ll d^C$ applications of \Defref{events_test_j}.

    This completes the proof of the lemma.
\end{proof}

\subsection{From Booleans to Gaussians} \label{sec:booleans_to_gaussians}
In this section, we prove some lemmas which enable us to approximate uniform samples from the Boolean hypercube by Gaussians.
At a high level, these are different specializations of the well-known Berry-Esseen central limit theorem to our feature learning setting.
In \Secref{basic_results}, we overview the Berry-Esseen theorem and some basic Gaussian approximations.
In \Secref{optimal_rate}, we show that the Berry-Esseen error term achieves the optimal $d^{-1/2}$ rate in our setting.
In \Secref{truncated_moment}, we apply the Berry-Esseen theorem to a certain truncated moment which will be important in our analysis.
In \Secref{lindeberg}, we use the Lindeberg exchange method to show that a certain population gradient with respect to a weight vector $\bw$ is approximately parallel to $\bw$.

\subsubsection{Basic Results} \label{sec:basic_results}
One can see that inner products with high-dimensional Rademacher vectors are approximately Gaussian via the following special case of the Berry-Esseen central limit theorem.
\begin{theorem} \label{thm:berry_esseen}
    For any $\bv\in\R^d$ and $\mu\in\R$, we have
    \begin{equation*}
        \sup_{k\in \R} \left| \P_{\bxi\sim \Unif(\{\pm 1\}^{d})}\left(\bv^\top\bxi - \mu \leq k \right) - \P_{G\sim \Nc(-\mu, \norm{\bv}_2^2)}\left(G \leq k \right) \right| \lesssim \frac{\norm{\bv}_3^3}{\norm{\bv}_2^3}.
    \end{equation*}
\end{theorem}

Following an application of \Thmref{berry_esseen}, we will often want to upper bound the Gaussian term as follows.
\begin{lemma} \label{lem:gaussian_taylor}
    For any $k\in \R$ we have
    \begin{equation*}
        \P_{G\sim \Nc(-\mu, \sigma^2)}(|G|\leq k)\lesssim \frac{k}{\sigma}\exp\left(\frac{-\mu^2}{2\sigma^2}\right) + \frac{k^2}{\sigma^2}
    \end{equation*}
\end{lemma}
\begin{proof}
    Let $\Phi$ denote the standard Gaussian cumulative distribution function and $\phi$ denote the standard Gaussian probability density function.
    Also, let $a=\mu/\sigma$ and $h=k/\sigma$.
    Noting that $\Phi''(\cdot)$ is upper bounded by a constant, Taylor's theorem gives us
    \begin{equation*}
        \Phi(a+h) -\Phi(a) - h \phi(a) \lesssim h^2.
    \end{equation*}
    In particular,
    \begin{equation*}
        \P_{G\sim \Nc(-\mu, \sigma^2)}(|G|\leq k)=\Phi(a+h)-\Phi(a-h)\lesssim h\phi(a)+h^2.
    \end{equation*}
    The lemma then follows by definition of $\phi$, $a$, and $h$.
\end{proof}

We combine \Thmref{berry_esseen} and \Lemref{gaussian_taylor} in the following useful form.
\begin{lemma} \label{lem:boolean_to_gaussian}
    Suppose $\bv\in\R^d$ and we pick $k\lesssim \norm{\bv}_2 d^{-1/4}$.
    Then for any $\mu\in\R$, we have
    \begin{equation*}
        \P_{\bxi\sim \Unif(\{\pm 1\}^{d})}\left(|\bv^\top\bxi - \mu| \leq k \right) \lesssim \frac{k}{\norm{\bv}_2}\exp\left(\frac{-\mu^2}{2\norm{\bv}_2^2}\right) + \frac{\norm{\bv}_3^3}{\norm{\bv}^3_2}.
    \end{equation*}
\end{lemma}
\begin{proof}
    By \Thmref{berry_esseen}, we have
    \begin{equation*}
        \P_{\bxi\sim \Unif(\{\pm 1\}^{d})}\left(|\bv^\top\bxi - \mu| \leq k \right) \lesssim \P_{G\sim \Nc(-\mu, \norm{\bv}_2^2)}(|G|\leq k) + \frac{\norm{\bv}_3^3}{\norm{\bv}_2^3}.
    \end{equation*}
    Using $k\lesssim \norm{\bv}_2 d^{-1/4}$, we further have
    \begin{equation*}
        \frac{k^2}{\norm{\bv}_2^2}\lesssim \frac{1}{\sqrt{d}} \leq \frac{\norm{\bv}_3^3}{\norm{\bv}_2^3}.
    \end{equation*}
    Applying \Lemref{gaussian_taylor} with $\sigma=\norm{\bv}_2$ completes the proof of the lemma.
\end{proof}

\subsubsection{Achieving the Optimal Rate} \label{sec:optimal_rate}
\Lemref{boolean_to_gaussian} is very close to the result we will need.
However, the $\tfrac{\norm{\bv}_3^3}{\norm{\bv}_2^3}$ error term can be problematic for us as $\norm{\bv}_\infty$ grows.
We can avoid this via a condition that the deviation $\bDelta$ from a well-behaved initialization $\bvzero$ has small $\ell_2$-norm.
This ensures that the set of ``bad'' indices in $\bv$ which cause the error term to stray from the optimal $d^{-1/2}$ rate is small enough that it can be integrated out as a constant.

Before proving the main result, we introduce a lemma which controls certain vector norms in this regime.
Note that the requirements on $\bvzero$ are satisfied with high probability when $\bvzero\sim \mathbb{S}^{d-1}(\theta)$.\footnote{We will not exactly have $\bvzero\sim \mathbb{S}^{d-1}(\theta)$: in our application we have $\bvzero=\bwperpzero=\bwzero_{4:}$ for a vector $\bwzero\sim \mathbb{S}^{d-1}(\theta)$, but the norm bounds will still hold (with high probability and up to constants).}
\begin{lemma} \label{lem:boolean_to_gaussian_delta_norms}
    Suppose we can write $\bv=\bvzero+\bDelta$ such that $\norm{\bvzero}_2\asymp \theta$, $\norm{\bvzero}_\infty\lesssim \theta\log^{1/2}(d)d^{-1/2}$, $\norm{\bvzero}_3^3\lesssim \theta^3 d^{-1/2}$ and $\norm{\bDelta}_2\ll \theta \log^{-1/2}(d)$.
    Define the ``bad'' index set $B\coloneqq \{ i: |\Delta_i| \geq \theta d^{-1/2}\}$.
    Define $\bv_B\coloneqq (v_i)_{i\in B}$ and $\bv_{\setminus B} \coloneqq (v_i)_{i\notin B}$.
    Then, we have:
    \crefalias{enumi}{lemma}
    \begin{enumerate}[label=(\roman*), ref=\thelemma(\roman*)]
        \item \label{lem:boolean_to_gaussian_delta_norms_B} $|B|\ll d\log^{-1}(d)$.
        \item \label{lem:boolean_to_gaussian_delta_norms_B_2} $\norm{\bv_B}_2\ll \theta$.
        \item \label{lem:boolean_to_gaussian_delta_norms_minusB_3} $\norm{\bv_{\setminus B}}_3^3\lesssim \theta^3 d^{-1/2}$.
        \item \label{lem:boolean_to_gaussian_delta_norms_minusB_2} $\norm{\bv_{\setminus B}}_2\asymp \theta$.
    \end{enumerate}
\end{lemma}
\begin{proof}
    For \Lemref{boolean_to_gaussian_delta_norms_B}, by the definition of $B$ we have $\Delta_i^2 \geq \theta^2d^{-1}$ for all $i\in B$.
    Using $\norm{\bDelta}_2\ll \theta\log^{-1/2}(d)$ we have
    \begin{equation*}
        |B| \theta^2d^{-1} \leq \sum_{i\in B}\Delta_i^2 \leq \norm{\bDelta}_2^2 \ll \theta^2\log^{-1}(d),
    \end{equation*}
    so $|B|\ll d\log^{-1}(d)$.
    For \Lemref{boolean_to_gaussian_delta_norms_B_2}, using $(a+b)^2\leq 2(a^2+b^2)$ for $a,b\geq 0$ we have
    \begin{align*}
        \norm{\bv_B}_2^2 &= \sum_{i\in B} v_i^2 \\
        &= \sum_{i\in B} (\vzero_i+\Delta_i)^2 \\
        &\lesssim \sum_{i\in B} (\vzero_i)^2 + \sum_{i\in B}\Delta_i^2 \\
        &\leq |B|\norm{\bvzero}_\infty^2 + \norm{\bDelta}_2^2 \\
        &\ll d\log^{-1}(d)\cdot \theta^2 \log(d) d^{-1} + \theta^2\log^{-1}(d) \\
        &\lesssim \theta^2,
    \end{align*}
    where we used \Lemref{boolean_to_gaussian_delta_norms_B}, $\norm{\bvzero}_\infty\lesssim \theta\log^{1/2}(d)d^{-1/2}$, and $\norm{\bDelta}_2\ll \theta\log^{-1/2}(d)$.
    For \Lemref{boolean_to_gaussian_delta_norms_minusB_3}, using the triangle inequality and $(a+b)^3\leq 4(a^3+b^3)$ for $a,b\geq 0$ we have
    \begin{align*}
        \norm{\bv_{\setminus B}}_3^3 &= \sum_{i\notin B} |v_i|^3 \\
        &\leq \sum_{i\notin B}(|\vzero_i|+|\Delta_i|)^3 \\
        &\lesssim \sum_{i \notin B} |\vzero_i|^3 + \sum_{i\notin B} |\Delta_i|^3 \\
        &\leq \norm{\bvzero}_3^3 + \max_{i\notin B} |\Delta_i| \norm{\bDelta}_2^2 \\
        &\lesssim \theta^3 d^{-1/2} + \theta^3 \log^{-1}(d) d^{-1/2} \\
        &\lesssim \theta^3 d^{-1/2},
    \end{align*}
    where we used $\norm{\bvzero}_3^3\lesssim \theta^3 d^{-1/2}$, $|\Delta_i|\leq \theta d^{-1/2}$ for all $i\notin B$, and $\norm{\bDelta}_2\ll \theta\log^{-1/2}(d)$.
    For \Lemref{boolean_to_gaussian_delta_norms_minusB_2}, we have
    \begin{align*}
        \norm{\bv_{\setminus B}}_2^2 &= \norm{\bv}_2^2 - \norm{\bv_B}_2^2 \\
        &= \norm{\bvzero+\bDelta}_2^2 - \norm{\bv_B}_2^2 \\
        &= \norm{\bvzero}_2^2 + 2(\bvzero)^\top\bDelta + \norm{\bDelta}_2^2 - \norm{\bv_B}_2^2.
    \end{align*}
    Using $\norm{\bvzero}_2^2\asymp \theta^2$ and the Cauchy-Schwarz inequality we have
    \begin{align*}
        | \norm{\bv_{\setminus B}}_2^2 - \theta^2 | &\lesssim (\bvzero)^\top\bDelta + \norm{\bDelta}_2^2 + \norm{\bv_B}_2^2 \\
        &\leq \norm{\bvzero}_2\norm{\bDelta}_2+ \norm{\bDelta}_2^2 + \norm{\bv_B}_2^2 \\
        &\ll \theta^2\log^{-1}(d)+\theta^2\log^{-1}(d)+\theta^2 \\
        &\ll \theta^2,
    \end{align*}
    where we also used $\norm{\bDelta}_2\ll \theta\log^{-1/2}(d)$ and \Lemref{boolean_to_gaussian_delta_norms_B_2}.
    Thus $\norm{\bv_{\setminus B}}_2^2 \asymp \theta^2(1\pm o(1))$ so $\norm{\bv_{\setminus B}}_2\asymp \theta$.
    This completes the proof of the lemma.
\end{proof}

Now, we can prove the main lemma of this section.
The technique is along the same lines as~\cite[Lemma B.4]{glasgow2024sgd}, though the result is modified for our purposes.
\begin{lemma} \label{lem:boolean_to_gaussian_delta}
    There exists a constant $C>0$ such that the following holds.
    Suppose we can write $\bv=\bvzero+\bDelta$ such that $\norm{\bvzero}_2\asymp \theta$, $\norm{\bvzero}_\infty\lesssim \theta\log^{1/2}(d)d^{-1/2}$, $\norm{\bvzero}_3^3\lesssim \theta^3 d^{-1/2}$ and $\norm{\bDelta}_2\ll \theta \log^{-1/2}(d)$.
    Moreover, suppose $k\lesssim \theta d^{-1/4}$.
    Then for any $\mu\in\R$, we have
    \begin{equation*}
        \P_{\bxi\sim \Unif(\{\pm 1\}^{d})}\left(|\bv^\top\bxi - \mu| \leq k \right) \lesssim \frac{k}{\theta}\exp\left(\frac{-\mu^2}{C\theta^2}\right) + \frac{1}{\sqrt{d}}.
    \end{equation*}
\end{lemma}
\begin{proof}
    Define the ``bad'' index set $B\coloneqq \{ i: |\Delta_i| \geq \theta d^{-1/2}\}$.
    Define $S_B\coloneqq \sum_{i\in B} v_i\xi_i$ and $S_{\setminus B} \coloneqq \sum_{i\notin B} v_i\xi_i$.
    Similarly, define $\bv_B\coloneqq (v_i)_{i\in B}$ and $\bv_{\setminus B} \coloneqq (v_i)_{i\notin B}$ with $\bxi_B \coloneqq (\xi_i)_{i \in B}$ and $\bxi_{\setminus B} \coloneqq (\xi_i)_{i \notin B}$.
    By independence of the $\xi_i$'s, we can condition on $S_B=z$ to find
    \begin{equation*}
        |\bv^\top\bxi - \mu| \leq k \iff |S_{\setminus B} - (\mu-z) | \leq k.
    \end{equation*}
    By the law of total probability,
    \begin{equation} \label{eq:total_probability}
        \P_{\bxi}\left(|\bv^\top\bxi - \mu| \leq k \right) = \sum_{z} \P_{\bxi_B}(S_B=z)\P_{\bxi_{\setminus B}}\left(|S_{\setminus B} - (\mu-z) | \leq k \right),
    \end{equation}
    where the sum is over all $z$ in the range of $S_B$, which is finite because $S_B$ is a discrete random variable.
    By \Lemref{boolean_to_gaussian}, for a constant $c>0$, we have
    \begin{align*}
        \P_{\bxi_{\setminus B}}(|S_{\setminus B} - (\mu-z) | \leq k) &\lesssim \frac{k}{\norm{\bv_{\setminus B}}_2}\exp\left( \frac{-(\mu-z)^2}{2\norm{\bv_{\setminus B}}_2^2} \right) + \frac{\norm{\bv_{\setminus B}}_3^3}{\norm{\bv_{\setminus B}}_2^3} \\
        &\lesssim \frac{k}{\theta}\exp\left(\frac{-(\mu-z)^2}{c\theta^2}\right)+\frac{1}{\sqrt{d}},
    \end{align*}
    where we used \Lemref{boolean_to_gaussian_delta_norms_minusB_3}, \Lemref{boolean_to_gaussian_delta_norms_minusB_2}, and $k\lesssim \theta d^{-1/4}$.
    Returning to \Eqref{total_probability}, we have
    \begin{equation*}
        \P_{\bxi}\left(|\bv^\top\bxi - \mu| \leq k \right) \lesssim \frac{k}{\theta} \E_{S_B}\left[\exp\left(\frac{-(\mu-S_B)^2}{c\theta^2}\right) \right] + \frac{1}{\sqrt{d}}.
    \end{equation*}
    It remains to show that
    \begin{equation*}
        \E_{S_B}\left[\exp\left(\frac{-(\mu-S_B)^2}{c\theta^2}\right) \right] \lesssim \exp\left( \frac{-\mu^2}{C\theta^2} \right).
    \end{equation*}
    We have
    \begin{align*}
        (\mu-S_B)^2-(\mu^2/2-S_B^2) &= \mu^2 - 2\mu S_B + S_B^2 - \mu^2/2+S_B^2 \\
        &= \mu^2/2-2\mu S_B + 2S_B^2 \\
        &= (\mu-2S_B)^2/2 \\
        &\geq 0,
    \end{align*}
    so $(\mu-S_B)^2 \geq \mu^2/2-S_B^2$.
    In particular,
    \begin{equation*}
        \E_{S_B}\left[\exp\left(\frac{-(\mu-S_B)^2}{c\theta^2}\right) \right] \leq \exp\left(\frac{-\mu^2}{2c\theta^2}\right) \E_{S_B}\left[ \exp\left(\frac{S_B^2}{c\theta^2} \right)\right].
    \end{equation*}
    We claim that $\E_{S_B}\left[ \exp\left(\frac{S_B^2}{c\theta^2} \right)\right]\lesssim 1$, which setting $C=2c$ would prove the lemma.
    For any $\beta>0$ and any $x$ we have
    \begin{align*}
        e^{\beta x^2} &= \frac{1}{\sqrt{4\pi\beta }} e^{\beta x^2} \sqrt{4\pi\beta}\\
        &= \frac{1}{\sqrt{4\pi\beta}}e^{\beta x^2}\int_{-\infty}^\infty e^{-t^2/(4\beta)}dt \\
        &= \frac{1}{\sqrt{4\pi\beta}}e^{\beta x^2}\int_{-\infty}^\infty e^{-(t-2\beta x)^2/(4\beta)}dt \\
        &= \frac{1}{\sqrt{4\pi\beta}}\int_{-\infty}^\infty e^{tx-t^2/(4\beta)}dt,
    \end{align*}
    where we completed the square in $t$ and used the standard Gaussian integral $\int_{-\infty}^\infty e^{-zt^2}dt=\sqrt{\pi/z}$ for any $z$.
    Taking expectations over $S_B$ and applying Fubini's theorem, we have
    \begin{equation*}
        \E_{S_B}[e^{\beta S_B^2}] = \frac{1}{\sqrt{4\pi\beta}}\int_{-\infty}^\infty \E_{S_B}[e^{tS_B}] e^{-t^2/(4\beta)}dt.
    \end{equation*}
    Since $S_B$ (as a Rademacher sum) is subgaussian with parameter $\norm{\bv_B}_2^2$, we have $\E_{S_B}[e^{tS_B}] \leq e^{t^2\norm{\bv_B}_2^2/2}$, so
    \begin{align*}
        \E_{S_B}[e^{\beta S_B^2}] &\leq \frac{1}{\sqrt{4\pi\beta}}\int_{-\infty}^\infty \exp\left(-t^2 \left(\frac{1}{4\beta} - \frac{\norm{\bv_B}_2^2}{2}\right)\right)dt \\
        &= \frac{1}{\sqrt{1-2\beta\norm{\bv_{B}}_2^2}},
    \end{align*}
    where we used the Gaussian integral again, for $\beta <1/(2\norm{\bv_B}_2^{2})$.
    Setting $\beta=1/(c\norm{\bvzero}_2^{2})$, and recalling that $\norm{\bv_B}_2^2\ll \theta^2 \asymp \norm{\bvzero}_2^2$ by \Lemref{boolean_to_gaussian_delta_norms_B_2}, we have
    \begin{equation*}
        \E_{S_B}\left[\exp\left(\frac{S_B^2}{c\theta^2}\right)\right]\leq \frac{1}{\sqrt{1-\frac{2\norm{\bv_{B}}_2^2}{c\norm{\bvzero}_2^2}}} \lesssim \frac{1}{\sqrt{1-o(1)}} \lesssim 1.
    \end{equation*}
    This completes the proof of the lemma.
\end{proof}

\subsubsection{Approximation of a Truncated Moment} \label{sec:truncated_moment}
In this section, we can apply \Thmref{berry_esseen} to a certain truncated moment which will be important in our analysis.
The proof technique is a bulk/tail argument on the integral of the difference between the Rademacher and Gaussian cumulative distribution functions.
\begin{lemma} \label{lem:berry_esseen_tail}
    For any $\bv\in\R^d$ and $k\geq 0$, write
    \begin{equation*}
        \Delta(\bv, k) \coloneqq \bigg| \E_{\bxi\sim\Unif(\{\pm 1\}^d)} \left[|\bv^\top\bxi| \ind\left(|\bv^\top\bxi| \geq k\right)\right] - \E_{G\sim\Nc(0,\norm{\bv}_2^2)} [|G| \ind(|G|\geq k)] \bigg|.
    \end{equation*}
    Then, for any $\bv\in\R^d$ we have
    \begin{equation*}
        \sup_{k\geq 0} \Delta(\bv,k) \lesssim \norm{\bv}_\infty \log^{1/2}\left(\frac{\norm{\bv}_2}{\norm{\bv}_\infty}\right).
    \end{equation*}
\end{lemma}
\begin{proof}
    Recall that for any non-negative random variable $X$ and $k\geq 0$ we may write
    \begin{equation} \label{eq:berry_esseen_bulk_tail}
        \E[X \ind(X\geq k)] = k\P(X\geq k) + \int_k^\infty \P(X\geq u) du. 
    \end{equation}
    In particular, we have by the triangle inequality that
    \begin{equation*}
        \Delta(\bv, k) \leq k \left|\P \left(|\bv^\top\bxi| \geq k\right)-\P(|G|\geq k) \right| + \int_k^\infty \left|\P\left(|\bv^\top\bxi| \geq u\right)-\P(|G|\geq u) \right| du.
    \end{equation*}
    For the first term, by \Thmref{berry_esseen} we have
    \begin{equation} \label{eq:berry_esseen_term_1}
        k\left|\P\left(|\bv^\top\bxi| \geq k\right)-\P(|G|\geq k)\right| \lesssim k\frac{\norm{\bv}_3^3}{\norm{\bv}_2^3}.
    \end{equation}
    For the second term, we will split the integral at $T= \sqrt{2} \norm{\bv}_2\log^{1/2}(\norm{\bv}_2^3/\norm{\bv}_3^3)$.
    For the first integral, by \Thmref{berry_esseen} we have
    \begin{equation} \label{eq:berry_esseen_bulk}
        \int_k^T \left|\P\left(|\bv^\top\bxi| \geq u\right)-\P(|G|\geq u)\right| du \lesssim (T-k) \frac{\norm{\bv}^3_3}{\norm{\bv}_2^3}.  
    \end{equation}
    For the second integral, by the Gaussian tail we have $\P(|G|\geq u) \leq 2\exp\Big(\frac{-u^2}{2\norm{\bv}_2^2}\Big)$.
    Moreover, by Hoeffding's inequality and independence of each $\xi_i$, we have $\P\left(|\bv^\top\bxi|\geq u\right)\leq 2\exp\Big(\frac{-u^2}{2\norm{\bv}_2^2}\Big)$.
    Hence, using the standard upper bound on the Gaussian tail integral, we obtain
    \begin{equation*}
        \int_T^\infty \left|\P\left(|\bv^\top\bxi| \geq u\right)-\P(|G|\geq u) \right| du \lesssim \int_T^\infty \exp\left(\frac{-u^2}{2\norm{\bv}_2^2}\right) du \leq \frac{\norm{\bv}_2^2}{T}\exp\left(\frac{-T^2}{2\norm{\bv}_2^2}\right).
    \end{equation*}
    By our choice of $T$ we have
    \begin{equation} \label{eq:berry_esseen_tail}
        \int_T^\infty \left|\P\left(|\bv^\top\bxi| \geq u \right)-\P(|G|\geq u)\right| du \lesssim \frac{\norm{\bv}_2^2}{T}\cdot\frac{\norm{\bv}_3^3}{\norm{\bv}_2^3}.
    \end{equation}
    Combining \Eqref{berry_esseen_bulk} and \Eqref{berry_esseen_tail} we have
    \begin{equation} \label{eq:berry_esseen_term_2}
        \int_k^\infty \left|\P\left(|\bv^\top\bxi| \geq u \right)-\P(|G|\geq u) \right| du \lesssim (T-k) \frac{\norm{\bv}^3_3}{\norm{\bv}_2^3}+\frac{\norm{\bv}_2^2}{T}\cdot\frac{\norm{\bv}_3^3}{\norm{\bv}_2^3} \asymp T \frac{\norm{\bv}^3_3}{\norm{\bv}_2^3}. 
    \end{equation}
    Combining \Eqref{berry_esseen_term_1} and \Eqref{berry_esseen_term_2} we have
    \begin{equation*}
        \Delta(\bv, k) \lesssim (k+T)\frac{\norm{\bv}_3^3}{\norm{\bv}_2^3}.
    \end{equation*}
    Now, let us supremize over $k\geq 0$.
    If $k\leq T$, we have $\Delta(\bv, k) \lesssim T\frac{\norm{\bv}_3^3}{\norm{\bv}_2^3}$.
    On the other hand, if $k>T$ the entirety of the error is contained in the tail (\cf \Eqref{berry_esseen_tail}).
    Formally, by the triangle inequality we have
    \begin{equation*}
        \Delta(\bv,k) \leq \E\left[|\bv^\top\bxi| \ind\left(|\bv^\top\bxi| \geq k \right) \right] + \E[|G| \ind(|G|\geq k)]. 
    \end{equation*}
    By \Eqref{berry_esseen_bulk_tail} and the tail bounds used previously, we have
    \begin{equation*}
        \Delta(\bv,k) \lesssim k\exp\left(\frac{-k^2}{2\norm{\bv}_2^2}\right) + \int_T^\infty \exp\left(\frac{-u^2}{2\norm{\bv}_2^2}\right)du
    \end{equation*}
    The first term is maximized at $k=T$, i.e.,
    \begin{equation} \label{eq:berry_esseen_bigk_term_1}
        k\exp\left(\frac{-k^2}{2\norm{\bv}_2^2}\right) \leq T\exp\left(\frac{-T^2}{2\norm{\bv}_2^2}\right) \lesssim T \frac{\norm{\bv}^3_3}{\norm{\bv}_2^3}.
    \end{equation}
    For the second term we have by \Eqref{berry_esseen_tail} that
    \begin{equation} \label{eq:berry_esseen_bigk_term_2}
        \int_T^\infty \exp\left(\frac{-u^2}{2\norm{\bv}_2^2}\right)du \lesssim \frac{\norm{\bv}_2^2}{T}\cdot \frac{\norm{\bv}^3_3}{\norm{\bv}_2^3}.
    \end{equation}
    Combining the $k\leq T$ case with \Eqref{berry_esseen_bigk_term_1} and \Eqref{berry_esseen_bigk_term_2} we have
    \begin{equation*}
        \sup_{k\geq 0}\Delta(\bv, k) \lesssim T\frac{\norm{\bv}_3^3}{\norm{\bv}_2^3}.
    \end{equation*}
    To obtain our final result, we use the definition of $T$ and standard norm inequalities such that
    \begin{equation*}
        \sup_{k\geq 0}\Delta(\bv, k) \lesssim \norm{\bv}_2 \frac{\norm{\bv}_3^3}{\norm{\bv}_2^3}\log^{1/2}\left(\frac{\norm{\bv}_2^3}{\norm{\bv}^3_3}\right) \lesssim \norm{\bv}_\infty \log^{1/2}\left(\frac{\norm{\bv}_2}{\norm{\bv}_\infty}\right).
    \end{equation*}
    This completes the proof of the lemma.
\end{proof}

\subsubsection{Approximation of a Gradient Projection} \label{sec:lindeberg}
Finally, we will leverage the Lindeberg exchange method (typically used in the proof of central limit theorems, \eg\cite{tao2015central}) to show that a certain population gradient with respect to a weight vector $\bw$ is approximately parallel to $\bw$.
\begin{lemma} \label{lem:lindeberg}
    Define $W\coloneqq \E_{\bxi\sim\textnormal{Unif}(\{\pm 1\}^d)}\left[\bxi h(\bw^\top\bxi)\right]$ where $h(t)\coloneqq \sgn(t)\ind(|t|\geq \kappa)$.
    Write $W_\perp$ for the projection of $W$ onto the subspace perpendicular to $\bw$.
    Then, $\norm{W_\perp}_2 \lesssim \left(\frac{\norm{\bw}_\infty}{\norm{\bw}_2}\right)^{1/7}$.
\end{lemma}
\begin{proof}
    By the dual formulation and linearity of expectation, we have
    \begin{equation*}
        \norm{W_\perp}_2 = \sup_{\bz^\top\bw=0, \norm{\bz}=1} \E_{\bxi}\left[\bz^\top\bxi \cdot h(\bw^\top\bxi) \right].
    \end{equation*}
    Let us fix some $\bz\in \R^d$ such that $\bz^\top\bw=0$ and $\norm{\bz}=1$ and upper bound the right-hand side.
    Note that if we instead took the expectation over a standard Gaussian vector $\bg\sim \Nc(0,\bI)$, we would have $\bz^\top\bg$ and $\bw^\top\bg$ independent such that
    \begin{equation*}
        \E_{\bg}\left[\bz^\top\bg \cdot h(\bw^\top\bg) \right]=\E_{\bg}\left[\bz^\top\bg\right]\cdot \E_{\bg} \left[h(\bw^\top\bg) \right] = 0,
    \end{equation*}
    in other words, $W$ is exactly parallel to $\bw$ for Gaussians.
    We will leverage this fact to show that the distance to a Gaussian analogue of $W$ is small using the Lindeberg exchange method.
    We will need to Taylor expand $h$, but it is discontinuous, so we define a smooth relaxation parameterized by $\tau$ (to be optimized later): let $h_\tau\in C^\infty$ be such that $h_\tau=h$ outside $[-\kappa-\tau,-\kappa+\tau]$ and $[\kappa-\tau, \kappa+\tau]$ and $\norm{h_\tau^{(k)}}_\infty \lesssim \tau^{-k}$ for $k=1,2,3$.
    This exists via convolution of $h$ with two mollifiers, since $h$ is just a step function with two steps.
    Specifically, choose any bounded $C^\infty$ mollifier $\rho:\R\to [0,\infty)$ such that $\rho$ is supported on $[-1,1]$ and $\int_{-1}^{1} \rho(t)dt=1$.
    We rescale this mollifier into $[-\tau,\tau]$ as
    \begin{equation*}
        \rho_\tau(t)\coloneqq\frac{1}{\tau}\rho\left(\frac{t}{\tau}\right).
    \end{equation*}
    Then, write $S$ for the Heaviside step function and define $S_\tau(t)\coloneqq (S*\rho_\tau)(t)$ where $*$ is the convolution operator.
    We argue that $h_\tau(t)\coloneqq S_\tau(t-\kappa)-S_\tau(-t-\kappa)$ has the desired properties.
    In particular, $h_\tau$ is $C^\infty$ because $\rho_\tau$ is $C^\infty$, and it has $h_\tau=h$ outside $[-\kappa-\tau,-\kappa+\tau]$ and $[\kappa-\tau, \kappa+\tau]$ because $\rho_\tau$ is supported on $[-\tau, \tau]$.
    For the first three derivatives, noting that $\rho$ has bounded derivatives since it is $C^\infty$, we have
    \begin{align*}
        S'_\tau(t)=\rho_\tau(t)=\frac{1}{\tau}\rho\left(\frac{t}{\tau}\right) &\implies h'_\tau(t)\lesssim \frac{1}{\tau}, \\
        S''_\tau(t)=\rho_\tau'(t)=\frac{1}{\tau^2}\rho'\left(\frac{t}{\tau} \right) &\implies h''_\tau(t)\lesssim \frac{1}{\tau^2}, \\
        S'''_\tau(t)=\rho_\tau''(t)=\frac{1}{\tau^3}\rho''\left(\frac{t}{\tau} \right) &\implies h'''_\tau(t)\lesssim \frac{1}{\tau^3}.
    \end{align*}
    Now, define $H(\bx)\coloneqq \bz^\top\bx h(\bw^\top\bx)$ and $H_\tau(\bx)\coloneqq \bz^\top\bx h_\tau(\bw^\top\bx)$ such that
    \begin{align} 
        \E_{\bxi}[\bz^\top\bxi \cdot h(\bw^\top\bxi)] &= \E_{\bxi}[H(\bxi)] \nonumber \\
        &= \big(\E_{\bxi}[H(\bxi)] -\E_{\bxi}[H_\tau(\bxi)]\big) + \big(\E_{\bg}[H_\tau(\bg)]-\E_{\bg}[H(\bg)] \big) + \big(\E_{\bxi}[H_\tau(\bxi)]-\E_{\bg}[H_\tau(\bg)] \big). \label{eq:lindeberg_full}
    \end{align}
    The first two terms are smoothing errors and the final term is the Gaussian analogue error.
    For the first term, recall that by definition of $h_\tau$ we have $|H_\tau(\bxi)-H(\bxi)|$ nonzero only when $|\bw^\top\bxi|\in [\kappa-\tau, \kappa+\tau]$.
    Using this fact with Jensen's inequality, the Cauchy-Schwarz inequality, and isotropy of $\bxi$, we have
    \begin{align*}
        \big|\E_{\bxi}[H(\bxi)] -\E_{\bxi}[H_\tau(\bxi)] \big| &\leq \E_{\bxi}\left[|\bz^\top\bxi|\cdot \ind\left(|\bw^\top\bxi|\in [\kappa-\tau, \kappa+\tau]\right)\right] \\
        &\leq \sqrt{\E_{\bxi}\left[(\bz^\top\bxi)^2\right]}\sqrt{\E_{\bxi}\left[\ind\left(|\bw^\top\bxi|\in [\kappa-\tau, \kappa+\tau]\right)\right]} \\
        &=\sqrt{\P_{\bxi}\left(|\bw^\top\bxi|\in [\kappa-\tau, \kappa+\tau]\right)}.
    \end{align*}
    Now, we have two cases.
    First, if $\tau< \kappa$, combining \Thmref{berry_esseen} and \Lemref{gaussian_taylor} we find
    \begin{equation*}
        \P_{\bxi}\left(|\bw^\top\bxi|\in [\kappa-\tau, \kappa+\tau]\right) \lesssim \frac{\tau}{\norm{\bw}_2}+\frac{\tau^2}{\norm{\bw}_2^2}+\frac{\norm{\bw}_3^3}{\norm{\bw}_2^3}.
    \end{equation*}
    Second, if $\tau \geq \kappa$, combining \Thmref{berry_esseen} and \Lemref{gaussian_taylor} we find
    \begin{align*}
        \P_{\bxi}\left(|\bw^\top\bxi|\in [\kappa-\tau, \kappa+\tau]\right) &= \P_{\bxi}\left(|\bw^\top\bxi|\in [0, \kappa+\tau]\right) \\
        &\lesssim \frac{\kappa+\tau}{\norm{\bw}_2}+\frac{(\kappa+\tau)^2}{\norm{\bw}_2^2}+\frac{\norm{\bw}_3^3}{\norm{\bw}_2^3} \\
        &\lesssim \frac{\tau}{\norm{\bw}_2}+\frac{\tau^2}{\norm{\bw}_2^2}+\frac{\norm{\bw}_3^3}{\norm{\bw}_2^3},
    \end{align*}
    where we used that $\tau\geq \kappa$ in the last step.
    In either case we have
    \begin{equation} \label{eq:lindeberg_term1}
        \big|\E_{\bxi}[H(\bxi)] -\E_{\bxi}[H_\tau(\bxi)] \big| \lesssim \frac{\tau^{1/2}}{\norm{\bw}^{1/2}_2}+\frac{\tau}{\norm{\bw}_2}+\frac{\norm{\bw}_3^{3/2}}{\norm{\bw}_2^{3/2}}.
    \end{equation}
    The Gaussian smoothing error follows in the same two cases, just without the application of Berry-Esseen: repeating the Cauchy-Schwarz step and using \Lemref{gaussian_taylor} we have
    \begin{equation} \label{eq:lindeberg_term2}
        \big|\E_{\bg}[H_\tau(\bg)] -\E_{\bg}[H(\bg)] \big| \lesssim \frac{\tau^{1/2}}{\norm{\bw}^{1/2}_2}+\frac{\tau}{\norm{\bw}_2}.
    \end{equation}
    Now, we execute the Lindeberg exchange method to bound the Gaussian analogue error.
    Let us replace coordinates one at a time such that
    \begin{equation*}
        \bk_i\coloneqq (\xi_1,\dots,\xi_i,g_{i+1},\dots,g_d)
    \end{equation*}
    and $\bk_0\coloneqq\bg$.
    Then we have the telescoping sum
    \begin{equation} \label{eq:lindeberg_telescope}
        \E_{\bxi}[H_\tau(\bxi)]-\E_{\bg}[H_\tau(\bg)] = \sum_{i=1}^d \E_{\bxi,\bg}[H_\tau(\bk_i)-H_\tau(\bk_{i-1})].
    \end{equation}
    Define
    \begin{equation*}
        q_i(t)\coloneqq H_\tau(\xi_1,\dots,\xi_{i-1},t,g_{i+1},\dots,g_d),
    \end{equation*}
    then we have
    \begin{equation*}
        H_\tau(\bk_i)-H_\tau(\bk_{i-1}) = q_i(\xi_i)-q_i(g_i).
    \end{equation*}
    By a third-order Taylor expansion about zero (which we can do since $\rho_\tau$ is $C^\infty$), we have
    \begin{equation*}
        q_i(t)=q_i(0)+q_i'(0)t+\frac{1}{2}q_i''(0)t^2+R_3(t).
    \end{equation*}
    Since $\xi_i$ and $g_i$ are both centered and isotropic, we have
    \begin{align*}
        q_i(0)-q_i(0)&=0, \\
        \E[q'_i(0)\xi_i-q'_i(0)g_i] = q'_i(0)(\E[\xi_i]-\E[g_i]) &= 0, \\
        \E[\frac{1}{2}q''_i(0)\xi_i^2-\frac{1}{2}q''_i(0)g_i^2]=\frac{1}{2}q''_i(0)(\E[\xi_i^2]-\E[g_i^2]) &= 0,
    \end{align*}
    so the first three terms cancel in expectation.
    By Taylor's theorem with remainder, for some $\zeta_1,\zeta_2\in(0,1)$ we have
    \begin{equation*}
        \E[q_i(\xi_i)-q_i(g_i)]=\E[R_3(\xi_i)-R_3(g_i)]\leq \E[|\xi_i|^3\cdot |q'''_i(\zeta_1 \xi_i)|] + \E[|g_i|^3 \cdot |q'''_i(\zeta_2 g_i)|].
    \end{equation*}
    Now, writing $A\coloneqq \bz_{\setminus i}^\top\bx_{\setminus i}$ and $B\coloneqq \bw^\top_{\setminus i}\bx_{\setminus i}$ for the contributions from the other coordinates in $\bx=\bk_i$, we differentiate $H_\tau(\bx)\coloneqq \bz^\top\bx \cdot h_\tau(\bw^\top\bx)$ to find
    \begin{equation*}
        q'''_i(t)=\partial_i^3 H_\tau(\bx)|_{x_i=t} = 3z_i w_i^2 h_\tau''(w_i t + B) + (z_it + A) w_i^3 h_\tau'''(w_it+B).
    \end{equation*}
    Using $\norm{h_\tau^{(k)}}_\infty \lesssim \tau^{-k}$ by definition, we have
    \begin{equation*}
        |q'''_i(t)| \lesssim \frac{|z_i|w_i^2}{\tau^2} + \frac{(|z_i||t|+|A|)|w_i|^3}{\tau^3}.
    \end{equation*}
    Since $|\xi_i|=1$ we have
    \begin{equation*}
        \E[|\xi_i|^3|q'''_i(\zeta_1 \xi_i)|]\lesssim \frac{|z_i|w_i^2}{\tau^2} + \frac{|z_i||w_i|^3}{\tau^3} + \frac{\E[|A|]|w_i|^3}{\tau^3}.
    \end{equation*}
    On the other hand, since $\E[|g_i|^3]=\sqrt{8/\pi}$ and $\E[g_i^4]=3$ we have
    \begin{equation*}
         \E[|g_i|^3|q'''_i(\zeta_2 g_i)|]\lesssim \frac{|z_i|w_i^2}{\tau^2} + \frac{|z_i||w_i|^3}{\tau^3} + \frac{\E[|A|]|w_i|^3}{\tau^3}.
    \end{equation*}
    Both sides give the same bound up to constants.
    By the Cauchy-Schwarz inequality with $\norm{\bz}=1$ and isotropy of $\bx_{\setminus i}$ we have
    \begin{equation*}
        \E[|A|]=\E[|\bz^\top_{\setminus i}\bx_{\setminus i}|]\leq \norm{\bz_{\setminus i}}\leq 1,
    \end{equation*}
    and $|z_i|\leq 1$, so
    \begin{equation*}
        \E[q_i(\xi_i)-q_i(g_i)] \lesssim \frac{|z_i|w_i^2}{\tau^2} + \frac{|w_i|^3}{\tau^3}.
    \end{equation*}
    Returning to \Eqref{lindeberg_telescope} and summing over the coordinates we have
    \begin{align}
        \big|\E_{\bxi}[H_\tau(\bxi)]-\E_{\bg}[H_\tau(\bg)] \big| &\lesssim \frac{1}{\tau^2}\sum_{i=1}^d|z_i|w_i^2 + \frac{1}{\tau^3}\sum_{i=1}^d |w_i|^3 \nonumber \\
        &\leq \frac{\norm{\bw}_4^2}{\tau^2} + \frac{\norm{\bw}_3^3}{\tau^3}, \label{eq:lindeberg_term3}
    \end{align}
    where we used the Cauchy-Schwarz inequality and $\norm{\bz}=1$ in the last step.
    Substituting \Eqref{lindeberg_term1}, \Eqref{lindeberg_term2}, and \Eqref{lindeberg_term3} into \Eqref{lindeberg_full}, we have
    \begin{equation*}
        \norm{W_\perp}_2 \lesssim \frac{\tau^{1/2}}{\norm{\bw}^{1/2}_2}+\frac{\tau}{\norm{\bw}_2}+\frac{\norm{\bw}_3^{3/2}}{\norm{\bw}_2^{3/2}} + \frac{\norm{\bw}_4^2}{\tau^2} + \frac{\norm{\bw}_3^3}{\tau^3}.
    \end{equation*}
    Using the basic inequalities $\norm{\bw}_3^3\leq \norm{\bw}_\infty\norm{\bw}^2_2$ and $\norm{\bw}_4^2\leq \norm{\bw}_\infty\norm{\bw}_2$ we have
    \begin{equation*}
        \norm{W_\perp}_2 \lesssim \frac{\tau^{1/2}}{\norm{\bw}^{1/2}_2}+ \frac{\tau}{\norm{\bw}_2}  + \frac{\norm{\bw}_\infty^{1/2}}{\norm{\bw}_2^{1/2}} + \frac{\norm{\bw}_\infty\norm{\bw}_2}{\tau^2} + \frac{\norm{\bw}_\infty\norm{\bw}_2^2}{\tau^3}.
    \end{equation*}
    Equalizing the first and last terms we choose $\tau\asymp \norm{\bw}_\infty^{2/7}\norm{\bw}_2^{5/7}$.
    Therefore,
    \begin{equation*}
        \norm{W_\perp}_2 \lesssim \frac{\norm{\bw}_\infty^{1/7}}{\norm{\bw}^{1/7}_2}+ \frac{\norm{\bw}_\infty^{2/7}}{\norm{\bw}^{2/7}_2} +\frac{\norm{\bw}_\infty^{1/2}}{\norm{\bw}_2^{1/2}} + \frac{\norm{\bw}_\infty^{3/7}}{\norm{\bw}_2^{3/7}} + \frac{\norm{\bw}_\infty^{1/7}}{\norm{\bw}_2^{1/7}} \lesssim \frac{\norm{\bw}_\infty^{1/7}}{\norm{\bw}^{1/7}_2},
    \end{equation*}
    where the last step follows by $\norm{\cdot}_\infty\leq \norm{\cdot}_2$.
    This completes the proof of the lemma.
\end{proof}

\subsection{Simultaneous Training of Two-Layer ReLU Neural Networks} \label{sec:simultaneous_training}

In this section, we provide norm bounds and concentration for training ReLU neural networks with Lipschitz loss functions.
We begin with a short lemma involving local Lipschitz behavior of the sigmoid function $\psi(u)\coloneqq 1/(1+e^{-u})$.
\begin{lemma} \label{lem:sigmoid_lipschitz}
    If $\gamma_1,\gamma_2>0$ or $\gamma_1,\gamma_2<0$, we have
    \begin{equation*}
        |\psi(\gamma_1)-\psi(\gamma_2)|\leq \exp\big(-\min(|\gamma_1|,|\gamma_2|) \big)\cdot |\gamma_1-\gamma_2|.
    \end{equation*}
\end{lemma}
\begin{proof}
    By the mean value theorem, for some $\gamma_1\leq \zeta \leq \gamma_2$, we have
    \begin{equation*}
        |\psi(\gamma_1)-\psi(\gamma_2)|\leq |\psi'(\zeta)| \cdot |\gamma_1-\gamma_2|,
    \end{equation*}
    so it suffices to control $\psi'$.
    In particular, we have
    \begin{equation*}
        \psi'(z)=\psi(z)(1-\psi(z))=\frac{e^{-z}}{(1+e^{-z})^2}.
    \end{equation*}
    If $z\geq 0$ then $(1+e^{-z})^2\geq 1$ so $\psi'(z)\leq e^{-z}$.
    If $z< 0$ then using the symmetry $\psi'(z)=\psi'(-z)$ we have $\psi'(z)\leq e^{-|z|}$.
    So globally we have $\psi'(z)\leq e^{-|z|}$.
    Now, since $\gamma_1$ and $\gamma_2$ have the same sign, the segment between them lies in either $(-\infty, 0]$ or $[0,\infty)$.
    In each of these regions, $e^{-|z|}$ is decreasing, and so is maximized at the point with smallest magnitude, \ie $|\psi'(\zeta)|\leq \exp\big(-\min(|\gamma_1|,|\gamma_2|) \big)$.
    This completes the proof of the lemma.
\end{proof}

The next lemma comprises a \emph{balancedness} condition (\cf\cite[Lemma B.13]{glasgow2024sgd}).
\begin{lemma} \label{lem:balanced}
    Suppose we train via online SGD with batch size $m\gg d\log^2(d)$ and learning rate $\eta<\frac{1}{\sqrt{2}}$ under the $\ell_\rho$ loss.
    Then, for any neuron $(\at,\bwt)$, the following hold:
    \crefalias{enumi}{lemma}
    \begin{enumerate}[label=(\roman*), ref=\thelemma(\roman*)]
        \item \label{lem:balanced_i} $\norm{\nabla_{\bwt} L_\rho} \leq 2|\at|$.
        \item \label{lem:balanced_ii} $|\partial_{\at} L_\rho| \leq 2\norm{\bwt}$.
        \item \label{lem:balanced_iii} Under the event $\Ebatcht$ (\Defref{events_batch_t}), if $|\at|\leq\norm{\bwt}$, then $|\atplus|\leq\norm{\bwtplus}$.
        \item \label{lem:balanced_iv} Under the event $\Ebatcht$ (\Defref{events_batch_t}), we have \begin{equation*}
            \norm{\bwtplus}^2-(\atplus)^2\leq 10\eta^2(\at)^2+\norm{\bwt}^2-(\at)^2.
        \end{equation*}
    \end{enumerate}
\end{lemma}
\begin{proof}
    Let us begin with \Lemref{balanced_i}.
    Recall the expansion of $L_\rho$ from \Lemref{empirical_concentration}.
    Since $\ell_\rho$ is $2$-Lipschitz, $\sigma'(z)\leq 1$ for the ReLU, and $\bx$ is isotropic, we have
    \begin{equation*}
        \frac{1}{|\at|}\norm{\nabla_{\bwt} L_\rho} = \sup_{\bv:\norm{\bv}=1} \E_{\bx}\left[\ell^{(1)}_\rho(x) \cdot y(\bx)\cdot \sigma'(\bw^{(t)\top}\bx) \cdot \bx^\top\bv \right] \leq 2 \sup_{\bv:\norm{\bv}=1} \E_{\bx}[|\bx^\top\bv|] = 2.
    \end{equation*}
    
    For \Lemref{balanced_ii}, using that $\ell_\rho$ is $2$-Lipschitz and $\bx$ is isotropic again, we have
    \begin{equation*}
        |\partial_{\at} L_\rho| = \Big|\E_{\bx}\left[\ell^{(1)}_\rho(x) \cdot y(\bx) \cdot \sigma(\bw^{(t)\top}\bx) \right]\Big| \leq 2\E_{\bx}[|\bw^{(t)\top}\bx|] \leq 2\norm{\bwt}.
    \end{equation*}
    
    For \Lemref{balanced_iii}, by the definition of gradient descent and direct expansion of the square we can write
    \begin{equation*}
        (\atplus)^2=(\at-\eta\partial_{\at} \widehat{L}_\rho)^2=(\at)^2-2\eta\at \partial_{\at} \widehat{L}_\rho+\eta^2(\partial_{\at} \widehat{L}_\rho)^2
    \end{equation*}
    and
    \begin{equation*}
        \norm{\bwtplus}^2=\norm{\bwt-\eta\nabla_{\bwt} \widehat{L}_\rho}^2=\norm{\bwt}^2-2\eta \bw^{(t)\top} \nabla_{\bwt} \widehat{L}_\rho+\eta^2\norm{\nabla_{\bwt} \widehat{L}_\rho}^2.
    \end{equation*}
    Using $1$-homogeneity of the ReLU we have $a \partial_{a} \widehat{L}_\rho = \bw^\top \nabla_{\bw} \widehat{L}_\rho$ for any neuron $(a,\bw)$.
    Using this together with the Cauchy-Schwarz inequality, we have
    \begin{align}
        (\atplus)^2 - \norm{\bwtplus}^2 &= (\at)^2 - \norm{\bwt}^2 + \eta^2 \left( (\partial_{\at} \widehat{L}_\rho)^2 - \norm{\nabla_{\bwt} \widehat{L}_\rho}^2 \right) \label{eq:balanced_1} \\
        &\leq (\at)^2 - \norm{\bwt}^2 + \eta^2 \left( (\partial_{\at} \widehat{L}_\rho)^2 - \frac{1}{\norm{\bwt}^2} (\bw^{(t)\top}\nabla_{\bwt} \widehat{L}_\rho)^2 \right) \nonumber \\
        &= (\at)^2 - \norm{\bwt}^2 + \eta^2 \left( (\partial_{\at} \widehat{L}_\rho)^2 - \frac{(\at)^2}{\norm{\bwt}^2} (\partial_{\at} \widehat{L}_\rho)^2 \right) \nonumber \\
        &= (\at)^2 - \norm{\bwt}^2 + \frac{\eta^2 (\partial_{\at} \widehat{L}_\rho)^2}{\norm{\bwt}^2} \left( \norm{\bwt}^2 - (\at)^2 \right) \nonumber \\
        &= \left((\at)^2 - \norm{\bwt}^2 \right) \left(1 - \frac{\eta^2 (\partial_{\at} \widehat{L}_\rho)^2}{\norm{\bwt}^2} \right). \label{eq:balanced_2}
    \end{align}
    Under the event $\Ebatcht$, using also \Lemref{balanced_ii}, we have
    \begin{align*}
        (\partial_{\at} \widehat{L}_\rho)^2 &\leq 2(\partial_{\at} L_\rho)^2 + 2(\partial_{\at} L_\rho - \partial_{\at} \widehat{L}_\rho)^2 \\
        &\leq 2\norm{\bwt}^2\left(1 + \frac{d\log^2(d)}{m}\right) \\
        &= 2\norm{\bwt}^2(1+o(1)),
    \end{align*}
    where the last line follows from $m\gg d\log^2(d)$.
    Substituting into \Eqref{balanced_2} and using $\eta< \frac{1}{\sqrt{2}}$, for $d$ large enough we have
    \begin{equation*}
        \sgn\left((\atplus)^2 - \norm{\bwtplus}^2 \right) = \sgn\left((\at)^2 - \norm{\bwt}^2\right),
    \end{equation*}
    which gives the desired statement of \Lemref{balanced_iii}.

    Finally, for \Lemref{balanced_iv}, starting from \Eqref{balanced_1} we have
    \begin{align*}
        \norm{\bwtplus}^2-(\atplus)^2 - \norm{\bwt}^2 + (\at)^2 &= \eta^2 \left(\norm{\nabla_{\bwt} \widehat{L}_\rho}^2 - (\partial_{\at} \widehat{L}_\rho)^2 \right) \\
        &\leq \eta^2 \norm{\nabla_{\bwt} \widehat{L}_\rho}^2 \\
        &\leq 2\eta^2 \left( \norm{\nabla_{\bwt} L_\rho}^2 + \norm{\nabla_{\bwt} L_\rho - \nabla_{\bwt} \widehat{L}_\rho}^2 \right).
    \end{align*}
    Under the event $\Ebatcht$, using also \Lemref{balanced_i} and $m\gg d\log^2(d)$, we have
    \begin{equation*}
        \norm{\bwtplus}^2-(\atplus)^2 - \norm{\bwt}^2 + (\at)^2 \leq 2\eta^2 (4(\at)^2+(\at)^2)\leq 10\eta^2(\at)^2,
    \end{equation*}
    which is the desired statement of \Lemref{balanced_iv}.
    This completes the proof of the lemma.
\end{proof}

The next lemma provides some helpful bounds on the magnitude of the neural network output $|f_\rho(x)|$.
\begin{lemma} \label{lem:frho_bounds}
    The following hold:
    \crefalias{enumi}{lemma}
    \begin{enumerate}[label=(\roman*), ref=\thelemma(\roman*)]
        \item \label{lem:frho_bounds_i} $\E_{\bx}[f_\rho(\bx)^2]\leq (\E_{(a,\bw) \sim \rho}[\norm{a\bw}])^2$.
        \item \label{lem:frho_bounds_ii} $\E_{\bx}[|f_\rho(\bx)|]\leq\E_{(a,\bw) \sim \rho}[\norm{a\bw}]$.
        \item \label{lem:frho_bounds_iii} $|f_\rho(\bx)|\leq \sqrt{d}\E_{(a,\bw) \sim \rho}[\norm{a\bw}]$ for any $\bx\in\{\pm1\}^d$.
        \item \label{lem:frho_bounds_iv} $\P_{\bx}(|f_\rho(\bx)|\geq \log(d) \E_{(a,\bw) \sim \rho}[\norm{a\bw}]) \leq d^{-C}$ for any fixed $C>0$.
    \end{enumerate}
\end{lemma}

\begin{proof}
    Let us begin with \Lemref{frho_bounds_i}.
    By isotropy of $\bx$, we have for any neuron $(a,\bw)$ that $\E_{\bx}[\sigma(\bw^\top\bx)^2]\leq \E_{\bx}[(\bw^\top\bx)^2] = \norm{\bw}^2$.
    Using this fact with the Cauchy-Schwarz inequality, we have
    \begin{align*}
        \E_{\bx}[f_\rho(\bx)^2]&=\E_{\bx}\left[\left(\E_{(a,\bw) \sim \rho}[a\sigma(\bw^\top\bx)] \right)^2 \right] \\
        &= \frac{1}{p^2}\sum_{j,k}a_ja_k \E_{\bx}\left[\sigma(\bw_j^\top\bx)\cdot\sigma(\bw_k^\top\bx) \right] \\
        &\leq \frac{1}{p^2}\sum_{j,k}|a_j||a_k| \cdot \sqrt{\E_{\bx}\left[\sigma(\bw_j^\top\bx)^2\right]}\cdot\sqrt{\E_{\bx}\left[\sigma(\bw_k^\top\bx)^2\right]} \\
        &\leq \frac{1}{p^2}\sum_{j,k} |a_j| |a_k| \cdot \norm{\bw_j}\norm{\bw_k} \\
        &= (\E_{(a,\bw) \sim \rho}[\norm{a\bw}])^2,
    \end{align*}
    as desired.
    
    For \Lemref{frho_bounds_ii}, using \Lemref{frho_bounds_i} and Jensen's inequality, we have
    \begin{equation*}
        \E_{\bx}[|f_\rho(\bx)|]\leq \sqrt{\E_{\bx}[f_\rho(\bx)^2]}\leq \E_{(a,\bw) \sim \rho}[\norm{a\bw}].
    \end{equation*}
    
    For \Lemref{frho_bounds_iii}, using the triangle inequality, $\sigma(x)\leq |x|$, and $|\bw^\top\bx|\leq\norm{\bw}_1\leq\sqrt{d}\norm{\bw}_2$ for $\bx\in\{\pm 1\}^d$, we have
    \begin{equation*}
        |f_\rho(\bx)|=\Big|\E_{(a,\bw) \sim \rho}[a\sigma(\bw^\top\bx)]\Big|\leq\E_{(a,\bw) \sim \rho}[|a||\bw^\top\bx|]\leq\sqrt{d}\E_{(a,\bw) \sim \rho}[\norm{a\bw}].
    \end{equation*}
    
    For \Lemref{frho_bounds_iv}, using \Lemref{frho_bounds_ii} we have
    \begin{align*}
        \log(d) \E_{(a,\bw) \sim \rho}[\norm{a\bw}] - \E_{\bx}[|f_\rho(\bx)|] &\geq \log(d) \E_{(a,\bw) \sim \rho}[\norm{a\bw}] - \E_{(a,\bw) \sim \rho}[\norm{a\bw}] \\
        &= (\log(d) - 1)\E_{(a,\bw) \sim \rho}[\norm{a\bw}].
    \end{align*}
    In particular,
    \begin{equation*}
        \P_{\bx}\Big(|f_\rho(\bx)|\geq \log(d) \E_{(a,\bw) \sim \rho}[\norm{a\bw}]\Big)\\
        \leq\P_{\bx}\Big(|f_\rho(\bx)|-\E_{\bx}[|f_\rho(x)|]\geq (\log(d) - 1)\E_{(a,\bw) \sim \rho}[\norm{a\bw}]\Big).
    \end{equation*}
    For any $i\in[d]$, we have the following bounded differences condition upon flipping the $i^{th}$ bit of $\bx$.
    Using the triangle inequality and $1$-Lipschitzness of the ReLU, we obtain
    \begin{equation*}
        \Big||f_\rho(\bx)|-|f_\rho(\bx-2x_i\be_i)|\Big| \leq \E_{(a,\bw) \sim \rho}\left[|a|\cdot|\sigma(\bw^\top\bx)-\sigma(\bw^\top(\bx-2x_i\be_i))|\right] \leq 2\E_{(a,\bw) \sim \rho}[|aw_i|].
    \end{equation*}
    Therefore, by McDiarmid's inequality and Jensen's inequality,\footnote{There is a subtle complication here: unlike~\cite{glasgow2024sgd} we do not have independence on all coordinates of $\bx$. However, for the purposes of McDiarmid's inequality, we can combine $x_1x_2x_3$ into a single ``super-variable'' $\tilde{\bx}$ which is independent from the rest of $\bx$. Moreover, the bounded differences condition on $\tilde{\bx}$ is satisfied by the sum of the bounded differences of $x_1,x_2,x_3$, so the dependency ends up not changing the bound.}
    we have
    \begin{align*}
        \P_{\bx}\Big(|f_\rho(\bx)|-\E_{\bx}[|f_\rho(x)|] \geq (\log(d) - 1)\E_{(a,\bw) \sim \rho}[\norm{a\bw}]\Big) &\leq \exp\left(-\frac{(\log(d) - 1)^2(\E_{(a,\bw) \sim \rho}[\norm{a\bw}])^2}{2\sum_{i=1}^d (\E_{(a,\bw) \sim \rho}[|aw_i|])^2}\right) \\
        &\leq \exp\left(-\frac{1}{2}(\log(d) - 1)^2\right) \\
        &\leq d^{-C},
    \end{align*}
    which is the desired statement of \Lemref{frho_bounds_iv}.
    This completes the proof of the lemma.
\end{proof}

\subsection{Linear Non-Homogeneous Recurrence Solution} \label{sec:misc_lemmas}
In this section, we provide a derivation of the solution to the linear non-homogeneous recurrences of \Defref{phase1_hypothesis}.
\begin{lemma} \label{lem:phase1_recursion}
    Consider $z,\eta,\mu, T>0$, and let $(\eps_t)_{t\geq 0}$ be a sequence with $|\eps_t|\leq \delta < 1$ such that $\delta\eta\mu T \ll 1$.
    Given a sequence $\{w^{(t)}\}_{t \geq 0}$ satisfying the recurrence
    \begin{equation*}
        w^{(t+1)}-w^{(t)} = \eta\mu (1+\eps_t) \cdot \sgn(w^{(t)})(|w^{(t)}| + z),
    \end{equation*}
    it holds that 
    \begin{equation*}
        |w^{(T)}| = (1\pm 3\delta \eta\mu T) \cdot (1+\eta\mu)^T (|w^{(0)}|+z)-z.
    \end{equation*}
\end{lemma}
\begin{proof}
    Since $|\eps_t|\leq \delta$ we have $1-\delta \leq 1+\eps_t \leq 1+\delta$ for all $t\leq T$.
    Write $x^{(t)}\coloneqq |w^{(t)}|+z$ and note that $x^{(t)} > 0$ for all $t \geq 0$.
    We will first upper and lower bound $x^{(T)}$.

    For the upper bound, by the triangle inequality we have $|w^{(t+1)}| \leq |w^{(t)}|+\eta\mu(1+\delta) \cdot(|w^{(t)}| + z)$.
    Rewriting, we find that $x^{(t+1)}\leq (1+\eta\mu(1+\delta))x^{(t)}$.
    Unrolling the recurrence, we have $x^{(T)} \leq (1+\eta\mu(1+\delta))^Tx^{(0)}$. %

    For the lower bound, the increment $w^{(t+1)}-w^{(t)}$ has the same sign as $w^{(t)}$, so $w^{(t)}$ never changes sign across all $t \geq 0$. Therefore, $|w^{(t+1)}| \geq |w^{(t)}|+\eta\mu(1-\delta)\cdot (|w^{(t)}| + z)$.
    Rewriting, we find $x^{(t+1)}\geq (1+\eta\mu(1-\delta))x^{(t)}$.
    Unrolling the recurrence, we have $x^{(T)} \geq (1+\eta\mu(1-\delta))^Tx^{(0)}$. %

    Therefore we have established that
    \begin{equation*}
        (1+\eta\mu(1-\delta))^T x^{(0)} \leq x^{(T)} \leq (1+\eta\mu(1+\delta))^T x^{(0)}.
    \end{equation*}
    The ratio of the upper to lower geometric factor satisfies
    \begin{equation*}
        \log \frac{(1+\eta\mu(1+\delta))^T}{(1+\eta\mu(1-\delta))^T} \leq 2\delta\eta\mu T.
    \end{equation*}
    Applying $e^x\leq 1+x+x^2$ for $x\in [0,1]$ and $\delta\eta\mu T\ll 1$ we obtain
    \begin{equation*}
        \frac{(1+\eta\mu(1+\delta))^T}{(1+\eta\mu(1-\delta))^T} \leq 1+3\delta\eta\mu T.
    \end{equation*}
    Hence
    \begin{equation*}
        x^{(T)} = (1\pm 3\delta \mu \eta T) \cdot (1+\eta\mu)^T x^{(0)},
    \end{equation*}
    wherein substituting $x^{(T)}\coloneqq |w^{(T)}|+z$ and $x^{(0)}=|w^{(0)}|+z$ yields the result.
\end{proof}

\subsection{Idealized Crossing-Time Analysis} \label{sec:crossing_time}
In this section, we provide the proof of \Lemref{constant_lambda_idealized} from \Secref{constant_lambda_idealized}.
Recall the following two-variable reduction based on an ODE formulation of \Lemref{wsp_phase2_induction} and \Lemref{sigopp_phase2_induction}.
Fix $\lambda\in(0,\tfrac{1}{2})$ and let $\psi(u)\coloneqq 1/(1+e^{-u})$ denote the sigmoid.
For a single neuron $(a,\bw)$ define the majority and minority group margins
\begin{align*}
    \gamma_{\textnormal{maj}} &\coloneqq |a|(\norm{\bwsig}+\norm{\bwsp}) \\
    \gamma_{\textnormal{min}} &\coloneqq |a|(\norm{\bwsig}-\norm{\bwsp}),
\end{align*}
and the loss weights
\begin{align*}
    \Gamma_{\textnormal{sig}} &\coloneqq (1-\lambda)\psi(-\gamma_{\textnormal{maj}}) + \lambda\psi(-\gamma_{\textnormal{min}}) \\
    \Gamma_{\textnormal{sp}} &\coloneqq (1-\lambda)\psi(-\gamma_{\textnormal{maj}}) - \lambda\psi(-\gamma_{\textnormal{min}}).
\end{align*}
Then, the idealized Phase II dynamics are
\begin{align*}
    \norm{\bwsigtplus}-\norm{\bwsigt} &= \eta (\at)^2 \Gamma_{\textnormal{sig}}^{(t)} \norm{\bwsigt} \\
    \norm{\bwsptplus} - \norm{\bwspt} &= \eta |\at| \Gamma_{\textnormal{sp}}^{(t)},
\end{align*}
with balancedness condition $|\at|=\norm{\bwspt}$.
We write $\tau=\eta t$ and take $\eta\to 0$ such that
\begin{equation} \label{eq:idealized_ode}
    \frac{d}{d\tau}\norm{\bwsig} = |a|^2\,\Gamma_{\textnormal{sig}}\,\norm{\bwsig} \qquad \frac{d}{d\tau}\norm{\bwsp} = |a|\,\Gamma_{\textnormal{sp}} \qquad |a|=\norm{\bwsp}.
\end{equation}

Each component of the reduction corresponds to a statement from our Phase II analysis, with lower-order terms dropped:
\begin{itemize}
    \item \emph{Balancedness.} The condition $|a|=\norm{\bwsp}$ is \Lemref{phase2_norms_a}, which gives $|a|=(1\pm o(\log^{-2}(d)))\cdot\norm{\bwsp}$ throughout Phase II.
    \item \emph{Margins.} By \Lemref{phase2_margin}, the majority and minority group margins concentrate at $\pm\bar{\gamma}$, where $\bar{\gamma}\asymp|a|\norm{\bwsp}$ since all neurons have spurious components of the same order. Our definitions of $\gamma_{\textnormal{maj}}$ and $\gamma_{\textnormal{min}}$ recover $\pm|a|\norm{\bwsp}$ when $\norm{\bwsig}=0$, and additionally include the signal contribution $|a|\norm{\bwsig}$, which is lower-order throughout Phase II but not near the crossing.
    \item \emph{Spurious update.} When $\norm{\bwsig}=0$ we have $\Gamma_{\textnormal{sp}}=1-\lambda-\psi(\bar{\gamma})$, so the spurious update is the recurrence $\wsptplus-\wspt=\eta(1\pm o(1))\cdot\wspt(1-\lambda-\psi(\bgammat))$ of \Lemref{wsp_phase2_induction} (derived from \Lemref{sp_lp}), with the $(1\pm o(1))$ factor dropped.
    \item \emph{Signal update.} For a representative neuron with $|a|=\norm{\bwsp}$, the bound of \Lemref{sigopp_phase2_induction} reads $\norm{\bwsigtplus-\bwsigt}\lesssim\eta(\at)^2e^{-c\bgammat}\norm{\bwsigt}+\eta\theta\log^{-1}(d)d^{-1/2}$. The idealized update takes this bound to hold with equality, drops the lower-order additive term, and replaces the margin factor $e^{-c\bar{\gamma}}$ with the loss weight $\Gamma_{\textnormal{sig}}$. Since $\Gamma_{\textnormal{sig}}\asymp e^{-\bar{\gamma}}+\lambda$ when $\norm{\bwsig}=0$, the two agree up to constants throughout Phase II, where $\lambda\ll e^{-\bar{\gamma}}$ (\Defref{phase2_scalings_strength}); for constant $\lambda$, the additional term $\lambda\psi(\bar{\gamma})\approx\lambda$ models the minority group contribution that keeps the signal growing once $\bar{\gamma}$ saturates.
    \item \emph{Initial condition.} At the end of Phase I, \Propref{phase1} (see \Remref{phase2_basecase}) gives $\norm{\bwsp}\asymp\theta\log^C(d)$ and $\norm{\bwsig}\lesssim\theta\log^{2C}(d)d^{-1/2}$, so the reduction starts from $\norm{\bwsig}\ll\norm{\bwsp}\lesssim 1$; since Phase I holds for constant $\lambda$, so does this initial condition. \Lemref{constant_lambda_idealized}(ii) instead starts from the spurious equilibrium; we discuss the transient between the two in the remarks following the proof.
    \item \emph{Continuous time.} Taking $\eta\to 0$ with $\tau=\eta t$ mirrors the continuous approximation used in the proof of \Propref{phase2}.
\end{itemize}

We will now prove \Lemref{constant_lambda_idealized}.
\begin{lemma}[Restatement of \Lemref{constant_lambda_idealized}]
    Fix $\lambda\in(0,\tfrac{1}{2})$ and let $\gamma^\star\coloneqq\log\big(\tfrac{1-\lambda}{\lambda}\big)$ and $R(\lambda)\coloneqq 2\lambda(1-\lambda)\gamma^\star$.
    Under the idealized dynamics of \Eqref{idealized_ode},
    \begin{enumerate}[label=(\roman*)]
        \item In the absence of signal, $\norm{\bwsp}$ converges monotonically to the equilibrium $\sqrt{\gamma^\star}$.
        \item Suppose $\norm{\bwsp^{(0)}}=\sqrt{\gamma^\star}$ and $\norm{\bwsig^{(0)}}\ll 1$, and neglect the feedback of $\norm{\bwsig}$ onto the margins. Then $\norm{\bwsig}$ grows geometrically at rate $R(\lambda)$, and $\norm{\bwsig^{(\tau^\star)}}=\norm{\bwsp^{(\tau^\star)}}$ at
        \begin{equation*}
            \tau^\star = \frac{\log(\sqrt{\gamma^\star}/\norm{\bwsig^{(0)}})}{R(\lambda)}.
        \end{equation*}
    \end{enumerate}
\end{lemma}

\begin{proof}
    For (i), with $\norm{\bwsig}=0$ and $|a|=\norm{\bwsp}$ we have $\gamma_{\textnormal{maj}}=\norm{\bwsp}^2$ and $\gamma_{\textnormal{min}}=-\norm{\bwsp}^2$, so using $\psi(-x)=1-\psi(x)$,
    \begin{equation*}
        \Gamma_{\textnormal{sp}} = (1-\lambda)\big(1-\psi(\norm{\bwsp}^2)\big) - \lambda\,\psi(\norm{\bwsp}^2) = 1-\lambda-\psi(\norm{\bwsp}^2).
    \end{equation*}
    Since $\psi$ is increasing with $\psi(\gamma^\star)=1-\lambda$, the right-hand side of $\frac{d}{d\tau}\norm{\bwsp}=\norm{\bwsp}\big(1-\lambda-\psi(\norm{\bwsp}^2)\big)$ is positive for $\norm{\bwsp}^2<\gamma^\star$ and negative for $\norm{\bwsp}^2>\gamma^\star$. Hence $\norm{\bwsp}$ moves monotonically toward $\sqrt{\gamma^\star}$.

    For (ii), neglecting the feedback of $\norm{\bwsig}$ onto the margins means setting $\gamma_{\textnormal{maj}}=\norm{\bwsp}^2$ and $\gamma_{\textnormal{min}}=-\norm{\bwsp}^2$, so that $\Gamma_{\textnormal{sp}}$ and $\Gamma_{\textnormal{sig}}$ depend on $\norm{\bwsp}$ alone. In particular, as in (i), $\frac{d}{d\tau}\norm{\bwsp}=\norm{\bwsp}\big(1-\lambda-\psi(\norm{\bwsp}^2)\big)$, which vanishes at $\norm{\bwsp}=\sqrt{\gamma^\star}$. Since $\norm{\bwsp^{(0)}}=\sqrt{\gamma^\star}$, we have $\norm{\bwsp^{(\tau)}}=\sqrt{\gamma^\star}$ for all $\tau\geq 0$. Using $\psi(\gamma^\star)=1-\lambda$ and $\psi(-\gamma^\star)=\lambda$, we then have
    \begin{equation*}
        \Gamma_{\textnormal{sig}} = (1-\lambda)\lambda+\lambda(1-\lambda) = 2\lambda(1-\lambda), \qquad \frac{d}{d\tau}\norm{\bwsig} = |a|^2\,\Gamma_{\textnormal{sig}}\,\norm{\bwsig} = R(\lambda)\norm{\bwsig},
    \end{equation*}
    so $\norm{\bwsig^{(\tau)}}=\norm{\bwsig^{(0)}}\,e^{R(\lambda)\tau}$. Since $\norm{\bwsig^{(0)}}\ll 1$ while $\sqrt{\gamma^\star}>0$ is a constant for fixed $\lambda$, we have $\norm{\bwsig^{(0)}}<\sqrt{\gamma^\star}$, and solving $\norm{\bwsig^{(0)}}\,e^{R(\lambda)\tau^\star}=\sqrt{\gamma^\star}$ for $\tau^\star$ gives the claim.
\end{proof}

We conclude with two remarks on the fidelity of our idealized dynamics.

\begin{remark}[Accuracy of the crossing time] \normalfont
    The only approximation in \Lemref{constant_lambda_idealized}(ii) is neglecting the feedback of $\norm{\bwsig}$ onto the margins.
    In the dynamics of \Eqref{idealized_ode}, this feedback enters in two ways.
    First, at the saddle point $(0, \sqrt{\gamma^\star})$ we have
    \begin{equation*}
        \frac{\partial}{\partial\norm{\bwsig}}\,\frac{d}{d\tau}\norm{\bwsp} = -(1-2\lambda)\,\gamma^\star\psi'(\gamma^\star) = -\frac{1-2\lambda}{2}\,R(\lambda) < 0,
    \end{equation*}
    so a growing signal pushes $\norm{\bwsp}$ below its equilibrium.
    Second, $\Gamma_{\textnormal{sig}}$ decreases as $\norm{\bwsig}$ grows, which slows the signal.
    Both effects are of order $\norm{\bwsig}$ and hence negligible while $\norm{\bwsig}\ll\norm{\bwsp}$.
    Since $\norm{\bwsig}$ grows geometrically, this holds for all but $O(1)$ time before the crossing, so the crossing time $\tau_{\textnormal{cross}}\coloneqq\inf\{\tau:\norm{\bwsig^{(\tau)}}\geq\norm{\bwsp^{(\tau)}}\}$ of the full dynamics satisfies $\tau_{\textnormal{cross}}=\tau^\star+O(1)$, where the $O(1)$ term depends on $\lambda$ but not on $\norm{\bwsig^{(0)}}$.
    This can be made rigorous via a standard Gr\"onwall argument, which we omit.
\end{remark}

\begin{remark}[Initialization from Phase I] \normalfont
    \Lemref{constant_lambda_idealized}(ii) starts from the spurious equilibrium, whereas at the end of Phase I we have $\norm{\bwsp}\asymp\theta\log^C(d)\ll\sqrt{\gamma^\star}$ (\Remref{phase2_basecase}).
    When $\norm{\bwsp}$ is small, both margins are close to zero, so $\Gamma_{\textnormal{sp}}$ is approximately $\tfrac{1}{2}-\lambda$ and $\Gamma_{\textnormal{sig}}$ is approximately $\tfrac{1}{2}$.
    Hence $\norm{\bwsp}$ grows geometrically at rate approximately $\tfrac{1}{2}-\lambda$ (matching \Defref{phase1_hypothesis_wsp}), reaching equilibrium after a transient of length $T_{\textnormal{sp}}\asymp\log\big(\sqrt{\gamma^\star}/\norm{\bwsp^{(0)}}\big)/(\tfrac{1}{2}-\lambda)\lesssim \log(1/\theta)$.
    Meanwhile, $\norm{\bwsig}$ grows at rate approximately $\norm{\bwsp}^2/2$; since $\norm{\bwsp}$ grows geometrically, $\int\norm{\bwsp}^2\,d\tau \lesssim 1$ over the transient, so $\norm{\bwsig}$ changes by only a constant factor.
    The crossing time from a Phase I initialization is therefore $T_{\textnormal{sp}}+\tau^\star+O(1)$, where $T_{\textnormal{sp}}$ corresponds to the $\log(1/\theta)$ term in the length of Phase II (\Propref{phase2}).
\end{remark}

\clearpage

\section{Additional Simulations}
In this section, we provide some additional simulations with different hyperparameters.

\begin{figure}[h!]
    \centering
    \begin{subfigure}[b]{0.325\textwidth}
        \centering
        \includegraphics[width=\textwidth, height=0.18\textheight]{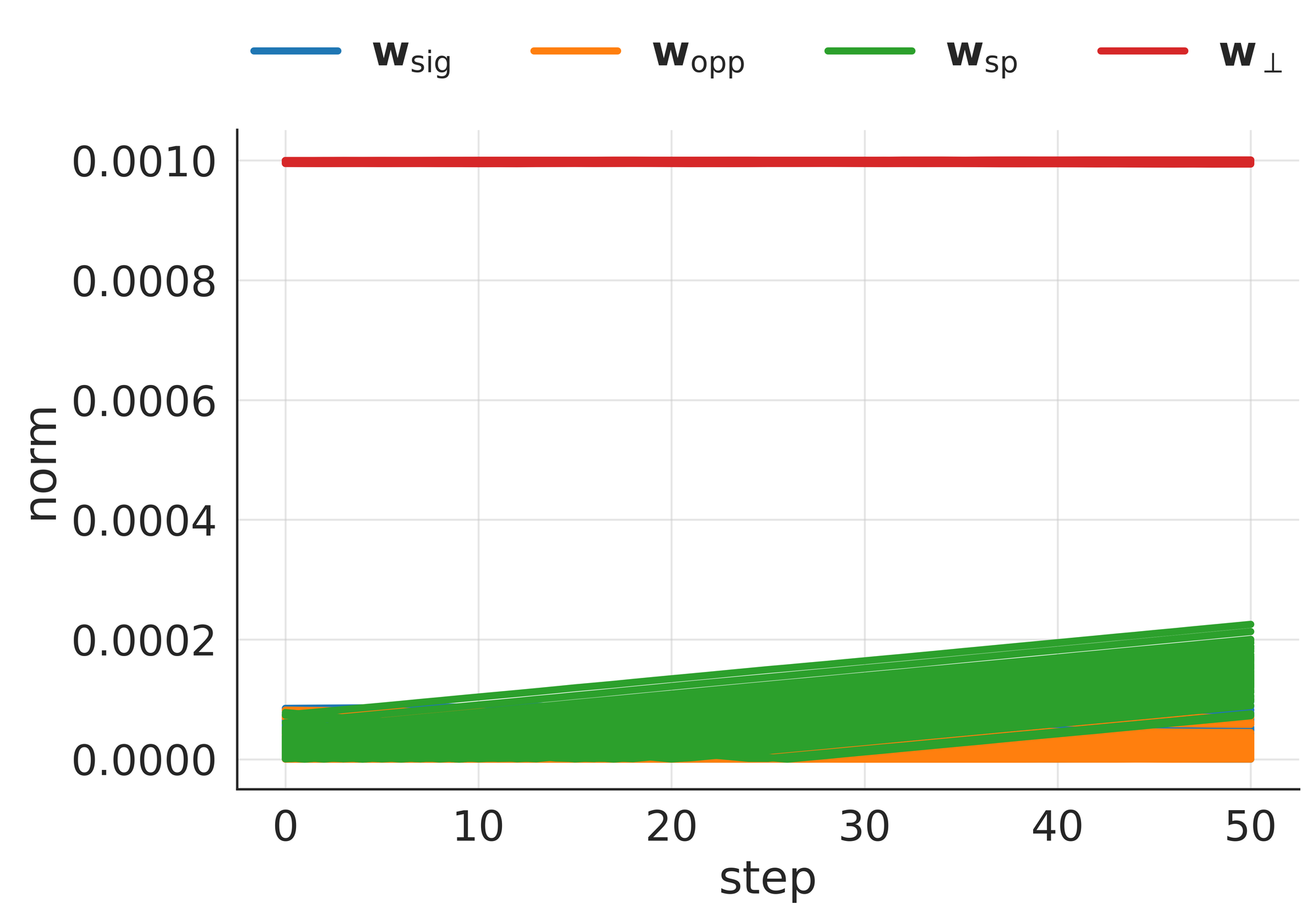}
        \caption{Phase Ia}
        \label{fig:phase1a_high}
    \end{subfigure}
    \begin{subfigure}[b]{0.325\textwidth}
        \centering
        \includegraphics[width=\textwidth, height=0.18\textheight]{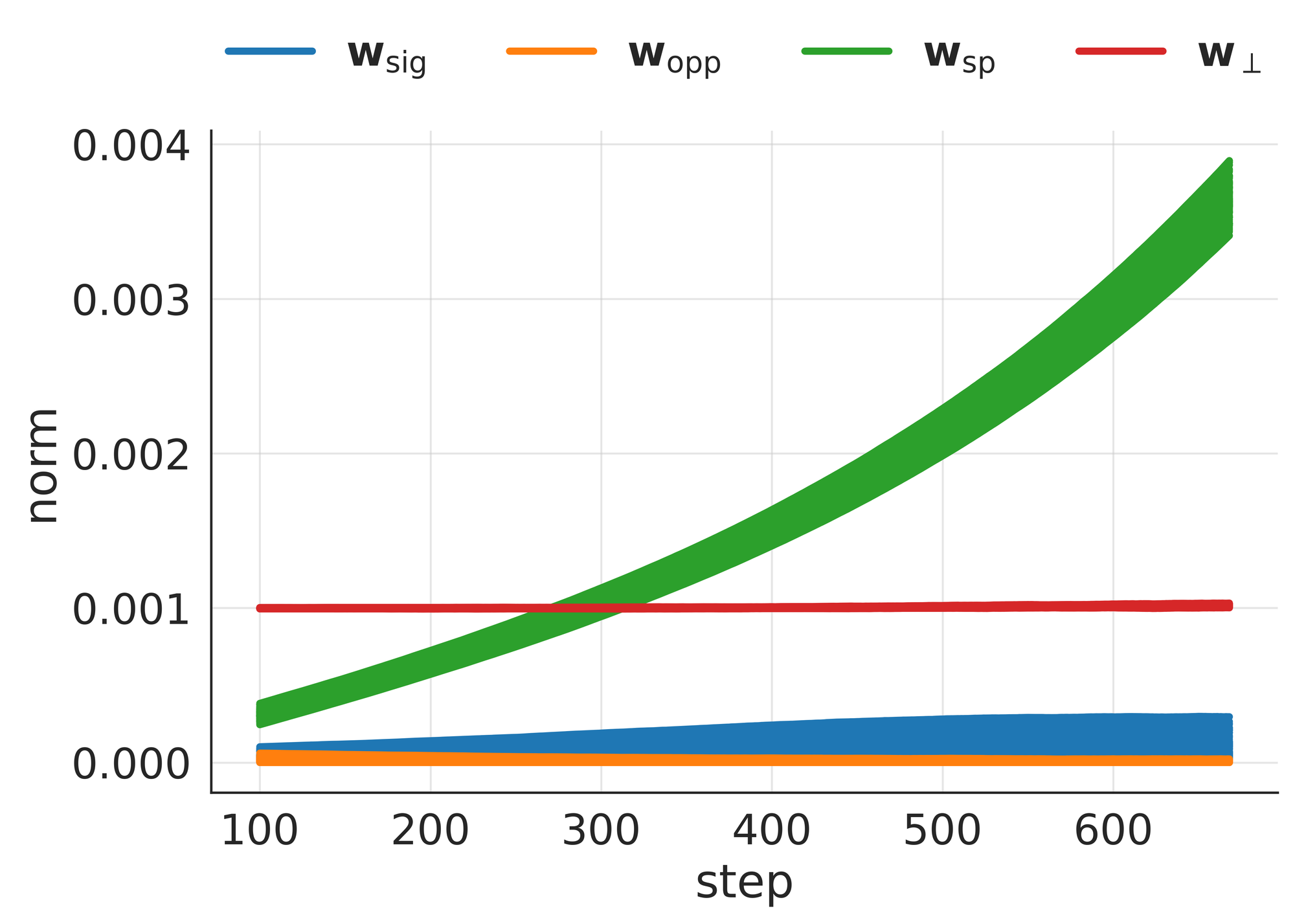}
        \caption{Phase Ib}
        \label{fig:phase1b_high}
    \end{subfigure}
    \begin{subfigure}[b]{0.325\textwidth}
        \centering
        \includegraphics[width=\textwidth, height=0.18\textheight]{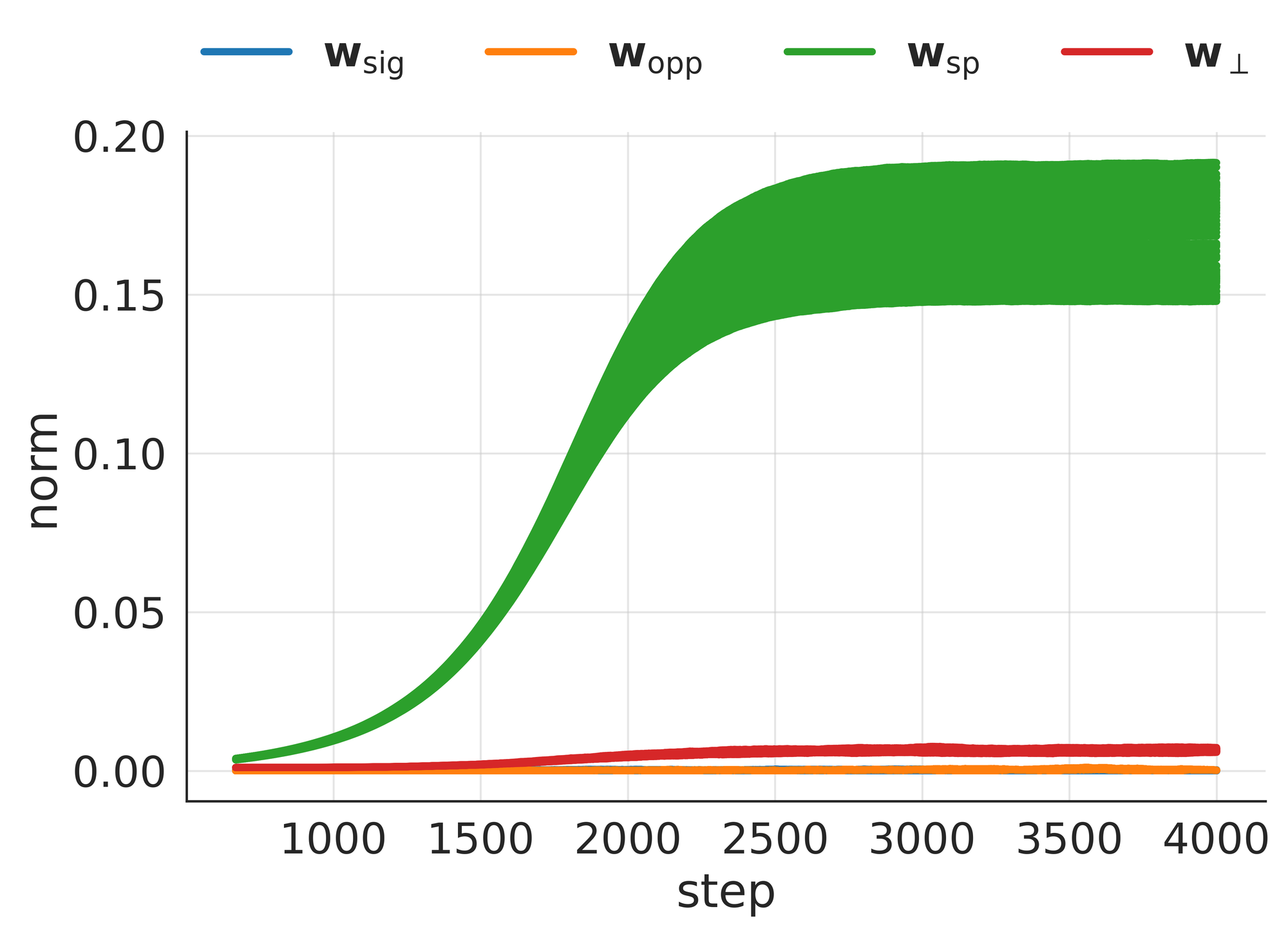}
        \caption{Phase II}
        \label{fig:phase2_high}
    \end{subfigure}
    \caption{\textbf{Scaled-up phase transitions.} We display the results of a training run with dimension $d=1000$, spurious correlation strength $\lambda=0.1$, learning rate $\eta=0.01$, width $p=100$, initialization scale $\theta=0.001$, and batch size $m=5000$. For each of the $p=100$ neurons, we plot $\norm{\bwsig}$, $\norm{\bwopp}$, $\norm{\bwsp}$, and $\norm{\bwperp}$ (defined in \Secref{feature_learning_analysis}).
    }
    \label{fig:scaled_up_phases}
\end{figure}

\begin{figure}[h!]
    \centering
    \begin{subfigure}[b]{0.325\textwidth}
        \centering
        \includegraphics[width=\textwidth, height=0.18\textheight]{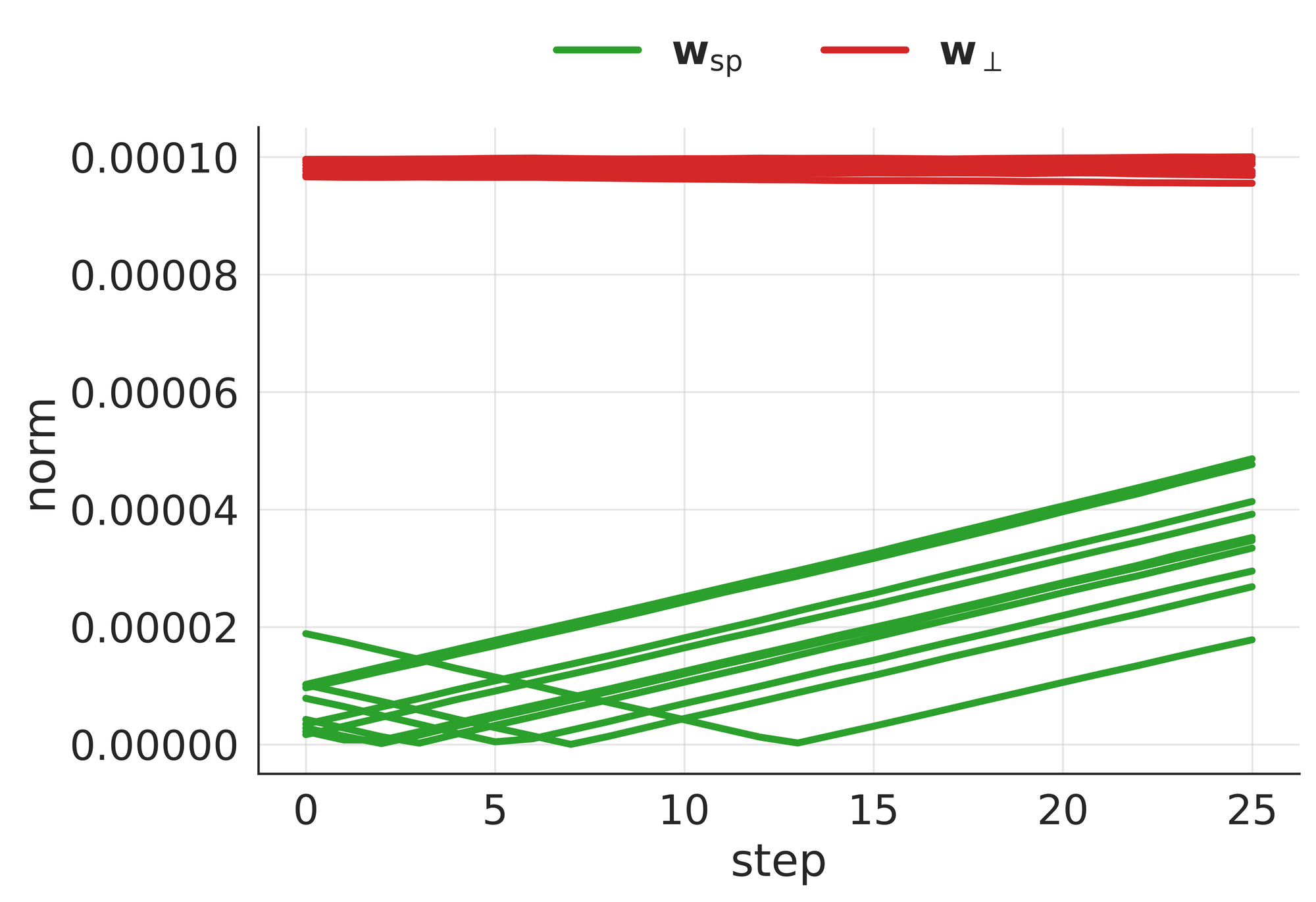}
        \caption{Phase Ia}
        \label{fig:phase1a_low}
    \end{subfigure}
    \begin{subfigure}[b]{0.325\textwidth}
        \centering
        \includegraphics[width=\textwidth, height=0.18\textheight]{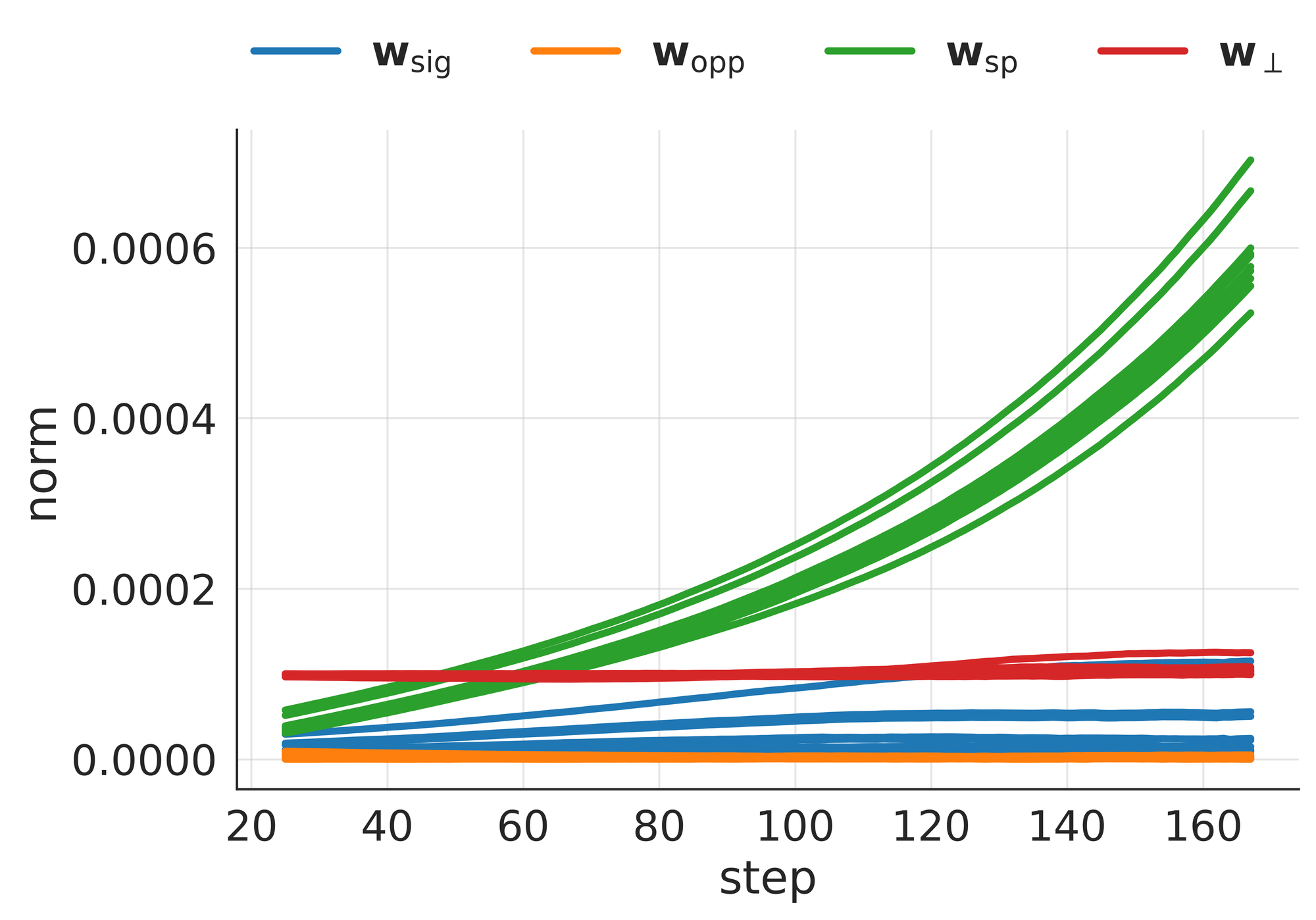}
        \caption{Phase Ib}
        \label{fig:phase1b_low}
    \end{subfigure}
    \begin{subfigure}[b]{0.325\textwidth}
        \centering
        \includegraphics[width=\textwidth, height=0.18\textheight]{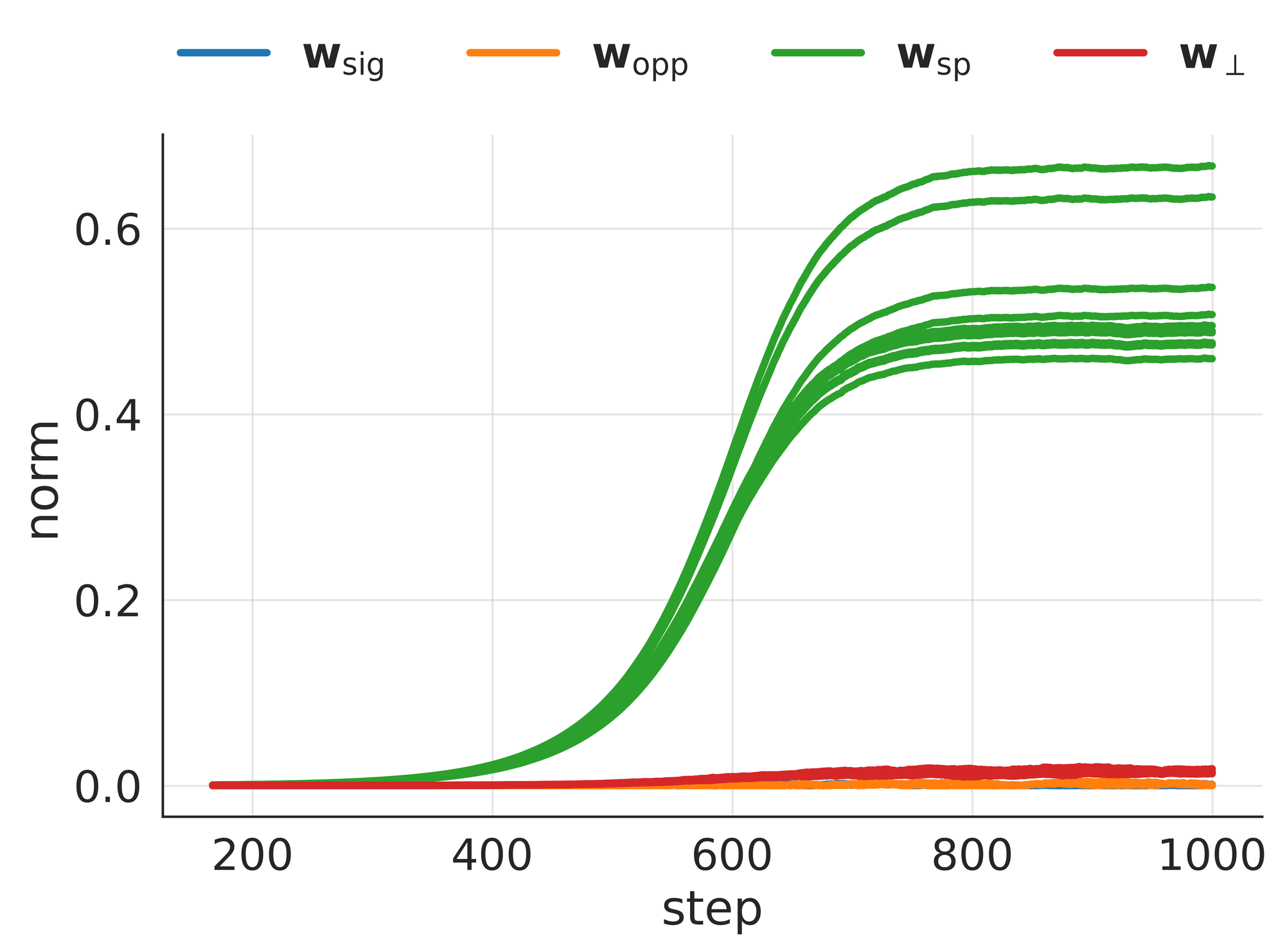}
        \caption{Phase II}
        \label{fig:phase2_low}
    \end{subfigure}
    \caption{\textbf{Small-initialization phase transitions.} We display the results of a training run with dimension $d=1000$, spurious correlation strength $\lambda=0.1$, learning rate $\eta=0.01$, width $p=10$, initialization scale $\theta=0.0001$, and batch size $m=5000$. For each of the $p=10$ neurons, we plot $\norm{\bwsig}$, $\norm{\bwopp}$, $\norm{\bwsp}$, and $\norm{\bwperp}$ (defined in \Secref{feature_learning_analysis}).
    }
    \label{fig:small_init_phases}
\end{figure}

\begin{figure}[h!]
    \centering
    \begin{subfigure}[b]{0.245\textwidth}
        \centering
        \includegraphics[width=\textwidth, height=0.15\textheight]{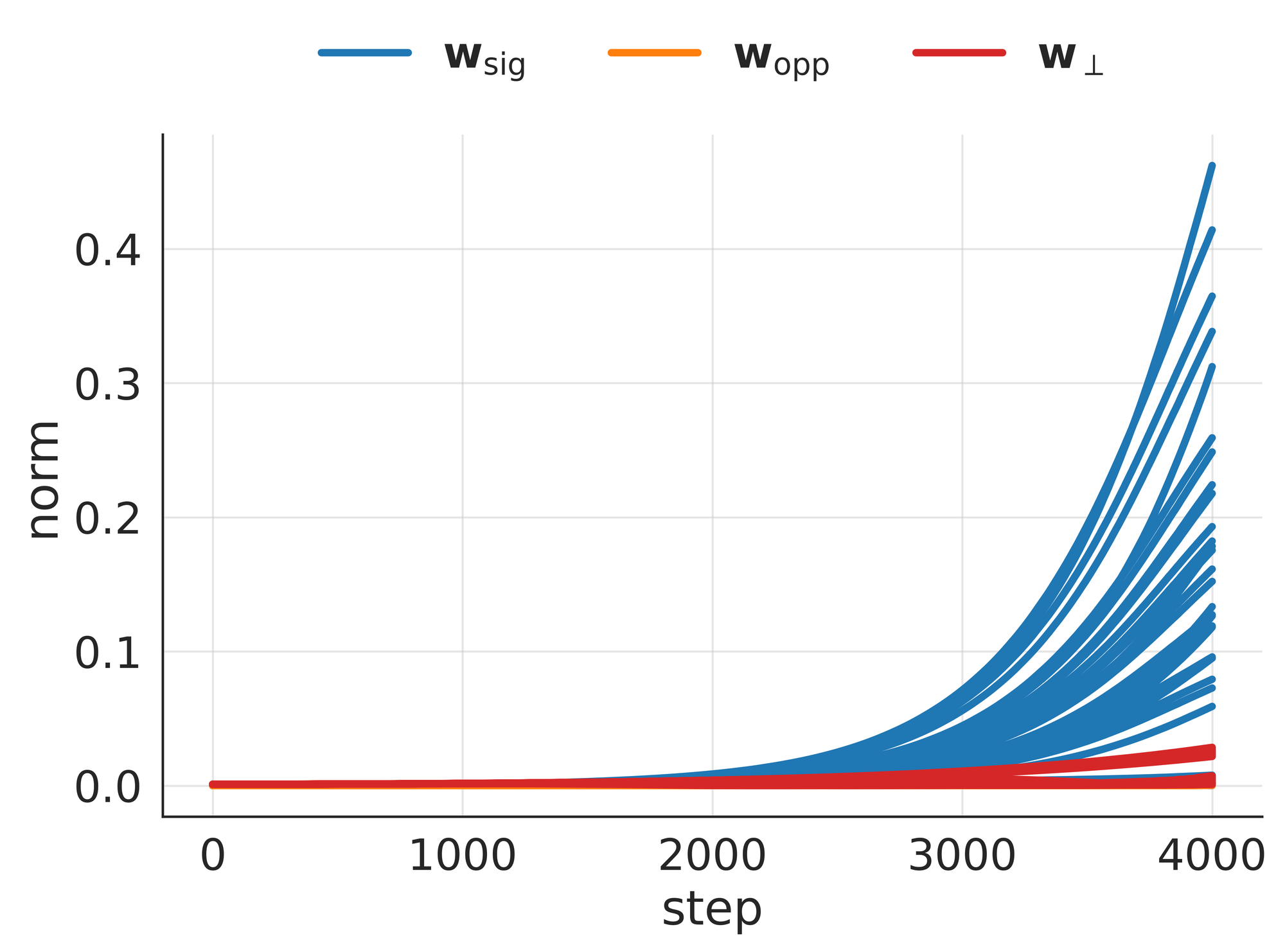}
        \caption{No spurious feature}
        \label{fig:glasgow_high}
    \end{subfigure}
    \begin{subfigure}[b]{0.245\textwidth}
        \centering
        \includegraphics[width=\textwidth, height=0.15\textheight]{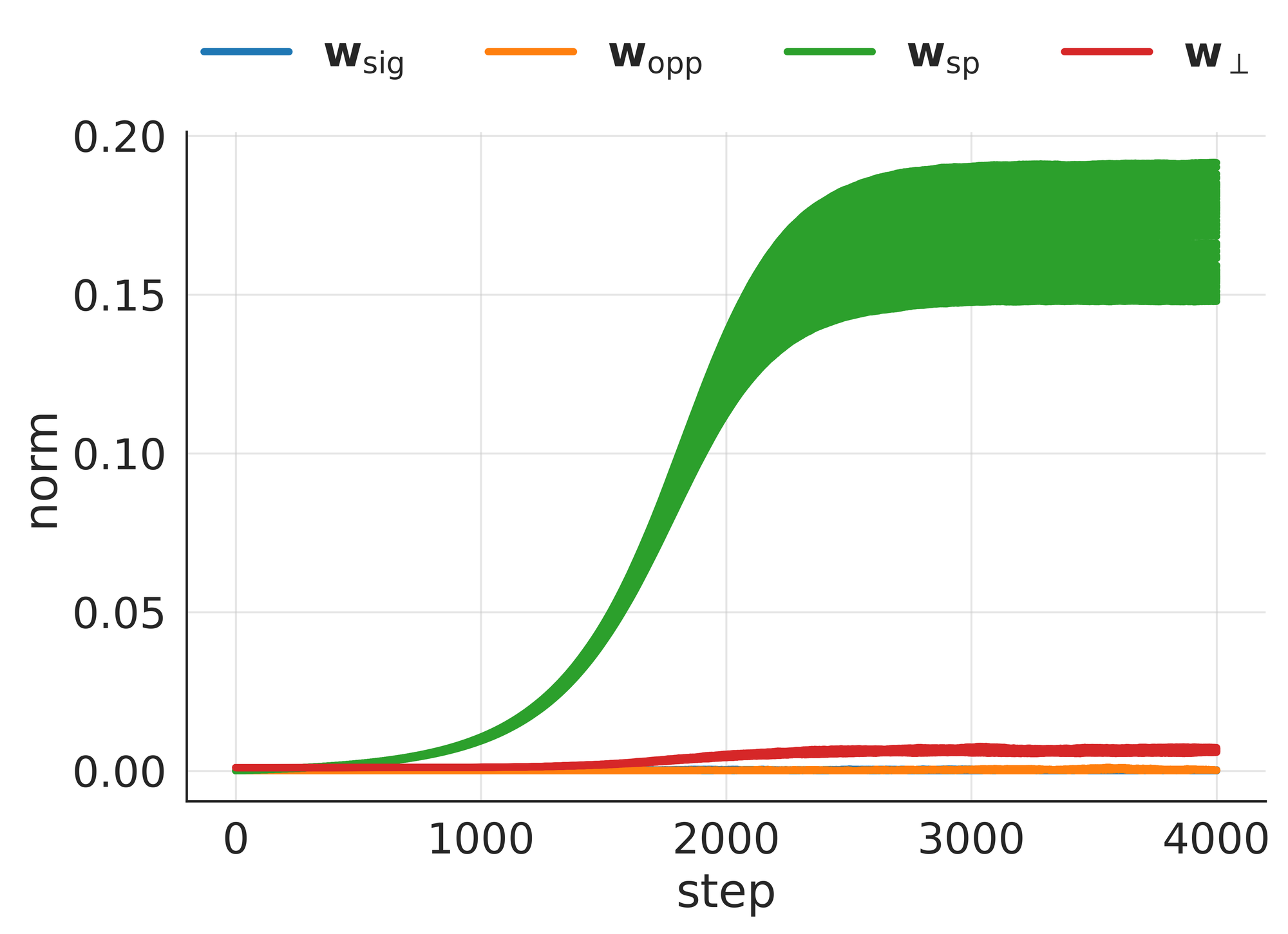}
        \caption{$\lambda=0.1$}
        \label{fig:lambda01_high}
    \end{subfigure}
    \begin{subfigure}[b]{0.245\textwidth}
        \centering
        \includegraphics[width=\textwidth, height=0.15\textheight]{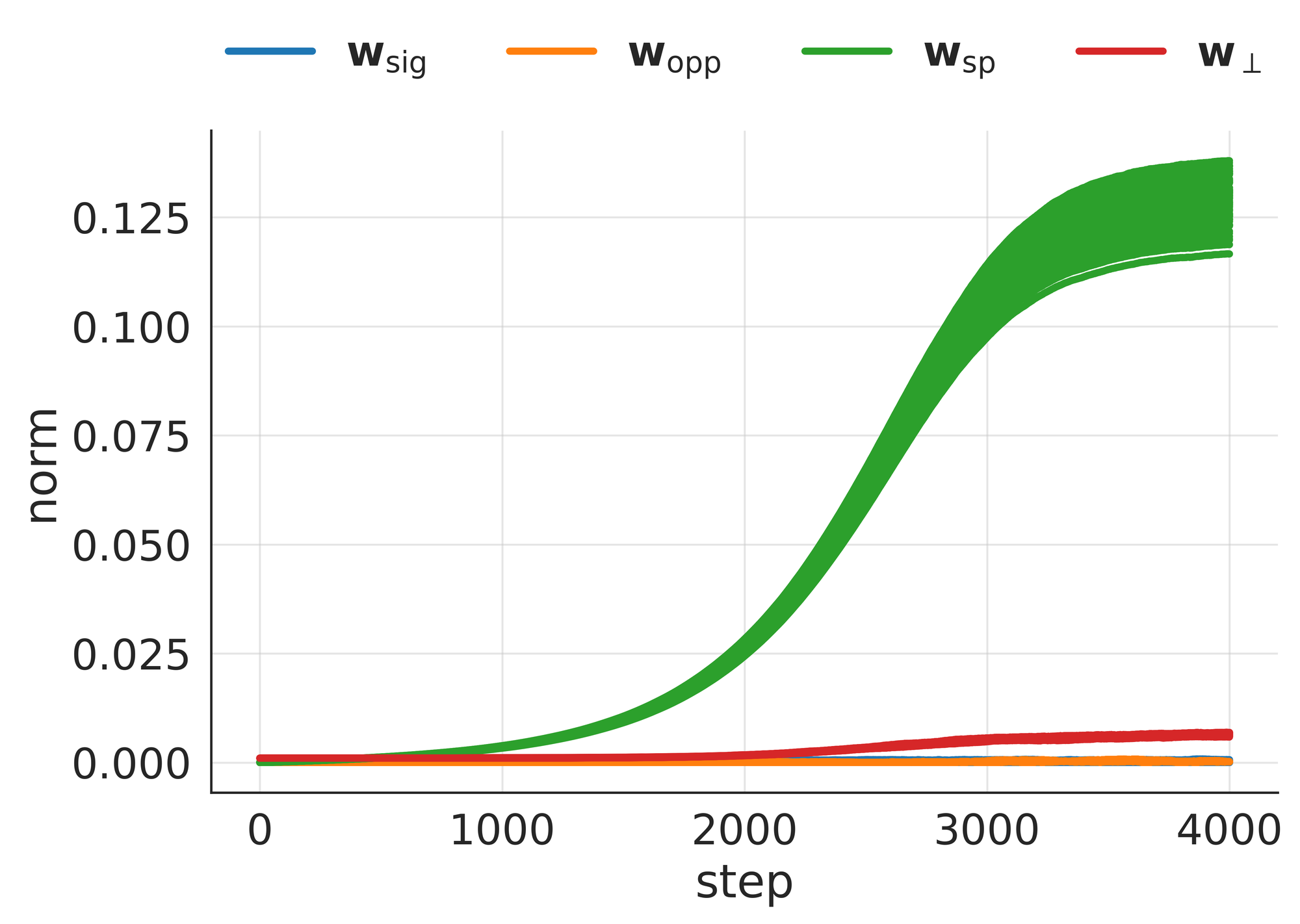}
        \caption{$\lambda=0.15$}
        \label{fig:lambda015_high}
    \end{subfigure}
    \begin{subfigure}[b]{0.245\textwidth}
        \centering
        \includegraphics[width=\textwidth, height=0.15\textheight]{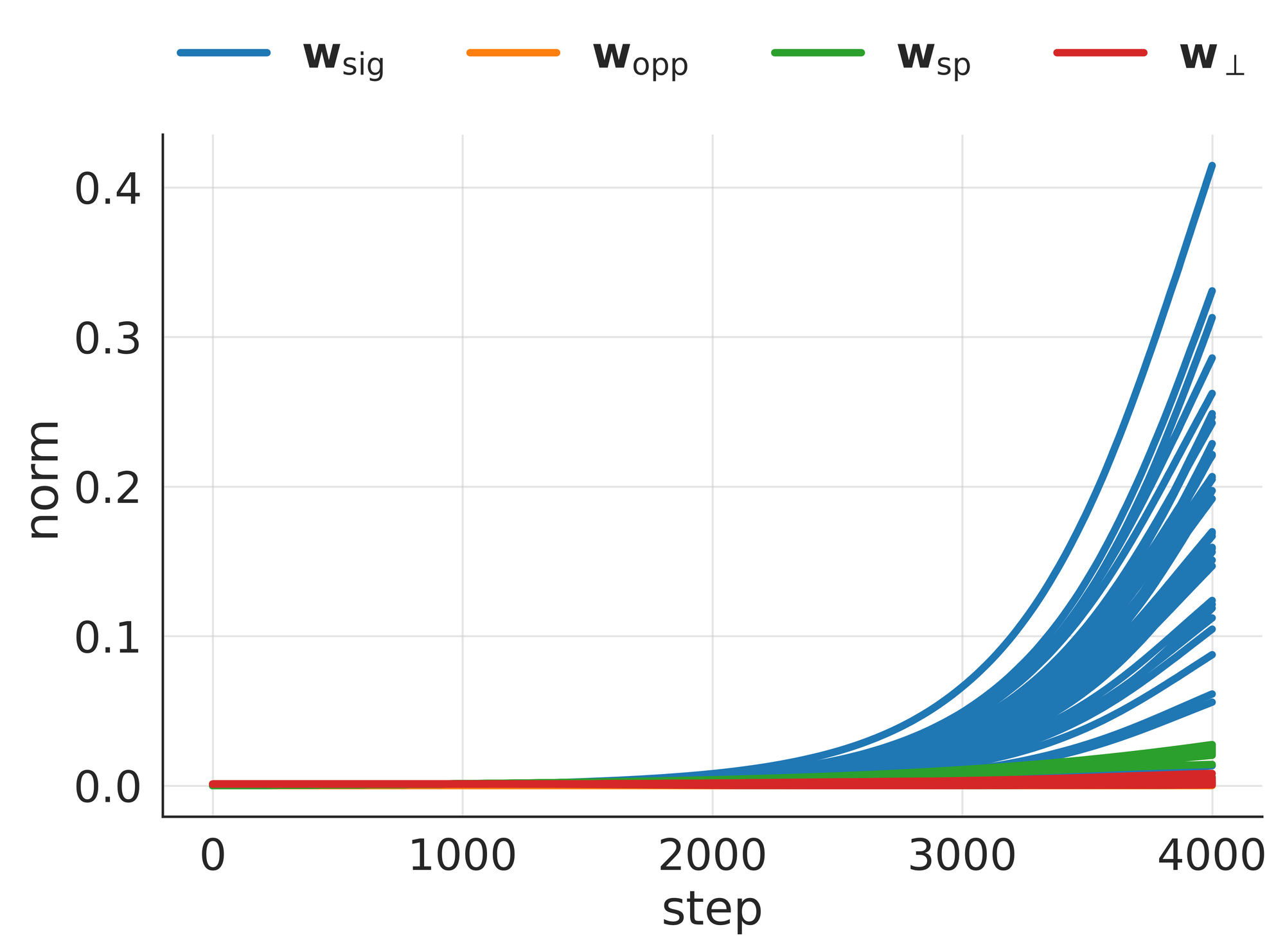}
        \caption{$\lambda=0.2$}
        \label{fig:lambda02_high}
    \end{subfigure}
    \caption{\textbf{Scaled-up $\lambda$.} We display the results of training runs with dimension $d=1000$, learning rate $\eta=0.01$, width $p=100$, initialization scale $\theta=0.001$, and batch size $m=5000$. For each of the $p=100$ neurons, we plot $\norm{\bwsig}$, $\norm{\bwopp}$, $\norm{\bwsp}$, and $\norm{\bwperp}$ (defined in \Secref{feature_learning_analysis}).}
    \label{fig:scaled_up_lambdas}
\end{figure}

\begin{figure}[h!]
    \centering
    \begin{subfigure}[b]{0.245\textwidth}
        \centering
        \includegraphics[width=\textwidth, height=0.15\textheight]{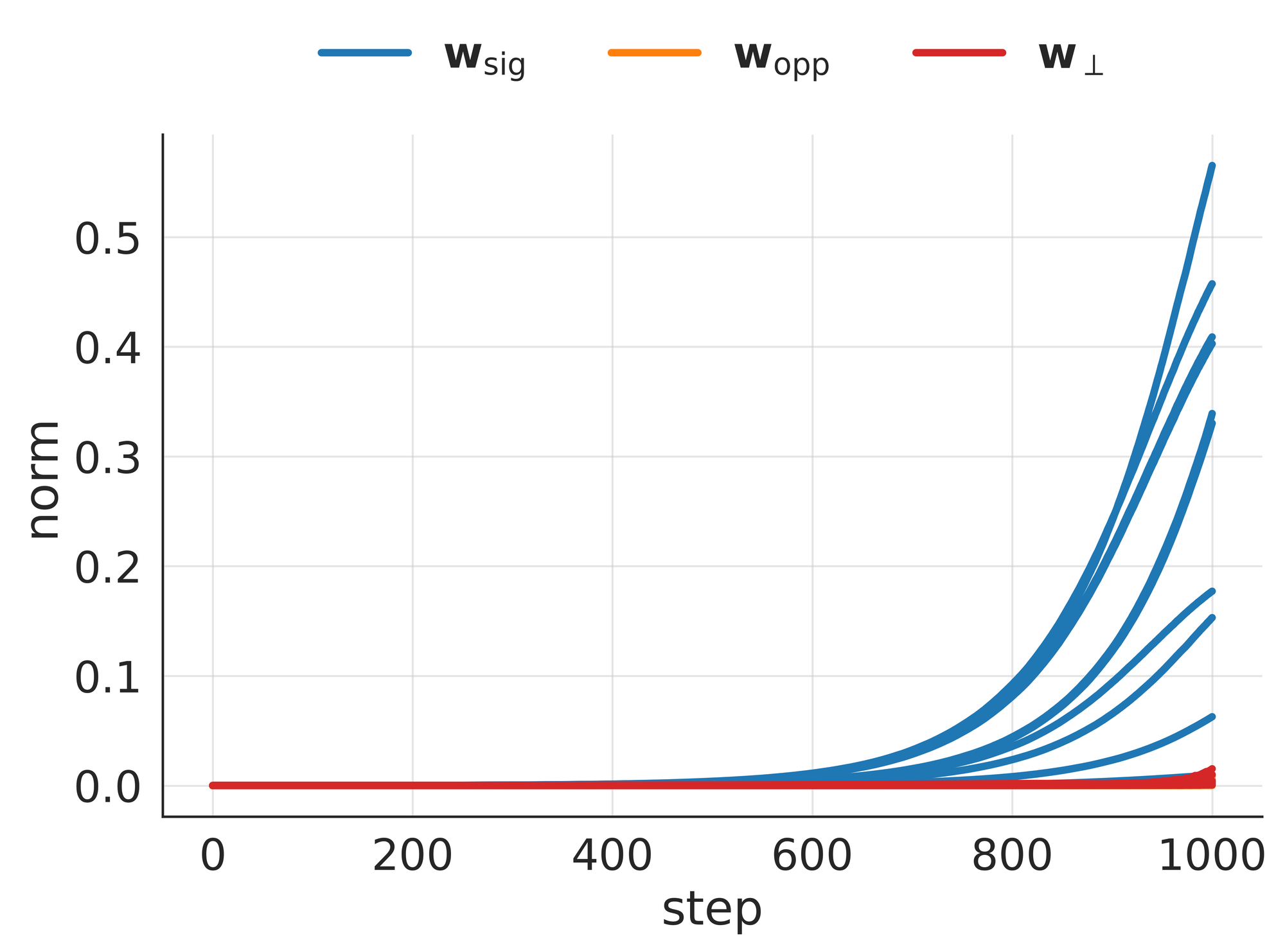}
        \caption{No spurious feature}
        \label{fig:glasgow_low}
    \end{subfigure}
    \begin{subfigure}[b]{0.245\textwidth}
        \centering
        \includegraphics[width=\textwidth, height=0.15\textheight]{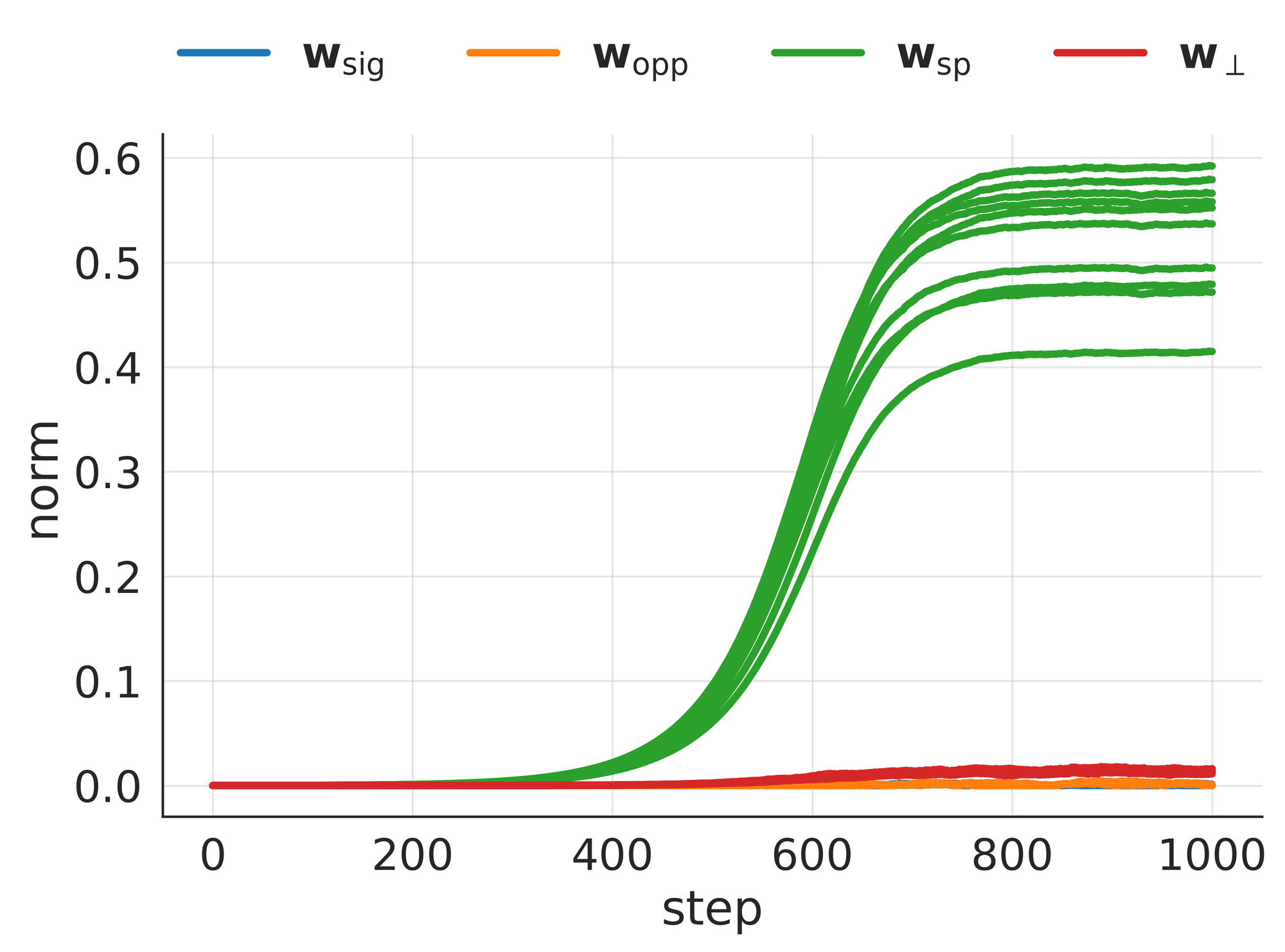}
        \caption{$\lambda=0.1$}
        \label{fig:lambda01_low}
    \end{subfigure}
    \begin{subfigure}[b]{0.245\textwidth}
        \centering
        \includegraphics[width=\textwidth, height=0.15\textheight]{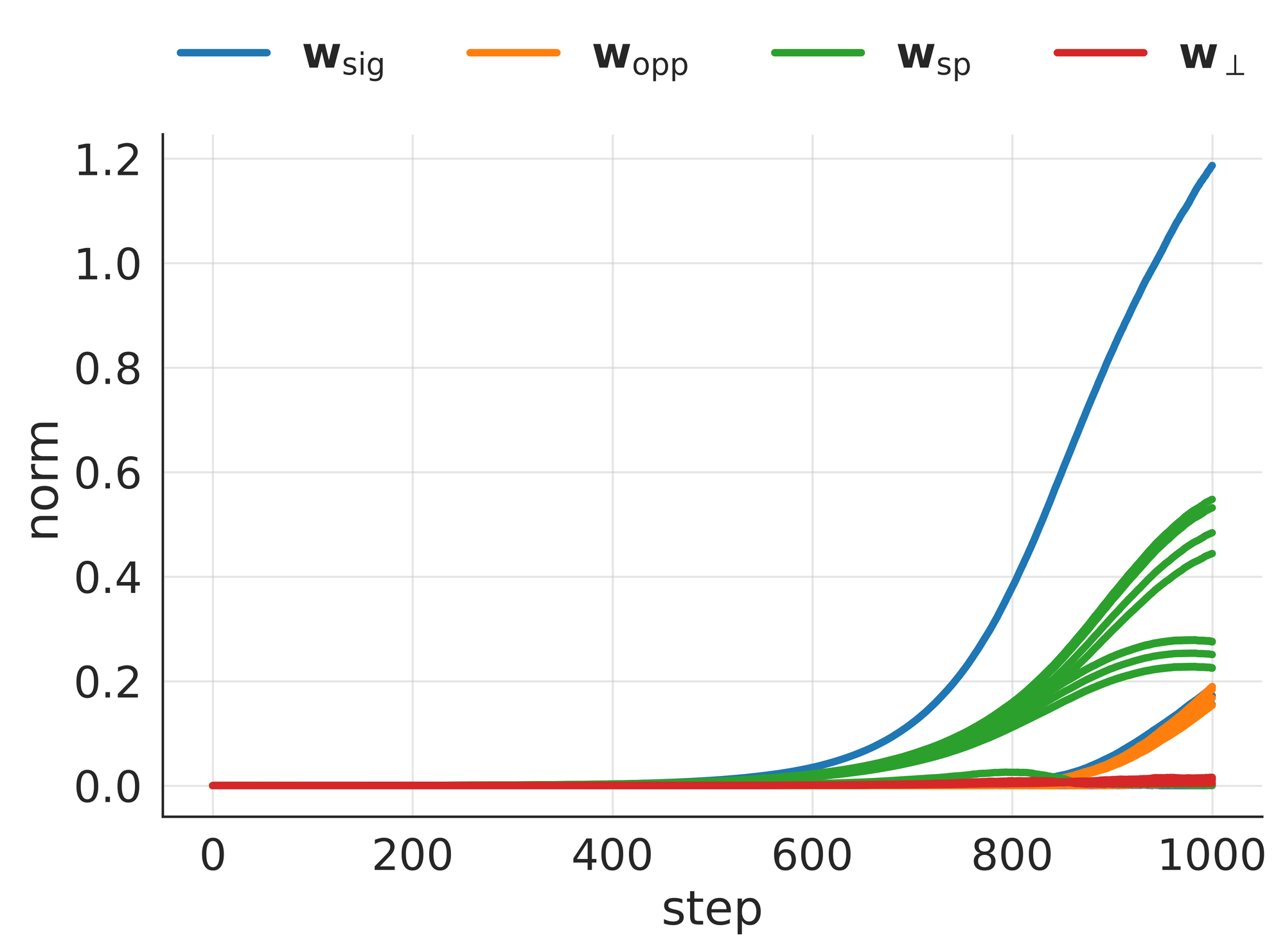}
        \caption{$\lambda=0.15$}
        \label{fig:lambda015_low}
    \end{subfigure}
    \begin{subfigure}[b]{0.245\textwidth}
        \centering
        \includegraphics[width=\textwidth, height=0.15\textheight]{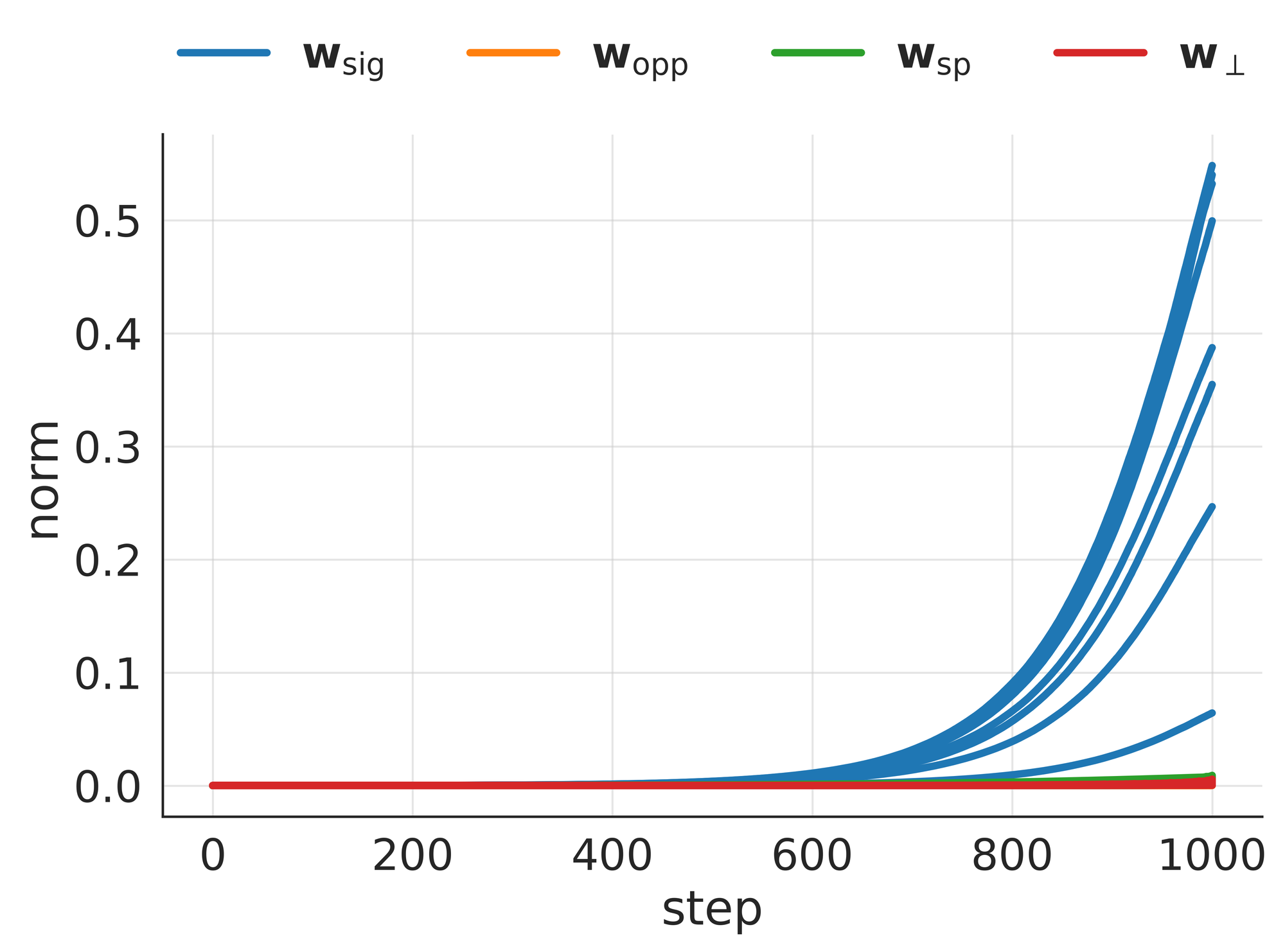}
        \caption{$\lambda=0.2$}
        \label{fig:lambda02_low}
    \end{subfigure}
    \caption{\textbf{Small-initialization $\lambda$.} We display the results of training runs with dimension $d=100$, learning rate $\eta=0.05$, width $p=10$, initialization scale $\theta=0.0001$, and batch size $m=5000$. For each of the $p=10$ neurons, we plot $\norm{\bwsig}$, $\norm{\bwopp}$, $\norm{\bwsp}$, and $\norm{\bwperp}$ (defined in \Secref{feature_learning_analysis}).}
    \label{fig:small-init_lambdas}
\end{figure}

\begin{figure}[h!]
    \centering
    \begin{subfigure}[b]{0.325\textwidth}
        \centering
        \includegraphics[width=\textwidth, height=0.18\textheight]{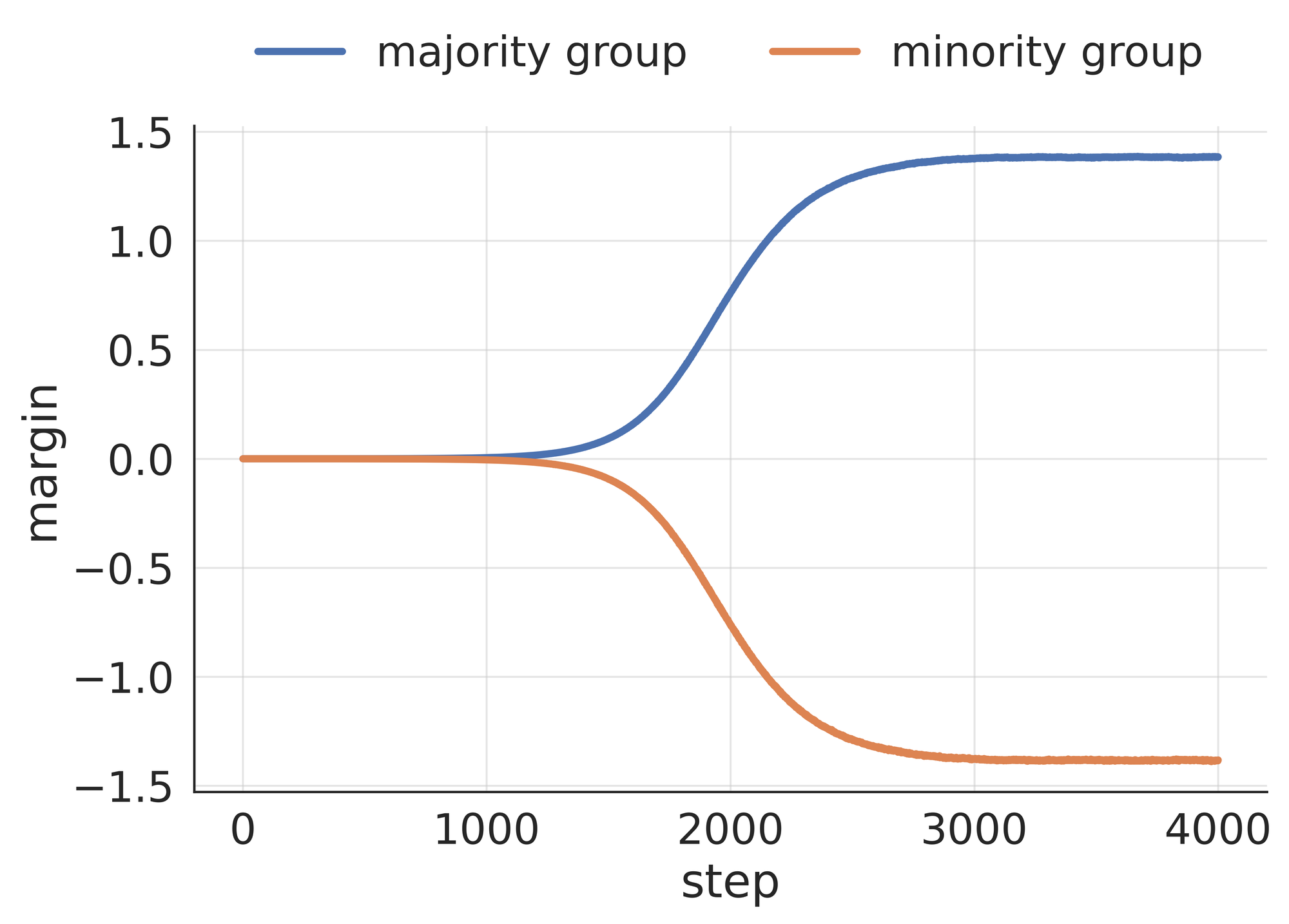}
        \caption{$\lambda=0.1$}
        \label{fig:lambda01_margin_high}
    \end{subfigure}
    \begin{subfigure}[b]{0.325\textwidth}
        \centering
        \includegraphics[width=\textwidth, height=0.18\textheight]{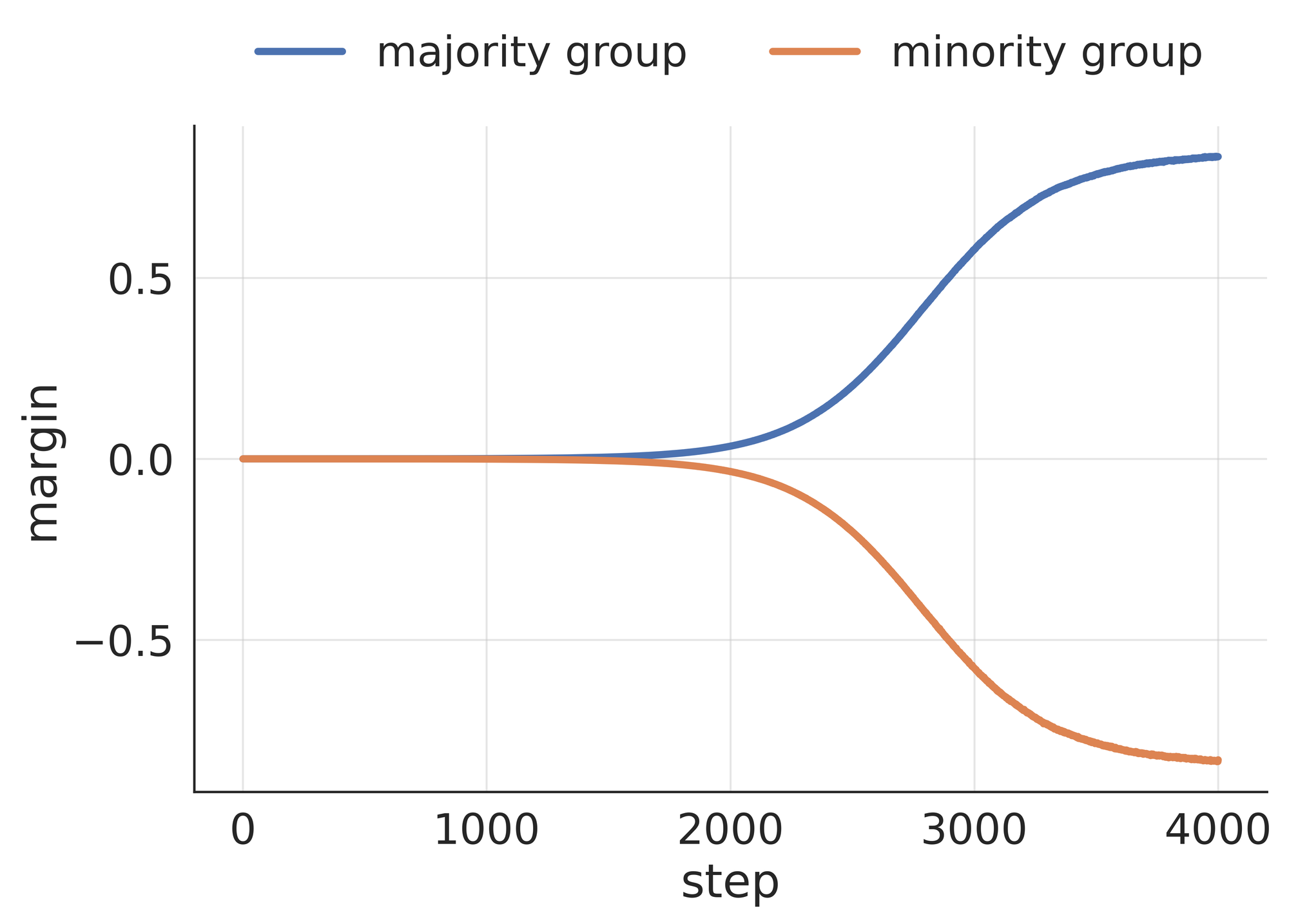}
        \caption{$\lambda=0.15$}
        \label{fig:lambda015_margin_high}
    \end{subfigure}
    \begin{subfigure}[b]{0.325\textwidth}
        \centering
        \includegraphics[width=\textwidth, height=0.18\textheight]{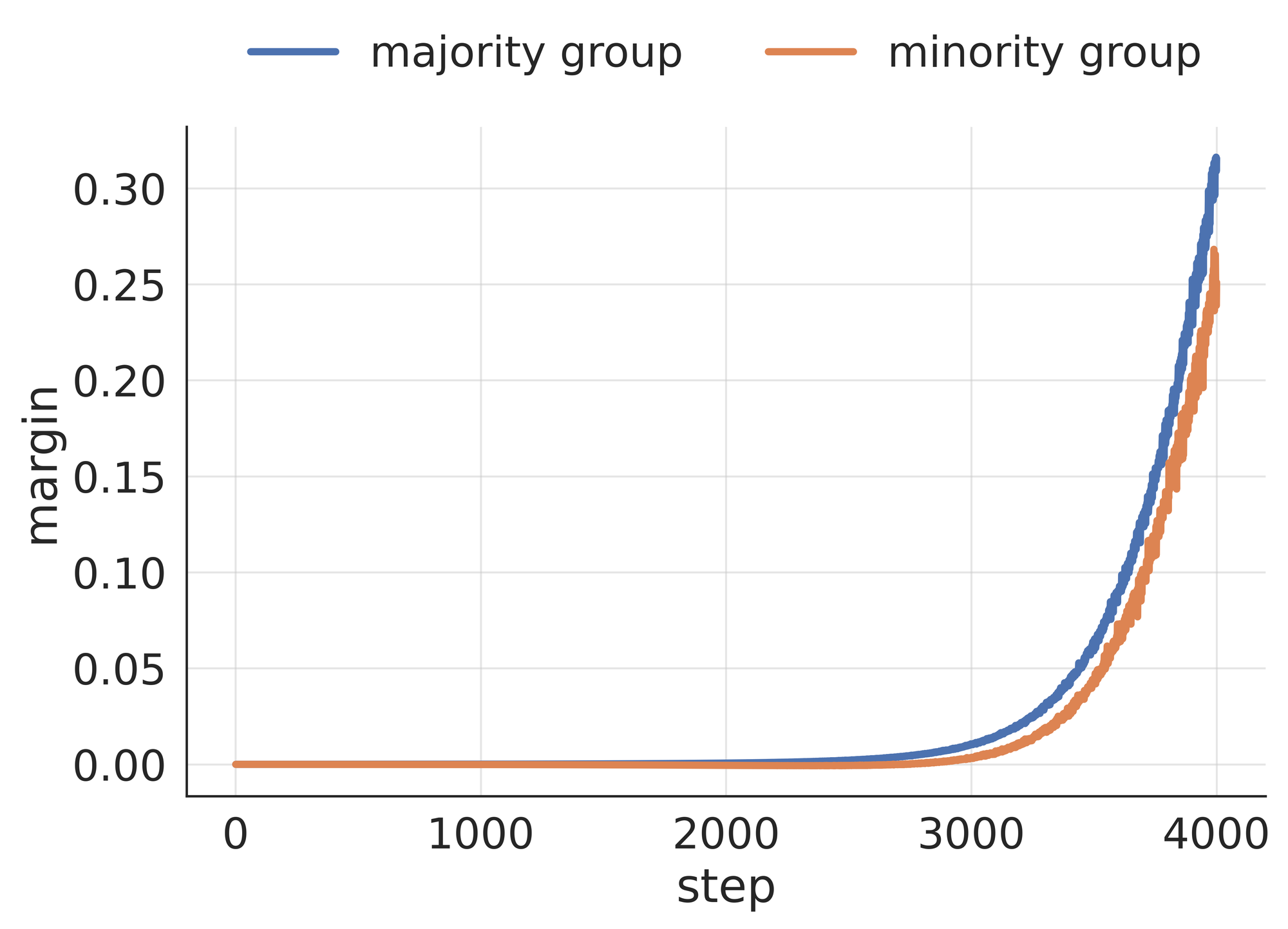}
        \caption{$\lambda=0.2$}
        \label{fig:lambda02_margin_high}
    \end{subfigure}
    \caption{\textbf{Scaled-up margins.} We display the results of training runs with dimension $d=1000$, learning rate $\eta=0.01$, width $p=100$, initialization scale $\theta=0.001$, and batch size $m=5000$. We plot the average per-group margin in each minibatch.}
    \label{fig:margins_high}
\end{figure}

\begin{figure}[h!]
    \centering
    \begin{subfigure}[b]{0.325\textwidth}
        \centering
        \includegraphics[width=\textwidth, height=0.18\textheight]{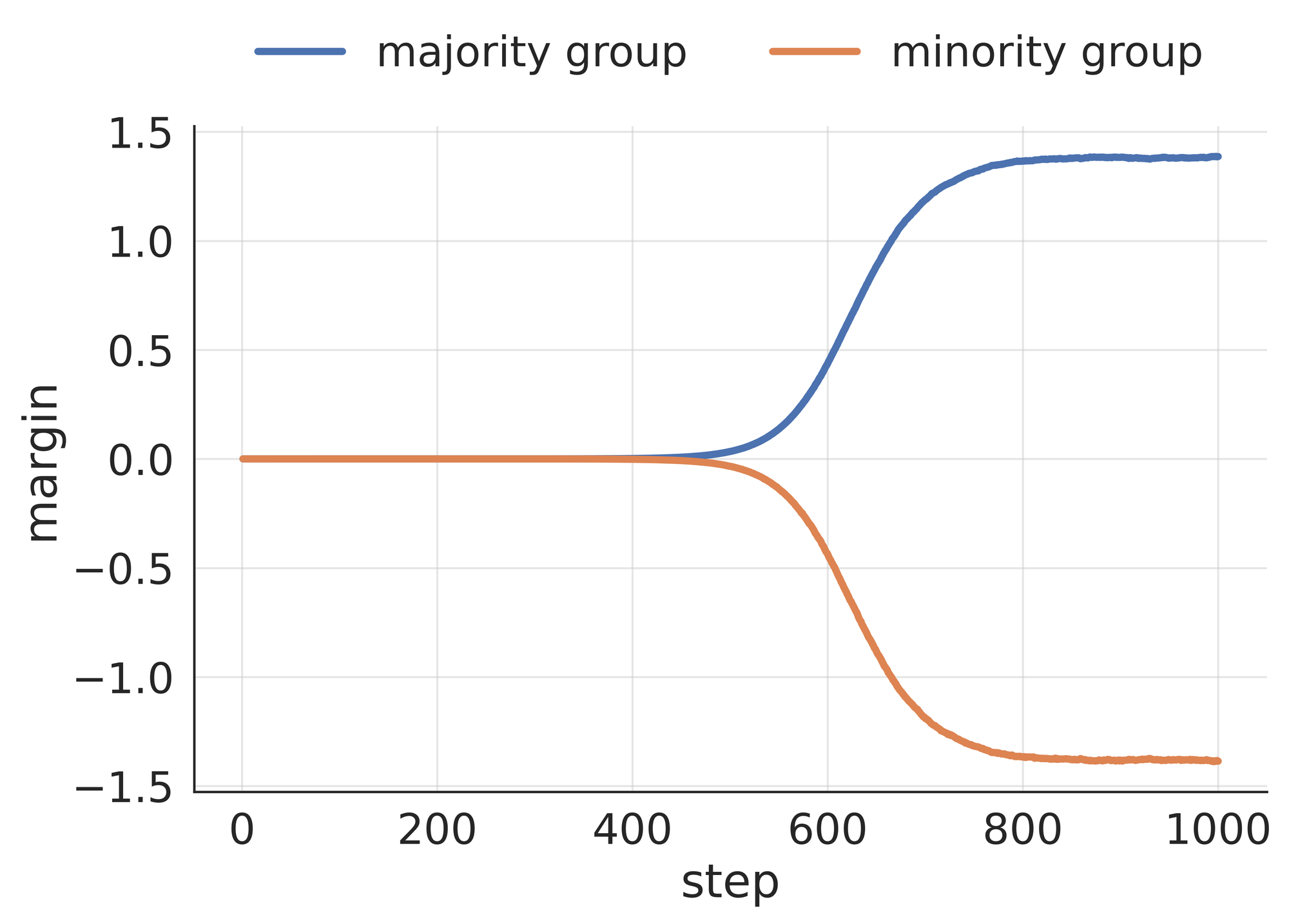}
        \caption{$\lambda=0.1$}
        \label{fig:lambda01_margin_low}
    \end{subfigure}
    \begin{subfigure}[b]{0.325\textwidth}
        \centering
        \includegraphics[width=\textwidth, height=0.18\textheight]{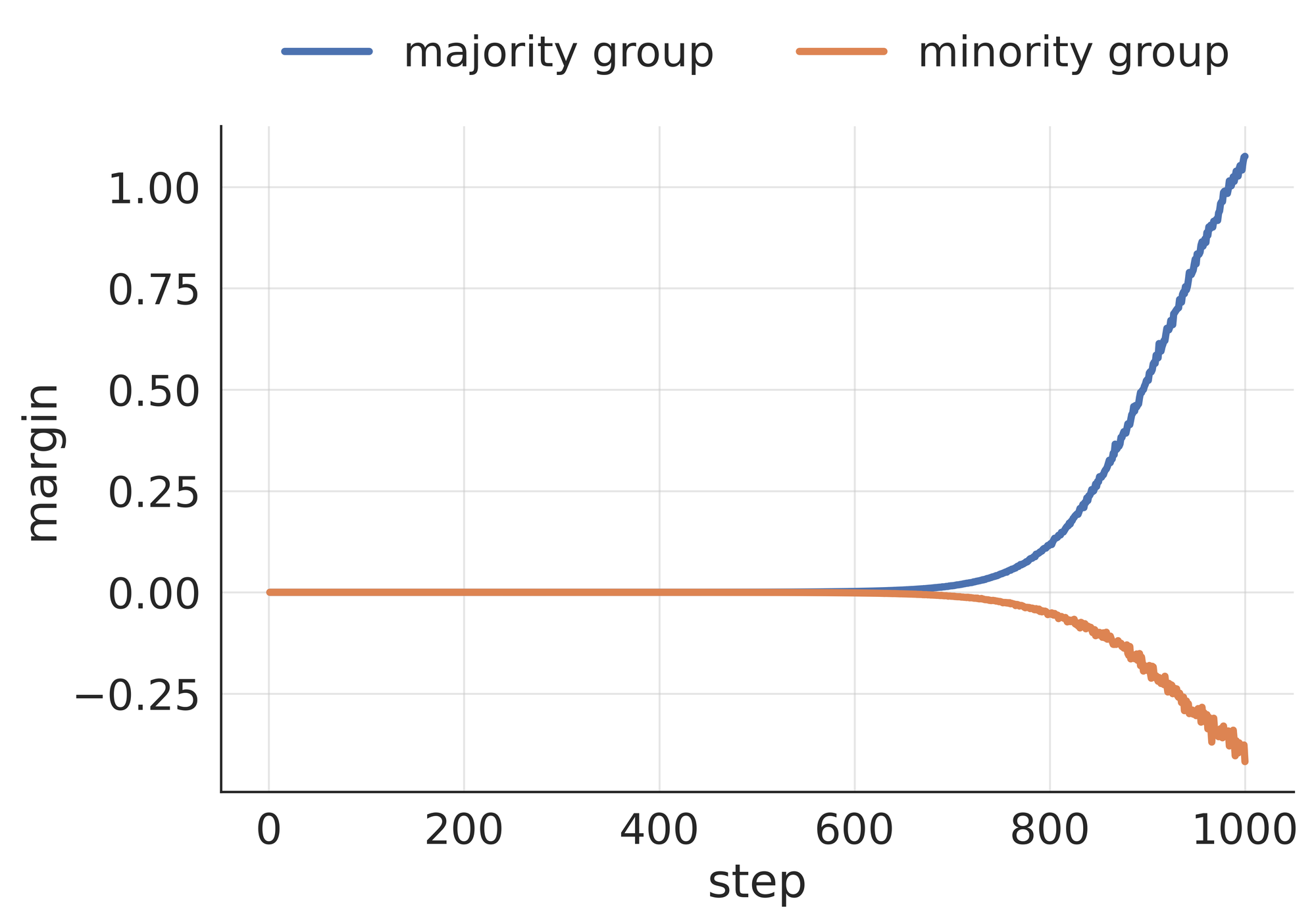}
        \caption{$\lambda=0.15$}
        \label{fig:lambda015_margin_low}
    \end{subfigure}
    \begin{subfigure}[b]{0.325\textwidth}
        \centering
        \includegraphics[width=\textwidth, height=0.18\textheight]{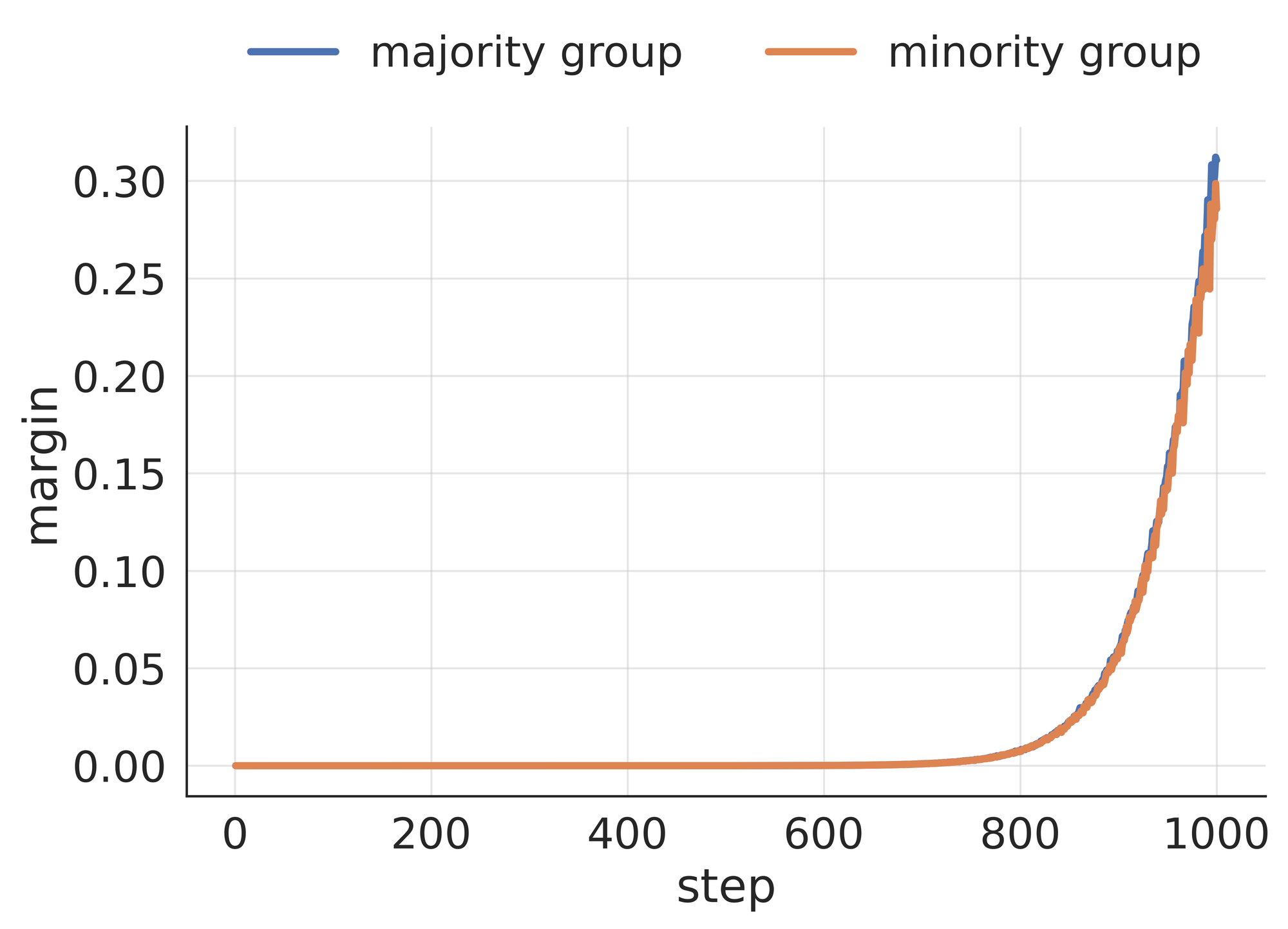}
        \caption{$\lambda=0.2$}
        \label{fig:lambda02_margin_low}
    \end{subfigure}
    \caption{\textbf{Small-initialization margins.} We display the results of training runs with dimension $d=100$, learning rate $\eta=0.05$, width $p=10$, initialization scale $\theta=0.0001$, and batch size $m=5000$. We plot the average per-group margin in each minibatch.}
    \label{fig:margins_low}
\end{figure}

\clearpage

\section{Limitations, Broader Impacts, and LLM Usage} \label{sec:limitations}

\paragraph{Limitations.} Our analysis is confined to two-layer ReLU neural networks; as is the case more generally in deep learning theory, extending the results to the three layers or more would be nontrivial (see, \eg\cite{nichani2023provable}).
Similarly, we study online minibatch SGD as is common in the literature, and handling batch sizes of $1$ or data re-use would require different theoretical techniques~\cite{glasgow2024sgd}.
Our results hold only for Boolean data, and not Gaussian data --- this is the key fact that enables us to write the population gradients $\nabla L_0$ and $\nabla L_\rho$ in closed form.
Finally, we focus on the setting with one signal feature and one spurious feature, whereas in practice there may be many signals and spurious features competing simultaneously~\cite{li2023whac, kim2024improving}.
We believe it is an exciting direction to extend our work to multiple features of differing complexities~\cite{qiu2024complexity}.

\paragraph{Broader impacts.} We hope this work contributes to the safe and equitable application of machine learning by improving understanding of how neural networks learn and rely on spurious correlations.
Our results may help motivate future research on robustness, fairness, and out-of-distribution generalization.
A potential negative outcome is that practitioners may over-interpret theoretical guarantees in simplified settings as evidence that real-world systems are robust to spurious correlations.
However, no theoretical framework can fully capture the complexity of practical deployment environments, and additional empirical evaluation remains necessary.

\paragraph{LLM usage.} Large language models (LLMs) including Claude Sonnet 4.6 Thinking, GPT 5.4 Thinking, and Gemini 3 Flash were used for discussion and refinement of some proof ideas.
In particular, these models contributed to mathematical derivations in \Lemref{v_ub}, \Lemref{boolean_to_gaussian_delta}, \Lemref{berry_esseen_tail}, and \Lemref{lindeberg}.
All proofs were written, critically reviewed, and independently verified by the authors, who assume full responsibility for their correctness.

\paragraph{Simulations.} Our simulations were run on an Nvidia A5000 GPU with 24GB VRAM, but this level of compute is not necessary. Our code is available at \url{https://github.com/tmlabonte/xor}.

\clearpage

\section{NeurIPS Paper Checklist}

\begin{enumerate}

\item {\bf Claims}
    \item[] Question: Do the main claims made in the abstract and introduction accurately reflect the paper's contributions and scope?
    \item[] Answer: \answerYes{} %
    \item[] Justification: Claims are stated clearly and supported by theoretical results and/or empirical evidence.
    \item[] Guidelines:
    \begin{itemize}
        \item The answer \answerNA{} means that the abstract and introduction do not include the claims made in the paper.
        \item The abstract and/or introduction should clearly state the claims made, including the contributions made in the paper and important assumptions and limitations. A \answerNo{} or \answerNA{} answer to this question will not be perceived well by the reviewers. 
        \item The claims made should match theoretical and experimental results, and reflect how much the results can be expected to generalize to other settings. 
        \item It is fine to include aspirational goals as motivation as long as it is clear that these goals are not attained by the paper. 
    \end{itemize}

\item {\bf Limitations}
    \item[] Question: Does the paper discuss the limitations of the work performed by the authors?
    \item[] Answer: \answerYes{} %
    \item[] Justification: We provide a discussion of limitations in \Secref{limitations}.
    \item[] Guidelines:
    \begin{itemize}
        \item The answer \answerNA{} means that the paper has no limitation while the answer \answerNo{} means that the paper has limitations, but those are not discussed in the paper. 
        \item The authors are encouraged to create a separate ``Limitations'' section in their paper.
        \item The paper should point out any strong assumptions and how robust the results are to violations of these assumptions (e.g., independence assumptions, noiseless settings, model well-specification, asymptotic approximations only holding locally). The authors should reflect on how these assumptions might be violated in practice and what the implications would be.
        \item The authors should reflect on the scope of the claims made, e.g., if the approach was only tested on a few datasets or with a few runs. In general, empirical results often depend on implicit assumptions, which should be articulated.
        \item The authors should reflect on the factors that influence the performance of the approach. For example, a facial recognition algorithm may perform poorly when image resolution is low or images are taken in low lighting. Or a speech-to-text system might not be used reliably to provide closed captions for online lectures because it fails to handle technical jargon.
        \item The authors should discuss the computational efficiency of the proposed algorithms and how they scale with dataset size.
        \item If applicable, the authors should discuss possible limitations of their approach to address problems of privacy and fairness.
        \item While the authors might fear that complete honesty about limitations might be used by reviewers as grounds for rejection, a worse outcome might be that reviewers discover limitations that aren't acknowledged in the paper. The authors should use their best judgment and recognize that individual actions in favor of transparency play an important role in developing norms that preserve the integrity of the community. Reviewers will be specifically instructed to not penalize honesty concerning limitations.
    \end{itemize}

\item {\bf Theory assumptions and proofs}
    \item[] Question: For each theoretical result, does the paper provide the full set of assumptions and a complete (and correct) proof?
    \item[] Answer: \answerYes{} %
    \item[] Justification: Complete proofs are provided in the Appendix.
    \item[] Guidelines:
    \begin{itemize}
        \item The answer \answerNA{} means that the paper does not include theoretical results. 
        \item All the theorems, formulas, and proofs in the paper should be numbered and cross-referenced.
        \item All assumptions should be clearly stated or referenced in the statement of any theorems.
        \item The proofs can either appear in the main paper or the supplemental material, but if they appear in the supplemental material, the authors are encouraged to provide a short proof sketch to provide intuition. 
        \item Inversely, any informal proof provided in the core of the paper should be complemented by formal proofs provided in appendix or supplemental material.
        \item Theorems and Lemmas that the proof relies upon should be properly referenced. 
    \end{itemize}

    \item {\bf Experimental result reproducibility}
    \item[] Question: Does the paper fully disclose all the information needed to reproduce the main experimental results of the paper to the extent that it affects the main claims and/or conclusions of the paper (regardless of whether the code and data are provided or not)?
    \item[] Answer: \answerYes{} %
    \item[] Justification: We have specified all training parameters and our code is available at \url{https://github.com/tmlabonte/xor}.
    \item[] Guidelines:
    \begin{itemize}
        \item The answer \answerNA{} means that the paper does not include experiments.
        \item If the paper includes experiments, a \answerNo{} answer to this question will not be perceived well by the reviewers: Making the paper reproducible is important, regardless of whether the code and data are provided or not.
        \item If the contribution is a dataset and\slash or model, the authors should describe the steps taken to make their results reproducible or verifiable. 
        \item Depending on the contribution, reproducibility can be accomplished in various ways. For example, if the contribution is a novel architecture, describing the architecture fully might suffice, or if the contribution is a specific model and empirical evaluation, it may be necessary to either make it possible for others to replicate the model with the same dataset, or provide access to the model. In general. releasing code and data is often one good way to accomplish this, but reproducibility can also be provided via detailed instructions for how to replicate the results, access to a hosted model (e.g., in the case of a large language model), releasing of a model checkpoint, or other means that are appropriate to the research performed.
        \item While NeurIPS does not require releasing code, the conference does require all submissions to provide some reasonable avenue for reproducibility, which may depend on the nature of the contribution. For example
        \begin{enumerate}
            \item If the contribution is primarily a new algorithm, the paper should make it clear how to reproduce that algorithm.
            \item If the contribution is primarily a new model architecture, the paper should describe the architecture clearly and fully.
            \item If the contribution is a new model (e.g., a large language model), then there should either be a way to access this model for reproducing the results or a way to reproduce the model (e.g., with an open-source dataset or instructions for how to construct the dataset).
            \item We recognize that reproducibility may be tricky in some cases, in which case authors are welcome to describe the particular way they provide for reproducibility. In the case of closed-source models, it may be that access to the model is limited in some way (e.g., to registered users), but it should be possible for other researchers to have some path to reproducing or verifying the results.
        \end{enumerate}
    \end{itemize}

\item {\bf Open access to data and code}
    \item[] Question: Does the paper provide open access to the data and code, with sufficient instructions to faithfully reproduce the main experimental results, as described in supplemental material?
    \item[] Answer: \answerYes{} %
    \item[] Justification: We have specified all training parameters and our code is available at \url{https://github.com/tmlabonte/xor}.
    \item[] Guidelines:
    \begin{itemize}
        \item The answer \answerNA{} means that paper does not include experiments requiring code.
        \item Please see the NeurIPS code and data submission guidelines (\url{https://neurips.cc/public/guides/CodeSubmissionPolicy}) for more details.
        \item While we encourage the release of code and data, we understand that this might not be possible, so \answerNo{} is an acceptable answer. Papers cannot be rejected simply for not including code, unless this is central to the contribution (e.g., for a new open-source benchmark).
        \item The instructions should contain the exact command and environment needed to run to reproduce the results. See the NeurIPS code and data submission guidelines (\url{https://neurips.cc/public/guides/CodeSubmissionPolicy}) for more details.
        \item The authors should provide instructions on data access and preparation, including how to access the raw data, preprocessed data, intermediate data, and generated data, etc.
        \item The authors should provide scripts to reproduce all experimental results for the new proposed method and baselines. If only a subset of experiments are reproducible, they should state which ones are omitted from the script and why.
        \item At submission time, to preserve anonymity, the authors should release anonymized versions (if applicable).
        \item Providing as much information as possible in supplemental material (appended to the paper) is recommended, but including URLs to data and code is permitted.
    \end{itemize}

\item {\bf Experimental setting/details}
    \item[] Question: Does the paper specify all the training and test details (e.g., data splits, hyperparameters, how they were chosen, type of optimizer) necessary to understand the results?
    \item[] Answer: \answerYes{} %
    \item[] Justification: We have specified experimental settings in \Secref{setting} and provided training parameters in each figure caption.
    \item[] Guidelines:
    \begin{itemize}
        \item The answer \answerNA{} means that the paper does not include experiments.
        \item The experimental setting should be presented in the core of the paper to a level of detail that is necessary to appreciate the results and make sense of them.
        \item The full details can be provided either with the code, in appendix, or as supplemental material.
    \end{itemize}

\item {\bf Experiment statistical significance}
    \item[] Question: Does the paper report error bars suitably and correctly defined or other appropriate information about the statistical significance of the experiments?
    \item[] Answer: \answerNo{} %
    \item[] Justification: Our main contribution is theoretical. Our experiments serve as visualization of our theoretical results, and primarily involve plotting feature norms over training a two-layer neural network, for which statistical significance tests are not required.
    \item[] Guidelines:
    \begin{itemize}
        \item The answer \answerNA{} means that the paper does not include experiments.
        \item The authors should answer \answerYes{} if the results are accompanied by error bars, confidence intervals, or statistical significance tests, at least for the experiments that support the main claims of the paper.
        \item The factors of variability that the error bars are capturing should be clearly stated (for example, train/test split, initialization, random drawing of some parameter, or overall run with given experimental conditions).
        \item The method for calculating the error bars should be explained (closed form formula, call to a library function, bootstrap, etc.)
        \item The assumptions made should be given (e.g., Normally distributed errors).
        \item It should be clear whether the error bar is the standard deviation or the standard error of the mean.
        \item It is OK to report 1-sigma error bars, but one should state it. The authors should preferably report a 2-sigma error bar than state that they have a 96\% CI, if the hypothesis of Normality of errors is not verified.
        \item For asymmetric distributions, the authors should be careful not to show in tables or figures symmetric error bars that would yield results that are out of range (e.g., negative error rates).
        \item If error bars are reported in tables or plots, the authors should explain in the text how they were calculated and reference the corresponding figures or tables in the text.
    \end{itemize}

\item {\bf Experiments compute resources}
    \item[] Question: For each experiment, does the paper provide sufficient information on the computer resources (type of compute workers, memory, time of execution) needed to reproduce the experiments?
    \item[] Answer: \answerYes{} %
    \item[] Justification: We provide a compute statement in \Secref{limitations}.
    \item[] Guidelines:
    \begin{itemize}
        \item The answer \answerNA{} means that the paper does not include experiments.
        \item The paper should indicate the type of compute workers CPU or GPU, internal cluster, or cloud provider, including relevant memory and storage.
        \item The paper should provide the amount of compute required for each of the individual experimental runs as well as estimate the total compute. 
        \item The paper should disclose whether the full research project required more compute than the experiments reported in the paper (e.g., preliminary or failed experiments that didn't make it into the paper). 
    \end{itemize}
    
\item {\bf Code of ethics}
    \item[] Question: Does the research conducted in the paper conform, in every respect, with the NeurIPS Code of Ethics \url{https://neurips.cc/public/EthicsGuidelines}?
    \item[] Answer: \answerYes{} %
    \item[] Justification: We have reviewed the code of ethics and confirm that our work follows them in every respect.
    \item[] Guidelines:
    \begin{itemize}
        \item The answer \answerNA{} means that the authors have not reviewed the NeurIPS Code of Ethics.
        \item If the authors answer \answerNo, they should explain the special circumstances that require a deviation from the Code of Ethics.
        \item The authors should make sure to preserve anonymity (e.g., if there is a special consideration due to laws or regulations in their jurisdiction).
    \end{itemize}

\item {\bf Broader impacts}
    \item[] Question: Does the paper discuss both potential positive societal impacts and negative societal impacts of the work performed?
    \item[] Answer: \answerYes{} %
    \item[] Justification: We provide a broader impacts statement in \Secref{limitations}.
    \item[] Guidelines:
    \begin{itemize}
        \item The answer \answerNA{} means that there is no societal impact of the work performed.
        \item If the authors answer \answerNA{} or \answerNo, they should explain why their work has no societal impact or why the paper does not address societal impact.
        \item Examples of negative societal impacts include potential malicious or unintended uses (e.g., disinformation, generating fake profiles, surveillance), fairness considerations (e.g., deployment of technologies that could make decisions that unfairly impact specific groups), privacy considerations, and security considerations.
        \item The conference expects that many papers will be foundational research and not tied to particular applications, let alone deployments. However, if there is a direct path to any negative applications, the authors should point it out. For example, it is legitimate to point out that an improvement in the quality of generative models could be used to generate Deepfakes for disinformation. On the other hand, it is not needed to point out that a generic algorithm for optimizing neural networks could enable people to train models that generate Deepfakes faster.
        \item The authors should consider possible harms that could arise when the technology is being used as intended and functioning correctly, harms that could arise when the technology is being used as intended but gives incorrect results, and harms following from (intentional or unintentional) misuse of the technology.
        \item If there are negative societal impacts, the authors could also discuss possible mitigation strategies (e.g., gated release of models, providing defenses in addition to attacks, mechanisms for monitoring misuse, mechanisms to monitor how a system learns from feedback over time, improving the efficiency and accessibility of ML).
    \end{itemize}
    
\item {\bf Safeguards}
    \item[] Question: Does the paper describe safeguards that have been put in place for responsible release of data or models that have a high risk for misuse (e.g., pre-trained language models, image generators, or scraped datasets)?
    \item[] Answer: \answerNA{} %
    \item[] Justification: Our main contribution is theoretical and does not pose such risks.
    \item[] Guidelines:
    \begin{itemize}
        \item The answer \answerNA{} means that the paper poses no such risks.
        \item Released models that have a high risk for misuse or dual-use should be released with necessary safeguards to allow for controlled use of the model, for example by requiring that users adhere to usage guidelines or restrictions to access the model or implementing safety filters. 
        \item Datasets that have been scraped from the Internet could pose safety risks. The authors should describe how they avoided releasing unsafe images.
        \item We recognize that providing effective safeguards is challenging, and many papers do not require this, but we encourage authors to take this into account and make a best faith effort.
    \end{itemize}

\item {\bf Licenses for existing assets}
    \item[] Question: Are the creators or original owners of assets (e.g., code, data, models), used in the paper, properly credited and are the license and terms of use explicitly mentioned and properly respected?
    \item[] Answer: \answerNA{} %
    \item[] Justification: We do not use existing assets.
    \item[] Guidelines:
    \begin{itemize}
        \item The answer \answerNA{} means that the paper does not use existing assets.
        \item The authors should cite the original paper that produced the code package or dataset.
        \item The authors should state which version of the asset is used and, if possible, include a URL.
        \item The name of the license (e.g., CC-BY 4.0) should be included for each asset.
        \item For scraped data from a particular source (e.g., website), the copyright and terms of service of that source should be provided.
        \item If assets are released, the license, copyright information, and terms of use in the package should be provided. For popular datasets, \url{paperswithcode.com/datasets} has curated licenses for some datasets. Their licensing guide can help determine the license of a dataset.
        \item For existing datasets that are re-packaged, both the original license and the license of the derived asset (if it has changed) should be provided.
        \item If this information is not available online, the authors are encouraged to reach out to the asset's creators.
    \end{itemize}

\item {\bf New assets}
    \item[] Question: Are new assets introduced in the paper well documented and is the documentation provided alongside the assets?
    \item[] Answer: \answerNA{} %
    \item[] Justification: We do not release new assets (except for our code, which has MIT license).
    \item[] Guidelines:
    \begin{itemize}
        \item The answer \answerNA{} means that the paper does not release new assets.
        \item Researchers should communicate the details of the dataset\slash code\slash model as part of their submissions via structured templates. This includes details about training, license, limitations, etc. 
        \item The paper should discuss whether and how consent was obtained from people whose asset is used.
        \item At submission time, remember to anonymize your assets (if applicable). You can either create an anonymized URL or include an anonymized zip file.
    \end{itemize}

\item {\bf Crowdsourcing and research with human subjects}
    \item[] Question: For crowdsourcing experiments and research with human subjects, does the paper include the full text of instructions given to participants and screenshots, if applicable, as well as details about compensation (if any)? 
    \item[] Answer: \answerNA{} %
    \item[] Justification: We do not use crowdsourcing or research with human subjects.
    \item[] Guidelines:
    \begin{itemize}
        \item The answer \answerNA{} means that the paper does not involve crowdsourcing nor research with human subjects.
        \item Including this information in the supplemental material is fine, but if the main contribution of the paper involves human subjects, then as much detail as possible should be included in the main paper. 
        \item According to the NeurIPS Code of Ethics, workers involved in data collection, curation, or other labor should be paid at least the minimum wage in the country of the data collector. 
    \end{itemize}

\item {\bf Institutional review board (IRB) approvals or equivalent for research with human subjects}
    \item[] Question: Does the paper describe potential risks incurred by study participants, whether such risks were disclosed to the subjects, and whether Institutional Review Board (IRB) approvals (or an equivalent approval/review based on the requirements of your country or institution) were obtained?
    \item[] Answer: \answerNA{} %
    \item[] Justification: We do not use crowdsourcing or research with human subjects.
    \item[] Guidelines:
    \begin{itemize}
        \item The answer \answerNA{} means that the paper does not involve crowdsourcing nor research with human subjects.
        \item Depending on the country in which research is conducted, IRB approval (or equivalent) may be required for any human subjects research. If you obtained IRB approval, you should clearly state this in the paper. 
        \item We recognize that the procedures for this may vary significantly between institutions and locations, and we expect authors to adhere to the NeurIPS Code of Ethics and the guidelines for their institution. 
        \item For initial submissions, do not include any information that would break anonymity (if applicable), such as the institution conducting the review.
    \end{itemize}

\item {\bf Declaration of LLM usage}
    \item[] Question: Does the paper describe the usage of LLMs if it is an important, original, or non-standard component of the core methods in this research? Note that if the LLM is used only for writing, editing, or formatting purposes and does \emph{not} impact the core methodology, scientific rigor, or originality of the research, declaration is not required.
    \item[] Answer: \answerYes{} %
    \item[] Justification: We provide an LLM usage statement in \Secref{limitations}.
    \item[] Guidelines:
    \begin{itemize}
        \item The answer \answerNA{} means that the core method development in this research does not involve LLMs as any important, original, or non-standard components.
        \item Please refer to our LLM policy in the NeurIPS handbook for what should or should not be described.
    \end{itemize}

\end{enumerate}

\end{document}